\documentclass[10pt]{article} % For LaTeX2e
\usepackage[accepted]{tmlr}
\usepackage{amsmath,amsfonts,bm}

\def\eqref#1{equation~\ref{#1}}
\def\1{\bm{1}}

\DeclareMathAlphabet{\mathsfit}{\encodingdefault}{\sfdefault}{m}{sl}
\SetMathAlphabet{\mathsfit}{bold}{\encodingdefault}{\sfdefault}{bx}{n}

\usepackage{etoc}
\usepackage{hyperref}
\usepackage{url}
\usepackage{colortbl}
\usepackage{xcolor}
\usepackage{amsmath}
\usepackage{amssymb}
\usepackage{mathtools}
\usepackage{nccmath}  % this defines \medmath
\usepackage{wrapfig}
\usepackage{longtable}
\usepackage{pifont}
\DeclareUnicodeCharacter{2713}{\ding{51}} % ✓
\DeclareUnicodeCharacter{2717}{\ding{55}} % ✗
\DeclareUnicodeCharacter{25B3}{$\triangle$} % △
\usepackage{bibentry}  
\usepackage{xcolor,soul,framed} %,caption
\colorlet{shadecolor}{yellow}
\usepackage[pdftex]{graphicx}
\graphicspath{{../pdf/}{../jpeg/}}
\DeclareGraphicsExtensions{.pdf,.jpeg,.png}
 \usepackage{booktabs}
\usepackage{tikz}
\usetikzlibrary{arrows.meta, positioning, calc, shapes, fit}

\usepackage{pgfplots}
\pgfplotsset{compat=1.18}
\usepgfplotslibrary{groupplots}
\usepgfplotslibrary{statistics}
\usepgfplotslibrary{dateplot}

\usepackage{array}
\usepackage{mdwmath}
\usepackage{mdwtab}
\usepackage{eqparbox}
\usepackage{url}
\usepackage{booktabs}
\usepackage{adjustbox}
\usepackage{tikz}
\usetikzlibrary{arrows.meta,positioning,shapes.geometric,fit}
\usepackage{siunitx}
\usetikzlibrary{positioning}
\usepackage{amssymb}
\usepackage{mathtools}
\usepackage{amsthm}
\providecommand{\medmath}[1]{\ensuremath{#1}}
\usepackage[utf8]{inputenc}
\usepackage[T1]{fontenc}
\usepackage{microtype}
\usepackage{nicefrac}
\usepackage{amsfonts}
\usepackage{algorithm}
\usepackage{algpseudocode} % from algorithmicx

\usepackage{pifont}
\usepackage[T1]{fontenc}
\usepackage[utf8]{inputenc}
\usepackage{microtype}
\usepackage{adjustbox}
\usetikzlibrary{shadows.blur,shadows}

\usepackage{array}
\usepackage{siunitx}
\usepackage{booktabs}
\usepackage{tabularx}
\usepackage{array}
\usepackage{multirow}
\usepackage{adjustbox}

\usepackage{xcolor}

\usepackage{caption}
\usepackage{subcaption}
\usepackage{wrapfig}

\makeatletter
\@ifundefined{labelindent}{}{}
\makeatother
\usepackage{enumitem}
\usepgfplotslibrary{polar,fillbetween,statistics}
\usetikzlibrary{patterns,shapes,arrows.meta}
\usepackage{url}

\usepackage{hyperref}
\usepackage[capitalize,noabbrev]{cleveref}

\usepackage[textsize=tiny]{todonotes}
\usepackage{tikz}
\usetikzlibrary{arrows.meta,positioning,shapes.geometric}
\definecolor{oi-blue}{RGB}{0,114,178}
\definecolor{oi-orange}{RGB}{230,159,0}
\definecolor{oi-sky}{RGB}{86,180,233}
\definecolor{oi-green}{RGB}{0,158,115}
\definecolor{oi-yellow}{RGB}{240,228,66}
\definecolor{oi-verm}{RGB}{213,94,0}
\definecolor{oi-purple}{RGB}{204,121,167}
\usepackage{amsmath,amssymb,mathtools}
\usepackage{amsthm}
\usepackage{bm}               % <-- defines \bm
\providecommand{\medmath}[1]{\ensuremath{#1}} % safe no-op if not defined
\theoremstyle{plain}
\newtheorem{theorem}{Theorem}       % Theorem 1, 2, 3...
\newtheorem{definition}{Definition} % Definition 1, 2, 3...
\newtheorem{assumption}{Assumption}  % Assumption 1, 2, 3...
\newtheorem{proposition}{Proposition} % Proposition 1, 2, 3...

\newtheorem{lemma}{Lemma}

\newtheorem{remark}{Remark}
\usepackage[normalem]{ulem}

\title{MCIR: A Feature Dependence-Aware  Explainability Method with Reliability Guarantees}

\author{Poushali Sengupta\\
Department of Informatics, University of Oslo\\
Oslo, Norway\\
\texttt{poushali.sengupta@ifi.uio.no}
\AND
Sabita Maharjan\\
Department of Informatics, University of Oslo\\
Oslo, Norway\\
\texttt{sabita.maharjan@ifi.uio.no}
\AND
Frank Eliassen\\
Department of Informatics, University of Oslo\\
Oslo, Norway\\
\texttt{frank.eliassen@ifi.uio.no}
\AND
Shashi Raj Pandey\\
Department of Electronic Systems, Aalborg University\\
Aalborg, Denmark\\
\texttt{srp@es.aau.dk}
\AND
Yan Zhang\\
Department of Informatics, University of Oslo\\
Oslo, Norway\\
\texttt{yanzhang@ieee.org}
}

\def\month{10}  % Insert correct month for camera-ready version
\def\year{2026} % Insert correct year for camera-ready version
\def\openreview{\url{https://openreview.net/TBD}} % Insert correct link to OpenReview for camera-ready version

\begin{document}
\maketitle
\begin{abstract}

As modern machine-learning models are deployed in high-stakes,
data-rich environments, interactions among features have become more
intricate and less amenable to traditional interpretation. Strong
feature dependence can make attribution sensitive to how shared
predictive information is allocated among correlated predictors. In the
presence of multicollinearity or near-duplicate predictors, attribution
tools such as SHAP, LIME, HSIC, MI/CMI, and SAGE may distribute
importance across redundant features, obscuring which variables contain
unique predictive information. This can produce unstable rankings and
increase computational cost. Correlation-aware approaches such as CIR
and BlockCIR address related aspects of dependence but do not directly
measure the unique conditional information contributed by an individual
feature. To address this, we propose the Mutual Correlation Impact Ratio
Method (MCIR-M), a dependence-aware global feature-importance procedure
that quantifies the unique predictive information contributed by each
feature beyond a selected dependence neighbourhood. MCIR-M introduces
the Mutual Correlation Impact Ratio (MCIR), which conditions each feature
on a small set of strongly dependent neighbours and computes a
normalized ratio of conditional to block-level information. The
population score lies in $[0,1]$ and equals zero under exact conditional
redundancy. This common numerical range supports comparison of reported
scores but does not by itself imply identical calibration across
datasets or tasks. In addition, we introduce a lightweight estimation
procedure that computes MCIR scores using a fraction of the available
data and evaluates agreement with full-data explanations through
ranking and faithfulness diagnostics. Across controlled synthetic
redundancy experiments and the UCI HAR benchmark, MCIR exhibits
dependence-aware ranking behaviour, with its clearest advantage appearing
under injected near-duplicate predictors; comparisons with SHAP
(independent and conditional), SAGE, HSIC, MI-based scores, and
CIR-family baselines are mixed across real-data evaluation criteria.
Using reduced explanation samples lowers the computational burden in
the evaluated configurations, while full-versus-lightweight agreement
is assessed separately through ranking, head-set, and faithfulness
diagnostics. These results support MCIR-M
as a practical dependence-aware diagnostic for global explanation under
strong feature dependence.

\end{abstract}

\section{Introduction}
Modern predictive models often use correlated or redundant features,
for which attribution depends on the target functional, treatment of
feature dependence, and perturbation or background distribution
~\cite{covert2021explaining}. Interventional
SHAP~\cite{lundberg2017shap}, conditional
SHAP~\cite{aas2021explaining}, SAGE~\cite{covert2020sage}, marginal
MI~\cite{cover2006elements,kraskov2004estimating},
HSIC~\cite{gretton2005measuring,hsic2007}, CMI, and CIR-family
scores~\cite{sengupta2025correlation,sengupta2025correlation2} may therefore produce different
allocations because they quantify distinct functionals. In particular,
perturbation- and coalition-based methods allocate shared information
according to their value function and background construction, whereas
marginal dependence measures do not isolate incremental information.
Although CMI provides such conditional isolation, its finite-sample
estimation is sensitive to sample size and conditioning dimension.
MCIR-M addresses this setting by measuring each feature's unique
predictive information relative to an explicit, screened dependence
neighbourhood and normalizing it by enlarged-block information. The
population score is bounded and vanishes under exact conditional
redundancy. For scalability, MCIR-M uses small local neighbourhoods and
supports reduced-sample evaluation whose fidelity is assessed through
separate representation, ranking, top-$K$, and behavioural diagnostics.

\par We propose MCIR-M (Mutual Correlation Impact Ratio Method), a
dependence-aware global attribution framework for explaining a fixed
predictive model. For each input feature, MCIR identifies a small
neighbourhood of related predictors and measures the predictive
information contributed by that feature beyond what is already
represented by its neighbours. This unique contribution is normalized
by the information carried jointly by the feature and its
neighbourhood, yielding a bounded score that vanishes under exact
conditional redundancy.
MCIR is built on population conditional mutual information
(CMI)~\cite{cover2006elements}. When regular conditional distributions exist, the same population CMI
functional accommodates continuous, discrete, and mixed-type variables.
Its application to jointly occurring variables of different types is
referred to later in the paper as \emph{mixed-variable CMI}. Its empirical reliability, however, depends on using an
estimator compatible with the observed variable types, sample size,
regularity conditions, and conditioning dimension
~\cite{gao2018demystifying}. MCIR-M does not remove these established
estimation challenges. Instead, it limits the conditioning dimension
through screened local neighbourhoods, states estimator-specific
conditions, and evaluates finite-sample sensitivity explicitly. Its
contribution is therefore the resulting redundancy-aware attribution
construction and its accompanying population, finite-sample, and
lightweight-computation analyses, rather than a new general-purpose CMI
estimator.
The population score lies in $[0,1]$ and vanishes under exact
conditional redundancy. Boundedness supplies a common numerical range,
not universal cross-task calibration. We establish population
properties and conditional finite-sample ranking bounds under explicit
component-error, denominator-separation, and ranking-margin
assumptions. Estimator switching and Auto-$\Phi$ are data-driven
selection heuristics rather than oracle-optimal procedures. A
lightweight implementation computes the same attribution functional on
a reduced explanation sample; its adequacy is evaluated empirically
through separate representation, ranking, head-set, and behavioural
diagnostics. Throughout, MCIR denotes the score and MCIR-M the complete
procedure. 

\par To improve practical scalability, we develop a lightweight computation strategy for estimating MCIR from a reduced environment and introduce an estimator-selection procedure for handling different data types and dependence structures. The main contributions of this work are summarized below. 
\begin{enumerate}[leftmargin=*, itemsep=1pt, topsep=0pt]
    \item \textbf{Population MCIR formulation.}
We define a bounded, dependence-aware global attribution functional
that measures the unique predictive information contributed by a
feature beyond a selected neighbourhood. The definition accommodates
continuous, discrete, and mixed-type random elements at the population
level.

\item \textbf{Qualified theoretical and computational analysis.}
We establish population boundedness, endpoint behaviour, exact
conditional-redundancy collapse, invariance, and a weak-dependence
reduction under explicit conditions. We also give estimator-specific
finite-sample perturbation and ranking results. Auto-$\Phi$ and
bootstrap estimator switching are presented as practical selection
heuristics, with their costs and limitations reported explicitly.

\item \textbf{Empirical evaluation under feature dependence.}
We evaluate controlled redundancy, UCI HAR, HouseEnergy-Sim,
CIFAR-10 representations, and Norwegian load forecasting using
ranking, group-level, deletion, estimator-sensitivity, and runtime
diagnostics. The clearest gains occur in the controlled redundancy
experiments; real-data comparisons are mixed across criteria. The lightweight experiments evaluate the computational and
attributional effects of reducing the explanation sample, with runtime
and full-versus-lightweight agreement reported as separate diagnostics.
\end{enumerate}

\par \textbf{Paper Overview:} 
The remainder of the paper is organized as follows. Section~\ref{sec:related} reviews related work and motivates the need for dependence-aware feature attribution. Section~\ref{sec:notation} introduces the statistical method and notation used throughout the paper. Section~\ref{prob} presents the proposed \textsc{MCIR-M} method together with its theoretical foundations and implementation considerations. Sections~\ref{sec:expts} and~\ref{res} describe the experimental setup and empirical evaluation across synthetic and real-world datasets. Finally, Sections~\ref{rep}, \ref{bro}, and~\ref{sec:conclusion} discuss reproducibility, broader impact, and concluding remarks. Additional proofs, algorithms, ablation studies, and extended experimental results are provided in the appendices.

% \par The appendices provide detailed supporting material for the proposed MCIR framework. Appendix~\ref{sec:theory1.1}--Appendix~\ref{sec:theory1.8} contains the proofs of the main theoretical results, including boundedness, redundancy collapse, weak-dependence behaviour, finite-sample stability, estimator consistency, and Auto-$\Phi$ guarantees. Appendix~\ref{sec:algo} presents supplementary implementation algorithms and estimator-switching procedures. Appendix~\ref{sec:appendix_experiments} provides additional experimental analyses, including synthetic redundancy studies, estimator ablations, runtime comparisons, deletion-faithfulness evaluations, lightweight-environment robustness, CIFAR-10 embedding experiments, and supplementary Norwegian load-zone results. Additional implementation details, visualizations, and reproducibility material are also included throughout the appendices.

\section{Background and Related Work}\label{sec:related}

In real-world datasets, correlated or redundant covariates complicate
global attribution because the resulting importance depends on the
target functional and on how feature dependence is treated.
Unconditional permutation importance and impurity-based importance can
be biased or difficult to interpret in the presence of correlated
predictors~\cite{breiman2001random,strobl2008conditional}.
Shapley-based explainers, including SHAP and SAGE, allocate contributions
through a coalition value function, but their treatment of dependence
depends on the selected value function, background distribution, and
conditional or interventional implementation
~\cite{lundberg2017shap,sundararajan2020many,
covert2020sage}. Standard KernelSHAP commonly uses a feature-independence
approximation, whereas conditional variants estimate the relevant
conditional distributions to preserve dependence
~\cite{aas2021explaining}. Consequently, correlated predictors may
receive different allocations under different valid Shapley
formulations; this reflects a difference in the attribution estimand
rather than a universal failure of Shapley values.

Perturbation-based procedures such as ROAR and deletion tests evaluate
whether predictive performance changes when selected features are
removed or replaced~\cite{hooker2019roar,samek2017evaluating}. Their
results depend on the perturbation, replacement, retraining, and
evaluation protocol and should therefore be distinguished from the
attribution functional itself. Information-theoretic and kernel
measures provide another perspective: marginal MI and HSIC quantify
unconditional dependence, whereas CMI explicitly quantifies dependence
remaining after conditioning
~\cite{cover2006elements,gretton2005measuring}. These quantities do not
share a universal unit-interval scale, and finite-sample estimation,
particularly for nonparametric CMI in larger conditioning dimensions,
requires estimator-specific regularity and sample-size conditions
~\cite{gao2018demystifying}. Conditional SHAP similarly depends on the
quality of the estimated conditional distribution
~\cite{aas2021explaining}.

Causal feature-attribution formulations target a different estimand and
require causal assumptions or structural information beyond the
observational joint law~\cite{janzing2020feature}. Finally, robustness
studies show that post-hoc explanations can change under small
prediction-preserving perturbations or deliberate model manipulation in
the settings examined by those studies
~\cite{ghorbani2019fragile,slack2020fooling}. These findings motivate
stating the attribution target, dependence treatment, estimator, and
perturbation protocol explicitly rather than treating different
importance measures as interchangeable.
\par More recent interaction-attribution methods, such as Archipelago,
model and attribute higher-order feature interactions
~\cite{tsang2020does}. This objective differs from MCIR's normalization
of unique conditional information relative to a selected dependence
neighbourhood.
Complementary statistical approaches define variable importance
through predictive-risk or sensitivity functionals.
Leave-one-covariate-out (LOCO) inference measures the change in
predictive performance when a covariate is omitted
~\cite{lei2018distributionfree}, whereas algorithm-agnostic
variable-importance functionals support flexible machine-learning
estimators and statistical inference
~\cite{williamson2021nonparametric}. Shapley effects allocate explained
variation among potentially dependent inputs using a cooperative-game
functional~\cite{owen2017shapley}, and tree-specific procedures improve
the estimation of conditional Shapley values for tree-based models
~\cite{amoukou2022accurate}. These approaches target different
estimands: predictive-risk change, allocation of explained variation,
and unique conditional information need not induce the same feature
ordering. Accordingly, the present study does not treat these functionals as
interchangeable. To examine their empirical differences under a common
controlled dependence structure, Section~\ref{subsec:loco-comparison}
compares MCIR with LOCO, a region-weighted R-LOCO-style procedure
\cite{amoukou2026regional}, and a dependent-input Shapley-effects
plug-in estimator. The comparison evaluates predictive-risk change,
regional loss variation, explained-variance allocation, and unique
conditional information as distinct attribution targets.
On the other hand, the Correlation Impact Ratio formulation
\cite{sengupta2025correlation} provides a bounded measure of
feature--output co-movement. We refer to its per-feature implementation
as PCIR in the experimental tables. The same formulation includes
BlockCIR for groupwise attribution and class-conditioned ExCIR for
multi-class outputs. These methods capture shared structure and
cross-covariance geometry but do not explicitly estimate the unique
conditional information contributed by a feature relative to
$X_{\Phi(i)}$.  
Existing methods target different attribution or dependence
functionals and therefore need not produce identical rankings.
Table~\ref{tab:rw_compare} summarizes their principal distinctions,
including the CIR-family baselines evaluated in this study. The entries
describe the cited formulations; implementation-dependent properties
are not inferred.

\begin{table*}[t]
\centering
\scriptsize
\setlength{\tabcolsep}{2.8pt}
\renewcommand{\arraystretch}{0.96}
\caption{Qualified comparison of attribution and dependence-measure
families. Method descriptions follow the cited formulations
~\cite{lundberg2017shap,ribeiro2016lime,
lei2018distributionfree,williamson2021nonparametric,
owen2017shapley,
sengupta2025correlation2}, and \cite{amoukou2022accurate}. Absence of a property means that it is not
established or evaluated here.}
\label{tab:rw_compare}
\begin{adjustbox}{width=\linewidth}
\begin{tabular}{
p{2.1cm}
p{3.2cm}
p{3.1cm}
p{3.2cm}
p{2.6cm}
p{3.7cm}
}
\toprule
\textbf{Method family} &
\textbf{Attribution target} &
\textbf{Dependence treatment} &
\textbf{Redundancy interpretation} &
\textbf{Numerical scale} &
\textbf{Finite-sample and practical scope} \\
\midrule

\textbf{SHAP / LIME} &
Local contribution or local surrogate attribution, optionally
aggregated &
Depends on the value function, background distribution, perturbation
scheme, and implementation &
Shared-signal allocation depends on the coalition or perturbation
definition &
Method- and model-output-dependent &
Depends on the explainer and sampling scheme; no full--lightweight
evaluation is reported here \\

\textbf{MI / HSIC / CMI} &
Statistical dependence or incremental conditional dependence &
MI and HSIC are typically marginal; CMI conditions on specified
variables &
Population CMI vanishes under the corresponding conditional-
independence relation; MI and HSIC do not isolate conditional effects &
Raw measures are not generally normalized to $[0,1]$ &
Depends on the estimator, sample size, regularity conditions, and
conditioning dimension \\
\textbf{LOCO} &
Global predictive-risk change after excluding one feature &
Accounts for dependence indirectly through refitting or reevaluating
the predictive procedure without the selected feature
~\cite{lei2018distributionfree,williamson2021nonparametric} &
Correlated predictors may compensate for the omitted feature, so a
small LOCO value need not imply absence of predictive association &
Expressed on the scale of the selected predictive-loss difference &
Requires a specified loss, training or refitting protocol, and
evaluation sample; it targets predictive-risk change rather than
unique conditional information \\

\textbf{Region-weighted LOCO} &
Regional predictive-loss change, aggregated over data-dependent
subpopulations &
Evaluates LOCO effects within identified regions of the feature space
~\cite{amoukou2025regional} &
Can reveal heterogeneous omission effects, but does not itself
establish feature-level conditional-redundancy collapse &
Expressed on the selected loss scale and depends on regional
aggregation &
Depends on region construction, regional sample sizes, and
aggregation weights. The implementation evaluated here is an
R-LOCO-style region-weighted procedure \\

\textbf{Shapley effects} &
Cooperative allocation of explained output variance among dependent
inputs &
Uses conditional coalition values to accommodate dependent inputs
~\cite{owen2017shapley} &
Allocates shared explained variance across features rather than
isolating unique information relative to a selected neighbourhood &
Typically normalized as a share of explained variance, subject to the
chosen formulation and estimator &
Depends on conditional-distribution or conditional-expectation
estimation. The experiment uses a finite-sample dependent-input
plug-in approximation \\

\textbf{CCA / PCIR} &
Global linear or canonical feature--output association &
Uses canonical dependence under the stated preprocessing
~\cite{hotelling1936relations,sengupta2025correlation} &
Does not explicitly isolate information conditional on a
feature-specific neighbourhood &
Uses the bounded normalization specified by the formulation &
Depends on preprocessing and estimation of covariance or association
quantities \\

\textbf{BlockCIR} &
Block-level aggregation of CIR-family quantities &
Uses prespecified feature blocks &
Summarizes correlated groups but does not establish feature-level
conditional-redundancy collapse &
Uses the corresponding CIR-family normalization &
Depends additionally on block construction; no separate end-to-end
complexity claim is made here \\

\textbf{CC-CIR} &
Class-conditional CIR-family attribution &
Evaluates association within the specified class partitions &
Class conditioning does not itself establish feature-level conditional-
redundancy collapse &
Uses the corresponding class-conditional formulation &
Depends on the class partition and available class-specific
observations \\

\textbf{MCIR-M} &
Global unique predictive association relative to a selected
neighbourhood &
Conditions explicitly on the screened neighbourhood
$X_{\Phi(i)}$ &
Population score is zero when
$Q\perp\!\!\!\perp X_i\mid X_{\Phi(i)}$ &
Bounded in $[0,1]$; not universally calibrated across tasks &
Conditional rank bounds require component-error, denominator, and
margin conditions. Full--lightweight fidelity is assessed through an
operational diagnostic profile, without an automatic preservation
guarantee \\

\bottomrule
\end{tabular}
\end{adjustbox}
\end{table*}

MCIR-M differs from grouping and class-conditioning approaches by
conditioning on a feature-specific dependence neighbourhood. Its
contribution is therefore a particular redundancy-aware attribution
functional and its associated analysis, rather than a claim that
different attribution targets should yield identical rankings.

\section{Preliminaries}
\label{sec:notation}

This section introduces the population-level quantities underlying MCIR,
their empirical approximations, and the full and lightweight explanation
environments. We distinguish throughout between
(i) the number of observations used to estimate an explanation and
(ii) the ambient dimension of the model-output representation.
This distinction is essential because the projection and embedding
distances of Cai and Lim~\cite{cai2022distances} concern probability
measures defined on spaces of different ambient dimensions, rather than
empirical measures containing different numbers of observations.

\paragraph{Statistical Setup and Model-Output Representations.}
Let $(\Omega,\mathcal A,\mathbb P)$ be a probability space. For
$i\in\{1,\ldots,k\}$, let
$X_i:\Omega\rightarrow\mathcal X_i$ be a random element taking values
in a standard Borel space $\mathcal X_i$, and define
\[
\mathcal X=\prod_{i=1}^{k}\mathcal X_i,
\qquad
X=(X_1,\ldots,X_k):\Omega\rightarrow\mathcal X.
\]
The continuous experiments specialize to
$\mathcal X_i\subseteq\mathbb R$. Let $P_X$ denote the law of $X$,
and let $Y^\star$ denote the ground-truth target. A fixed measurable
predictor
$M:\mathcal X\rightarrow\mathbb R^d$
produces the full model-output representation
$Z=M(X)\in\mathbb R^d$.

Throughout the paper, MCIR explains statistical dependence between the
input features and the model-output representation $Z$, rather than
causal dependence on the ground-truth target $Y^\star$. Accordingly,
MCIR should be interpreted as a model-explanation method and not as a
causal variable-importance procedure.

For a target feature $i\in\{1,\ldots,k\}$, let
$\medmath{\Phi(i)\subseteq\{1,\ldots,k\}\setminus\{i\}}$ denote its
conditioning neighbourhood, and let
$\medmath{X_{\Phi(i)}=(X_j)_{j\in\Phi(i)}}$ denote the corresponding
random feature block. Thus, $I(Z;X_i)$ is the marginal mutual
information between the model-output representation and $X_i$, whereas
$I(Z;X_i\mid X_{\Phi(i)})$ is the conditional mutual information
remaining after accounting for the selected neighbourhood. Unless
explicitly marked with a hat, all MI and CMI quantities denote
population functionals of the joint law $P_{X,Z}$. Empirical estimators
are introduced separately in Section~\ref{sec:computational}.

\paragraph{Full Population Environment.}
We define the full population explanation environment as
\begingroup
\setlength{\abovedisplayskip}{4pt}
\setlength{\belowdisplayskip}{4pt}
\begin{equation}
\medmath{
\mathcal{U}=(P_{X,Z},M,d),
}
\label{eq:full-population-environment}
\end{equation}
\endgroup
where $P_{X,Z}$ is the joint law induced by $X\sim P_X$ and
$Z=M(X)$, and $d$ is the ambient dimension of the model-output
representation. This definition makes explicit both the probability law
with respect to which the population attribution is defined and the
fixed predictor being explained.

\paragraph{Full Empirical Environment.}
Let
\(\medmath{\mathcal{D}_{N}=\{x^{(r)}\}_{r=1}^{N}}\)
denote the full explanation dataset containing \(N\) observations.
Here, the parenthesized superscript \(r\) is an observation index:
\(x^{(r)}\) denotes the \(r\)-th sample and not exponentiation. For each
observation, let
$\medmath{z^{(r)}=M(x^{(r)})\in\mathbb{R}^{d}}$. The empirical joint
law and corresponding full empirical environment are
\begingroup
\setlength{\abovedisplayskip}{4pt}
\setlength{\belowdisplayskip}{4pt}
\begin{subequations}
\begin{align}
\widehat P_{X,Z}^{\,N}
&=
\medmath{
\frac{1}{N}
\sum_{r=1}^{N}
\delta_{(x^{(r)},z^{(r)})}},
\label{eq:full-empirical-law}\\[-1mm]
\widehat{\mathcal{U}}_{N}
&=
\medmath{
(\widehat P_{X,Z}^{\,N},M,d)}.
\label{eq:full-empirical-environment}
\end{align}
\end{subequations}
\endgroup

\paragraph{Lightweight Environment.}
Let $\mathcal{I}_{n}\subseteq\{1,\ldots,N\}$ be a selected index set
containing $n<N$ observations. In addition to observation subsampling,
the lightweight procedure may use a measurable representation map
$\medmath{G:\mathbb{R}^{d}\rightarrow\mathbb{R}^{d'}}$, with
$d'\leq d$, which produces the lightweight representation
$\medmath{Z'=G(Z)\in\mathbb{R}^{d'}}$. For the selected observations,
write
$\medmath{z'^{(r)}=G(M(x^{(r)}))}$ for
$r\in\mathcal{I}_{n}$. The lightweight empirical joint law and
corresponding lightweight environment are
\begingroup
\setlength{\abovedisplayskip}{4pt}
\setlength{\belowdisplayskip}{4pt}
\begin{subequations}
\begin{align}
\widehat P_{X,Z'}^{\,n}
&=
\medmath{
\frac{1}{n}
\sum_{r\in\mathcal{I}_{n}}
\delta_{(x^{(r)},z'^{(r)})}},
\label{eq:lightweight-empirical-law}\\[-1mm]
\widehat{\mathcal{U}}_{n}'
&=
\medmath{
(\widehat P_{X,Z'}^{\,n},G\circ M,d')}.
\label{eq:lightweight-empirical-environment}
\end{align}
\end{subequations}
\endgroup

The symbols $N$ and $n$ denote numbers of observations, whereas
$d$ and $d'$ denote the ambient dimensions of the per-observation
model-output representations. Observation subsampling changes
$N$ to $n$. Representation reduction changes $d$ to $d'$.
These operations are conceptually distinct.
If the lightweight procedure only subsamples observations, then
$G$ is the identity map and $d'=d$. In that case, the full and
lightweight empirical laws are defined on the same ambient space, and
no different-dimensional projection is required. The Cai--Lim
construction \cite{cai2022distances} becomes relevant when the lightweight procedure also
changes the representation dimension, so that $d'<d$.

\paragraph{Output Laws.}
Let $\medmath{\mu=P_Z\in\mathcal{M}_p(\mathbb{R}^{d})}$ and
$\medmath{\nu=P_{Z'}\in\mathcal{M}_p(\mathbb{R}^{d'})}$ denote the full
and lightweight population output laws, respectively, where
$\mathcal{M}_p(\mathbb{R}^{r})$ is the set of Borel probability
measures on $\mathbb{R}^{r}$ with finite $p$th moment. Their empirical
counterparts are
\begingroup
\setlength{\abovedisplayskip}{4pt}
\setlength{\belowdisplayskip}{4pt}
\begin{subequations}
\begin{align}
\widehat\mu_N
&=
\medmath{
\frac{1}{N}\sum_{r=1}^{N}\delta_{z^{(r)}}
\in\mathcal{M}_p(\mathbb{R}^{d})},
\label{eq:full-output-law}\\[-1mm]
\widehat\nu_n
&=
\medmath{
\frac{1}{n}\sum_{r\in\mathcal{I}_{n}}\delta_{z'^{(r)}}
\in\mathcal{M}_p(\mathbb{R}^{d'})}.
\label{eq:lightweight-output-law}
\end{align}
\end{subequations}
\endgroup
The numbers of atoms, $N$ and $n$, do not determine the ambient
dimensions of these empirical measures, which are $d$ and $d'$,
respectively.

\paragraph{Projection and Embedding Distances.}
Assume $d'\leq d$. Following Cai and
Lim~\cite{cai2022distances}, define the Stiefel manifold
$\medmath{\mathsf{O}(d',d)
=\{V\in\mathbb{R}^{d'\times d}:VV^\top=I_{d'}\}}$.
For $V\in\mathsf{O}(d',d)$ and $b\in\mathbb{R}^{d'}$, define the
affine projection
\begingroup
\setlength{\abovedisplayskip}{4pt}
\setlength{\belowdisplayskip}{4pt}
\begin{equation}
\medmath{
\varphi_{V,b}:\mathbb{R}^{d}\rightarrow\mathbb{R}^{d'},
\qquad
\varphi_{V,b}(z)=Vz+b.
}
\label{eq:affine-projection}
\end{equation}
\endgroup
For any probability measure $\rho$ on $\mathbb{R}^{d}$,
$\varphi_{V,b\#}\rho$ denotes its pushforward through
$\varphi_{V,b}$. The admissible projection and embedding families and
their corresponding distances are
\begingroup
\setlength{\abovedisplayskip}{4pt}
\setlength{\belowdisplayskip}{4pt}
\begin{subequations}
\begin{align}
\Phi^{-}(\mu,d')
&=
\medmath{
\left\{
\varphi_{V,b\#}\mu:
V\in\mathsf{O}(d',d),\
b\in\mathbb{R}^{d'}
\right\}},
\label{eq:projection-family}\\[-1mm]
W_p^{-}(\nu,\mu)
&=
\medmath{
\inf_{\beta\in\Phi^{-}(\mu,d')}
W_p(\nu,\beta)},
\label{eq:cai-lim-projection-distance}\\[-1mm]
\Phi^{+}(\nu,d)
&=
\medmath{
\left\{
\alpha\in\mathcal{M}_p(\mathbb{R}^{d}):
\varphi_{V,b\#}\alpha=\nu
\text{ for some }
V\in\mathsf{O}(d',d),\
b\in\mathbb{R}^{d'}
\right\}},
\label{eq:embedding-family}\\[-1mm]
W_p^{+}(\nu,\mu)
&=
\medmath{
\inf_{\alpha\in\Phi^{+}(\nu,d)}
W_p(\alpha,\mu)}.
\label{eq:cai-lim-embedding-distance}
\end{align}
\end{subequations}
\endgroup
For the $p$-Wasserstein family, Cai and Lim~\cite{cai2022distances} establish that
$\medmath{W_p^{-}(\nu,\mu)=W_p^{+}(\nu,\mu)}$. We denote their common
value by
\begingroup
\setlength{\abovedisplayskip}{4pt}
\setlength{\belowdisplayskip}{4pt}
\begin{equation}
\medmath{
\widetilde W_p(\nu,\mu)
:=
W_p^{-}(\nu,\mu)
=
W_p^{+}(\nu,\mu),
}
\label{eq:augmented-wasserstein}
\end{equation}
\endgroup
and refer to it as the augmented Wasserstein distance between the
different-dimensional representation laws.

\paragraph{Similarity Between Environments.}
For a tolerance $\varepsilon_{\mathrm{emb}}\geq0$, we define
representation-level environment similarity by
\begingroup
\setlength{\abovedisplayskip}{4pt}
\setlength{\belowdisplayskip}{4pt}
\begin{equation}
\medmath{
\mathcal{U}'
\approx_{\varepsilon_{\mathrm{emb}}}
\mathcal{U}
\quad\Longleftrightarrow\quad
\widetilde W_p(P_{Z'},P_Z)
\leq
\varepsilon_{\mathrm{emb}}.
}
\label{eq:environment-similarity}
\end{equation}
\endgroup
At the empirical level, the corresponding diagnostic is
\begingroup
\setlength{\abovedisplayskip}{4pt}
\setlength{\belowdisplayskip}{4pt}
\begin{equation}
\medmath{
\widetilde W_p
\left(
\widehat\nu_n,
\widehat\mu_N
\right).
}
\label{eq:empirical-environment-similarity}
\end{equation}
\endgroup
This definition specifies the compared probability measures, the
projection class, the base distance, and the precise meaning of
approximate equality.
\begin{wrapfigure}{r}{0.48\linewidth}
\centering
\begin{adjustbox}{width=0.95\linewidth}
\begin{tikzpicture}[
    >=latex,
    node distance=5mm,
    scale=1.05,
    transform shape
]

% Colors
\definecolor{pcirblue}{RGB}{0,100,200}
\definecolor{pcirred}{RGB}{200,60,60}
\definecolor{blockgray}{RGB}{240,240,240}
\definecolor{lightblue}{RGB}{220,240,255}
\definecolor{lightred}{RGB}{255,240,240}

% Background panel
\node[
    rectangle,
    rounded corners=6pt,
    fill=blockgray,
    draw=gray,
    thick,
    drop shadow={opacity=0.25},
    minimum width=5.7cm,
    minimum height=4.6cm,
    label=above:{\textbf{Global (unconditioned) baseline}}
] (panel) {};

% Target feature
\node[
    rectangle,
    rounded corners=4pt,
    fill=lightblue,
    draw=pcirblue!80!black,
    thick,
    minimum width=1.6cm,
    minimum height=0.9cm,
    font=\large,
    text=pcirblue!80!black
] (xi) at ([xshift=-1.45cm,yshift=0.85cm]panel.center)
{$X_i$};

% Model-output representation
\node[
    circle,
    fill=lightblue,
    draw=pcirblue!80!black,
    thick,
    minimum size=13mm,
    font=\large,
    text=pcirblue!80!black
] (z) at ([xshift=1.55cm,yshift=0.85cm]panel.center)
{$Z$};

% Global association
\draw[
    <->,
    very thick,
    pcirblue!80!black,
    line width=1.4pt
] (xi) -- node[
    above,
    font=\scriptsize,
    text=pcirblue!90!black,
    fill=white,
    inner sep=2pt
] {global association} (z);

% Correlated neighbourhood
\node[
    rectangle,
    rounded corners=4pt,
    fill=lightred,
    draw=pcirred!80!black,
    thick,
    minimum width=1.8cm,
    minimum height=0.8cm,
    font=\large,
    text=pcirred!80!black
] (xphi) at ([xshift=-1.45cm,yshift=-0.65cm]panel.center)
{$X_{\Phi(i)}$};

% Dependence between target and neighbourhood
\draw[
    <->,
    thick,
    pcirred!80!black,
    dashed
] (xi) -- node[
    right,
    font=\scriptsize,
    text=pcirred!90!black
] {dependence} (xphi);

% Show that the neighbourhood is not conditioned on
\draw[
    -,
    thick,
    gray!65,
    dashed
] (xphi.east) -- ([xshift=-2mm]z.south west);

\node[
    font=\scriptsize,
    text=pcirred!90!black,
    fill=lightred,
    rounded corners=2pt,
    inner sep=3pt,
    align=center
] at ([yshift=-1.45cm]panel.center)
{The dependence neighbourhood is not\\
explicitly conditioned on by the global baseline.};

% Interpretation box
\node[
    rectangle,
    rounded corners=4pt,
    fill=white,
    draw=pcirblue!70!black,
    thick,
    drop shadow={opacity=0.15},
    minimum width=5.3cm,
    minimum height=0.9cm,
    font=\footnotesize,
    text=pcirblue!90!black,
    align=center
] (interpretation) at ([yshift=-2.65cm]panel.center)
{\textbf{PCIR:} global feature--output association\\
without conditional redundancy adjustment};

\end{tikzpicture}
\end{adjustbox}
\caption{Conceptual view of the PCIR baseline. PCIR evaluates global
feature--output association without explicitly conditioning on the
dependence neighbourhood $X_{\Phi(i)}$. The connecting lines represent
statistical association, not causal direction.}
\vspace{2mm}
\label{fig:pcir-intuition}
\end{wrapfigure}
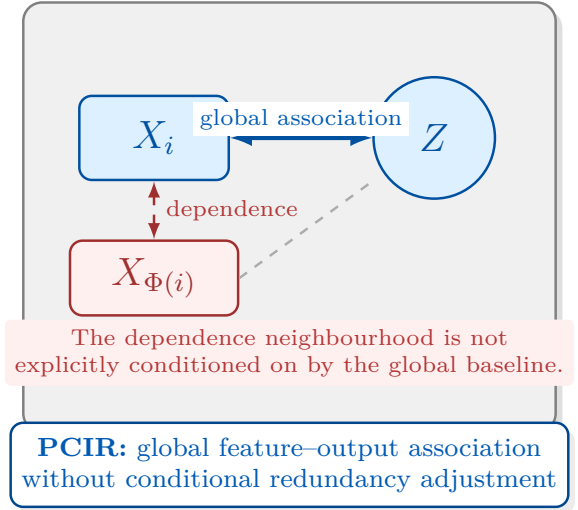

\begin{remark}[Scope of the augmented distance]
The augmented distance in
Equation~\ref{eq:augmented-wasserstein} compares probability laws
whose per-observation representations have dimensions $d'$ and $d$.
It does not treat the sample sizes $n$ and $N$ as ambient dimensions.
Moreover, $\widetilde W_p$ is not generally a metric on the original
space of probability measures: it may vanish for distinct measures
that are related through an admissible projection, rotation, or
translation. It should therefore be interpreted as a distance to an
equivalence class of representations.
\end{remark}

\begin{remark}[Effect of projection on information]
When $d'=d$, the transformation $\varphi_{V,b}$ is bijective for
orthogonal $V$. In this special case, it preserves Euclidean geometry
and the mutual information between the complete transformed output
representation and any feature block.

When $d'<d$, however, the projection is not bijective and may discard
predictive information. The conditional data-processing inequality
implies
$I(Z';X_i\mid X_{\Phi(i)})
    \leq
    I(Z;X_i\mid X_{\Phi(i)})$
when $Z'$ is a measurable function of $Z$ under the required Markov
structure. Therefore, dimensionality reduction is not assumed to
preserve MCIR automatically. The resulting information loss and
attribution discrepancy must be analyzed explicitly.
\end{remark}
\paragraph{Environment Similarity and Attribution Diagnostics.}
The representation-level similarity of the full and lightweight
environments is measured using the probability laws defined above:
\begingroup
\setlength{\abovedisplayskip}{4pt}
\setlength{\belowdisplayskip}{4pt}
\begin{subequations}
\begin{align}
d_{\mathrm{repr}}(\mathcal{U},\mathcal{U}')
&=
\medmath{
\begin{cases}
W_p(P_Z,P_{Z'}), &
d'=d \text{ with common coordinates},\\
\widetilde W_p(P_{Z'},P_Z), &
d'<d,
\end{cases}}
\label{eq:representation-discrepancy}\\[-1mm]
\widehat d_{\mathrm{repr}}
(\widehat{\mathcal{U}}_{N},\widehat{\mathcal{U}}_{n}')
&=
\medmath{
\begin{cases}
W_p(\widehat\mu_N,\widehat\nu_n), & d'=d,\\
\widetilde W_p(\widehat\nu_n,\widehat\mu_N), & d'<d.
\end{cases}}
\label{eq:empirical-representation-discrepancy}
\end{align}
\end{subequations}
\endgroup
Here, $W_p$ is the ordinary $p$-Wasserstein distance and
$\widetilde W_p$ is the augmented projection/embedding distance in
Equation~\ref{eq:augmented-wasserstein}.

Representation-level similarity alone does not guarantee attribution
similarity. MCIR depends on the joint feature--output law, whereas
Equation~\ref{eq:representation-discrepancy} compares only marginal
output laws. Explanation agreement is therefore evaluated separately.

Let
$\medmath{\widehat{\bm C}_{N}
=(\widehat C_{1,N},\ldots,\widehat C_{k,N})}$ and
$\medmath{\widehat{\bm C}_{n}'
=(\widehat C_{1,n}',\ldots,\widehat C_{k,n}')}$ denote the full and
lightweight MCIR score vectors. Their normalized Kendall discrepancy
and top-$K$ disagreement are
\begingroup
\setlength{\abovedisplayskip}{4pt}
\setlength{\belowdisplayskip}{4pt}
\begin{subequations}
\begin{align}
d_{\mathrm{rank}}
(\widehat{\bm C}_{N},\widehat{\bm C}_{n}')
&=
\medmath{
\frac{
1-\tau_b(\widehat{\bm C}_{N},\widehat{\bm C}_{n}')
}{2}},
\label{eq:rank-discrepancy}\\[-1mm]
d_{\mathrm{J}@K}
&=
\medmath{
1-
\frac{
\left|
\operatorname{TopK}(\widehat{\bm C}_{N})
\cap
\operatorname{TopK}(\widehat{\bm C}_{n}')
\right|
}{
\left|
\operatorname{TopK}(\widehat{\bm C}_{N})
\cup
\operatorname{TopK}(\widehat{\bm C}_{n}')
\right|
}}.
\label{eq:jaccard-discrepancy}
\end{align}
\end{subequations}
\endgroup
Here, $\tau_b$ is Kendall's coefficient with tie correction. The
quantities $d_{\mathrm{repr}}$, $d_{\mathrm{rank}}$, and
$d_{\mathrm{J}@K}$ assess representation-distribution agreement,
complete-ranking agreement, and head-ranking agreement, respectively.
They are reported jointly as diagnostics, but ranking agreement is not
included in the definition of environment similarity or assumed in
order to establish itself as a conclusion.

\paragraph{Population versus Empirical Quantities.}
Population MI and CMI quantities are defined with respect to the
appropriate joint laws. In the full and lightweight environments,
respectively, the relevant quantities are
\begingroup
\setlength{\abovedisplayskip}{4pt}
\setlength{\belowdisplayskip}{4pt}
\begin{equation}
\medmath{
\begin{aligned}
U_i
&=
I_{P_{X,Z}}(Z;X_i\mid X_{\Phi(i)}),
&
J_i
&=
I_{P_{X,Z}}(Z;X_{\Phi(i)\cup\{i\}}),\\[-1mm]
U_i'
&=
I_{P_{X,Z'}}(Z';X_i\mid X_{\Phi(i)}),
&
J_i'
&=
I_{P_{X,Z'}}(Z';X_{\Phi(i)\cup\{i\}}).
\end{aligned}}
\label{eq:population-information-quantities}
\end{equation}
\endgroup
Empirical estimators computed from $\widehat P_{X,Z}^{\,N}$ and
$\widehat P_{X,Z'}^{\,n}$ are denoted by $\widehat I_N$ and
$\widehat I_n'$, respectively. Population properties of MCIR are stated
separately from estimator assumptions. Estimator consistency,
concentration, and dimensional limitations are discussed in
Section~\ref{sec:computational} and
Appendix~\ref{appendix:statistical}.

\paragraph{PCIR as a Global Baseline.}
We include PCIR~\cite{sengupta2025correlation2} only as a historical
global baseline. PCIR evaluates feature--output association without
explicitly conditioning on the dependence neighbourhood
$X_{\Phi(i)}$ and therefore cannot separate information uniquely
contributed by $X_i$ from information shared with correlated features.

The original PCIR dispersion formulation combines feature and output
quantities within a common pooled construction. Its numerical value may
therefore depend on their relative parameterizations and preprocessing.
For transparency, its exact definition, preprocessing procedure, and
experimental implementation are provided in
Appendix~\ref{app:pcir}. PCIR is not used to establish the
measure-theoretic or information-theoretic properties of MCIR.

Figure~\ref{fig:pcir-intuition} illustrates this conceptual limitation:
PCIR evaluates the association between $X_i$ and the model-output
representation $Z$ without conditioning on the correlated block
$X_{\Phi(i)}$. MCIR instead measures the conditional information
contributed by $X_i$ beyond its selected dependence neighbourhood.
\section{MCIR: Formal Definition and Analysis}
\label{prob}

This section defines MCIR as a population feature-attribution
functional and establishes its structural properties. Finite-sample
estimation, ranking stability, neighbourhood selection, and
lightweight-environment fidelity are considered separately in
Section~\ref{sec:computational}. This distinction separates properties
of the population functional from assumptions imposed on particular
MI and CMI estimators.

\subsection{Population MCIR Formulation}
\label{sec:mcir_formulation}

We first specify the environment-dependent model-output representation
and its joint law with the input features.

\subsubsection{Environment-Specific Output Representation}

Let $Q$ denote the model-output representation associated with the
environment under consideration: $\medmath{Q=Z\in\mathbb{R}^{d}}$ in
the full environment and
$\medmath{Q=Z'\in\mathbb{R}^{d'}}$ in the lightweight environment.
The corresponding joint law is denoted generically by $P_{X,Q}$, with
$\medmath{P_{X,Q}=P_{X,Z}}$ or
$\medmath{P_{X,Q}=P_{X,Z'}}$, respectively. The MCIR definition
therefore applies to either environment through the appropriate joint
law.

Fix a feature index $i\in\{1,\ldots,k\}$. Let
$\medmath{\Phi(i)\subseteq\{1,\ldots,k\}\setminus\{i\}}$ denote its
conditioning neighbourhood and define
$\medmath{X_{\Phi(i)}=(X_j)_{j\in\Phi(i)}}$. The enlarged feature block
containing both the target feature and its neighbourhood is denoted by
$\medmath{X_{\Phi(i)\cup\{i\}}}$.

Here, $X_i$, $X_{\Phi(i)}$, and $Q$ are population random variables or
random vectors. Their observed realizations enter only through the
empirical estimators introduced later. Thus, columns of an observed
feature matrix represent samples from the corresponding random
variables rather than population quantities themselves.

\subsubsection{Standard-Borel Well-Posedness}

\begin{assumption}[Population well-posedness]
\label{asm:wellposed}
For each $Q\in\{Z,Z'\}$ and $i\in\{1,\ldots,k\}$:

\begin{enumerate}[label=(\roman*),leftmargin=*,nosep]
    \item the random elements $Q$, $X_i$, and $X_{\Phi(i)}$ take values
    in standard Borel spaces;

    \item the information carried by the enlarged feature block is
    finite:
    \begin{equation}
        \medmath{
        I_{P_{X,Q}}
        \left(
            Q;X_{\Phi(i)\cup\{i\}}
        \right)<\infty.
        }
        \label{eq:finite-block-information}
    \end{equation}
\end{enumerate}
\end{assumption}

The standard Borel-space formulation ensures the existence of regular
conditional probability kernels and accommodates continuous, discrete,
and mixed-type variables within a common measure-theoretic framework.
When densities or probability-mass functions are used, they are
understood as Radon--Nikodym derivatives with respect to the relevant
dominating measures.
By the mutual-information chain rule,
\begin{equation}
    \medmath{
    I_{P_{X,Q}}
    \left(
        Q;X_{\Phi(i)\cup\{i\}}
    \right)
    =
    I_{P_{X,Q}}
    \left(
        Q;X_{\Phi(i)}
    \right)
    +
    I_{P_{X,Q}}
    \left(
        Q;X_i\mid X_{\Phi(i)}
    \right).
    }
    \label{eq:block-mi-chain-rule}
\end{equation}
Since mutual information and conditional mutual information are
non-negative, Assumption~\ref{asm:wellposed} guarantees the finiteness
of every population information quantity entering MCIR.

Assumption~\ref{asm:wellposed} concerns only the population functional.
Conditions governing empirical consistency, concentration, convergence
rates, and dimensional dependence are stated separately in
Section~\ref{sec:computational}.

\subsubsection{Mixed-Variable Conditional Mutual Information}

MCIR combines two complementary information quantities. Conditional
mutual information measures the information about $Q$ contributed by
$X_i$ beyond the selected neighbourhood $X_{\Phi(i)}$, whereas joint
mutual information measures the total information about $Q$ carried by
the enlarged block $X_{\Phi(i)\cup\{i\}}$.
\begin{definition}[\textbf{Mixed-Variable Conditional Mutual Information}]
\label{def:cmi}
Let $X_i$, $X_{\Phi(i)}$, and $Q$ be random elements that may be
continuous, discrete, or mixed-type. Their conditional mutual
information under $P_{X,Q}$ is defined as
\begin{equation}
\medmath{
\begin{aligned}
I_{P_{X,Q}}
\left(
    Q;X_i\mid X_{\Phi(i)}
\right)
:=
\int
D_{\mathrm{KL}}
\Big(
&P_{Q,X_i\mid X_{\Phi(i)}=x_{\Phi}}
\\[-1mm]
&\big\|
P_{Q\mid X_{\Phi(i)}=x_{\Phi}}
\otimes
P_{X_i\mid X_{\Phi(i)}=x_{\Phi}}
\Big)
\,P_{X_{\Phi(i)}}(dx_{\Phi}).
\end{aligned}
}
\label{eq:mixed-cmi-kernel-definition}
\end{equation}
Equivalently, using conditional Radon--Nikodym derivatives,
\begin{equation}
\medmath{
I_{P_{X,Q}}
\left(
    Q;X_i\mid X_{\Phi(i)}
\right)
=
\mathbb{E}_{P_{Q,X_i,X_{\Phi(i)}}}
\left[
\log
\frac{
dP_{Q\mid X_i,X_{\Phi(i)}}
}{
dP_{Q\mid X_{\Phi(i)}}
}(Q)
\right].
}
\label{eq:mixed-cmi-rn-representation}
\end{equation}
When compatible conditional densities or probability-mass functions
exist, Equation~\ref{eq:mixed-cmi-rn-representation} becomes
\begin{equation}
\medmath{
I_{P_{X,Q}}
\left(
    Q;X_i\mid X_{\Phi(i)}
\right)
=
\mathbb{E}_{P_{Q,X_i,X_{\Phi(i)}}}
\left[
\log
\frac{
p(Q\mid X_i,X_{\Phi(i)})
}{
p(Q\mid X_{\Phi(i)})
}
\right].
}
\label{eq:mixed-cmi-density-representation}
\end{equation}
\end{definition}

The expectation in
Equation~\ref{eq:mixed-cmi-rn-representation} is taken with respect
to the joint law of $\medmath{(Q,X_i,X_{\Phi(i)})}$. Equivalently,
Equation~\ref{eq:mixed-cmi-kernel-definition} averages the
conditional KL divergence over the marginal law of $X_{\Phi(i)}$.
Consequently,
\begin{equation}
\medmath{
I_{P_{X,Q}}
\left(
    Q;X_i\mid X_{\Phi(i)}
\right)\geq 0,
}
\label{eq:mixed-cmi-nonnegative}
\end{equation}
with equality if and only if
$\medmath{Q\perp\!\!\!\perp X_i\mid X_{\Phi(i)}}$, up to
$P_{X_{\Phi(i)}}$-null sets.

Definition~\ref{def:cmi} provides a common measure-theoretic
formulation for continuous, discrete, and mixed-type variables. When
all variables are discrete, it reduces to the standard sum-based CMI
expression; when all variables are continuous, it reduces to the
corresponding integral-based expression. Mixed feature spaces are
handled through the appropriate combination of probability kernels,
densities, and probability-mass functions.
\begin{wrapfigure}{r}{0.7\linewidth}
\centering
\begin{adjustbox}{width=0.95\linewidth}
\begin{tikzpicture}[>=latex,scale=1.08,transform shape]

% Colors
\definecolor{mcirblue}{RGB}{0,100,200}
\definecolor{mcirred}{RGB}{200,60,60}
\definecolor{blockgray}{RGB}{240,240,240}
\definecolor{lightblue}{RGB}{220,240,255}
\definecolor{lightred}{RGB}{255,240,240}

% Selected dependence neighbourhood
\node[
    rectangle,
    rounded corners=6pt,
    fill=blockgray,
    draw=gray,
    thick,
    drop shadow={opacity=0.3},
    minimum width=5.8cm,
    minimum height=2.8cm,
    label=above:{\textbf{Selected dependence neighbourhood}}
] (block) at (0,0) {};

% Target feature
\node[
    rectangle,
    rounded corners=4pt,
    fill=lightblue,
    draw=mcirblue!80!black,
    thick,
    font=\large,
    minimum width=1.65cm,
    minimum height=0.9cm,
    text=mcirblue!80!black
] (xi) at (-1.95,0) {$X_i$};

% Neighbourhood features
\node[
    rectangle,
    rounded corners=4pt,
    fill=lightblue,
    draw=mcirblue!80!black,
    thick,
    font=\large,
    minimum width=1.65cm,
    minimum height=0.9cm,
    text=mcirblue!80!black
] (xphi1) at (0,0) {$X_{\phi_1}$};

\node[
    rectangle,
    rounded corners=4pt,
    fill=lightblue,
    draw=mcirblue!80!black,
    thick,
    font=\large,
    minimum width=1.65cm,
    minimum height=0.9cm,
    text=mcirblue!80!black
] (xphi2) at (1.95,0) {$X_{\phi_2}$};

% Shared information
\draw[
    thick,
    mcirred,
    decorate,
    decoration={zigzag,segment length=3pt,amplitude=2pt}
]
(xi.north east)
.. controls +(1.3,0.9) and +(-1.3,0.9) ..
(xphi2.north west);

\node[
    font=\scriptsize,
    text=mcirred!85!black,
    fill=lightred,
    rounded corners=2pt,
    inner sep=2pt
] at (0,0.95) {shared predictive information};

% Unique information
\node[
    font=\scriptsize,
    text=mcirblue!90!black,
    fill=lightblue,
    rounded corners=2pt,
    inner sep=2pt,
    align=center
] (unique) at (-1.95,-0.85)
{unique information\\given $X_{\Phi(i)}$};

% Model-output representation
\node[
    circle,
    draw=mcirblue!70!black,
    thick,
    fill=lightblue,
    minimum size=14mm,
    font=\large,
    text=mcirblue!80!black
] (q) at (0,-2.25) {$Q$};

% Statistical-dependence paths
\draw[->,thick,mcirblue!70,line width=1.2pt]
    (xi.south) to[out=-90,in=150] (q.north west);

\draw[->,thick,mcirblue!70,line width=1.2pt]
    (xphi1.south) to[out=-90,in=90] (q.north);

\draw[->,thick,mcirblue!70,line width=1.2pt]
    (xphi2.south) to[out=-90,in=30] (q.north east);

\node[
    font=\scriptsize,
    text=gray!80!black,
    fill=white,
    inner sep=1.5pt
] at (0,-1.35) {predictive dependence};

% MCIR ratio
\node[
    rectangle,
    rounded corners=4pt,
    fill=white,
    draw=mcirred!70!black,
    thick,
    drop shadow={opacity=0.2},
    minimum width=6.4cm,
    minimum height=1.45cm,
    font=\footnotesize,
    text=mcirred!90!black,
    align=center
] (ratio) at (0,-3.75)
{$\displaystyle
\operatorname{MCIR}_i(P_{X,Q})
=
\frac{2U_i(P_{X,Q})}
     {U_i(P_{X,Q})+J_i(P_{X,Q})}
\in[0,1]
$\\ 
\\
$\scriptstyle
U_i=I_{P_{X,Q}}(Q;X_i\mid X_{\Phi(i)}),\qquad
J_i=I_{P_{X,Q}}(Q;X_{\Phi(i)\cup\{i\}})
$};

\node[
    below=2mm of ratio,
    text=gray!75!black,
    font={\itshape\footnotesize},
    align=center
]
{Unique contribution is defined relative to
the selected neighbourhood $X_{\Phi(i)}$.};

\end{tikzpicture}
\end{adjustbox}
\caption{MCIR intuition. The target feature $X_i$ and its selected
dependence neighbourhood
$X_{\Phi(i)}=\{X_{\phi_1},X_{\phi_2}\}$ may carry shared predictive
information about the model-output representation $Q$. MCIR conditions
on $X_{\Phi(i)}$ to isolate the \emph{unique predictive information}
contributed by $X_i$ relative to that neighbourhood and normalizes it
to $[0,1]$. The arrows represent predictive dependence, not causal
effects.}
\label{fig:mcir-intuition}
\end{wrapfigure}
Mixed-variable CMI measures the \emph{unique predictive information}
about $Q$ contributed by $X_i$ after accounting for the selected
dependence neighbourhood $X_{\Phi(i)}$. Here, uniqueness is defined
relative to $X_{\Phi(i)}$: it is the predictive information that
remains available from $X_i$ after the information carried by the
selected neighbourhood has been taken into account.
If $X_i$ is conditionally redundant given $X_{\Phi(i)}$, then
\begin{equation}
\medmath{
P_{Q\mid X_i,X_{\Phi(i)}}
=
P_{Q\mid X_{\Phi(i)}}
\quad
P_{X_i,X_{\Phi(i)}}\text{-a.s.}
}
\label{eq:conditional-redundancy}
\end{equation}
and its unique predictive information is zero. Conversely, a positive
CMI indicates that including $X_i$ changes the conditional law of $Q$
beyond the information already contained in $X_{\Phi(i)}$. Therefore,
a positive value identifies predictive information uniquely
contributed by $X_i$ relative to its selected dependence
neighbourhood. This notion of uniqueness concerns statistical
predictive information and does not imply a causal effect.
The term \emph{mixed-variable CMI} refers to the evaluation of
the same population CMI functional when continuous, discrete, and
mixed-type variables occur within a common feature space. It is not
introduced as a new information measure. The methodological
contribution of MCIR is to use mixed-variable CMI to isolate the
\emph{unique predictive information} of a feature relative to its
selected dependence neighbourhood and to normalize this quantity into
a bounded, dependence-aware feature-attribution score.

\subsubsection{Joint Mutual Information}

\begin{definition}[\textbf{Joint Mutual Information}]
\label{def:jmi}
For an index set $S\subseteq\{1,\ldots,k\}$, let
$\medmath{X_S=(X_j)_{j\in S}}$. The joint mutual information between
$Q$ and $X_S$ is
\begin{equation}
\medmath{
\begin{aligned}
I_{P_{X,Q}}(Q;X_S)
&:=
D_{\mathrm{KL}}
\left(
    P_{Q,X_S}
    \,\middle\|\,
    P_Q\otimes P_{X_S}
\right)\\
&=
\mathbb{E}_{P_{Q,X_S}}
\left[
    \log
    \frac{
        dP_{Q,X_S}
    }{
        d(P_Q\otimes P_{X_S})
    }(Q,X_S)
\right]
\geq 0.
\end{aligned}
}
\label{eq:jmi-definition}
\end{equation}
When compatible densities or probability-mass functions exist, the
second line can be written as
\begin{equation}
\medmath{
I_{P_{X,Q}}(Q;X_S)
=
\mathbb{E}_{P_{Q,X_S}}
\left[
\log
\frac{p(Q,X_S)}
     {p(Q)\,p(X_S)}
\right].
}
\label{eq:jmi-density-representation}
\end{equation}
Equality holds if and only if
$\medmath{Q\perp\!\!\!\perp X_S}$.
\end{definition}

Joint mutual information measures the total predictive information
about $Q$ carried collectively by a feature block. In particular, for
$S=\Phi(i)\cup\{i\}$, it includes information uniquely contributed by
$X_i$ as well as information carried by or shared with its dependence
neighbourhood.

\subsubsection{Information Components and MCIR Score}

For each feature $i$, define,
\begin{equation}
\medmath{
\begin{aligned}
U_i(P_{X,Q})
&:=
I_{P_{X,Q}}
\left(
    Q;X_i\mid X_{\Phi(i)}
\right),\qquad
J_i(P_{X,Q})
&:=
I_{P_{X,Q}}
\left(
    Q;X_{\Phi(i)\cup\{i\}}
\right).
\end{aligned}
}
\label{eq:mcir-information-components}
\end{equation}
The quantity $U_i(P_{X,Q})$ measures the \emph{unique predictive
information} contributed by $X_i$ beyond its selected dependence
neighbourhood, whereas $J_i(P_{X,Q})$ measures the total predictive
information carried by the enlarged feature block
$X_{\Phi(i)\cup\{i\}}$. MCIR combines these complementary population
quantities through a dimensionless normalization that distinguishes
the unique predictive contribution of $X_i$ from information shared
with its correlated neighbourhood.
\begin{definition}[\textbf{Mutual Correlation Impact Ratio (MCIR)}]
\label{def:mcir}
For feature $X_i$ and its selected dependence neighbourhood
$X_{\Phi(i)}$, define
\begin{equation}
\medmath{
\begin{aligned}
C_i(P_{X,Q};\Phi(i))
&:=
\frac{
2U_i(P_{X,Q})
}{
U_i(P_{X,Q})+J_i(P_{X,Q})
}=
\frac{
2I_{P_{X,Q}}(Q;X_i\mid X_{\Phi(i)})
}{
I_{P_{X,Q}}(Q;X_i\mid X_{\Phi(i)})
+
I_{P_{X,Q}}(Q;X_{\Phi(i)\cup\{i\}})
}
\in[0,1],
\end{aligned}
}
\label{eq:mcir_raw}
\end{equation}
whenever $\medmath{U_i(P_{X,Q})+J_i(P_{X,Q})>0}$. If both
information quantities vanish, we set
$\medmath{C_i(P_{X,Q};\Phi(i))=0}$.
\end{definition}

By the chain rule,
\begin{equation}
\medmath{
J_i(P_{X,Q})
=
I_{P_{X,Q}}(Q;X_{\Phi(i)})
+
U_i(P_{X,Q})
\geq U_i(P_{X,Q}),
}
\label{eq:mcir-chain-relation}
\end{equation}
which yields the unit-interval range in
Equation~\ref{eq:mcir_raw}. The numerator represents the
\emph{unique predictive information} contributed by $X_i$ beyond its
selected neighbourhood, whereas the denominator normalizes this
contribution relative to the information carried by the enlarged
feature block.

All MI and CMI quantities below are population functionals of the joint
law $P_{X,Q}$ and are finite under
Assumption~\ref{asm:wellposed}. For brevity, write
\begin{equation}
\medmath{
U_i=I_{P_{X,Q}}(Q;X_i\mid X_{\Phi(i)}),\qquad
H_i=I_{P_{X,Q}}(Q;X_{\Phi(i)}),\qquad
J_i=H_i+U_i.
}
\label{eq:mcir-information-decomposition}
\end{equation}
The MCIR score can therefore be written as
\begin{equation}
\medmath{
C_i
=
\frac{2U_i}{U_i+J_i}
=
\frac{2U_i}{2U_i+H_i},
}
\label{eq:mcir-decomposed}
\end{equation}
with $C_i=0$ when $U_i+J_i=0$.

A value near zero indicates that $X_i$ contributes little unique
predictive information after conditioning on $X_{\Phi(i)}$. A value
near one indicates that the unique contribution of $X_i$ dominates the
predictive information carried by the enlarged block. Uniqueness is
always defined relative to the selected neighbourhood and does not
imply causal importance.

Throughout this paper, MCIR refers exclusively to
Definition~\ref{def:mcir} and Equation~\ref{eq:mcir_raw}. All
population results, estimators, algorithms, and experiments are linked
to this definition. Figure~\ref{fig:mcir-intuition} illustrates how
MCIR separates the unique predictive information contributed by
$X_i$ from information shared with $X_{\Phi(i)}$.

In summary, MCIR (i) conditions on a selected dependence neighbourhood
to isolate the \emph{unique predictive information} of a feature,
(ii) normalizes this information to a unit-interval score, and
(iii) assigns zero importance under exact conditional redundancy.
Population properties and estimator-dependent guarantees are
established separately in the following sections.

\paragraph{Relationship to Conditional Mutual Information.}
Conditional mutual information
$\medmath{I_{P_{X,Q}}(Q;X_i\mid X_{\Phi(i)})}$ measures the unique
predictive information contributed by $X_i$ beyond
$X_{\Phi(i)}$, but its magnitude is not bounded. MCIR converts this
quantity into a bounded attribution score by normalizing it jointly
with
$\medmath{I_{P_{X,Q}}(Q;X_{\Phi(i)\cup\{i\}})}$, the total predictive
information carried by the enlarged feature block. The normalization
therefore expresses the unique contribution of $X_i$ relative to the
information associated with its local dependence structure.

\paragraph{Role of Joint Normalization.}
Raw CMI quantifies the absolute conditional information contribution
but does not indicate its magnitude relative to the information carried
by the complete feature block. The joint-information term supplies this
reference quantity. The factor of two maps the effective ratio to
$[0,1]$, since
$\medmath{0\leq U_i(P_{X,Q})\leq J_i(P_{X,Q})}$. Thus, MCIR represents
relative unique predictive information rather than an unbounded
information magnitude.

\paragraph{Relationship to Partial Information Decomposition.}
Partial Information Decomposition (PID) seeks a complete decomposition
of information into unique, redundant, and synergistic components.
MCIR has a narrower attribution objective: it measures the information
uniquely contributed by one feature relative to a selected conditioning
neighbourhood and normalizes that quantity using block-level mutual
information. MCIR therefore does not claim to provide a full PID or to
identify separate synergistic components.

MCIR is constructed from mixed-variable CMI but differs from
raw CMI in purpose. Mixed-variable CMI provides a common
measure-theoretic treatment of continuous, discrete, and mixed-type
variables, while MCIR uses it to isolate the \emph{unique predictive
information} of $X_i$ relative to $X_{\Phi(i)}$ and converts that
quantity into a bounded dependence-aware attribution score. The
contribution is therefore the MCIR attribution functional, rather than
a new definition or estimator of conditional mutual information.

% \begin{remark}
% $I(Y'; f_i\mid f_\Phi)$ isolates the \emph{unique} contribution of $f_i$
% beyond its correlated partners $f_\Phi$. The joint term
% $I(Y'; f_{\Phi\cup\{i\}})$ stabilizes scale across tasks and dependence strengths.
% Thus $C$ reports the fraction of \emph{explainable association mass} uniquely attributable
% to $f_i$ after accounting for partners. 
% \end{remark}
Equation~\ref{eq:mcir-zero-characterization} establishes exact
redundancy collapse at the population level through conditional
independence. Approximate redundancy in finite samples is considered
separately through estimator-error and ranking-sensitivity diagnostics.

The preceding definition treats MCIR as a population functional under
a fixed joint law $P_{X,Q}$ and a fixed dependence neighbourhood
$\Phi(i)$. Its endpoint behaviour, exact redundancy collapse, and
weak-dependence reduction are developed below. Approximate redundancy
and ranking behaviour in finite samples are considered separately
through estimator-error and sensitivity analyses.

Population properties of the information functional are
distinguished from properties of particular finite-sample estimators.
In particular, rank-based invariance of the Gaussian--copula estimator
is treated as an estimator-specific property in
Section~\ref{sec:computational}.

\subsection{Population Properties}
\label{sec:population-properties}

We next summarize the structural population properties of MCIR.
Boundedness places all scores on the common interval $[0,1]$, although
it does not by itself imply identical calibration across different data
distributions or tasks. Exact conditional redundancy yields a score of
zero, whereas a score of one is attained when $X_i$ carries positive
predictive information and its selected neighbourhood carries no
predictive information about $Q$ without $X_i$. 
For exact duplicates, MCIR collapses the contribution of a feature
when that duplicate is contained in its conditioning neighbourhood.
Approximate duplicates require separate finite-sample analysis and are
therefore evaluated through estimator-error and ranking-sensitivity
diagnostics rather than an assumed KL-continuity argument. Moreover,
without an external ordering or an asymmetric neighbourhood rule,
statistical dependence alone cannot determine which of two exact
duplicates is the ``original'' feature. The following proposition
formalizes the population properties, while estimator-dependent
behaviour is considered separately in
Section~\ref{sec:computational}.
To help the reader navigate the technical results of Section~\ref{prob}, 
Table~\ref{tab:theory_glance} summarizes the core structural and statistical 
guarantees satisfied by MCIR. Detailed statements and proofs follow in the 
subsequent subsections and appendices.

\begin{table}[t]
\centering
\scriptsize
\caption{Theoretical properties of MCIR and their scope.}
\label{tab:theory_glance}
\begin{adjustbox}{width=\linewidth}
\begin{tabular}{p{2.0cm} p{6.8cm} p{5cm}}
\toprule
\textbf{Property} &
\textbf{Formal Statement} &
\textbf{Scope and Interpretation} \\
\midrule

\textbf{Boundedness} &
$\medmath{0\leq C_i\leq1}$ &
Provides a common numerical range, but not identical calibration across
different data-generating distributions. \\

\textbf{Unique-Information Collapse} &
$\medmath{
C_i=0
\Longleftrightarrow
I_{P_{X,Q}}(Q;X_i\mid X_{\Phi(i)})=0
\Longleftrightarrow
Q\perp\!\!\!\perp X_i\mid X_{\Phi(i)}
\ \mathrm{a.s.}}$ &
A feature receives zero attribution when it contributes no unique
predictive information beyond its fixed neighbourhood. \\

\textbf{Population Invariance} &
$\medmath{\widetilde C_i=C_i}$ under bimeasurable transformations of
$Q$, $X_i$, and $X_{\Phi(i)}$ &
Population MCIR is unchanged by equivalent reparameterizations,
including strictly monotone changes of units. \\

\textbf{Weak-Dependence Reduction} &
If $\medmath{U_i^{(m)}=u_i^{(m)}+o(1)}$, then
$\medmath{
C_i^{(m)}
=
2u_i^{(m)}/
\{2u_i^{(m)}+H_i^{(m)}\}
+o(1)}$ under denominator separation &
MCIR becomes a normalized marginal-information score. This does not
imply ordering equivalence with PCIR or another attribution method. \\

\textbf{Finite-Sample Rank Stability} &
$\medmath{
\mathbb P\!\left(
d_{\mathrm{inv}}
\leq16M\delta/c_0
\right)
\geq1-m_{\mathcal I}\alpha}$ &
Requires explicit estimator-error, denominator-separation, and
population-margin conditions from
Theorem~\ref{thm:rank_stability}. \\

\textbf{Empirical-Score Consistency} &
If $\medmath{\widehat U_{i,s}\xrightarrow{p}U_i}$,
$\medmath{\widehat J_{i,s}\xrightarrow{p}J_i}$, and
$\medmath{U_i+J_i>0}$, then
$\medmath{\widehat C_{i,s}\xrightarrow{p}C_i}$ &
Consistency is inherited from the selected MI and CMI estimators; no
common convergence rate is assumed. \\

\bottomrule
\end{tabular}
\end{adjustbox}
\end{table}

Estimator switching and lightweight fidelity are not listed as
population guarantees in Table~\ref{tab:theory_glance}. Estimator
switching is an empirically evaluated implementation rule, while the
lightweight fidelity profile is a multidimensional diagnostic.

\subsubsection{Boundedness and Endpoints}

Having decomposed the block information into the neighbourhood contribution
$H_i$ and the incremental contribution $U_i$, we first characterize the
range and endpoints of MCIR. This clarifies how the normalization represents
the relative information contributed by $X_i$ beyond $X_{\Phi(i)}$.
Since $U_i,H_i\geq0$, Equation~\ref{eq:mcir-decomposed} immediately
gives $\medmath{0\leq C_i\leq1}$. More precisely, the lower endpoint
occurs when $U_i=0$, meaning that $X_i$ contributes no additional
information beyond its neighbourhood. The upper endpoint occurs when
$U_i>0$ and $H_i=0$, meaning that the enlarged block's predictive
information is contributed entirely by $X_i$. When $U_i=H_i=0$, the
zero-denominator convention gives $C_i=0$. These endpoint properties
provide the basis for the conditional-redundancy result developed next.
\begin{equation}
\medmath{
C_i=0
\quad\Longleftrightarrow\quad
U_i=0
\quad\Longleftrightarrow\quad
Q\perp\!\!\!\perp X_i\mid X_{\Phi(i)}
\ \text{a.s.}
}
\label{eq:mcir-zero-characterization}
\end{equation}
Thus, MCIR assigns zero attribution when $X_i$ contributes no unique
predictive information beyond its selected dependence neighbourhood.

If $U_i>0$, then
\begin{equation}
\medmath{
C_i=1
\quad\Longleftrightarrow\quad
H_i=I_{P_{X,Q}}(Q;X_{\Phi(i)})=0.
}
\label{eq:mcir-one-characterization}
\end{equation}
Hence, MCIR reaches one when $X_i$ carries positive predictive
information while its selected neighbourhood alone carries no
predictive information about $Q$. Boundedness places MCIR scores on a
common numerical scale, but does not by itself imply identical
calibration across different data-generating distributions.

The endpoint properties in
Equations~\eqref{eq:mcir-zero-characterization} and
\eqref{eq:mcir-one-characterization} follow directly from the
population MCIR functional. Supporting derivations are provided in
Appendix~\ref{sec:theory1.1}--\ref{sec:theory1.3}.

\subsubsection{Exact Conditional-Redundancy Collapse}

The lower endpoint has a direct redundancy interpretation. If $X_i$ is
measurable with respect to $\sigma(X_{\Phi(i)})$, then its value is
already determined by the selected neighbourhood. Conditioning on
$X_{\Phi(i)}$ therefore leaves no additional information attributable
to $X_i$, and
\begin{equation}
\medmath{
I_{P_{X,Q}}(Q;X_i\mid X_{\Phi(i)})=0
\quad\Longrightarrow\quad
C_i=0.
}
\label{eq:exact-duplicate-collapse}
\end{equation}
This includes exact duplicates and, more generally, any measurable
transformation of variables contained in the neighbourhood. Thus, MCIR
assigns no additional population-level credit to a feature whose
information is already represented by $X_{\Phi(i)}$. Having established
this exact collapse property, we next examine whether the population
score is preserved under equivalent reparameterizations.

For another feature $X_j$, evaluated relative to its own neighbourhood
$X_{\Phi(j)}$, the implication
\begin{equation}
\medmath{
I_{P_{X,Q}}(Q;X_j\mid X_{\Phi(j)})>0
\quad\Longrightarrow\quad
C_j>0
}
\label{eq:positive-unique-information}
\end{equation}
follows directly from Definition~\ref{def:mcir}. Which feature retains
positive attribution therefore depends on the specified conditioning
neighbourhoods. Approximate duplicates are evaluated separately through
finite-sample estimator and neighbourhood-sensitivity analysis.

The term \emph{unique predictive information} refers to
information contributed by $X_i$ relative to its selected dependence
neighbourhood $X_{\Phi(i)}$. It does not identify a causal driver or
distinguish an ``original'' feature from an exact duplicate without
additional structural information.

\subsubsection{Population Invariance}

The population functional depends on the information carried by the
random elements rather than on a particular equivalent
parameterization of their state spaces.

\begin{proposition}[Population reparameterization invariance]
\label{prop:population-invariance}
Let $T_Q$, $T_i$, and $T_\Phi$ be bimeasurable bijections on the state
spaces of $Q$, $X_i$, and $X_{\Phi(i)}$, respectively. Define
\[
\widetilde Q=T_Q(Q),\qquad
\widetilde X_i=T_i(X_i),\qquad
\widetilde X_{\Phi(i)}=T_\Phi(X_{\Phi(i)}).
\]
Then
\[
I(\widetilde Q;\widetilde X_i\mid\widetilde X_{\Phi(i)})=U_i,
\qquad
I(\widetilde Q;
  \widetilde X_{\Phi(i)},\widetilde X_i)=J_i,
\]
and consequently the transformed population MCIR score satisfies
$\widetilde C_i=C_i$.
\end{proposition}

The result follows from invariance of relative entropy under
bimeasurable pushforwards. Its proof is given in
Appendix~\ref{sec:theory1.3}. This population property is distinct from
the estimator-specific monotone invariance of the rank--Gaussianized
Gaussian--copula estimator.

\subsubsection{Weak-Dependence Reduction}
The preceding properties address exact redundancy and reparameterization.
We now consider the complementary regime in which conditioning on the
selected neighbourhood has an asymptotically negligible effect on the
information attributed to $X_i$. The following result characterizes the
resulting MCIR expression without asserting equivalence with PCIR or any
other attribution method.
\begin{proposition}[\textbf{Weak-Dependence Reduction}]
\label{prop:weak_dep_correct}
Consider a sequence of joint laws $\{P_{X,Q}^{(m)}\}_{m\geq1}$ and
define
\begin{equation}
\medmath{
u_i^{(m)}
=
I_{P_{X,Q}^{(m)}}(Q;X_i),\qquad
H_i^{(m)}
=
I_{P_{X,Q}^{(m)}}(Q;X_{\Phi(i)}).
}
\end{equation}
Suppose
\begin{equation}
\medmath{
U_i^{(m)}
=
u_i^{(m)}+r_i^{(m)},\qquad
r_i^{(m)}\longrightarrow0,
}
\label{eq:weak-dependence-condition}
\end{equation}
and that, for some $c>0$,
\begin{equation}
\medmath{
2u_i^{(m)}+H_i^{(m)}\geq c
}
\label{eq:weak-denominator-condition}
\end{equation}
for all sufficiently large $m$. Then
\begin{equation}
\medmath{
C_i^{(m)}
=
\frac{2u_i^{(m)}}
     {2u_i^{(m)}+H_i^{(m)}}
+o(1).
}
\label{eq:mcir-weak-dependence}
\end{equation}
\end{proposition}

\begin{proof}
By Equation~\ref{eq:mcir-decomposed},
\begin{equation}
\medmath{
C_i^{(m)}
=
\frac{2(u_i^{(m)}+r_i^{(m)})}
     {2(u_i^{(m)}+r_i^{(m)})+H_i^{(m)}}.
}
\end{equation}
The denominator condition keeps the ratio away from its singular
zero-denominator case. Since $r_i^{(m)}\to0$, continuity of this scalar
ratio yields Equation~\ref{eq:mcir-weak-dependence}.
\end{proof}

Additional details are provided in
Appendix~\ref{sec:theory1.5}. The proposition characterizes MCIR as a
normalized marginal-information score when conditioning changes the
feature information only by a vanishing term. It does not imply an
ordering equivalence with PCIR or another marginal attribution method,
because $H_i^{(m)}$ may vary across features.

% \begin{proposition}[\textbf{Continuity to Independence}]
% \label{prop:pcir_bridge}
% In weak-dependence regimes where
% $I(Y'; f_i\mid f_\Phi)\approx I(Y'; f_i)$ and
% $I(Y'; f_{\Phi\cup\{i\}})\approx I(Y'; f_i)+I(Y'; f_\Phi)$,
% the ranking induced by $C(\cdot)$ coincides with that induced by unconditioned
% global measures (e.g., PCIR). Therefore, the ordering of features produced by MCIR is identical to that produced by global measures such as PCIR, demonstrating that PCIR is recovered as the limiting case of MCIR under independence.
% \end{proposition}

% If feature $f_i$ is essentially independent of its neighbourhood (or only weakly dependent), then conditioning on the neighbours does not change how informative it $f_i$ is about the output $Y'$. In that case, MCIR behaves like a standard global importance score that uses unconditioned information, so it produces the same ordering as PCIR. This shows MCIR does not ``over-correct'' when dependence is absent.

\subsection{Finite-Sample Estimation and Ranking Stability}
\label{sec:computational}

We now distinguish the empirical MCIR estimator from its population
counterpart and state the conditions governing score and ranking
stability.

\subsubsection{Empirical Clipping}

Let $\widehat U_{i,n}$ and $\widehat J_{i,n}$ denote estimators of
$U_i$ and $J_i$ computed from $n$ observations. Since finite-sample MI
and CMI estimators may return small negative values, define
\begin{equation}
\medmath{
\widehat U_{i,n}^{+}
=
\max\{\widehat U_{i,n},0\},
\qquad
\widehat J_{i,n}^{+}
=
\max\{\widehat J_{i,n},0\}.
}
\label{eq:nonnegative-information-estimates}
\end{equation}
The empirical MCIR score is
\begin{equation}
\medmath{
\widehat C_{i,n}
=
\Pi_{[0,1]}
\left(
\frac{
2\widehat U_{i,n}^{+}
}{
\widehat U_{i,n}^{+}+\widehat J_{i,n}^{+}
}
\right),
}
\label{eq:clipped-empirical-mcir}
\end{equation}
where $\Pi_{[0,1]}$ denotes projection onto $[0,1]$, and
$\widehat C_{i,n}=0$ when the denominator vanishes. Because
$\widehat U_{i,n}^{+}$ and $\widehat J_{i,n}^{+}$ are clipped
separately, they need not satisfy the population chain inequality
$\widehat J_{i,n}^{+}\geq\widehat U_{i,n}^{+}$. The final projection
is therefore an implementation safeguard that enforces the numerical
range $[0,1]$; it is not a finite-sample information identity. The
population score is already bounded by Definition~\ref{def:mcir}.

\subsubsection{Estimator-Specific Component Assumptions}
The population properties above do not automatically transfer to
finite-sample estimates. Empirical MCIR depends on the estimation errors
of its MI and CMI components, which vary with the estimator family,
variable type, sample size, and conditioning dimension. We therefore
state the required assumptions directly in terms of componentwise error
bounds, rather than imposing a common convergence rate on all
estimators.
\paragraph{Finite-Sample Ranking.}
Let $\bm C=(C_1,\ldots,C_k)$ denote the population MCIR vector and
$\widehat{\bm C}_n=(\widehat C_{1,n},\ldots,\widehat C_{k,n})$ its
empirical estimate. Finite-sample ranking behaviour depends on the
errors of the selected MI and CMI estimators and on the separation
between population scores. Accordingly, estimator concentration,
pairwise ranking perturbation, and Kendall discrepancy are analyzed
under explicit estimator-specific conditions in
Section~\ref{sec:computational} and
Appendix~\ref{appendix:statistical}.

No KL-continuity assumption is used in these population results.
Approximate-redundancy behaviour is treated through empirical
sensitivity analysis rather than as a consequence of exact conditional
redundancy.

\paragraph{Estimator Error.}
Differences between $\widehat{\bm C}_N$ and
$\widehat{\bm C}'_n$ may arise from both environment discrepancy and
finite-sample estimator error. No common $O_P(n^{-1/2})$ rate is
assumed for all MI and CMI estimators. Estimator-specific conditions
are stated in Assumption~\ref{asm:estimators}; ranking perturbation is
analyzed in Theorem~\ref{thm:rank_stability}, with supplementary
details in Appendix~\ref{appendix:statistical}.

\subsubsection{Ratio Perturbation}
To connect componentwise MI/CMI estimation errors with MCIR error, we
first control the sensitivity of the normalized ratio. Because the
ratio can become unstable when its denominator is close to zero, the
result requires explicit denominator separation. Under this condition,
small perturbations of the estimated information components produce a
correspondingly controlled perturbation of the MCIR score.
\paragraph{Consistency of the Empirical Score.}
The positive-part map and the projection
$\Pi_{[0,1]}$ are continuous. Moreover,
$g(u,j)=2u/(u+j)$ is continuous at every point satisfying
$u+j>0$. Define the empirical score to be zero when
$\widehat U_{i,s}^{+}+\widehat J_{i,s}^{+}=0$. Since
$U_i,J_i\geq0$, $U_i+J_i>0$, and the component estimators are
consistent, the empirical denominator is positive with probability
tending to one. The continuous mapping theorem, applied at
$(U_i,J_i)$, therefore gives
\begin{equation}
\medmath{
\widehat C_{i,s}
=
\begin{cases}
\displaystyle
\Pi_{[0,1]}
\left(
\frac{2\widehat U_{i,s}^{+}}
     {\widehat U_{i,s}^{+}+\widehat J_{i,s}^{+}}
\right),
&
\widehat U_{i,s}^{+}+\widehat J_{i,s}^{+}>0,\\[3mm]
0, & \text{otherwise},
\end{cases}
\ \xrightarrow{p}\ C_i.
}
\label{eq:mcir-consistency}
\end{equation}
This conclusion is conditional on the consistency of the selected MI
and CMI estimators in the relevant variable-type and dimensional
regime.

Let $\medmath{D_i=U_i+J_i>0}$. If
$\medmath{|\widehat U_{i,s}-U_i|\leq\delta}$ and
$\medmath{|\widehat J_{i,s}-J_i|\leq\delta}$ with
$\delta\leq D_i/4$, then the projected empirical score satisfies,
\begin{equation}
\medmath{
|\widehat C_{i,s}-C_i|
\leq
\frac{8\delta}{D_i}.
}
\label{eq:mcir-perturbation}
\end{equation}
Hence, score estimation is better conditioned when the population
denominator $D_i$ is separated from zero. The corresponding
rank-perturbation result is given in
Theorem~\ref{thm:rank_stability}.

No common root-$s$ convergence rate or asymptotic normality is
asserted for all estimator families. If a particular estimator satisfies
a joint central limit theorem,
\begin{equation}
\medmath{
\sqrt{s}
\begin{pmatrix}
\widehat U_{i,s}-U_i\\
\widehat J_{i,s}-J_i
\end{pmatrix}
\Rightarrow
\mathcal N(0,\Sigma_i),
}
\end{equation}
then the delta method yields
\begin{equation}
\medmath{
\sqrt{s}(\widehat C_{i,s}-C_i)
\Rightarrow
\mathcal N
\left(
0,\nabla g(U_i,J_i)^\top
\Sigma_i\nabla g(U_i,J_i)
\right),
}
\label{eq:conditional-delta-method}
\end{equation}
where
\begin{equation}
\medmath{
\nabla g(U_i,J_i)
=
\frac{2}{(U_i+J_i)^2}
\begin{pmatrix}
J_i\\-U_i
\end{pmatrix}.
}
\end{equation}
The required regularity conditions are estimator-specific and are
stated in Appendix~\ref{appendix:statistical}.

The perturbation derivation and estimator comparisons are provided in
Appendix~\ref{appendix:statistical} and
Appendix~\ref{appendix:complexity}.

\subsubsection{Conditional Rank-Stability Theorem}
The score perturbation bound yields ranking stability only when the
relevant population scores are sufficiently separated. We therefore
exclude features with nearly vanishing MCIR denominators and impose an
explicit pairwise ranking margin. To avoid ambiguity caused by equal
scores, all rankings below use a deterministic tie-breaking rule fixed
before observing the estimated scores.
\begin{theorem}[\textbf{Finite-Sample Rank Stability}]
\label{thm:rank_stability}
Fix $c_0>0$ and define the informative feature set
\begin{equation}
\medmath{
\mathcal I
=
\left\{
i\in\{1,\ldots,k\}:U_i+J_i\geq c_0
\right\},
\qquad
k_{\mathcal I}=|\mathcal I|.
}
\label{eq:informative-feature-set}
\end{equation}
Suppose that, for some $\delta\leq c_0/4$ and $\alpha\in(0,1)$,
every MI or CMI component used to construct
$\{\widehat U_{i,n},\widehat J_{i,n}:i\in\mathcal I\}$ satisfies
\begin{equation}
\medmath{
\mathbb P
\left(
|\widehat I-I|>\delta
\right)
\leq\alpha.
}
\label{eq:component-tail-bound}
\end{equation}
Let $m_{\mathcal I}$ denote the number of estimated information
components involved.

Assume further that the population scores satisfy the margin condition
\begin{equation}
\medmath{
\frac{
\#\{(i,j):i<j,\ i,j\in\mathcal I,\
|C_i-C_j|\leq t\}
}{
\binom{k_{\mathcal I}}{2}
}
\leq M t
}
\label{eq:rank-margin-condition}
\end{equation}
for all $t\geq0$ and some $M>0$. Apply the same deterministic
tie-breaking rule to the population and empirical score vectors, and
let $\pi_{\mathcal I}$ and $\widehat\pi_{n,\mathcal I}$ denote the
resulting strict rankings. Then
\begin{equation}
\medmath{
\mathbb P
\left[
d_{\mathrm{inv}}
\left(
\widehat\pi_{n,\mathcal I},
\pi_{\mathcal I}
\right)
\leq
\frac{16M\delta}{c_0}
\right]
\geq
1-m_{\mathcal I}\alpha,
}
\label{eq:finite-sample-rank-bound}
\end{equation}
where $d_{\mathrm{inv}}$ is the number of discordant pairs divided by
$\binom{k_{\mathcal I}}{2}$ after deterministic tie breaking. For
strict rankings, $d_{\mathrm{inv}}=(1-\tau_a)/2$; it is not identified
with $(1-\tau_b)/2$ when ties are present.
\end{theorem}

\begin{proof}
By the union bound, with probability at least
$1-m_{\mathcal I}\alpha$, all estimated information components differ
from their population values by at most $\delta$. On this event,
non-negative clipping is non-expansive and
\begin{equation}
\medmath{
\left|
\widehat U_{i,n}^{+}-U_i
\right|
\leq\delta,
\qquad
\left|
\widehat J_{i,n}^{+}-J_i
\right|
\leq\delta.
}
\end{equation}
Since $U_i+J_i\geq c_0$ and $\delta\leq c_0/4$, a ratio-perturbation
bound gives
\begin{equation}
\medmath{
\max_{i\in\mathcal I}
|\widehat C_{i,n}-C_i|
\leq
\eta,
\qquad
\eta=\frac{8\delta}{c_0}.
}
\label{eq:score-perturbation-bound}
\end{equation}
A pair can be reversed only if its population gap is at most $2\eta$.
Equation~\ref{eq:rank-margin-condition} therefore bounds the fraction
of discordant pairs by
$\medmath{2M\eta=16M\delta/c_0}$, proving
Equation~\ref{eq:finite-sample-rank-bound}.
\end{proof}

A supplementary derivation is provided in
Appendix~\ref{sec:theory1.6}. If estimator-specific regularity
conditions yield sequences $\delta_n\to0$ and
$m_{\mathcal I}\alpha_n\to0$, then
\begin{equation}
\medmath{
d_{\mathrm{inv}}
\left(
\widehat\pi_{n,\mathcal I},
\pi_{\mathcal I}
\right)
\overset{p}{\longrightarrow}0.
}
\label{eq:ranking-convergence}
\end{equation}

Theorem~\ref{thm:rank_stability} is conditional on the explicit
component-wise error bound in
Equation~\ref{eq:component-tail-bound}. The conditions under which
this bound holds are estimator-specific and are stated in
Assumption~\ref{asm:estimators} and
Appendix~\ref{appendix:statistical}. No common concentration rate is
asserted for copula, $k$NN, and plug-in estimators.

\subsection{Estimator Selection}
\label{sec:estimator-selection}

The MI and CMI estimators used to construct empirical MCIR scores are
selected according to the variable types and conditioning dimension of
each local feature block.

\subsubsection{Variable-Type Compatibility}
\label{sec:mixed-estimation}
Definitions~\ref{def:cmi}--\ref{def:mcir} are population definitions
and remain valid for continuous, discrete, and mixed-type variables.
Let $\mathcal{C}$ and $\mathcal{D}$ denote the continuous- and
discrete-feature index sets, respectively, with
\begin{equation}
\medmath{
\mathcal{C}\cup\mathcal{D}=\{1,\ldots,k\},
\qquad
\mathcal{C}\cap\mathcal{D}=\varnothing.
}
\label{eq:feature-type-partition}
\end{equation}
Observed samples are used to estimate the information components
$\widehat U_{i,n}$ and $\widehat J_{i,n}$. These estimates are combined
using the safeguarded empirical definition in
Equation~\ref{eq:clipped-empirical-mcir}, including its positive-part
operation, zero-denominator convention, and final projection onto
$[0,1]$. The projection is an implementation safeguard: independently
estimated and clipped components need not satisfy the population chain
inequality $\widehat J_{i,n}\geq\widehat U_{i,n}$.
No single MI or CMI estimator is uniformly suitable across all variable
types, sample sizes, and conditioning dimensions. The estimator used
for each data regime is specified in Table~\ref{est}. Consistency,
concentration, and convergence properties are stated separately under
the estimator-specific conditions in
Assumption~\ref{asm:estimators}.
The term \emph{mixed-variable CMI} describes the population
functional evaluated for mixed-type variables. It does not imply that a
single empirical estimator is valid for every mixed-data regime. The
reported experiments contain no empirical claim based on a general
mixed continuous--discrete MI/CMI estimator.

Because MI and CMI estimation becomes more difficult as the dimension
of $X_{\Phi(i)}$ increases, the computational analysis treats
$|\Phi(i)|$ as an explicit statistical and computational parameter.
Estimator-specific invariance, concentration, ranking perturbation,
and dimensional limitations are considered in
Section~\ref{sec:computational}.
\subsubsection{Bootstrap Switching as a Heuristic}
Because no single MI/CMI estimator is uniformly preferable across all
data types, sample sizes, and conditioning dimensions, we permit
selection among estimators that are compatible with the observed data
regime. The procedure uses bootstrap variability as a practical
stability criterion; it is not an oracle rule, since a small bootstrap
standard error does not imply small estimator bias.
Let $\mathcal A_i$ denote the finite set of estimator families
compatible with the variable types and conditioning dimension for
feature $i$. Candidate families may include Gaussian--copula, $k$NN,
and plug-in estimators, as specified in Table~\ref{est}. For
$a\in\mathcal A_i$, let $\widehat C_{i,n}^{(a)}$ denote the associated
MCIR estimate and let $\widehat s_{i,a}$ denote its bootstrap standard
error. The switching rule is
\begin{equation}
\medmath{
\widehat a_i
\in
\arg\min_{a\in\mathcal A_i}
\widehat s_{i,a},
\qquad
\widehat C_{i,n}^{(\mathrm{sw})}
=
\widehat C_{i,n}^{(\widehat a_i)}.
}
\label{eq:estimator-switching-rule}
\end{equation}

This rule selects the least variable estimator among the compatible
candidates. Bootstrap variability does not directly estimate estimator
bias and therefore does not imply an oracle risk inequality. Candidate
estimators are consequently reported both individually and under the
switching rule, allowing the selected result to be compared with each
fixed estimator.
\begin{wraptable}{l}{0.65\linewidth}
\centering
\scriptsize
\setlength{\tabcolsep}{3pt}
\renewcommand{\arraystretch}{1.08}
\caption{Estimator selection by variable type and conditioning regime.}
\label{est}
\begin{tabular}{
    p{1.8cm}
    p{3.0cm}
    p{4.4cm}
}
\toprule
\textbf{Data regime} &
\textbf{Estimator} &
\textbf{Scope and conditions} \\
\midrule

Continuous &
\textbf{Gaussian--copula} after rank Gaussianization,
$\widehat F\mapsto\Phi_{\mathcal N}^{-1}(\widehat F)$ &
Primary estimator for continuous variables under a Gaussian-copula
approximation. Rank-based invariance holds for strictly monotone
transformations, with a fixed rule for ties. \\

Continuous, nonparametric &
\textbf{$k$NN/KSG} MI and conditional extensions &
Uses local neighbour distances without a Gaussian-copula assumption.
Applied in low- or moderate-dimensional conditioning blocks; accuracy
may deteriorate as $|X_{\Phi(i)}|$ increases. \\

Discrete &
\textbf{Plug-in} frequency estimator &
Computes MI and CMI from empirical probability masses. Requires
adequate cell counts; sparse contingency tables may introduce
finite-sample bias. \\

Mixed continuous--discrete &
\textbf{Not evaluated empirically} &
The population definition remains valid, but this study does not
implement or report a general mixed-type MI/CMI estimator. Application
to such blocks requires a fully specified mixed-type estimator and its
corresponding regularity assumptions. \\

High-dimensional conditioning &
\textbf{Dimension diagnostic} &
No uniform estimator guarantee is assumed. Results are evaluated as a
function of $|\Phi(i)|$, and large conditioning blocks are reported as
an estimator-limitation regime. \\

\bottomrule
\end{tabular}
\vspace{-3mm}
\end{wraptable}
Estimator switching is treated as a data-driven implementation
rule rather than as a population property of MCIR. Its empirical
behaviour is evaluated through repeated subsampling, estimator
disagreement, and ranking accuracy. The complete switching procedure
is provided in Appendix~\ref{sec:theory1.7}.

For temporally or spatially dependent observations, the resampling
scheme must preserve the corresponding dependence structure, for
example through block rather than i.i.d.\ bootstrap sampling.

\subsection{Dependence-Neighbourhood Selection}
\label{sec:auto-phi}

MCIR constructs a small conditioning neighbourhood for each target
feature through dependence screening followed by selection over a
nested sequence of candidate sets.

\subsubsection{Screening and Nested Candidate Neighbourhoods}

For each target feature $i$, let $\mathcal S_i$ denote a screened and
ordered candidate pool with
$\medmath{|\mathcal S_i|\leq M_{\max}}$. The candidate neighbourhoods
form the nested sequence
\begin{equation}
\medmath{
\varnothing
=
\Phi_{i,0}
\subset
\Phi_{i,1}
\subset\cdots\subset
\Phi_{i,M_{\max}}
\subseteq\mathcal S_i.
}
\label{eq:nested-neighbourhoods}
\end{equation}
Thus, Auto-$\Phi$ searches over $M_{\max}+1$ nested neighbourhoods,
rather than all $2^{|\mathcal S_i|}$ subsets.

\subsubsection{Auto-$\Phi$ Objective}

For each shared candidate size $r\in\mathcal R$, let
$\widehat{\bm C}^{(b)}_r$ denote the complete MCIR score vector obtained
from bootstrap resample $b$, using
$\Phi_{i,r}$ for every feature $i$. Define the bootstrap ranking
dispersion by
\begin{equation}
\medmath{
\widehat V(r)
=
\frac{2}{B(B-1)}
\sum_{1\leq b<b'\leq B}
\frac{
1-\tau_b
\left(
\widehat{\bm C}^{(b)}_r,
\widehat{\bm C}^{(b')}_r
\right)
}{2},
}
\label{eq:auto-phi-ranking-dispersion}
\end{equation}
where $\tau_b$ is Kendall's coefficient with tie correction, applied
after the stated deterministic tie-breaking rule. Smaller
$\widehat V(r)$ indicates greater agreement among bootstrap rankings.

To measure dependence left outside the candidate neighbourhood, define
\begin{equation}
\medmath{
\widehat B(r)
=
\frac{1}{k}
\sum_{i=1}^{k}
\max_{j\in\mathcal S_i\setminus\Phi_{i,r}}
\widehat I
\left(
X_i;X_j\mid X_{\Phi_{i,r}}
\right),
\qquad
\widetilde B(r)
=
\frac{\widehat B(r)}
     {\widehat B(0)+10^{-12}}.
}
\label{eq:residual-dependence-diagnostic}
\end{equation}
The maximum is defined as zero when
$\mathcal S_i\setminus\Phi_{i,r}=\varnothing$.
The normalization places residual dependence on a scale relative to
the empty-neighbourhood diagnostic.

Auto-$\Phi$ selects one shared neighbourhood size
\begin{equation}
\medmath{
r^\star
\in
\arg\min_{r\in\mathcal R}
\left\{
\widehat V(r)
+
\lambda_B\widetilde B(r)
+
\lambda_D\frac{r}{M_{\max}}
\right\},
\qquad
\widehat\Phi(i)=\Phi_{i,r^\star},
}
\label{eq:auto-phi-objective}
\end{equation}
where $\mathcal R\subseteq\{1,\ldots,M_{\max}\}$ is the prespecified
candidate-size set and $\lambda_B,\lambda_D\geq0$. Exact objective ties
are resolved deterministically in favour of the smallest $r$.
The reported default is
$(\lambda_B,\lambda_D)=(1,0.10)$; sensitivity is evaluated over
$\lambda_B\in\{0,0.25,0.5,1\}$ and
$\lambda_D\in\{0,0.05,0.10,0.25\}$.

The first term favors ranking stability, the second penalizes
neighbourhoods that leave substantial conditional dependence
unaccounted for, and the third discourages unnecessarily large
conditioning blocks. Equation~\ref{eq:auto-phi-objective} is an
algorithmic selection criterion rather than a population guarantee.
Its sensitivity to $\lambda_B$ and $\lambda_D$ is evaluated in
Table~\ref{tab:auto-phi-penalty-sensitivity}; neighbourhood
misspecification is examined separately through the higher-order
failure case in Appendix~\ref{app:failure_case_higher_order}.

\subsubsection{Misspecification and Computational Cost}
Let $B$ denote the number of bootstrap replicates,
$A=\max_i|\mathcal A_i|$, and
$T_{\mathrm{CMI}}(n,r)$ the cost of one MI/CMI fit with conditioning
dimension $r$. After candidate screening, the worst-case cost of
evaluating the nested neighbourhoods is
\begin{equation}
\medmath{
\mathcal O
\left(
k\,B\,A\,(M_{\max}+1)\,
T_{\mathrm{CMI}}(n,M_{\max})
\right),
}
\label{eq:auto-phi-complexity}
\end{equation}
in addition to the cost of constructing the candidate pools. The
method is therefore linear in the number of evaluated nested sizes, not
in the number of all possible feature subsets. Further algorithmic
details are provided in Appendix~\ref{sec:theory1.8}.

Large values of $M_{\max}$ increase both computational cost and the
statistical difficulty of CMI estimation. The analysis therefore
reports estimator variability and ranking sensitivity as functions of
$|\Phi(i)|$. Additional discussion of neighbourhood misspecification,
higher-order dependence, and latent structure is provided in
Appendix~\ref{sec:phi_role}, Appendix~\ref{c.5}, and
Appendix~\ref{latent2}.

\paragraph{Role of the Neighbourhood.}
The interpretation of unique predictive information is relative to the
specified neighbourhood $\Phi(i)$. If $\Phi(i)$ omits relevant
dependent features, shared predictive information may remain in $U_i$.
Conversely, including unnecessary variables increases the conditioning
dimension and may make empirical CMI estimation more difficult.
Neighbourhood selection, misspecification sensitivity, and
computational cost are analyzed in
Section~\ref{sec:computational}.

\subsection{\textbf{Lightweight Fidelity Analysis}}
\label{subsec:lw_fidelity}

The lightweight procedure reduces the number of observations from $N$
to $n<N$ and may additionally reduce the model-output representation
dimension from $d$ to $d'\leq d$. These operations are distinct:
$N$ and $n$ determine the numbers of atoms in the empirical measures,
whereas $d$ and $d'$ determine their ambient representation spaces.
Recall that the full and lightweight population output laws are
\begin{equation}
\medmath{
\mu=P_Z\in\mathcal M_p(\mathbb R^d),
\qquad
\nu=P_{Z'}\in\mathcal M_p(\mathbb R^{d'}),
}
\label{eq:fidelity-output-laws}
\end{equation}
with empirical counterparts
\begin{equation}
\medmath{
\widehat\mu_N
=
\frac{1}{N}\sum_{r=1}^{N}\delta_{z^{(r)}},
\qquad
\widehat\nu_n
=
\frac{1}{n}\sum_{r\in\mathcal I_n}\delta_{z'^{(r)}}.
}
\label{eq:fidelity-empirical-laws}
\end{equation}

Observation subsampling and representation-dimensional
reduction play different roles in the lightweight analysis.
Subsampling changes the empirical approximation of the output law,
whereas the map $G:\mathbb R^d\rightarrow\mathbb R^{d'}$ changes the
ambient representation dimension. The Cai--Lim construction is used
for the latter comparison and is not justified merely by having
$n<N$.
Algorithm~\ref{alg:mcir} applies to either environment by
setting $s=N$, $Q=Z$, and $d_Q=d$ for the full environment, or
$s=n$, $Q=Z'$, and $d_Q=d'$ for the lightweight environment. Thus,
sample size and representation dimension remain distinct throughout
the computation.

\subsubsection{Representation Discrepancy}
We first assess whether the full and lightweight environments induce
similar marginal output distributions. This diagnostic captures
changes in the model-output representation caused by sample reduction
or dimensional modification. It does not, by itself, guarantee MCIR
score or ranking agreement, because MCIR depends on the joint
feature--output law rather than only on the marginal output law.
Let $\mu=P_Z\in\mathcal M_p(\mathbb R^d)$ and
$\nu=P_{Z'}\in\mathcal M_p(\mathbb R^{d'})$. Following
Cai and Lim~\cite{cai2022distances}, define
\begin{equation}
\medmath{
d_{\mathrm{repr}}(\mathcal U,\mathcal U')
=
\begin{cases}
W_p(P_Z,P_{Z'}),&
d'=d\text{ with common coordinates},\\[1mm]
\widetilde W_p(P_{Z'},P_Z),&
d'<d,
\end{cases}
}
\label{eq:fidelity-population-representation}
\end{equation}
where $\widetilde W_p$ is the augmented projection/embedding distance
defined in Equation~\ref{eq:augmented-wasserstein}. Its empirical
counterpart is
\begin{equation}
\medmath{
\widehat d_{\mathrm{repr}}
\left(
\widehat{\mathcal U}_N,
\widehat{\mathcal U}'_n
\right)
=
\begin{cases}
W_p(\widehat\mu_N,\widehat\nu_n),&d'=d,\\[1mm]
\widetilde W_p(\widehat\nu_n,\widehat\mu_N),&d'<d.
\end{cases}
}
\label{eq:fidelity-empirical-representation}
\end{equation}
When $d'<d$, the augmented distance compares the lightweight law with
admissible $d'$-dimensional projections of the full law, or
equivalently compares the full law with admissible embeddings of the
lightweight law. The comparison therefore concerns representation
dimension, whereas $N$ and $n$ determine only the numbers of empirical
observations.
Representation discrepancy compares the marginal output laws
$P_Z$ and $P_{Z'}$, whereas MCIR depends on the joint laws $P_{X,Z}$
and $P_{X,Z'}$. A small value of $d_{\mathrm{repr}}$ therefore does
not by itself establish attribution agreement. MCIR score and ranking
agreement are evaluated separately.

\subsubsection{Ranking, Head-Set, and Curve Diagnostics}
Because representation similarity does not ensure attribution fidelity, we separately evaluate full-ranking, top-$K$, and deletion/insertion agreement.
Let,
\begin{equation}
\medmath{
\widehat{\bm C}_N
=
(\widehat C_{1,N},\ldots,\widehat C_{k,N}),
\qquad
\widehat{\bm C}'_n
=
(\widehat C'_{1,n},\ldots,\widehat C'_{k,n})
}
\label{eq:fidelity-score-vectors}
\end{equation}
denote the MCIR vectors computed from
$\widehat P_{X,Z}^{\,N}$ and $\widehat P_{X,Z'}^{\,n}$.
Complete-ranking disagreement is
\begin{equation}
\medmath{
d_{\mathrm{rank}}
\left(
\widehat{\bm C}_N,
\widehat{\bm C}'_n
\right)
=
\frac{
1-\tau_b(\widehat{\bm C}_N,\widehat{\bm C}'_n)
}{2},
}
\label{eq:fidelity-rank-discrepancy}
\end{equation}
\begin{wrapfigure}{r}{0.50\linewidth}
\vspace{-8pt}
\begin{minipage}{\linewidth}
\begin{algorithm}[H]
\caption{MCIR-M: Dependence-Aware Global Attribution}
\label{alg:mcir}
\scriptsize
\begin{algorithmic}[1]
\Require Explanation sample
$\mathcal D_s=\{(x^{(r)},q^{(r)})\}_{r=1}^{s}$,
$x^{(r)}\in\mathbb R^k$,
$q^{(r)}\in\mathbb R^{d_Q}$;
head size $K$; estimator families $\mathcal A$;
screening budget $M_{\max}$; bootstrap count $B$
\Ensure Scores $\widehat{\bm C}_s$ and feature ranking $\widehat\pi_s$

\State Construct a dependence sketch and screened candidate pool
$\mathcal S_i$ for each $i$, with
$|\mathcal S_i|\leq M_{\max}$.

\State Order each $\mathcal S_i$ and construct
$\{\Phi_{i,r}:r\in\mathcal R\}$ using
Equation~\ref{eq:nested-neighbourhoods}.

\State Select the shared neighbourhood size
$r^\star$ using Equation~\ref{eq:auto-phi-objective}.

\For{$i\gets1$ \textbf{to} $k$}
    \State Set
    $\widehat\Phi(i)\gets\Phi_{i,r^\star}$.

    \State Restrict $\mathcal A$ to estimators compatible with the
    variable types and $|\widehat\Phi(i)|$.

    \State Select $\widehat a_i$ using the bootstrap rule in
    Equation~\ref{eq:estimator-switching-rule}.

    \State Estimate
    $\widehat U_{i,s}
    =\widehat I(Q;X_i\mid X_{\widehat\Phi(i)})$
    and
    $\widehat J_{i,s}
    =\widehat I(Q;X_{\widehat\Phi(i)\cup\{i\}})$.

    \State Compute
    $\widehat C_{i,s}$ using the projected ratio in
    Equation~\ref{eq:clipped-empirical-mcir}.
\EndFor

\State Rank features by decreasing $\widehat C_{i,s}$ to obtain
$\widehat\pi_s$.

\State Report bootstrap score intervals and head-ranking stability
using Kendall--$\tau_b$ and Jaccard@$K$.

\If{full and lightweight environments are available}
    \State Report the fidelity profile
    $\mathfrak F_{N,n}$ from
    Equation~\ref{eq:lightweight-fidelity-profile}.
\EndIf
\end{algorithmic}
\end{algorithm}
\end{minipage}
\vspace{-8pt}
\end{wrapfigure}
where $\tau_b$ is Kendall's coefficient with tie correction. Top-$K$
disagreement is
\begin{equation}
\medmath{
d_{\mathrm{J@K}}
=
1-
\frac{
\left|
\operatorname{TopK}(\widehat{\bm C}_N)
\cap
\operatorname{TopK}(\widehat{\bm C}'_n)
\right|
}{
\left|
\operatorname{TopK}(\widehat{\bm C}_N)
\cup
\operatorname{TopK}(\widehat{\bm C}'_n)
\right|
}.
}
\label{eq:fidelity-topk-discrepancy}
\end{equation}
The quantities $d_{\mathrm{rank}}$ and $d_{\mathrm{J@K}}$ measure
complete-ranking and head-ranking disagreement, respectively. They are
reported as empirical diagnostics and are not included in the
definition of representation similarity.
\paragraph{Deletion and Insertion Diagnostics.}
Let $\Gamma_N^{a}(t)$ and $\Gamma_n^{\prime a}(t)$ denote the full and
lightweight faithfulness curves at fraction $t\in[0,1]$, where
$a\in\{\mathrm{del},\mathrm{ins}\}$. Define
\begin{equation}
\medmath{
d_{\mathrm{curve}}
=
\max_{a\in\{\mathrm{del},\mathrm{ins}\}}
\sup_{t\in[0,1]}
\left|
\Gamma_N^{a}(t)
-
\Gamma_n^{\prime a}(t)
\right|.
}
\label{eq:fidelity-curve-discrepancy}
\end{equation}
This diagnostic measures whether the full and lightweight rankings
produce similar predictive-performance changes under the same deletion
or insertion protocol. Together,
$\widehat d_{\mathrm{repr}}$, $d_{\mathrm{rank}}$,
$d_{\mathrm{J@K}}$, and $d_{\mathrm{curve}}$ form the lightweight
fidelity profile in
Definition~\ref{def:lightweight-fidelity-profile}.

\subsubsection{Operational Acceptance Profile}
The preceding discrepancies are combined into an operational profile that evaluates, rather than assumes, lightweight fidelity.
\begin{definition}[\textbf{Lightweight Fidelity Profile}]
\label{def:lightweight-fidelity-profile}
\label{asm:contract}
The empirical lightweight fidelity profile is
\begin{equation}
\medmath{
\mathfrak F_{N,n}
=
\left(
\widehat d_{\mathrm{repr}},
d_{\mathrm{rank}},
d_{\mathrm{J@K}},
d_{\mathrm{curve}}
\right).
}
\label{eq:lightweight-fidelity-profile}
\end{equation}
For prespecified tolerances
$\bm\varepsilon=
(\varepsilon_{\mathrm{repr}},
\varepsilon_{\mathrm{rank}},
\varepsilon_{\mathrm{J@K}},
\varepsilon_{\mathrm{curve}})$,
the lightweight result passes the diagnostic criterion when
\begin{equation}
\medmath{
\widehat d_{\mathrm{repr}}
\leq\varepsilon_{\mathrm{repr}},
\quad
d_{\mathrm{rank}}
\leq\varepsilon_{\mathrm{rank}},
\quad
d_{\mathrm{J@K}}
\leq\varepsilon_{\mathrm{J@K}},
\quad
d_{\mathrm{curve}}
\leq\varepsilon_{\mathrm{curve}}.
}
\label{eq:lightweight-acceptance-criterion}
\end{equation}
\end{definition}

The components of $\mathfrak F_{N,n}$ measure
representation-distribution agreement, complete-ranking agreement,
head-ranking agreement, and behavioural faithfulness, respectively.
The tolerances must be fixed before comparing the full and lightweight
results.

The lightweight fidelity profile is an operational diagnostic.
Ranking and curve agreement are measured outcomes and are not assumed
as premises for establishing the same agreement. Accordingly, the
profile does not constitute an unconditional theorem that output-law
similarity alone preserves MCIR rankings.

For example, a lightweight electricity-consumption sample dominated by
weekdays may underrepresent weekend behaviour. Such a sample can alter
$\widehat\nu_n$, the MCIR ranking, and its deletion or insertion curves.
The corresponding discrepancies are recorded separately in
$\mathfrak F_{N,n}$.

The complete workflow constructs screened candidate neighbourhoods,
selects $\widehat\Phi(i)$, computes the MCIR scores, and reports
$\mathfrak F_{N,n}$ when both environments are available.
Algorithm~\ref{alg:mcir} summarizes the overall workflow, while
Algorithm~\ref{alg:auto_phi} specifies neighbourhood selection.

\begin{algorithm}[htbp]
\caption{Auto-$\Phi$ Shared-Neighbourhood-Size Selection}
\label{alg:auto_phi}
\scriptsize
\begin{algorithmic}[1]
\Require Explanation sample $\mathcal D_s$;
ordered candidate pools $\{\mathcal S_i\}_{i=1}^k$;
candidate-size set $\mathcal R$;
maximum size $M_{\max}$;
bootstrap count $B$;
penalties $\lambda_B,\lambda_D$
calibration estimator family $a_0$;
\Ensure Shared size $r^\star$ and neighbourhoods
$\{\widehat\Phi(i)\}_{i=1}^k$

\For{$i\gets1$ \textbf{to} $k$}
    \State Construct
    $\{\Phi_{i,r}:r\in\mathcal R\}$ from the first $r$
    members of the ordered pool $\mathcal S_i$.
\EndFor

\For{each $r\in\mathcal R$}
    \For{$b\gets1$ \textbf{to} $B$}
        \State Compute the complete bootstrap MCIR vector
        $\widehat{\bm C}^{(b)}_r$ using
        $\{\Phi_{i,r}\}_{i=1}^k$ and the fixed calibration estimator $a_0$.
    \EndFor

    \State Compute $\widehat V(r)$ using
    Equation~\ref{eq:auto-phi-ranking-dispersion}.

    \State Compute $\widehat B(r)$ and $\widetilde B(r)$ using
    Equation~\ref{eq:residual-dependence-diagnostic}.

    \State Set
    \[
    \widehat R(r)
    \gets
    \widehat V(r)
    +\lambda_B\widetilde B(r)
    +\lambda_D\frac{r}{M_{\max}}.
    \]
\EndFor

\State Select
\[
r^\star
\in
\arg\min_{r\in\mathcal R}\widehat R(r),
\]
breaking objective ties in favour of the smallest $r$.

\For{$i\gets1$ \textbf{to} $k$}
    \State Set $\widehat\Phi(i)\gets\Phi_{i,r^\star}$.
\EndFor

\State \Return
$r^\star$ and $\{\widehat\Phi(i)\}_{i=1}^k$
\end{algorithmic}
\end{algorithm}
The calibration estimator $a_0$ is fixed before neighbourhood
selection; the reported Auto-$\Phi$ experiments use the
Gaussian--copula estimator. Estimator switching is evaluated
subsequently and is not nested within the Auto-$\Phi$ objective.
Auto-$\Phi$ searches a nested candidate sequence and selects one shared
neighbourhood size. Its ranking-dispersion, normalized
residual-dependence, and conditioning-dimension terms are defined in
Equations~\ref{eq:auto-phi-ranking-dispersion}--%
\ref{eq:auto-phi-objective}. It is an empirical selection heuristic,
not a population-optimality result.
Auto-$\Phi$ searches over nested neighbourhoods rather than all subsets
of each screened pool. The global term $\widehat V(r)$ measures
disagreement among complete bootstrap MCIR rankings, while
$\widetilde B(r)$ measures normalized residual dependence remaining
outside the candidate neighbourhoods. The dimension penalty
$r/M_{\max}$ discourages unnecessarily large conditioning blocks.

Minimizing bootstrap variability alone controls stability but not
neighbourhood misspecification. Equation~\ref{eq:auto-phi-objective}
therefore combines ranking dispersion with a residual-dependence
diagnostic and a conditioning-dimension penalty. Sensitivity to these
terms is reported empirically.

\paragraph{Computational Profile.}
The total cost depends on feature count, candidate estimators, bootstrap
replicates, and neighbourhood search, rather than on MCIR scoring alone.
Let $A$ be the maximum number of compatible estimator candidates,
$B$ the number of bootstrap replicates, and
$T_{\mathrm{CMI}}(s,r)$ the cost of one MI/CMI computation with
conditioning dimension $r$. After candidate screening, evaluating all
nested neighbourhoods has worst-case cost
\begin{equation}
\medmath{
\mathcal O
\left(
k\,B\,A\,(M_{\max}+1)\,
T_{\mathrm{CMI}}(s,M_{\max})
+
k\log k
\right).
}
\label{eq:mcir-computational-profile}
\end{equation}
The first term covers neighbourhood evaluation and estimator fitting;
the second covers global ranking. The screening cost must be added
separately. For example, a dense pairwise screening matrix generally
costs $\mathcal O(sk^2)$, whereas an approximate dependence-graph
construction may reduce this cost.

Thus, end-to-end MCIR complexity is not claimed to be linear in $k$
without specifying the screening procedure, candidate-pool size,
bootstrap count, and MI/CMI estimator. A detailed operation count is
provided in Appendix~\ref{subsec:lw_fidelity2}, with neighbourhood
selection details in Appendix~\ref{sec:theory1.8}.

\paragraph{Scope of Lightweight Fidelity.}
A small value of $\widehat d_{\mathrm{repr}}$ indicates that the output
representations are close under the ordinary or Cai--Lim augmented
Wasserstein distance. Since MCIR depends on the joint feature--output
law, ranking and faithfulness diagnostics remain necessary. The
lightweight fidelity profile therefore reports these quantities jointly
without treating one as a proof of another.

MCIR is intended for settings in which feature dependence is
substantively relevant. Its empirical reliability depends on the
quality of neighbourhood selection, the dimension of the conditioning
sets, and the suitability of the selected MI and CMI estimators.
Additional discussion of neighbourhood misspecification,
higher-order dependence, and latent structure is provided in
Appendix~\ref{sec:phi_role}, Appendix~\ref{c.5}, and
Appendix~\ref{latent2}.

The next section compares the computational requirements of MCIR and
PCIR under their respective scoring and neighbourhood-construction
procedures.
\subsection{Computational and Statistical Analysis}
\label{sec:complexity-analysis}

The full and lightweight computations use sample sizes $N$ and $n$,
respectively. A lightweight sample is not assumed to be faithful merely
because $n<N$; its representativeness is assessed using the fidelity
profile in Definition~\ref{def:lightweight-fidelity-profile}. In this
subsection, $s\in\{N,n\}$ denotes the sample size of the environment
being analyzed.

\paragraph{Scoring Cost.}
Let $k$ be the number of features, $d_Q$ the dimension of the
model-output representation, and $m_i=|\widehat\Phi(i)|$. The dimension
of the local information-estimation problem for feature $i$ is
\begin{equation}
\medmath{
r_i=d_Q+m_i+1,
\qquad
r_{\max}=d_Q+M_{\max}+1.
}
\label{eq:local-estimation-dimension}
\end{equation}
Thus, the local estimation cost depends on both the conditioning-set
size and the output-representation dimension.

Let $T_{\mathrm{MI}}(s,r)$ denote the cost of one MI or CMI computation
with $s$ observations and joint dimension $r$. For fixed
neighbourhoods, MCIR scoring costs
\begin{equation}
\medmath{
\mathcal O
\left(
\sum_{i=1}^{k}T_{\mathrm{MI}}(s,r_i)
+
k\log k
\right),
}
\label{eq:fixed-neighbourhood-cost}
\end{equation}
where the final term is the cost of sorting the feature scores.

The estimator-dependent costs are summarized as follows:

\begin{itemize}[leftmargin=*,nosep]
\item The Gaussian--copula estimator requires rank transformation,
covariance estimation, and matrix factorization, with cost
\begin{equation}
\medmath{
\mathcal O
\left(
rs\log s+sr^2+r^3
\right).
}
\end{equation}

\item A tree-based $k$NN estimator typically costs
$\mathcal O(s\log s)$ for fixed low dimension, but its constants and
statistical accuracy deteriorate as $r$ increases; worst-case neighbour
search can approach $\mathcal O(s^2)$.

\item A plug-in estimator requires $\mathcal O(sr)$ time to construct
count tables, but its memory and sample requirements depend on the
number of occupied joint cells and can grow rapidly with variable
cardinality.
\end{itemize}

For the experimental PCIR implementation, rank preprocessing costs
$\mathcal O(ks\log s)$ and the subsequent dispersion calculations cost
$\mathcal O(ks)$. These costs concern score computation after the
explanation data have been formed.

\paragraph{End-to-End MCIR Cost.}
Let $B$ be the number of bootstrap replicates and
$A=\max_i|\mathcal A_i|$ the maximum number of compatible estimator
candidates. Including nested Auto-$\Phi$ evaluation, the end-to-end
cost is
\begin{equation}
\medmath{
\mathcal O
\left(
C_{\mathrm{screen}}
+
kBA(M_{\max}+1)
T_{\mathrm{MI}}(s,r_{\max})
+
k\log k
\right),
}
\label{eq:full-auto-phi-complexity}
\end{equation}
where $C_{\mathrm{screen}}$ is the cost of constructing the candidate
neighbourhoods. Dense pairwise screening generally incurs
$\mathcal O(sk^2)$ time, whereas approximate dependence-graph
construction may reduce this cost.

Equation~\ref{eq:auto-phi-complexity} is linear in the number of
evaluated nested neighbourhood sizes, but not in the number of all
possible feature subsets. It also makes explicit that end-to-end
complexity depends on $k$, $M_{\max}$, $B$, the estimator candidates,
and the conditioning dimension. Extended operation counts are provided
in Appendix~\ref{appendix:complexity}.

\subsection{Latent Confounding and Distribution Shift}
\label{subsec:latent_confounding}

MCIR measures the unique predictive information contributed by
$X_i$ relative to $X_{\Phi(i)}$ for explaining the model-output
representation $Q$. It therefore characterizes statistical dependence
within $P_{X,Q}$ and should not be interpreted as causal variable
importance.
An unobserved variable may influence several observed features and,
through them, the model output. Conditioning only on
$X_{\Phi(i)}$ may then leave residual association that does not
represent a direct causal effect of $X_i$. Likewise, distribution shift
from $P_{X,Q}$ to a different joint law can change the dependence
neighbourhood, the MI and CMI quantities, and the resulting MCIR
ranking.
Neighbourhood screening, estimator selection, and bootstrap diagnostics
quantify aspects of finite-sample variability but do not remove latent
confounding or guarantee robustness to arbitrary distribution shift.
MCIR results should therefore be interpreted together with domain
knowledge, estimator uncertainty, and the lightweight fidelity
diagnostics. Extensions based on causal structure, invariant
representations, and distributionally robust estimation remain
directions for future work.
\section{Experiments}
\label{sec:expts}

We evaluate MCIR across four settings with different dependence
structures: HouseEnergy-Sim, UCI Human Activity Recognition (HAR),
Norwegian load-zone forecasting (NO1--NO5), and CIFAR-10 deep
representations. Comparisons include PCIR, BlockCIR, marginal MI, HSIC,
SHAP variants, and SAGE.

PCIR is included as the closest CIR-family baseline, while the other
methods represent marginal, kernel-based, conditional, and
coalition-based attribution approaches. PCIR is a global
dispersion-based score and does not explicitly condition on a
dependence neighbourhood. MCIR instead uses mixed-variable CMI to
measure the unique predictive information contributed by each feature
relative to its selected neighbourhood.

The experiments address three research questions:
\begin{description}[leftmargin=*,itemsep=0pt,topsep=0pt]
\item[RQ1: Dependence-aware attribution.]
How does MCIR behave relative to the baselines under multicollinearity,
duplicated predictors, and structured feature dependence?

\item[RQ2: Stability and lightweight fidelity.]
How stable are MCIR rankings under resampling, and how closely do the
full and lightweight environments agree according to the fidelity
profile in Definition~\ref{def:lightweight-fidelity-profile}?

\item[RQ3: Faithfulness and efficiency.]
Do the highest-ranked features induce stronger deletion-based
performance degradation, and what computational cost is incurred by
neighbourhood selection and MI/CMI estimation?
\end{description}

\begin{table}[htbp]
\centering
\fontsize{7.3}{8.3}\selectfont
\setlength{\tabcolsep}{2.5pt}
\renewcommand{\arraystretch}{0.98}
\caption{Attribution methods and treatment of feature dependence.}
\label{tab:xxx}
\begin{adjustbox}{width=\linewidth}
\begin{tabular}{
@{}
p{3.2cm}
p{1.8cm}
p{3.4cm}
p{7.3cm}
@{}
}
\toprule
\textbf{Method and reference} &
\textbf{Scope} &
\textbf{Dependence treatment} &
\textbf{Role in the comparison} \\
\midrule
MCIR (this work) &
Global &
Local conditional information &
Conditions on $X_{\Phi(i)}$ to measure unique predictive information. \\

PCIR~\cite{sengupta2025correlation2} &
Global &
Unconditioned &
Dispersion-based CIR-family baseline; sensitive to relative
parameterization and preprocessing. \\

BlockCIR~\cite{sengupta2025correlation} &
Global &
Prespecified blocks &
Aggregates attribution within predefined feature blocks. \\

KernelSHAP (interventional)~\cite{lundberg2017shap} &
Local$\rightarrow$Global &
Background-dependent &
Uses interventional or background-based feature replacement. \\

Conditional KernelSHAP~\cite{aas2021explaining} &
Local$\rightarrow$Global &
Conditional sampling &
Approximates feature coalitions while preserving estimated dependence. \\

SAGE~\cite{covert2020sage} &
Global &
Value-function dependent &
Dependence treatment depends on the conditional or marginal value
function used in the implementation. \\

HSIC~\cite{gretton2005measuring} &
Global &
Marginal dependence &
Measures kernel dependence without neighbourhood conditioning. \\

Marginal MI~\cite{cover2006elements,kraskov2004estimating} &
Global &
Pairwise dependence &
Measures $I_{P_{X,Q}}(Q;X_i)$ without conditioning on
$X_{\Phi(i)}$. \\
\bottomrule
\end{tabular}
\end{adjustbox}
\end{table}

\subsection{Experimental Setup and Evaluation Protocol}
\label{setup}

\paragraph{Datasets.}
The experiments include two primary tabular benchmarks and two
additional scalability and application studies.

\textbf{UCI HAR (classification):}
The UCI HAR dataset contains 561 smartphone-sensor-derived features for
six activities: walking, walking upstairs, walking downstairs, sitting,
standing, and laying. Measurements were recorded at 50\,Hz and
processed into 2.56-second segments. A random-forest classifier is used
as the fixed predictive model whose outputs are explained.

\textbf{HouseEnergy-Sim (regression):}
HouseEnergy-Sim is a controlled synthetic benchmark representing hourly
residential electricity consumption. Its features include time,
appliance-load proxies, and weather variables with configurable
dependence and redundancy. A random-forest regressor is trained using
the designated training split, and its fixed outputs are explained in
the full and lightweight environments.

\textbf{Norwegian residential demand:}
The NO1--NO5 experiment uses aggregated hourly residential-demand
series for the five Norwegian bidding areas, together with correlated
weather, temporal, and load-history predictors
\cite{hofmann2023norwegian}. A separate XGBoost model is fitted for
each area.

\textbf{CIFAR-10 representations:}
The CIFAR-10 experiment evaluates MCIR on high-dimensional
representations extracted from a trained ResNet-50 model. This setting
is used to examine computational and estimator sensitivity as the
feature and conditioning dimensions increase.

All methods use the same predictive model outputs, preprocessing,
train/test partitions, and random seeds within each dataset.

\paragraph{Lightweight Protocol.}
The predictive model $M$ is trained once on the designated training
split and remains fixed during attribution. Let $N$ denote the size of
the full explanation dataset and let $n<N$ denote the lightweight
sample size. Observation-only lightweight experiments retain all
features and use the same output coordinates; hence $d'=d$, and
representation discrepancy is measured using ordinary Wasserstein
distance.

When a representation map
$G:\mathbb R^d\rightarrow\mathbb R^{d'}$ with $d'<d$ is explicitly
applied, the resulting output laws are compared using the Cai--Lim
augmented Wasserstein distance $\widetilde W_p$. Observation subsampling
alone does not invoke the different-dimensional Cai--Lim construction.

Lightweight fractions and sampling rules are fixed before attribution
agreement is evaluated. Ranking agreement, top-$K$ overlap, and
deletion or insertion agreement are reported as components of
$\mathfrak F_{N,n}$; they are not used to select the same lightweight
sample whose fidelity is subsequently assessed.

\begin{figure}[t]
\centering
\includegraphics[width=\linewidth]
{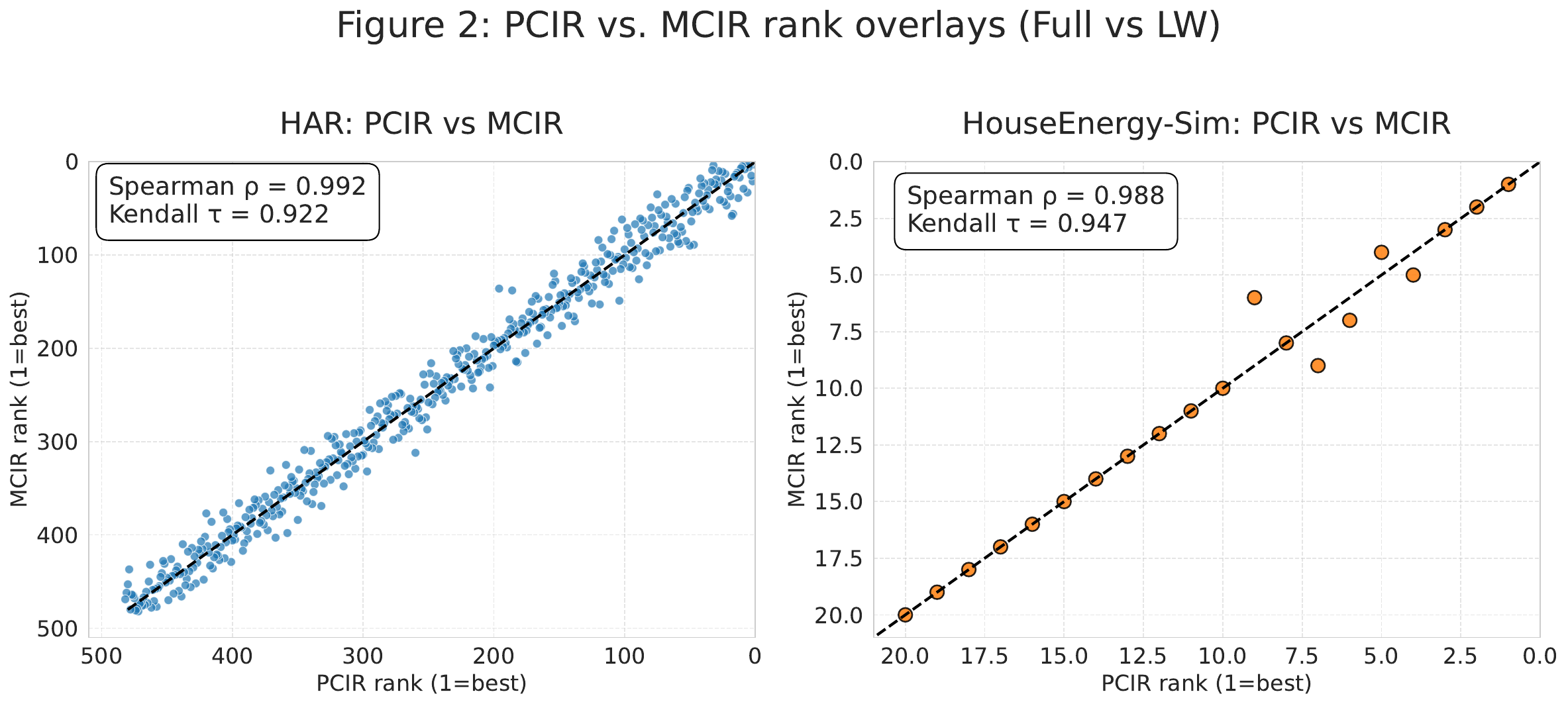}
\caption{PCIR--MCIR rank overlays for UCI HAR (left) and
HouseEnergy-Sim (right). Each point represents one feature, with its
PCIR rank on the horizontal axis and MCIR rank on the vertical axis.
Proximity to the diagonal therefore indicates agreement between the two
methods. Full and lightweight results are distinguished by marker type;
the displacement between corresponding markers indicates sensitivity
to lightweight subsampling. Full-versus-lightweight agreement is
quantified separately using Kendall--$\tau_b$ and Jaccard@$K$.}
\label{fig:expts-overlay}
\end{figure}

\paragraph{Auto-$\Phi$ and Neighbourhood Selection.}
We evaluate fixed neighbourhood sizes
$|\Phi(i)|\in\{0,1,3,5\}$ as a sensitivity analysis. The case
$|\Phi(i)|=0$ is retained only as the degenerate endpoint described in
Equation~\ref{eq:empty-neighbourhood-degeneracy}. When the screened pool
is nonempty, Auto-$\Phi$ searches over prespecified positive candidate
sizes using Equation~\ref{eq:auto-phi-objective}.

Results are reported for every candidate size in addition to the
selected neighbourhood. This reveals whether a stable choice remains
sensitive to omitted dependence and whether increasing $|\Phi(i)|$
degrades CMI estimation. Table~\ref{tab:estimator_ci} reports direct
agreement between fixed estimators,
Table~\ref{tab:switch_selection} reports estimator-selection
frequencies, and Appendix Table~\ref{tab:estimator_switching_all} reports the
performance of the switching rule. These tables therefore evaluate
different aspects of the estimation procedure.

\paragraph{Estimators and Stability Evaluation.}
MCIR estimates
$I_{P_{X,Q}}(Q;X_i\mid X_{\Phi(i)})$ and
$I_{P_{X,Q}}(Q;X_{\Phi(i)\cup\{i\}})$ within each selected local block.
Gaussian--copula estimation is used for continuous variables under the
stated copula approximation, with deterministic treatment of ties.
$k$NN estimators are evaluated for continuous, low- or
moderate-dimensional blocks, while plug-in estimators are used for
discrete variables with adequate cell counts.

For each lightweight sample size $n$, we use $B=50$ bootstrap
replicates to measure score and ranking variability. Estimator
sensitivity is evaluated using 200 additional resamples and reported
through Kendall--$\tau_b$, Spearman's $\rho$, Jaccard@$K$, and
confidence intervals. Because $k$NN-based CMI estimation may
deteriorate as $|\Phi(i)|$ increases, estimator variability is reported
as a function of conditioning dimension rather than assuming a common
concentration rate.

\paragraph{Faithfulness and Deletion Metrics.}
Deletion and insertion tests evaluate how predictive performance
changes when features are removed or added according to their
attribution ranking. For deletion, a faster performance decrease and a
lower deletion AUC indicate stronger faithfulness under the specified
perturbation protocol. We additionally report Spearman's $\rho$,
Kendall's $\tau_b$, Jaccard@$K$, Group-Jaccard@$K$, allocation drift,
runtime, and scalability.

\paragraph{Group-Level Stability.}
Exact Jaccard@$K$ measures agreement between feature identities.
Under strong dependence, different but highly correlated features may
be selected across subsamples while representing similar predictive
information. We therefore additionally report Group-Jaccard@$K$, which
measures agreement after grouping correlated features. Correlation
groups are constructed using an absolute Pearson-correlation threshold
of $0.85$, with sensitivity evaluated over thresholds from $0.60$ to
$0.85$.

Let $g(j)$ denote the correlation group containing feature $X_j$. For
top-$K$ index sets $S_K$ and $\widetilde S_K$, define
\begin{equation}
\medmath{
\operatorname{GJ@K}
=
\frac{
\left|
\{g(j):j\in S_K\}
\cap
\{g(j):j\in\widetilde S_K\}
\right|
}{
\left|
\{g(j):j\in S_K\}
\cup
\{g(j):j\in\widetilde S_K\}
\right|
}.
}
\label{eq:group-jaccard}
\end{equation}
Exact Jaccard@$K$ and Group-Jaccard@$K$ are reported together because
they measure feature-level and group-level agreement, respectively.

\paragraph{Runtime Protocol.}
Runtime is measured separately for candidate screening, neighbourhood
selection, MI/CMI estimation, bootstrap evaluation, and final ranking.
Within each reported timing comparison, all methods are evaluated in
the same recorded runtime session using the same explanation sample,
feature set, and repetition count. Wall-clock measurements remain
implementation- and hardware-specific.

MCIR runtime is not assumed to be uniformly linear in $n$ or $k$. Its
end-to-end cost depends on the screening procedure, bootstrap count
$B$, estimator candidates, maximum neighbourhood size $M_{\max}$, and
local estimation dimension. The corresponding operation count is given
in Equation~\ref{eq:auto-phi-complexity}. Runtime is therefore
reported empirically as a function of $n$, $k$, $|\Phi(i)|$, and
estimator family.

% ---- Baseline references (compact)

% ---- HAR lightweight utility summary

% ---- Aggregate agreement figure (requested)

% ============================================
% Results
% ============================================

\subsection{Implementation Details and Reproducibility}
\label{subsec:implementation}

Table~\ref{tab:reproducibility_summary} summarizes the common
implementation settings, while
Table~\ref{tab:experiment_protocol} reports dataset-specific models,
estimators, neighbourhood parameters, metrics, and repetition schemes.
All reported comparisons use fixed data partitions and recorded random
seeds.

 Repeated experiments are summarized using the stated sampling unit,
repetition count, descriptive effect estimates, and confidence
intervals or standard errors where available.
Experiments are executed through a common entry-point script, with
environment specifications and regeneration scripts available at
\url{https://anonymous.4open.science/r/MCIR-79B4/README.md}.

\begin{table*}[t]
\centering
\footnotesize
\setlength{\tabcolsep}{4pt}
\renewcommand{\arraystretch}{0.94}
\setlength{\aboverulesep}{0pt}
\setlength{\belowrulesep}{0pt}
\caption{Experimental configuration and reproducibility summary.}
\label{tab:reproducibility_summary}
\begin{tabular}{p{4.0cm}p{11.2cm}}
\toprule
\textbf{Category} & \textbf{Specification} \\
\midrule

\textbf{Software environment} &
Python 3.10; NumPy 1.26; SciPy 1.11; scikit-learn 1.3;
pandas 2.1; PyTorch; Matplotlib; and Seaborn. Exact package versions
are pinned in the repository environment file. \\

\textbf{Execution environment} &
Linux-based Google Colab Pro sessions. Deep-model training uses the
GPU assigned to the recorded session; attribution computations run on
CPU. The exact CPU, GPU, memory, and runtime configuration is recorded
with each experiment log. \\

\textbf{Randomness control} &
Seed 7 for train/test partitions and seed 42 for subsampling,
bootstrap, estimator, NumPy, scikit-learn, and PyTorch randomness.
Deterministic operations are used where supported. \\

\textbf{Gaussian--copula estimator} &
Rank transformation followed by empirical-CDF Gaussianization
$\widehat F\mapsto\Phi_{\mathcal N}^{-1}(\widehat F)$, deterministic
tie handling, and ridge-regularized covariance estimation. Multivariate
MI and CMI are computed through covariance determinant or Schur
complement formulas. Ridge values are reported per experiment in
Table~\ref{tab:experiment_protocol}. \\

\textbf{$k$NN estimator} &
KSG-type MI and conditional extensions for continuous, low- or
moderate-dimensional blocks; $k_{\mathrm{NN}}\in\{5,10,20\}$ with
Euclidean distance. Results are reported as a function of
$|\Phi(i)|$ because estimator behaviour may degrade with conditioning
dimension. \\

\textbf{Discrete estimator} &
Plug-in MI and CMI estimation from empirical probability masses.
Finite-sample corrections and minimum cell-count requirements are
reported with the corresponding experiment. \\

\textbf{Bootstrap and repetitions} &
$B_{\mathrm{rank}}=50$ replicates for ranking-stability analysis and
$B_{\mathrm{est}}=200$ replicates for estimator-sensitivity intervals.
Estimator-switching ablations use 30 independent repetitions.
Dataset-level results use five train/test splits unless otherwise
specified in Table~\ref{tab:experiment_protocol}. \\

\textbf{Candidate screening} &
Dependence screening is performed on the training data only. Pearson
screening uses $|\rho|\geq0.60$ and retains at most
$M_{\max}=6$ candidates per target feature. Alternative screening
statistics are reported for the corresponding experiments in
Table~\ref{tab:experiment_protocol}. \\

\textbf{Auto-$\Phi$ selection} &
Nested candidate neighbourhoods are evaluated using
Equation~\ref{eq:auto-phi-objective}, which combines bootstrap ranking
dispersion, normalized residual conditional dependence, and a
conditioning-size penalty. When the screened pool is nonempty,
Auto-$\Phi$ searches over the prespecified positive sizes
$|\Phi(i)|\in\{1,3,5\}$ permitted by that pool. The case
$|\Phi(i)|=0$ is included only as an unconditioned baseline in the
fixed-size sensitivity analysis. \\

\textbf{Dimensionality control} &
The selected conditioning set satisfies
$|\widehat\Phi(i)|\leq M_{\max}$. Estimator variability and ranking
sensitivity are reported by conditioning dimension rather than
assuming dimension-independent MI/CMI accuracy. \\

\textbf{Lightweight sampling} &
Uniform sampling without replacement for regression and stratified
sampling for classification. Prespecified fractions are
$n/N\in\{0.10,0.25,0.50\}$, with $0.25$ used as the default where
stated. The predictive model remains fixed during explanation. \\

\textbf{Representation comparison} &
Observation-only experiments have $d'=d$ and use ordinary Wasserstein
distance. Experiments applying an explicit map
$G:\mathbb R^d\rightarrow\mathbb R^{d'}$, $d'<d$, report $G$, $d'$,
and the Cai--Lim augmented Wasserstein distance
$\widetilde W_p$ separately. \\

\textbf{Ranking metrics} &
Kendall's $\tau_b$, Spearman's $\rho$, Jaccard@$K$ for
$K\in\{5,10\}$, and Group-Jaccard@$K$ under the stated correlation
threshold. \\

\textbf{Faithfulness metrics} &
Deletion and insertion curves, deletion AUC, classification accuracy
and macro-F1, and regression $R^2$ and RMSE. Predictive ROC-AUC is
reported only for applicable classification experiments and is
distinguished from deletion AUC. \\

\textbf{Statistical reporting} &
 Repeated experiments use the
stated sampling unit and repetition count and report descriptive effect
estimates, confidence intervals, or standard errors where available. \\

\textbf{Reported dimensions} &
Redundancy robustness, estimator sensitivity, rank stability,
lightweight fidelity, faithfulness, and computational cost. \\

\textbf{Code and artifacts} &
Implementations of MCIR and PCIR; MI/CMI estimators; screening and
Auto-$\Phi$ procedures; lightweight sampling; bootstrap protocols;
environment files; run logs; and scripts regenerating all tables and
figures. \\

\bottomrule
\end{tabular}
\end{table*}

The lightweight sample size $(N\rightarrow n)$ and output
representation dimension $(d\rightarrow d')$ are recorded separately.
The Cai--Lim different-dimensional comparison is used only when an
explicit representation map $G$ changes the output dimension; row
subsampling alone is evaluated in the common output space.
\begin{table*}[t]
\centering
\scriptsize
\setlength{\tabcolsep}{3pt}
\renewcommand{\arraystretch}{0.96}
\caption{Per-experiment protocol. All neighbourhoods are constructed
from training data, and predictive models remain fixed during
attribution.}
\label{tab:experiment_protocol}
\begin{adjustbox}{width=\linewidth}
\begin{tabular}{
p{0.6cm}
p{2.6cm}
p{2.3cm}
p{3.1cm}
p{2.3cm}
p{2.8cm}
p{4.2cm}
p{1.8cm}
}
\toprule
\textbf{ID} &
\textbf{Dataset} &
\textbf{Model} &
\textbf{Neighbourhood protocol} &
\textbf{Estimator} &
\textbf{Key settings} &
\textbf{Evaluation} &
\textbf{Repetitions} \\
\midrule

E1 &
Synthetic redundancy benchmark &
Linear model / RF &
Pearson candidate graph; top
$M_{\max}=3$ neighbours &
Gaussian--copula &
$B=200$ bootstrap replicates &
Duplicate-injection response, unique-information collapse, rank
stability, and comparison with marginal MI and SHAP &
20 runs \\

E2 &
UCI HAR &
Random forest &
Distance-correlation screening followed by Auto-$\Phi$ &
Gaussian--copula and $k$NN &
$k_{\mathrm{NN}}=5$; fixed tie rule &
Kendall-$\tau_b$, Spearman-$\rho$, Jaccard@$K$,
Group-Jaccard@$K$, and deletion faithfulness &
10 subsamples \\

E3 &
HouseEnergy-Sim &
Random forest &
Pearson screening; fixed-size and Auto-$\Phi$ sensitivity &
Gaussian--copula; $k$NN in the estimator ablation &
Neighbourhood sizes are study-specific; Auto-$\Phi$ penalty
sensitivity uses $\mathcal R=\{1,2,3,5\}$ &
Rank stability, redundancy robustness, estimator sensitivity,
deletion AUC, and runtime &
15 subsamples unless otherwise stated \\

E4 &
Norwegian residential demand &
XGBoost regressor &
Training-data correlation screening; fixed local neighbourhood &
Gaussian--copula &
Chronological $80/20$ split; $n=200$ explanation observations;
$|\Phi(i)|=5$; top-$K=8$; seed $42$ &
Full--lightweight rank and head-set agreement, Group-Jaccard@$8$,
cumulative $R^2$-drop AUC, and runtime &
10 explanation samples \\

E5 &
CIFAR-10 representations &
Fixed ResNet-50 representation and MLP probe &
Correlation screening with a fixed local neighbourhood &
Gaussian--copula &
$|\Phi(i)|=5$ &
Deletion AUC and runtime under the fixed Gaussian--copula configuration &
5 runs \\
E6 &
HouseEnergy-Sim &
Fixed nonlinear output $Q$; XGBoost surrogate refits &
Pearson screening; top three neighbours for MCIR &
Gaussian--copula MCIR; LOCO; region-weighted LOCO; Shapley-effects
plug-in &
$n_{\mathrm{train}}=1800$, $n_{\mathrm{cal}}=800$,
$n_{\mathrm{test}}=1400$; top-$4$; 24 permutations &
Core-feature recovery, positive attribution allocation,
cumulative $R^2$-drop AUC, cross-method agreement, and runtime &
10 runs \\

E7 &
High-dimensional HouseEnergy-Sim &
Fixed nonlinear output $Q$ &
Pearson screening from the estimator-specific reference sample &
Gaussian--copula and continuous $k$NN MI/CMI &
$k=24$ features; $n\in\{500,1000,2000\}$;
$|\Phi|\in\{1,2,3,5,8,12\}$; reference $n=5000$;
$k_{\mathrm{NN}}=5$ &
Agreement with the estimator-specific reference, component behaviour,
top-$4$ overlap, core-feature recovery, and runtime &
8 runs per setting \\
\bottomrule
\end{tabular}
\end{adjustbox}
\end{table*}

\section{Results}
\label{sec:results}

The results are organized by research question. For \textbf{RQ1}, the
controlled redundancy experiments show that MCIR assigns lower
conditional attribution to duplicated and near-duplicated features when
their information is already represented in $X_{\Phi(i)}$. This is an
empirical finite-sample result; exact population redundancy collapse is
characterized separately by
Equation~\ref{eq:mcir-zero-characterization}. Full-versus-lightweight attribution agreement is evaluated separately
in Section~\ref{subsubsec:rq2-lightweight} using complete-ranking,
top-$K$, and faithfulness diagnostics. Cross-method agreement between
PCIR and MCIR is reported separately and is not used as evidence of
lightweight fidelity.

For \textbf{RQ2}, the observation-only lightweight experiments reduce
the explanation sample from $N$ to $n$ while retaining the same
features and output coordinates. Reducing the explanation sample lowers runtime in the evaluated
configurations, with the observed computational effect depending on the
dataset, estimator, feature count, and conditioning neighbourhood.
Full-versus-lightweight agreement is evaluated separately through the
components of $\mathfrak F_{N,n}$. Empirically, agreement was stronger on the synthetic benchmark and more
moderate at the exact feature level on UCI HAR. Representation, ranking, and deletion-curve
agreement are reported as separate components of
$\mathfrak F_{N,n}$; output-law similarity is not treated as proof of
ranking preservation.

For \textbf{RQ3}, removing highly ranked MCIR features produces
systematic degradation in the fixed model's predictive performance.
Faithfulness is quantified using deletion AUC together with accuracy,
macro-F1, or $R^2$, as appropriate. We additionally compare MCIR with
predictive-risk and variance-allocation functionals on HouseEnergy-Sim.
LOCO measures the loss incurred after excluding a covariate, regional
LOCO examines whether such loss varies across regions of the input
space, and Shapley effects allocate explained output variance among
dependent inputs. These targets complement MCIR's measurement of unique
conditional information and need not induce the same feature ordering.
All perturbation results measure predictive association under the
stated protocol and do not establish causal importance.

\subsection{UCI HAR and HouseEnergy-Sim Results}
\label{res}

We first examine rank agreement, redundancy sensitivity, deletion
faithfulness, estimator sensitivity, and runtime on UCI HAR and
HouseEnergy-Sim. Experimental settings are specified in
Section~\ref{setup}.
\subsubsection{Dependence-Aware Attribution under Strong Dependence (RQ1)}
\label{subsubsec:rq1-dependence}

Block-level aggregation can obscure feature-level conditional
contributions when several predictors encode similar information. MCIR
instead evaluates the unique predictive information contributed by
$X_i$ relative to its selected neighbourhood $X_{\Phi(i)}$.
\begin{figure}[h]
\centering
\begin{minipage}{.48\linewidth}
\centering
\includegraphics[width=\linewidth]{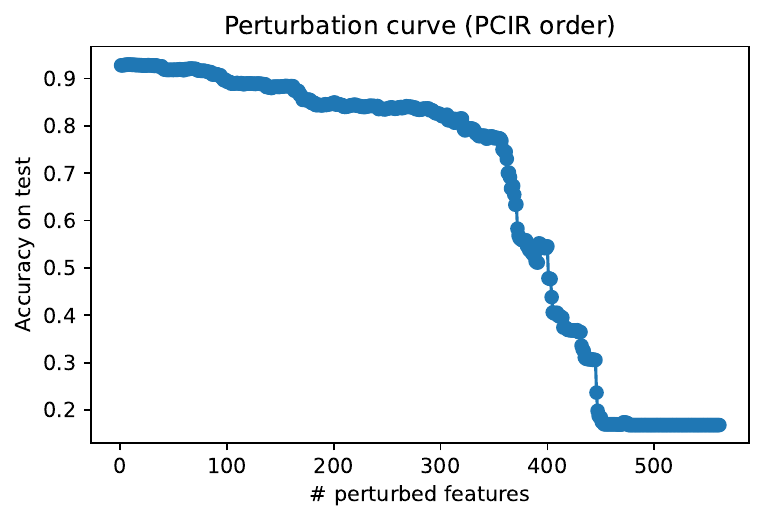}\\[2pt]
\includegraphics[width=\linewidth]{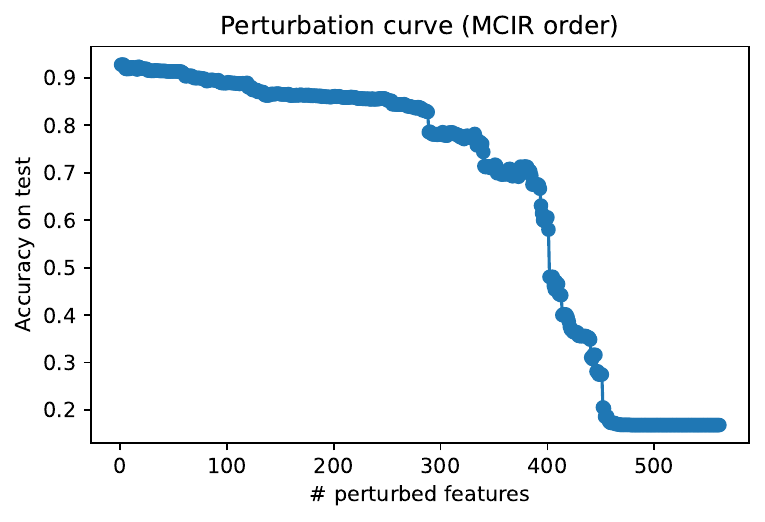}
\captionof{figure}{\textbf{UCI HAR perturbation analysis.} Test
accuracy as progressively higher-ranked PCIR and MCIR features are
perturbed. A faster decrease indicates stronger faithfulness under the
specified perturbation protocol; ranking stability is evaluated
separately.}
\label{fig:expts-har-faith}
\end{minipage}
\hfill
\begin{minipage}{.48\linewidth}
\centering
\includegraphics[width=\linewidth]
{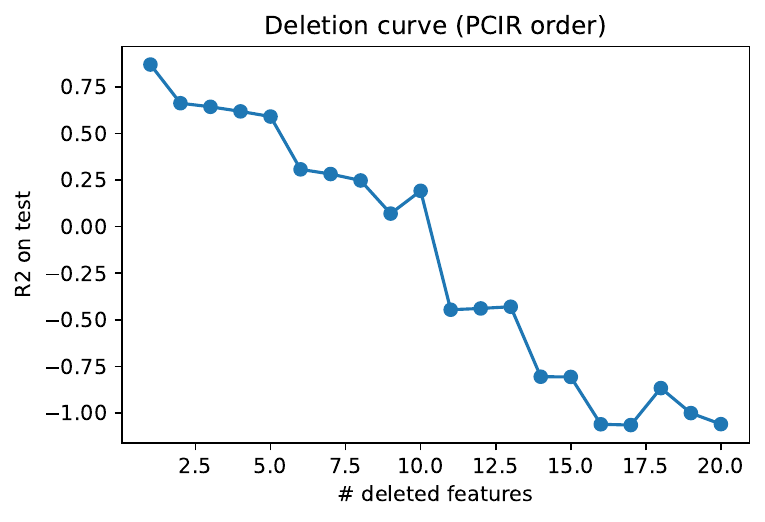}\\[2pt]
\includegraphics[width=\linewidth]
{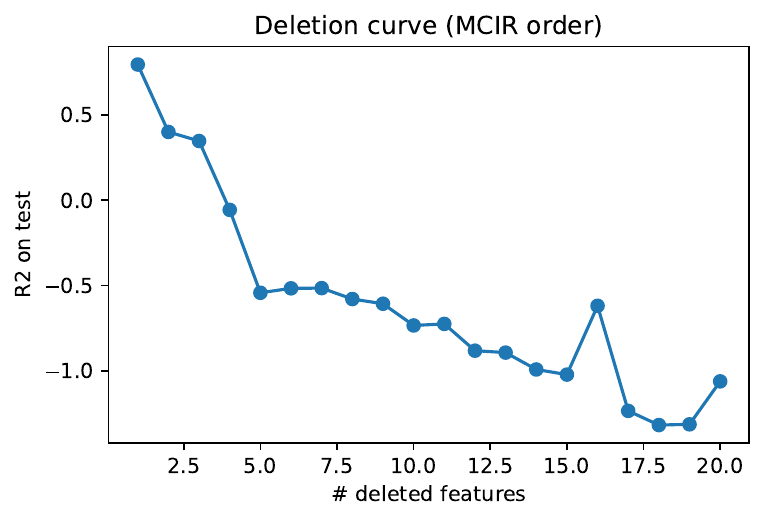}
\captionof{figure}{\textbf{HouseEnergy-Sim deletion analysis.} Test
$R^2$ as progressively higher-ranked PCIR and MCIR features are
removed. Lower deletion AUC indicates stronger faithfulness under the
specified deletion protocol.}
\label{fig:expts-house-faith}
\end{minipage}
\end{figure}

Figure~\ref{fig:expts-overlay} compares PCIR and MCIR ranks, while
full-versus-lightweight stability is measured separately through the
displacement between paired markers and the corresponding
Kendall-$\tau_b$ and Jaccard@$K$ statistics. Thus, proximity to the
diagonal measures agreement between PCIR and MCIR, not agreement
between the full and lightweight environments.
The deletion and perturbation results in
Figures~\ref{fig:expts-har-faith} and
\ref{fig:expts-house-faith} show that removing or perturbing highly
ranked MCIR features reduces predictive performance. Comparative curve
steepness varies across datasets and baselines; predictive faithfulness
is therefore assessed using the prespecified deletion protocol and AUC
rather than interpreted as a universal advantage of MCIR.

The deletion experiment measures predictive faithfulness rather
than causal influence. The early part of each curve reflects the effect
of removing the highest-ranked features. Later portions may be affected
by feature interactions, remaining redundancy, and increasingly
out-of-distribution perturbations and are therefore interpreted
separately from the initial deletion response and deletion AUC.
Table~\ref{tab:sota-merged} reports cross-method ranking agreement under
the common lightweight configuration. PCIR and MCIR obtain
$\rho=0.83$, $\tau_b=0.66$, and Jaccard@$20=0.95$; these values do not
measure full--lightweight MCIR fidelity. Separately,
Table~\ref{tab:har_lw} reports Jaccard@$30=0.622$ in a
reduced-training sensitivity experiment involving a differently fitted
model and is therefore not treated as fixed-model fidelity evidence.
Exact and Group-Jaccard are both reported because correlated features
may act as interchangeable representatives of similar predictive
information. The lower agreement of MCIR with KernelSHAP, SAGE, and
HSIC indicates different ranking behaviour, not the correctness or
superiority of any method.
\paragraph{Controlled Redundancy Experiment.}
Let
\begin{equation}
\medmath{
X_2
=
\alpha X_1+\sqrt{1-\alpha^2}\,\widetilde Z,
\qquad
Q
=
\beta_1X_1+\beta_3X_3+\varepsilon,
}
\label{eq:synthetic-redundancy-model}
\end{equation}
where $X_1$, $X_3$, $\widetilde Z$, and $\varepsilon$ are mutually
independent and $\alpha\in[0,1]$. Then
\begin{equation}
\medmath{
Q\perp\!\!\!\perp X_2\mid X_1,
\qquad
I(Q;X_2\mid X_1)=0,
}
\label{eq:synthetic-conditional-redundancy}
\end{equation}
for every $\alpha$. Increasing $\alpha$ strengthens the marginal
association between $X_2$ and $Q$, while its population unique
predictive information relative to $X_1$ remains zero. This experiment
therefore tests whether a finite-sample MCIR estimator remains close to
the conditional-redundancy value while marginal methods increasingly
assign attribution to $X_2$.
If neighbourhoods are constructed symmetrically so that
$X_1\in\Phi(2)$ and $X_2\in\Phi(1)$, both features become conditionally
redundant as $\alpha\to1$. Identifying one as the representative feature
requires an asymmetric neighbourhood rule, group attribution, or
external domain information.
The experiment shows that conditional and marginal attribution answer
different questions. MCIR measures information unique to $X_2$ beyond
$X_1$, whereas marginal MI, HSIC, permutation importance, and global
SHAP may assign nonzero attribution to $X_2$ through its association
with $X_1$. Figure~\ref{fig:collapse} visualizes this controlled population
construction. When $\Phi(2)=\{1\}$, the redundant proxy satisfies
$C_2=0$ for every $\alpha$, while its normalized marginal information
increases with $\alpha$. The anchor score relative to the proxy
decreases as the proxy becomes more accurate and vanishes at the
exact-duplicate limit.
\begin{figure}[t]
\centering
\includegraphics[width=\linewidth]{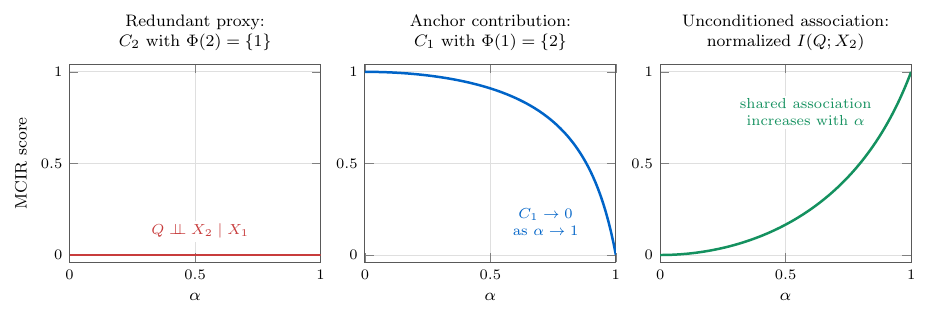}
\caption{Population behaviour under a Gaussian redundancy construction.
Let $X_1,X_3,E_2$ be independent standard Gaussian variables,
$X_2=\alpha X_1+\sqrt{1-\alpha^2}E_2$, and
$Q=X_1+0.5X_3+\varepsilon$, where
$\varepsilon\sim\mathcal N(0,0.25)$ is independent.
\textbf{Left:} because
$Q\perp\!\!\!\perp X_2\mid X_1$, the redundant proxy has
$C_2=0$ for every $\alpha$ when $\Phi(2)=\{1\}$.
\textbf{Middle:} the unique predictive information of $X_1$ relative
to $X_2$ decreases as the proxy becomes more accurate and vanishes at
the exact-duplicate limit $\alpha=1$.
\textbf{Right:} the normalized marginal information
$I(Q;X_2)/I(Q;X_1)$ increases with $\alpha$, illustrating how an
unconditioned association measure can credit a redundant proxy.
At $\alpha=1$, neither feature can be identified as the ``original''
from statistical dependence alone when each is conditioned on the
other.}
\label{fig:collapse}
\vspace{-3mm}
\end{figure}
\paragraph{Ablation on Neighbourhood Size.}
The case $|\Phi(i)|=0$ is not a marginal-ranking version of MCIR.
Indeed,
\begin{equation}
\medmath{
\Phi(i)=\varnothing,\quad I(Q;X_i)>0
\quad\Longrightarrow\quad
U_i=J_i=I(Q;X_i),\quad C_i=1.
}
\label{eq:empty-neighbourhood-degeneracy}
\end{equation}
It is therefore included only as an endpoint diagnostic and is excluded
from comparative MCIR ranking claims.
\begin{figure}[t]
\centering
\includegraphics[width=0.5\linewidth]{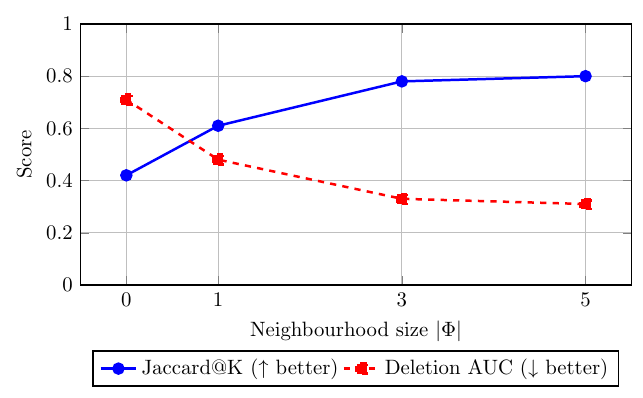}
\caption{Sensitivity to neighbourhood size. The case
$|\Phi(i)|=0$ is shown only as a degenerate endpoint: whenever
$I(Q;X_i)>0$, Definition~\ref{def:mcir} gives $C_i=1$. Nonempty
neighbourhoods enable MCIR to distinguish unique from shared
information. Moderate sizes provide the best empirical balance between
residual redundancy and CMI estimation variability in this experiment.}
\label{fig:phi_ablation}
\end{figure}
Among nonempty candidates, Figure~\ref{fig:phi_ablation} shows that
$|\Phi(i)|\in\{1,3\}$ improves redundancy suppression and deletion
behaviour relative to underspecified neighbourhoods, while larger
conditioning sets provide diminishing returns and greater estimator
variability. Accordingly, when the screened candidate pool is nonempty,
Auto-$\Phi$ searches over $|\Phi(i)|\geq1$; the empty neighbourhood is
used only when no candidate passes screening. Selection follows
Equation~\ref{eq:auto-phi-objective} and is evaluated together with
neighbourhood-size sensitivity.
\paragraph{Auto-$\Phi$ Penalty Sensitivity.}
We evaluated the Auto-$\Phi$ penalties over
$\lambda_B\in\{0,0.25,0.5,1\}$ and
$\lambda_D\in\{0,0.05,0.10,0.25\}$ using
$n\in\{500,1000,2000\}$, candidate sizes
$\mathcal R=\{1,2,3,5\}$, 30 bootstrap resamples, and 20 independent
repetitions per sample size. Across all 48 sample-size--penalty
configurations, the criterion selected $r^\star=1$ in every repetition.
At $r=1$, normalized residual dependence ranged from $0.0156$ to
$0.0196$, while bootstrap ranking dispersion was lower than for
$r\in\{2,3,5\}$. Increasing the neighbourhood size further reduced
residual dependence but increased ranking variability. Thus, selection
was insensitive to the evaluated penalty values in this controlled
experiment. This result does not show that the residual-dependence or
dimension penalty improves upon variance-only selection, nor does it
establish population-optimal neighbourhood recovery.

\begin{table}[htbp]
\centering
\scriptsize
\setlength{\tabcolsep}{3.5pt}
\renewcommand{\arraystretch}{0.96}
\caption{Auto-$\Phi$ penalty sensitivity on HouseEnergy-Sim. Values are
means over 20 repetitions. Every evaluated penalty pair selected
$r^\star=1$ in every repetition.}
\label{tab:auto-phi-penalty-sensitivity}
\begin{tabular}{ccccc}
\toprule
$n$ & $r^\star$ & Kendall $\tau_b$ & Jaccard@10 &
Driver Recall@10 \\
\midrule
500  & 1 & 0.391 & 0.514 & 0.538 \\
1000 & 1 & 0.445 & 0.583 & 0.638 \\
2000 & 1 & 0.514 & 0.656 & 0.650 \\
\bottomrule
\end{tabular}
\end{table}
\paragraph{Redundancy Stress Test: Ordinal Stability and Allocation Drift.}

We evaluate robustness to explicit predictor redundancy using the
nonlinear HouseEnergy-Sim benchmark with $N=4000$ observations. The
four original drivers are \texttt{Base\_load}, \texttt{HVAC},
\texttt{Fridge}, and \texttt{Solar}. The ground-truth target is
generated as
\begin{equation}
\medmath{
Y^\star
=
5.0\,X_{\mathrm{Base}}
+
2.5\,X_{\mathrm{HVAC}}
+
1.2\,X_{\mathrm{Fridge}}
+
1.5\,X_{\mathrm{Base}}X_{\mathrm{HVAC}}
-
0.8\,X_{\mathrm{Solar}}^2
+
\varepsilon,
\qquad
\varepsilon\sim\mathcal N(0,0.1^2).
}
\label{eq:houseenergy-generator}
\end{equation}
A fixed predictive model is trained to approximate $Y^\star$, and MCIR
explains its output representation $Q=M(X)$. The coefficient structure
designates \texttt{Base\_load} as the anchor feature while retaining a
nonlinear interaction with \texttt{HVAC}.
Redundancy is introduced by adding $r$ noisy copies of the anchor:
\begin{equation}
\medmath{
X_{\mathrm{dup},j}
=
X_{\mathrm{Base}}+\eta_j,
\qquad
\eta_j\overset{\mathrm{i.i.d.}}{\sim}\mathcal N(0,0.02^2),
\qquad
r\in\{0,5,10,20,40\}.
}
\label{eq:duplicate-injection}
\end{equation}
\begin{wraptable}{r}{0.50\linewidth}
\centering
\small
\setlength{\tabcolsep}{3pt}
\renewcommand{\arraystretch}{0.96}
\caption{HouseEnergy-Sim redundancy stress test. Values are means over
the stated repetitions; lower is better.}
\label{tab:redundancy_stress}
\begin{tabular}{@{}lcccccc@{}}
\toprule
$r$ &
\multicolumn{2}{c}{\textbf{MCIR}} &
\multicolumn{2}{c}{\textbf{MI}} &
\multicolumn{2}{c}{\textbf{PCIR}} \\
\cmidrule(lr){2-3}
\cmidrule(lr){4-5}
\cmidrule(lr){6-7}
&
$\Delta_{\rm rank}$ &
$\Delta_{\rm alloc}$ &
$\Delta_{\rm rank}$ &
$\Delta_{\rm alloc}$ &
$\Delta_{\rm rank}$ &
$\Delta_{\rm alloc}$ \\
\midrule
0  & 0.0000 & 0.0105 & 0.0000 & 0.0103 & 0.0000 & 0.0015 \\
5  & 0.4667 & 0.4473 & 0.0000 & 0.4495 & 1.0000 & 0.2178 \\
10 & 0.4667 & 0.5261 & 1.3333 & 0.5288 & 1.0000 & 0.2499 \\
20 & 0.4667 & 0.5785 & 1.2667 & 0.5789 & 1.0000 & 0.2747 \\
40 & 0.5333 & 0.6095 & 1.3333 & 0.6077 & 1.0000 & 0.2910 \\
\bottomrule
\end{tabular}
\end{wraptable}
The experiment measures stability on the four original features using
\begin{equation}
\medmath{
\Delta_{\mathrm{rank}}(r)
=
1-\tau_b(\bm s_0,\bm s_r),
\qquad
\Delta_{\mathrm{alloc}}(r)
=
\|\bm s_0-\bm s_r\|_2,
}
\label{eq:redundancy-stability-metrics}
\end{equation}
where $\bm s_0$ is the reference attribution vector and $\bm s_r$ is
the corresponding vector after injecting $r$ duplicates. The same
within-method normalization is applied before computing allocation
drift. Lower values indicate greater ordinal or allocation stability.
Unlike $d_{\mathrm{rank}}$ in
Equation~\ref{eq:rank-discrepancy},
$\Delta_{\mathrm{rank}}=1-\tau_b$ ranges from $0$ to $2$.

\begin{figure}[t]
\centering
\includegraphics[width=0.95\linewidth]{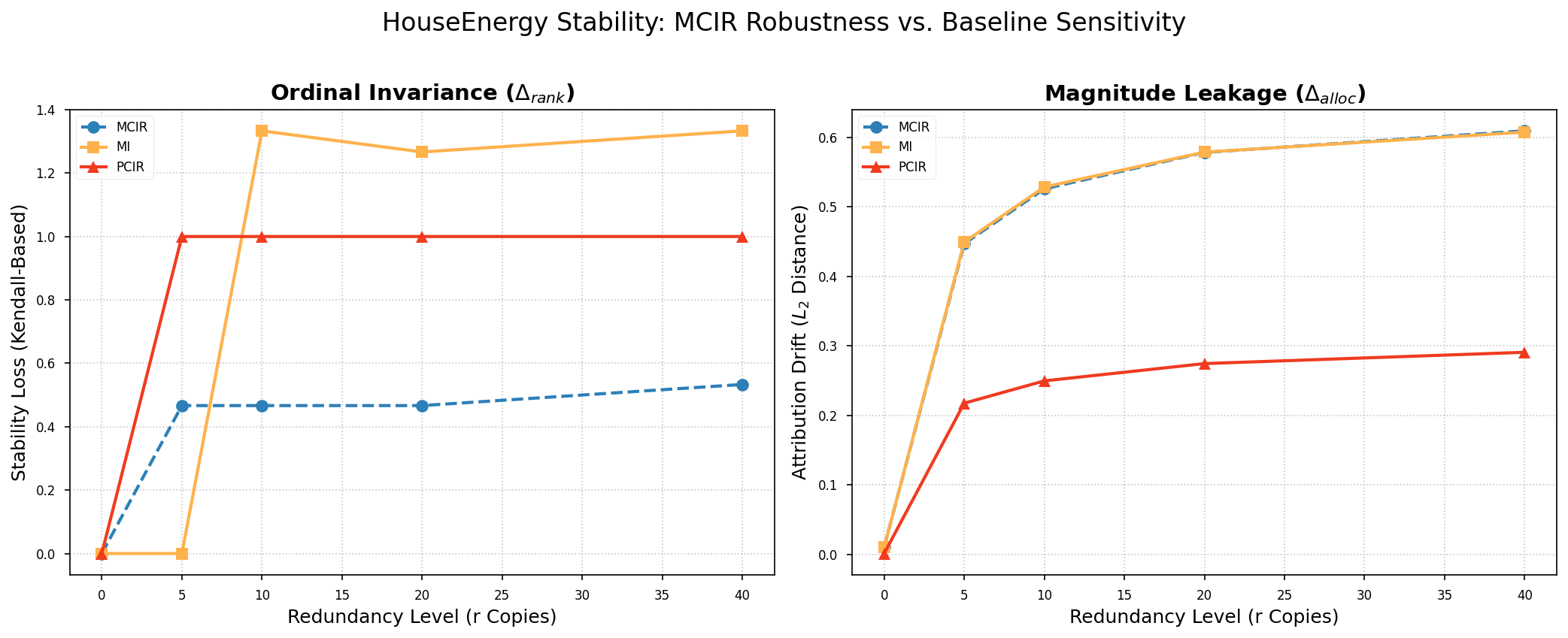}
\caption{Redundancy stress test on HouseEnergy-Sim.
\textbf{Left:} mean ordinal instability
$\Delta_{\mathrm{rank}}=1-\tau_b$ on the original features.
\textbf{Right:} mean allocation drift
$\Delta_{\mathrm{alloc}}$ on the same features. Lower values indicate
greater stability. MCIR exhibits lower ordinal instability than MI and
PCIR at the larger redundancy levels, while MCIR and MI show similar
allocation drift.}
\label{fig:redundancy_stress}
\end{figure}

\paragraph{Ordinal Stability.}
At $r=0$, ordinal instability is zero by construction because each
method is compared with its own no-duplicate reference ranking. After
duplicates are introduced, PCIR has
$\Delta_{\mathrm{rank}}=1.0$ for all evaluated $r\geq5$. MI remains
unchanged at $r=5$ but reaches
$\Delta_{\mathrm{rank}}\in\{1.2667,1.3333\}$ at the larger redundancy
levels. MCIR has
$\Delta_{\mathrm{rank}}=0.4667$ for $r\in\{5,10,20\}$ and $0.5333$ at
$r=40$. Thus, MCIR exhibits lower observed ordinal instability than
both baselines under the larger redundancy configurations.

\paragraph{Allocation Drift.}
MCIR and MI exhibit similar allocation drift as redundancy increases.
At $r=40$, their respective values are $0.6095$ and $0.6077$. PCIR has
a smaller allocation drift of $0.2910$ but a larger ordinal
instability. This difference shows that allocation drift and ranking
instability measure distinct behaviours: a relatively small change in
attribution magnitude does not necessarily preserve feature ordering.

The small nonzero allocation values at $r=0$ reflect variability across
the repeated estimation runs relative to the fixed reference vector.
If $\bm s_0$ is compared with itself without resampling, these values
must equal zero.

Figure~\ref{fig:redundancy_stress} and
Table~\ref{tab:redundancy_stress} therefore provide complementary
evidence. In this controlled benchmark, MCIR preserves the ordering of
the designated original drivers more effectively than MI and PCIR at
the larger tested redundancy levels, while not eliminating allocation
drift. Because the anchor is designated by construction, this
experiment evaluates robustness relative to a known reference feature;
it does not imply that MCIR can identify an ``original'' feature among
statistically indistinguishable duplicates without additional
structure.
\begin{table}[htbp]
\centering
\footnotesize
\setlength{\tabcolsep}{5pt}
\renewcommand{\arraystretch}{1.05}
\caption{\textbf{UCI HAR training-subsample sensitivity.} The reduced
configuration uses half of the training rows and therefore evaluates a
different fitted model. This table is reported separately from the
fixed-model lightweight attribution analysis.}
\label{tab:har_lw}
\begin{tabular}{lcccccc}
\toprule
\textbf{Setting} &
\textbf{Training rows} &
\textbf{Accuracy} &
\textbf{Macro-F1} &
$\boldsymbol{D_{\mathrm{KL}}}$ &
\textbf{Jaccard@30} &
\textbf{F1 ratio} \\
\midrule
Full training &
7{,}352 & 0.930 & 0.928 & 0.000 & 1.000 & 1.000 \\
Reduced training &
3{,}676 & 0.917 & 0.914 & 7.174 & 0.622 & 0.985 \\
\bottomrule
\end{tabular}
\end{table}
\subsubsection{Lightweight Agreement and Runtime Profile (RQ2)}
\label{subsubsec:rq2-lightweight}

The observation-only lightweight experiments use a fixed predictive
model and reduce the explanation sample from $N$ to $n<N$, while
retaining the same features and output coordinates. Thus, the
experiments compare full and lightweight explanation
\emph{environments}, not separately trained full and lightweight
models.

\begin{figure*}[t]
\centering
\includegraphics[width=\linewidth]{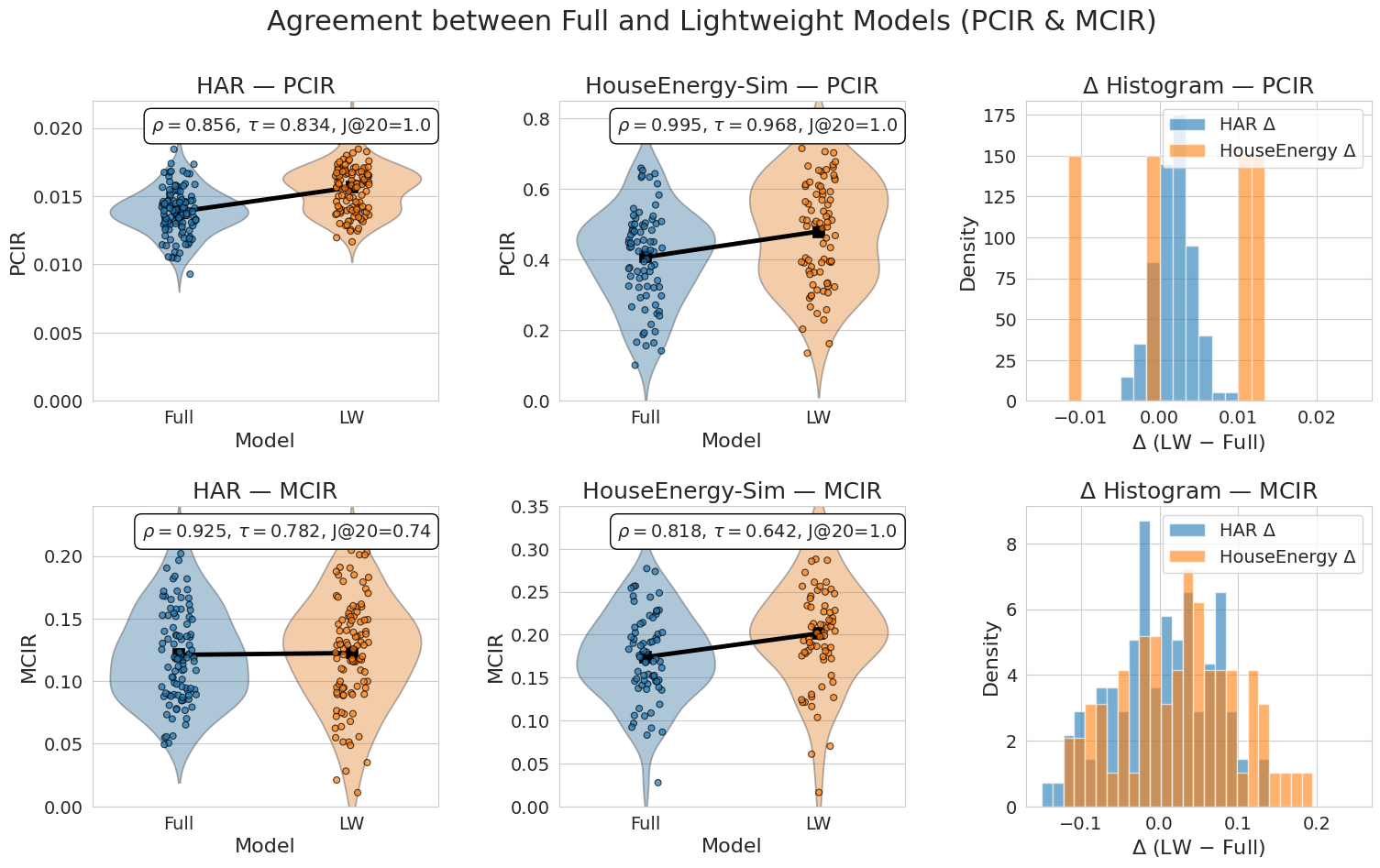}
\caption{\textbf{Full-versus-lightweight attribution comparison.}
The first row reports PCIR results for UCI HAR and HouseEnergy-Sim, and
the second row reports MCIR results. Each panel presents full and
lightweight score distributions, paired feature-level scores, and
mean confidence intervals. The $\Delta$ histograms show
$\widehat C_{i,n}-\widehat C_{i,N}$. Inset badges report Spearman's
$\rho$, Kendall's $\tau_b$, and Jaccard@$20$. These quantities measure
different aspects of score and ranking agreement.}
\label{fig:metric-first}
\end{figure*}

Across the evaluated lightweight configurations, MCIR achieves
Kendall-$\tau_b$ values of $0.72$--$0.89$ for head rankings and
$0.55$--$0.76$ for complete rankings. Runtime varies substantially
with feature count, conditioning size, and estimator: at $n=2000$,
Gaussian--copula MCIR requires $0.4487$\,s on HouseEnergy-Sim
($k=20$) and $451.412$\,s on UCI HAR ($k=561$), whereas $k$NN MCIR
requires $4.3769$\,s and $139.738$\,s, respectively. The faster
estimator therefore depends on the dataset and implementation, so these
measurements do not imply a universal runtime ordering or speedup.

Figure~\ref{fig:metric-first} reports score and ranking agreement,
whereas deletion-curve agreement is evaluated separately in
Figures~\ref{fig:expts-har-faith} and
\ref{fig:expts-house-faith}. Together with the representation
discrepancy, these quantities form the fidelity profile
$\mathfrak F_{N,n}$. None of its individual components is treated as
proof of the others.

The UCI HAR results in Table~\ref{tab:har_lw} involve reduced model
training and are therefore reported as a separate training-subsample
sensitivity analysis. They are not used as evidence for fixed-model
lightweight attribution fidelity.

\begin{table}[htbp]
\centering
\scriptsize
\setlength{\tabcolsep}{3pt}
\renewcommand{\arraystretch}{0.97}
\caption{Score-computation costs on a lightweight explanation sample
of size $n$. Here, $k$ is the number of features,
$m_\Phi$ is the conditioning size,
$d_Q$ is the output-representation dimension, and
$r=d_Q+m_\Phi+1$. Screening, Auto-$\Phi$, bootstrap, and estimator
selection costs are excluded and accounted for separately in
Equation~\ref{eq:auto-phi-complexity}.}
\label{tab:complexity}
\begin{adjustbox}{width=\linewidth}
\begin{tabular}{p{3.5cm}p{5.2cm}p{7.0cm}}
\toprule
\textbf{Method} &
\textbf{Score-computation cost} &
\textbf{Scope and qualifications} \\
\midrule

PCIR &
$\mathcal O(nk)$ after preprocessing &
Uses vectorized feature-wise dispersion calculations. Rank
preprocessing, when applied, adds up to $\mathcal O(nk\log n)$. \\

MCIR (Gaussian--copula) &
$\mathcal O\!\left(
k[rn\log n+nr^2+r^3]
\right)$ &
Includes rank Gaussianization, covariance estimation, and matrix
factorization for fixed neighbourhoods. \\

MCIR ($k$NN) &
Expected $\mathcal O(kn\log n)$ for fixed low $r$ &
Tree-based neighbour search; constants and estimator accuracy
deteriorate with dimension, and worst-case search may approach
$\mathcal O(kn^2)$. \\

MCIR (plug-in) &
$\mathcal O(knr)$ for fixed alphabets &
Count construction is linear in observations, but memory and sample
requirements grow with the number of occupied joint cells. \\

BlockCIR &
$\mathcal O(nk)$ after block construction &
Covers within-block statistics and aggregation; block discovery is
excluded. \\

KernelSHAP &
$\mathcal O(S(k)\,\mathsf E(n))$ &
Requires $S(k)$ model evaluations; conditional variants add dependence
estimation and conditional-sampling overhead. \\

SAGE &
$\mathcal O(S(k)\,\mathsf E(n))$ &
Depends on the coalition-sampling budget and the cost
$\mathsf E(n)$ of evaluating the predictive model. \\

HSIC &
$\mathcal O(kn^2)$ without approximation &
Quadratic Gram-matrix computation; low-rank, random-feature, or
Nystr{\"o}m approximations can reduce this cost. \\

Marginal KSG-MI &
Expected $\mathcal O(kn\log n)$ in low dimension &
Requires one neighbour-search problem per feature; performance depends
on sample and representation dimension. \\

\bottomrule
\end{tabular}
\end{adjustbox}
\end{table}

Reducing $N$ to $n$ lowers the number of observations processed by each
score computation, but the realized speed-up also depends on fixed
overheads, screening, bootstrap replication, and estimator behaviour.
Accordingly, Table~\ref{tab:complexity} reports score-only asymptotic
costs, while end-to-end runtime is measured directly for each
experimental configuration.
\paragraph{Lightweight Stability under Feature Dependence.}

This experiment evaluates ranking stability under reduced data
availability, whereas deletion AUC evaluates predictive faithfulness
under feature removal. Exact Jaccard@$K$ measures agreement between
feature identities, while Group-Jaccard@$K$ measures agreement after
mapping correlated features to common groups. These metrics therefore
capture complementary aspects of explanation quality and are not
interpreted interchangeably.

For example, two highly correlated features $X_1$ and $X_2$
may act as interchangeable representatives within a fitted model. Two
rankings may select different representatives and consequently have low
Exact Jaccard but high Group-Jaccard. Similarly, deleting only one
representative may produce a gradual performance decrease when another
correlated feature remains available. This behaviour does not invalidate
deletion analysis; rather, it distinguishes identity-level stability,
group-level stability, and deletion faithfulness.

We compare full and lightweight rankings under identical predictive
models, data partitions, feature spaces, and subsampling protocols.
Table~\ref{tab:jaccard_combined}(a) reports Exact and Group-Jaccard@$30$
at lightweight fractions $n/N\in\{0.20,0.40,0.60\}$. Correlation groups
are constructed using $|\rho|\geq0.85$ in the main experiment. Because
group membership is based on correlation, Group-Jaccard measures
agreement between correlated groups and does not establish that the
selected features are causally or functionally identical.

MCIR does not obtain the highest Exact Jaccard in every configuration.
At the group level, it obtains the highest value at the $20\%$ and
$40\%$ fractions, while HSIC obtains the highest value at $60\%$.
Accordingly, both feature-level and group-level results are reported
without treating either metric as a complete measure of explanation
quality.

\begin{table*}[t]
\centering
\fontsize{7}{7.6}\selectfont
\setlength{\tabcolsep}{2pt}
\renewcommand{\arraystretch}{0.86}
\caption{UCI HAR stability under subsampling and sensitivity to the
correlation-group threshold. In panel (a), entries are
Exact/Group-Jaccard@$30$.}
\label{tab:jaccard_combined}
\label{tab:jaccard_group}
\label{tab:group_threshold_sensitivity}
\vspace{-1.5mm}

\begin{minipage}[t]{0.49\linewidth}
\centering
\begin{adjustbox}{width=\linewidth}
\begin{tabular}{lccc}
\toprule
\multicolumn{4}{c}{\textbf{(a) Subsampling stability: Exact / Group}}\\
\midrule
\textbf{Method} & \textbf{20\%} & \textbf{40\%} & \textbf{60\%} \\
\midrule
MCIR &
0.585/\textbf{0.641} &
0.624/\textbf{0.671} &
0.627/0.655 \\

PCIR &
0.673/0.617 &
0.710/0.668 &
0.767/0.725 \\

SHAP &
0.644/0.574 &
0.698/0.636 &
0.747/0.694 \\

HSIC &
\textbf{0.674}/0.609 &
\textbf{0.711}/0.652 &
\textbf{0.787}/\textbf{0.742} \\
\bottomrule
\end{tabular}
\end{adjustbox}
\end{minipage}
\hfill
\begin{minipage}[t]{0.49\linewidth}
\centering
\begin{adjustbox}{width=\linewidth}
\begin{tabular}{ccccc}
\toprule
\multicolumn{5}{c}{
\textbf{(b) Group-Jaccard@30 threshold sensitivity}
}\\
\midrule
\textbf{Threshold} &
\textbf{MCIR} &
\textbf{PCIR} &
\textbf{SHAP} &
\textbf{HSIC} \\
\midrule
0.60 & 0.6405 & 0.5036 & \textbf{0.9200} & 0.4284 \\
0.70 & 0.6261 & 0.4648 & \textbf{0.9667} & 0.4334 \\
0.75 & 0.5792 & 0.4204 & \textbf{0.9741} & 0.4582 \\
0.80 & 0.5946 & 0.4720 & \textbf{0.9788} & 0.5201 \\
0.85 & 0.5821 & 0.4321 & \textbf{0.9662} & 0.4994 \\
\bottomrule
\end{tabular}
\end{adjustbox}
\end{minipage}
\vspace{-2mm}
\end{table*}

Table~\ref{tab:jaccard_combined}(b) evaluates thresholds from
$0.60$ to $0.85$. The extended analysis at thresholds $0.90$ and $0.95$ is reported in
Appendix~\ref{app:threshold-sensitivity}. Changing the threshold changes the number and
composition of the correlation groups and therefore changes the
absolute Group-Jaccard values. Lower thresholds generally form larger
groups, whereas higher thresholds form smaller and more selective
groups.
\begin{table}[htbp]
\centering
\scriptsize
\setlength{\tabcolsep}{3pt}
\renewcommand{\arraystretch}{0.96}
\caption{Pairwise agreement between attribution methods under the same
lightweight configuration. Jaccard@20 is omitted for HouseEnergy-Sim
because its 20 evaluated features make the top-20 set identical for
every method.}
\label{tab:sota-merged}
\begin{adjustbox}{width=0.82\linewidth}
\begin{tabular}{llccc}
\toprule
\textbf{Dataset} &
\textbf{Method pair} &
$\boldsymbol{\rho}$ &
$\boldsymbol{\tau_b}$ &
\textbf{Jaccard@20} \\
\midrule
\multirow{9}{*}{\textbf{UCI HAR}}
& PCIR vs.\ MCIR                     & 0.83 & 0.66 & 0.95 \\
& PCIR vs.\ KernelSHAP$_{\rm cond}$  & 0.14 & 0.09 & 0.05 \\
& PCIR vs.\ KernelSHAP$_{\rm indep}$ & 0.08 & 0.05 & 0.03 \\
& PCIR vs.\ SAGE                     & 0.09 & 0.06 & 0.04 \\
& PCIR vs.\ HSIC                     & 0.10 & 0.07 & 0.06 \\
& MCIR vs.\ KernelSHAP$_{\rm cond}$  & 0.13 & 0.08 & 0.05 \\
& MCIR vs.\ KernelSHAP$_{\rm indep}$ & 0.07 & 0.05 & 0.04 \\
& MCIR vs.\ SAGE                     & 0.10 & 0.07 & 0.05 \\
& MCIR vs.\ HSIC                     & 0.11 & 0.08 & 0.05 \\
\midrule
\multirow{9}{*}{\textbf{HouseEnergy-Sim}}
& PCIR vs.\ MCIR                     & 0.99 & 0.98 & -- \\
& PCIR vs.\ KernelSHAP$_{\rm cond}$  & 0.19 & 0.13 & -- \\
& PCIR vs.\ KernelSHAP$_{\rm indep}$ & 0.16 & 0.12 & -- \\
& PCIR vs.\ SAGE                     & 0.24 & 0.18 & -- \\
& PCIR vs.\ HSIC                     & 0.36 & 0.27 & -- \\
& MCIR vs.\ KernelSHAP$_{\rm cond}$  & 0.21 & 0.15 & -- \\
& MCIR vs.\ KernelSHAP$_{\rm indep}$ & 0.18 & 0.13 & -- \\
& MCIR vs.\ SAGE                     & 0.27 & 0.19 & -- \\
& MCIR vs.\ HSIC                     & 0.35 & 0.28 & -- \\
\bottomrule
\end{tabular}
\end{adjustbox}
\end{table}
Across the evaluated thresholds, MCIR ranges from $0.579$ to
$0.641$, a variation of approximately $0.06$. SHAP achieves the highest
absolute Group-Jaccard throughout this sensitivity analysis, while MCIR
shows a similar degree of variation across thresholds. Group-Jaccard is
reported as a complementary robustness diagnostic rather than an
optimization target or evidence that one attribution ranking is
correct. Taken together, panels (a) and (b) show how the conclusions
change under feature-level matching, group-level matching, and
alternative group definitions.
\begin{figure}[htbp]
\centering
\includegraphics[width=0.68\linewidth]{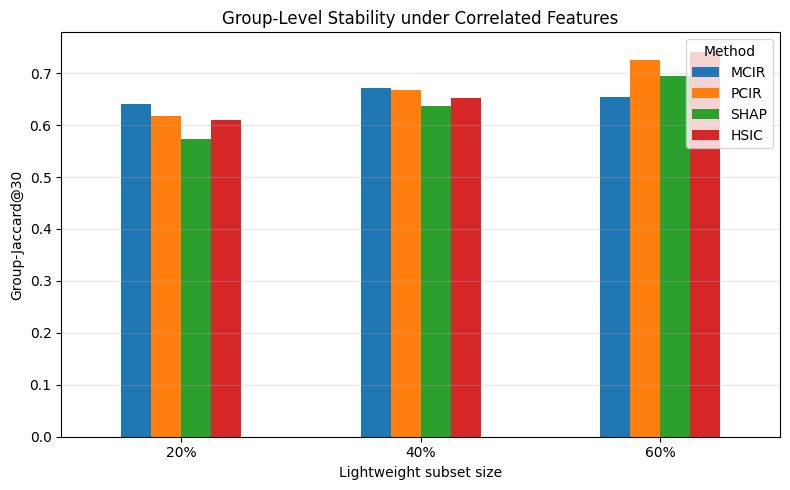}
\caption{Group-Jaccard@$30$ under feature dependence. MCIR obtains the
highest group-level agreement at the $20\%$ and $40\%$ lightweight
fractions, while HSIC obtains the highest value at $60\%$.}
\label{fig:group_stability}
\end{figure}

Figure~\ref{fig:group_stability} complements the exact feature-level
results by comparing rankings after correlated features are mapped to
common groups. MCIR's lower Exact Jaccard in some configurations can
therefore coexist with higher Group-Jaccard when different subsamples
select different members of the same correlation group. Group-Jaccard
measures agreement under this grouping rule; it does not establish that
group members are causally or functionally interchangeable.

Deletion results for MCIR, PCIR, SHAP, and HSIC are computed using the
same model, test set, perturbation rule, and deletion schedule; extended
curves are provided in Appendix~\ref{sec:appendix_deletion}. These
curves evaluate predictive faithfulness and are interpreted separately
from ranking stability. A sharp initial decrease indicates that the
first removed features are important under the deletion protocol,
whereas a gradual decrease may reflect remaining correlated predictors,
feature interactions, or perturbation effects. It is not, by itself,
evidence of better or worse group-level stability.

Exact Jaccard, Group-Jaccard, and deletion AUC are therefore reported as
complementary diagnostics of identity-level agreement, correlation-group
agreement, and predictive faithfulness, respectively. Rank overlays and
the complete deletion comparisons are provided in the appendix.
\subsubsection{Predictive Faithfulness and Estimator Sensitivity (RQ3)}
\label{subsubsec:rq3-faithfulness}

Deletion and perturbation experiments evaluate whether highly ranked
features are important to the fixed model's predictions under the
specified intervention protocol. On UCI HAR, perturbing highly ranked
MCIR features decreases test accuracy
(Fig.~\ref{fig:expts-har-faith}), while on HouseEnergy-Sim, deleting
highly ranked features decreases test $R^2$
(Fig.~\ref{fig:expts-house-faith}). The deletion curves and their AUC
summaries provide the primary predictive-faithfulness comparison. These
results concern predictive influence under the specified perturbations
and do not establish causal importance.
\begin{wraptable}{l}{0.50\linewidth}
\centering
\scriptsize
\setlength{\tabcolsep}{3pt}
\renewcommand{\arraystretch}{0.96}
\caption{Bootstrap agreement between Gaussian--copula and $k$NN MCIR
estimates on HouseEnergy-Sim.}
\label{tab:estimator_ci}
\begin{tabular}{lccc}
\toprule
\textbf{Metric} & \textbf{Mean} & \textbf{2.5\%} & \textbf{97.5\%} \\
\midrule
Kendall $\tau_b$ & 0.58 & 0.51 & 0.64 \\
Spearman $\rho$  & 0.73 & 0.67 & 0.78 \\
Jaccard@10       & 0.40 & 0.30 & 0.50 \\
\bottomrule
\end{tabular}
\end{wraptable}
Table~\ref{tab:sota-merged} reports pairwise agreement between
attribution methods computed under the same lightweight configuration.
It measures cross-method ranking similarity, not
full-versus-lightweight stability or attribution correctness. High
agreement indicates similar orderings, whereas low agreement indicates
that the methods rank features differently.
PCIR and MCIR produce similar rankings in both datasets, with
$\rho=0.83$ and $\tau_b=0.66$ on UCI HAR and $\rho=0.99$ and
$\tau_b=0.98$ on HouseEnergy-Sim. Their lower agreement with
KernelSHAP, SAGE, and HSIC reflects differences in attribution targets
and treatments of feature dependence; it does not establish that either
ranking is intrinsically superior. Predictive faithfulness is assessed
separately through the deletion and perturbation experiments.

\paragraph{Estimator Sensitivity: Gaussian--copula versus $k$NN.}
\label{subsec:estimator_sensitivity}
We evaluate whether the MCIR ranking changes when its MI and CMI
components are estimated using Gaussian--copula or $k$NN estimators.
Both estimators use the same HouseEnergy-Sim lightweight sample, fixed
model outputs, and conditioning neighbourhoods $\Phi(i)$. The
Gaussian--copula estimator applies rank Gaussianization followed by
covariance-determinant calculations, whereas the $k$NN estimator uses
local neighbour distances within the corresponding low-dimensional
feature-output blocks.
Agreement is evaluated using Spearman's $\rho$, Kendall's $\tau_b$, and
Jaccard@10, with bootstrap means and $95\%$ confidence intervals
reported in Table~\ref{tab:estimator_ci}. The two estimators obtain
$\rho=0.73$ and $\tau_b=0.58$, indicating moderate agreement across
the complete ranking. Jaccard@10 is $0.40$, showing that estimator
choice has a more substantial effect on the highest-ranked features.
Figure~\ref{fig:ablation-panel}(b) further shows that agreement depends
on $K$: overlap is higher among the first few features and changes as
\begin{wrapfigure}{l}{0.48\linewidth}
\centering
\includegraphics[width=\linewidth]{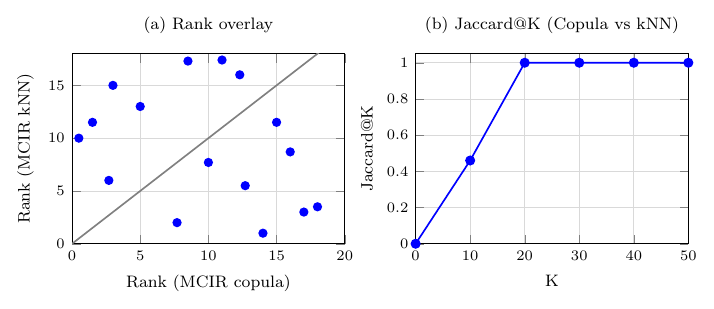}
\caption{HouseEnergy-Sim estimator sensitivity:
\textbf{(a)} Gaussian--copula versus $k$NN rank overlay and
\textbf{(b)} Jaccard@$K$ as the selection size varies.}
\label{fig:ablation-panel}
\vspace{-2mm}
\end{wrapfigure}
mid-ranked features are included. The Gaussian--copula implementation
is used as the computational default in the reported experiments
because of its lower cost and invariance to strictly monotone
componentwise transformations under rank Gaussianization, subject to
the stated tie-handling procedure. The $k$NN results are retained as an
estimator-sensitivity analysis. The observed agreement does not imply
that the two estimators are interchangeable or uniformly accurate.

Figure~\ref{fig:lw-bands} illustrates the resulting attribution
profiles. Under the fitted model and selected neighbourhoods, MCIR
assigns high scores to \textit{Space\_heater},
\textit{Water\_heater}, \textit{Washing\_machine}, and
\textit{HVAC\_load}, while assigning lower scores to several correlated
proxies. These scores quantify unique predictive association with the
model-output representation $Q$ relative to the selected neighbourhood;
they do not represent causal effects. PCIR produces a different profile
because it evaluates global unconditioned dispersion rather than unique
conditional information.
Overall, estimator choice has a measurable effect on MCIR, particularly
on top-$K$ selection. The results therefore support reporting
estimator-specific rankings and uncertainty intervals rather than
claiming estimator independence. They apply to the evaluated
low-dimensional HouseEnergy-Sim neighbourhoods and do not establish
robustness of $k$NN CMI estimation in high-dimensional conditioning
spaces.
\paragraph{Comparison with Predictive-Risk and Variance-Allocation
Functionals.}
\label{subsec:loco-comparison}

MCIR is designed to isolate unique conditional information, whereas
several established variable-importance functionals answer different
questions. LOCO quantifies the increase in held-out predictive loss
after excluding one covariate. A region-weighted R-LOCO-style procedure
extends this analysis by estimating featurewise loss changes within

regions having similar LOCO profiles. Shapley effects instead allocate
the explained variance of $Q$ among potentially dependent inputs.
Including these methods therefore separates redundancy suppression from
predictive-risk change and variance allocation under a common data-
generating process.
\begin{figure}[htbp]
\centering
\includegraphics[width=0.85\linewidth]{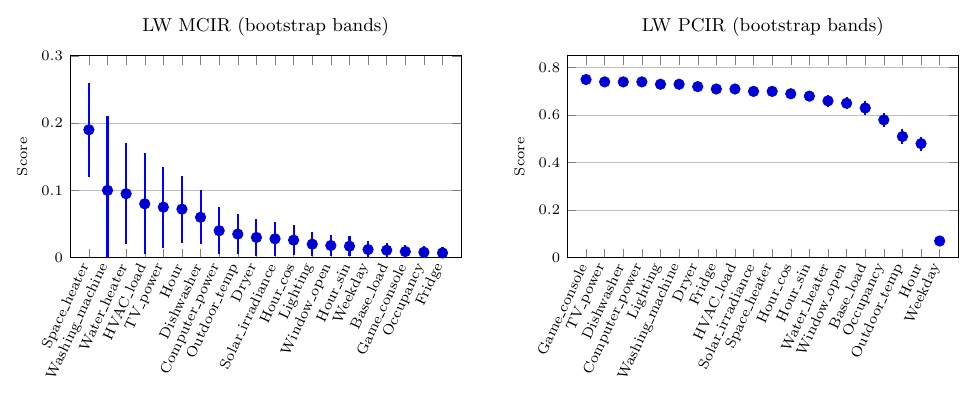}
\caption{HouseEnergy-Sim attribution profiles. MCIR reports unique
predictive information relative to each selected neighbourhood,
normalized to $[0,1]$. PCIR reports an unconditioned dispersion-based
score whose value depends on the stated preprocessing.}
\label{fig:lw-bands}
\end{figure}
The comparison uses the nonlinear HouseEnergy-Sim output $Q$ and four
prespecified functional drivers:
\textit{Base\_load}, \textit{HVAC}, \textit{Fridge}, and
\textit{Solar}. The remaining variables comprise correlated proxies and
independent noise features. MCIR is computed using Gaussian--copula
MI/CMI and $|\Phi(i)|=3$. LOCO uses the held-out squared-loss increase
after refitting XGBoost without each feature. For the regional
procedure, observation-level LOCO-loss vectors are clustered and the
regional scores are aggregated using their held-out frequencies.
Dependent-input Shapley effects are approximated using a cross-fitted
permutation plug-in estimator of
$\operatorname{Var}\{\mathbb E(Q\mid X_S)\}/\operatorname{Var}(Q)$.
All methods use identical training, calibration, and test partitions.

\begin{table*}[t]
\centering
\scriptsize
\setlength{\tabcolsep}{3.5pt}
\renewcommand{\arraystretch}{0.96}
\caption{HouseEnergy-Sim comparison of distinct importance
functionals. Values are mean $\pm$ standard deviation over ten
independent runs. Original and proxy mass are computed after applying
the positive-part normalization to each estimated attribution vector.
Higher Recall@$4$ and cumulative $R^2$-drop AUC indicate stronger
recovery and faster predictive degradation, respectively. Runtime
includes method-specific estimation after $Q$ is available.}
\label{tab:loco-shapley-houseenergy}
\begin{tabular}{lccccc}
\toprule
\textbf{Method} &
\textbf{Core Recall@4} &
\textbf{Original mass} &
\textbf{Proxy mass} &
\textbf{$R^2$-drop AUC} &
\textbf{Runtime (s)} \\
\midrule
MCIR &
$0.825\pm0.206$ &
$0.983\pm0.016$ &
$0.013\pm0.013$ &
$0.660\pm0.122$ &
$0.071\pm0.038$ \\

LOCO &
$0.775\pm0.079$ &
$\mathbf{0.984}\pm0.016$ &
$\mathbf{0.009}\pm0.011$ &
$0.481\pm0.118$ &
$1.801\pm0.892$ \\

Region-weighted LOCO &
$\mathbf{0.850}\pm0.129$ &
$0.977\pm0.020$ &
$0.012\pm0.011$ &
$0.535\pm0.147$ &
$2.029\pm0.917$ \\

Shapley-effects plug-in &
$0.375\pm0.132$ &
$0.331\pm0.061$ &
$0.661\pm0.057$ &
$\mathbf{0.680}\pm0.121$ &
$14.337\pm3.693$ \\
\bottomrule
\end{tabular}
\end{table*}

The results differ by evaluation target. Region-weighted LOCO obtains
the highest mean core-feature Recall@$4$, while MCIR and LOCO retain
approximately $98\%$ of their positive attribution mass on the four
original functional drivers. The Shapley-effects plug-in obtains the
largest cumulative $R^2$-drop AUC but allocates substantially more
importance to correlated proxies. Thus, no method is uniformly superior
across recovery, allocation, faithfulness, and runtime.

Cross-method agreement is correspondingly moderate. MCIR obtains mean
Kendall agreement of $0.436$ with LOCO and $0.449$ with the
region-weighted procedure, while its agreement with the Shapley-effects
plug-in is $-0.071$. These differences are consistent with the methods
allocating shared predictive information according to different
functionals; they are not evidence that one functional is intrinsically
correct.

The regional analysis selects two regions in eight runs and four
regions in two runs. However, the largest held-out region has mean
weight $0.999$, so the weighted result behaves predominantly as a
global LOCO summary in this data-generating process. The Shapley-effects
plug-in also produces some negative finite-sample marginal
contributions, including for functional drivers. Its allocation results
are therefore interpreted as estimator-sensitivity evidence rather than
as exact analytical Shapley effects.
\paragraph{Conditioning-Dimension Stress Test.}
\label{subsec:highdim-cmi}

We next examine how finite-sample MCIR estimation changes as the local
conditioning dimension increases. The experiment uses 24 continuous
features and varies the explanation-sample size over
$n\in\{500,1000,2000\}$ and the neighbourhood size over
$|\Phi|\in\{1,2,3,5,8,12\}$. Gaussian--copula and $k$NN MCIR rankings
are compared with estimator-specific reference rankings computed from
5000 observations. Consequently, agreement measures finite-sample
stability relative to the corresponding empirical reference; it does
not identify either estimator as population truth.

\begin{figure*}[t]
\centering
\includegraphics[width=\linewidth]{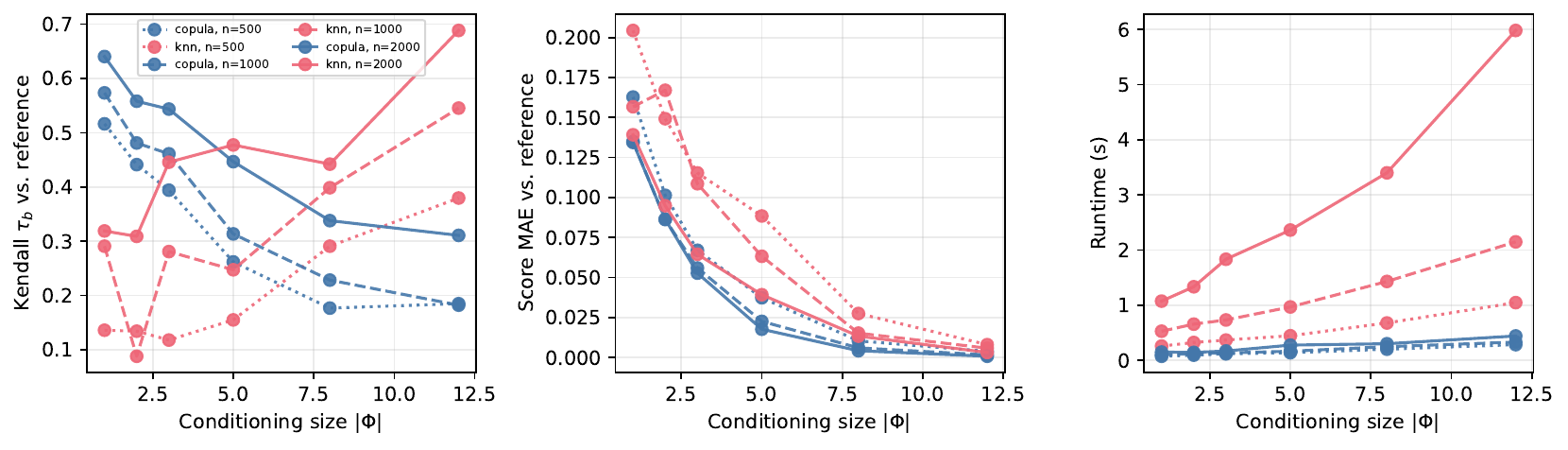}
\caption{High-dimensional CMI sensitivity on HouseEnergy-Sim.
\textbf{Left:} Kendall-$\tau_b$ agreement with the estimator-specific
5000-observation reference ranking.
\textbf{Middle:} mean absolute MCIR-score difference from that
reference.
\textbf{Right:} measured attribution runtime.
Line styles distinguish explanation-sample sizes. The decreasing score
error at large $|\Phi|$ is considered together with rank agreement,
because compression of the score vector can reduce absolute error
without preserving its ordering.}
\label{fig:highdim-cmi-ablation}
\end{figure*}

For Gaussian--copula MCIR at $n=2000$, Kendall agreement decreases from
$0.640$ at $|\Phi|=1$ to $0.311$ at $|\Phi|=12$. The $k$NN estimator
shows a different, non-monotone pattern and produces negative
pre-clipping MI/CMI component estimates in approximately $10\%$--$56\%$
of the evaluated feature components, depending on $n$ and $|\Phi|$.
At $n=2000$ and $|\Phi|=12$, $k$NN requires $5.982$ seconds compared
with $0.444$ seconds for the Gaussian--copula implementation. The
corresponding $k$NN rank agreement is $0.688$, but $21.6\%$ of its
pre-clipping component estimates are negative.

These results show that score error, ranking agreement, component
behaviour, and runtime need not change in parallel as the conditioning
dimension increases. In particular, the small score differences
observed at large $|\Phi|$ do not alone demonstrate improved
estimation, because complete-ranking agreement can remain modest.
High-dimensional MCIR results should therefore report the estimator,
conditioning dimension, component clipping frequency, rank agreement,
and runtime jointly.
\paragraph{Repeated-Subsampling Evaluation of Estimator Switching.}
\label{par:switching-evaluation}

In addition to comparing the two fixed estimators, we evaluate the
bootstrap-based estimator-switching rule over repeated lightweight
subsamples. The study considers sample sizes
$n\in\{200,500,1000,2000,3000\}$ and neighbourhood sizes
$|\Phi(i)|\in\{1,2,3,5\}$. For each setting, MCIR is computed using
(i) the Gaussian--copula estimator, (ii) the Kraskov-type $k$NN
estimator, and (iii) the switching rule. The switching rule selects
between the two candidates using its bootstrap stability criterion
computed only from the lightweight sample; the full-data ranking and
ground-truth labels are used exclusively for evaluation.

Each of the $5\times4=20$ design settings is repeated over 30
independent subsamples, yielding 600 repeated-subsample instances and
1,800 estimator-level evaluations. We report Kendall's $\tau_b$,
Spearman's $\rho$, Group-Jaccard@10, ground-truth Recall@10, runtime,
and estimator-selection frequency. Rank and group-overlap metrics compare
each lightweight result with the corresponding full-data reference
produced by the same estimator alternative.

\begin{table*}[t]
\centering
\scriptsize
\setlength{\tabcolsep}{3pt}
\renewcommand{\arraystretch}{0.95}
\caption{Repeated-subsampling evaluation of the estimator-switching
rule. Left: mean Kendall $\tau_b$ relative to the corresponding
full-data reference. Right: percentage of repetitions selecting the
Gaussian--copula/$k$NN estimator (C/K). Each cell summarizes 30
subsamples.}
\label{tab:switch_summary}
\label{tab:switch_selection}
\begin{minipage}[t]{0.47\linewidth}
\centering
\textbf{(a) Mean Kendall $\tau_b$}\\[2pt]
\begin{tabular}{ccccc}
\toprule
$n$ & $|\Phi|=1$ & $|\Phi|=2$ & $|\Phi|=3$ & $|\Phi|=5$ \\
\midrule
200  & 0.530 & 0.000 & 0.000 & 0.000 \\
500  & 0.141 & 0.431 & 0.000 & 0.000 \\
1000 & 0.102 & 0.718 & 0.000 & 0.000 \\
2000 & 0.100 & 0.737 & 0.667 & 0.000 \\
3000 & 0.153 & 0.443 & 1.000 & 0.000 \\
\bottomrule
\end{tabular}
\end{minipage}
\hfill
\begin{minipage}[t]{0.49\linewidth}
\centering
\textbf{(b) Selection frequency: C/K (\%)}\\[2pt]
\begin{tabular}{ccccc}
\toprule
$n$ & $|\Phi|=1$ & $|\Phi|=2$ & $|\Phi|=3$ & $|\Phi|=5$ \\
\midrule
200  & 13.3/86.7 & 0/100 & 0/100 & 0/100 \\
500  & 86.7/13.3 & 0/100 & 0/100 & 0/100 \\
1000 & 100/0     & 0/100 & 0/100 & 0/100 \\
2000 & 100/0     & 13.3/86.7 & 0/100 & 0/100 \\
3000 & 100/0     & 66.7/33.3 & 0/100 & 0/100 \\
\bottomrule
\end{tabular}
\end{minipage}
\end{table*}

The switching behaviour varies substantially with sample size and
conditioning dimension. For $|\Phi(i)|=1$, the rule increasingly selects
the Gaussian--copula estimator as $n$ grows. For larger neighbourhoods,
it predominantly selects $k$NN. This preference does not uniformly
translate into accurate rankings: agreement is consistently zero for
$|\Phi(i)|=5$ and remains zero for $|\Phi(i)|=3$ until $n=2000$.
The strongest agreement occurs for $|\Phi(i)|=2$ at $n=1000$ and
$n=2000$, and for $|\Phi(i)|=3$ at $n=3000$. Thus, the switching rule
is an empirical heuristic whose reliability depends on both sample size
and conditioning dimension; it does not provide an oracle guarantee or
uniformly dominate the fixed estimators.

Appendix Table~\ref{tab:estimator_switching_all} gives the complete
comparison of the fixed Gaussian--copula, fixed $k$NN, and switching
alternatives. Group-Jaccard@10 measures agreement between the
lightweight and corresponding full-data top-10 correlated feature
groups, using the grouping rule defined in the experimental protocol.
Ground-truth Recall@10 measures the proportion of generator-defined
relevant groups recovered among the top-10 groups. The ground-truth
metric is available because HouseEnergy-Sim is synthetic; it is not
used by the switching rule.

% \begin{center}
% \scriptsize
% \setlength{\LTleft}{0pt}
% \setlength{\LTright}{0pt}
% \setlength{\tabcolsep}{2.5pt}
% \renewcommand{\arraystretch}{0.94}

% \begin{longtable}{cclccccc}
% \caption{Complete estimator-switching evaluation on HouseEnergy-Sim.
% Entries are means over 30 independent lightweight subsamples. GT denotes
% generator-defined ground-truth recovery.}
% \label{tab:estimator_switching_all}\\
% \toprule
% $n$ & $|\Phi|$ & \textbf{Estimator} &
% $\boldsymbol{\tau_b}$ &
% $\boldsymbol{\rho}$ &
% \textbf{Group-J@10} &
% \textbf{GT Recall@10} &
% \textbf{Time (s)} \\
% \midrule
% \endfirsthead

% \multicolumn{8}{c}{
% \tablename\ \thetable\ -- continued
% }\\
% \toprule
% $n$ & $|\Phi|$ & \textbf{Estimator} &
% $\boldsymbol{\tau_b}$ &
% $\boldsymbol{\rho}$ &
% \textbf{Group-J@10} &
% \textbf{GT Recall@10} &
% \textbf{Time (s)} \\
% \midrule
% \endhead

% \midrule
% \multicolumn{8}{r}{Continued on next page}
% \endfoot

% \bottomrule
% \endlastfoot

% % Retain the existing numerical rows here.

% \end{longtable}
% \end{center}

The complete comparison should be interpreted as a sensitivity
analysis rather than evidence of estimator-independent behaviour.
In particular, poor agreement for larger $|\Phi(i)|$ is consistent
with the increasing difficulty of finite-sample CMI estimation as the
conditioning dimension grows. The results motivate using small
neighbourhoods and reporting the selected estimator, ranking
uncertainty, and runtime for each experimental setting.

The results provide mixed rather than uniformly favourable evidence
for automatic estimator switching. For $|\Phi(i)|=1$, the rule
increasingly selects the Gaussian--copula estimator as $n$ grows,
although fixed $k$NN produces substantially higher rank agreement from
$n=500$ onward. For $|\Phi(i)|=2$, switching matches fixed $k$NN at
$n\leq1000$ and remains relatively close at $n=2000$, but underperforms
it at $n=3000$ after selecting the copula estimator in $66.7\%$ of the
repetitions. Thus, the bootstrap criterion does not uniformly select
the estimator with the highest agreement against the full-data
reference.

For $|\Phi(i)|=3$, switching always selects $k$NN and therefore
matches its reported performance, including perfect rank agreement at
$n=3000$. For $|\Phi(i)|=5$, switching and $k$NN obtain
Group-Jaccard@10 of $1.000$ and ground-truth Recall@10 of $0.769$, while
their complete-ranking correlations remain zero. These results show
that recovery of the top feature-group set can coexist with
uninformative complete-ranking agreement, particularly when rankings
contain extensive ties or degeneracy. The zero correlations therefore
cannot be presented as evidence of full-ranking stability; the tie
structure and top-set metrics must be interpreted separately.

The runtime results reveal a substantial cost for estimator
switching. The fixed Gaussian--copula implementation is consistently
the fastest, while fixed $k$NN is more expensive but often achieves
higher rank agreement. Switching is the most expensive alternative
because its bootstrap criterion evaluates both candidates repeatedly
before selecting one. At $n=3000$, switching requires between
$82.000$ and $296.926$ seconds, compared with less than $0.06$ seconds
for the fixed copula estimator and between $9.048$ and $32.963$ seconds
for fixed $k$NN. Consequently, the switching procedure should be
viewed as a calibration heuristic rather than a computationally
efficient substitute for fixed-estimator selection.

Appendix~\ref{app:auto_phi_runtime} further decomposes the runtime
of neighbourhood screening, Auto-$\Phi$ calibration, and final MCIR
evaluation on UCI HAR. Screening and final score computation remain
comparatively inexpensive, whereas repeated bootstrap calibration is
the dominant cost. When the data distribution, model, and feature
representation remain fixed, the selected neighbourhoods and estimator
configuration can be reused for subsequent MCIR evaluations. Reuse
under distribution shift is not assumed and requires renewed
diagnostic evaluation.

The threshold-sensitivity and repeated-subsampling experiments
provide complementary finite-sample diagnostics. The former evaluates
whether the resulting rankings change under different neighbourhood
thresholds, while the latter shows that estimator preference and
ranking agreement depend strongly on both $n$ and $|\Phi(i)|$. These
results illustrate the practical limitations of the switching
heuristic; they do not establish an oracle guarantee, uniform
estimator robustness, or a universal improvement over fixed
estimators. Additional experimental details are provided in
Appendix~\ref{app:additional}.
\subsection*{Runtime in Lightweight Environments}
\label{sec:runtime-lw}

Runtime is evaluated by computing attribution on a lightweight
explanation sample of $n<N$ observations while keeping the trained
model and feature space fixed. Let
$\medmath{r_i=d_Q+|\Phi(i)|+1}$ denote the dimension of the local
feature--output block used for feature $i$, where $d_Q$ is the dimension
of the model-output representation. For fixed $d_Q$ and bounded
$|\Phi(i)|$, the local dimension remains small, but the total runtime
still depends on both $n$ and the number of evaluated features $k$.
\begin{wraptable}{r}{0.56\linewidth}
\centering
\scriptsize
\setlength{\tabcolsep}{3pt}
\renewcommand{\arraystretch}{0.96}
\caption{Leading attribution costs for $n$ lightweight observations and
$k$ features. Here $r=d_Q+m_\Phi+1$ is treated as bounded.}
\label{tab:complexity1}
\begin{adjustbox}{width=\linewidth}
\begin{tabular}{lcc}
\toprule
\textbf{Method} &
\textbf{Per feature} &
\textbf{All $k$ features} \\
\midrule
MCIR (copula) &
$\mathcal{O}(rn\log n+nr^2+r^3)$ &
$\mathcal{O}\!\left(k(rn\log n+nr^2+r^3)\right)$ \\
MCIR ($k$NN) &
$\mathcal{O}(n\log n)$ expected &
$\mathcal{O}(kn\log n)$ expected \\
PCIR &
$\mathcal{O}(n)$ &
$\mathcal{O}(kn)$ \\
HSIC (RBF) &
$\mathcal{O}(n^2)$ &
$\mathcal{O}(kn^2)$ \\
\bottomrule
\end{tabular}
\end{adjustbox}
\end{wraptable}
For the Gaussian--copula implementation, rank Gaussianization and local
covariance operations give a per-feature cost of
$\medmath{\mathcal{O}(r_i n\log n+n r_i^2+r_i^3)}$ before accounting
for reusable preprocessing. For fixed $r_i$, this is nearly linear in
$n$, up to rank-sorting costs. The $k$NN implementation has expected
cost $\medmath{\mathcal{O}(n\log n)}$ per feature in a fixed
low-dimensional block, although its constants and statistical
performance deteriorate as the conditioning dimension increases.
Table~\ref{tab:complexity1} summarizes these costs.

The wall-clock measurements in Table~\ref{tab:runtime-lw-merged}
include the implemented attribution pipeline rather than isolated
estimator kernels. HouseEnergy-Sim and UCI HAR use the same lightweight
sample size but differ substantially in feature count, neighbourhood
structure, and implementation overhead. The measurements therefore
illustrate practical runtime at the reported settings; they should not
be interpreted as an empirical proof of a universal scaling law.

\begin{table}[htbp]
\centering
\scriptsize
\setlength{\tabcolsep}{4pt}
\renewcommand{\arraystretch}{0.96}
\caption{Observed attribution runtime on HouseEnergy-Sim
($n=2000$, $k=20$) and UCI HAR ($n=2000$, $k=561$). Times include the
implemented scoring pipeline.}
\label{tab:runtime-lw-merged}
\begin{tabular}{lcccc}
\toprule
\multirow{2}{*}{\textbf{Method}} &
\multicolumn{2}{c}{\textbf{HouseEnergy-Sim}} &
\multicolumn{2}{c}{\textbf{UCI HAR}} \\
\cmidrule(lr){2-3}\cmidrule(lr){4-5}
& \textbf{Time (s)} & \textbf{Setting}
& \textbf{Time (s)} & \textbf{Setting} \\
\midrule
MCIR (copula) & 0.4487 & rank Gaussianization
              & 451.412 & rank Gaussianization \\
MCIR ($k$NN)  & 4.3769 & $k_{\rm NN}=5$
              & 139.738 & $k_{\rm NN}=5$ \\
HSIC (RBF)    & 7.0033 & median bandwidth
              & 207.823 & median bandwidth \\
\bottomrule
\end{tabular}
\end{table}

Reducing $N$ to $n$ lowers the number of observations processed by each
attribution estimator without changing the trained model or feature
space. The resulting speedup depends on the estimator, feature count,
local conditioning dimension, and implementation. Attribution fidelity
is evaluated separately using the representation, ranking, and
deletion-curve diagnostics defined in
Section~\ref{subsec:lw_fidelity}. The Macro-F1 ratio reported in
Table~\ref{tab:har_lw} concerns predictive sensitivity to reduced
training data and is not, by itself, evidence of explanation fidelity.

\subsection{Results Across Cross-Domain and Real-World Settings}
\label{accross}

We next evaluate MCIR in two additional settings: high-dimensional
image representations and electricity-load forecasting with correlated
temporal and meteorological covariates. These experiments assess
predictive faithfulness under the stated protocols; they do not extend
the population guarantees beyond their assumptions.

\subsubsection{CIFAR-10/ResNet-50 Representation Experiment}
\label{subsec:cifar}

A ResNet-50 model is fine-tuned on CIFAR-10, and its
2048-dimensional penultimate-layer representations are extracted. An
MLP probe trained on these representations obtains $95.9\%$ test
accuracy. MCIR explains the output of this fixed probe by ranking the
embedding coordinates using bounded local neighbourhoods.
Deleting the 128 coordinates ranked highest by MCIR reduces probe
accuracy from $0.96$ to $0.84$ and produces a deletion AUC of $0.887$.
This is the lowest AUC among the evaluated methods in
Table~\ref{tab:cifar_auc}, indicating the strongest deletion
faithfulness within this experiment. The result shows that the
MCIR-ranked coordinates are predictively influential to the fixed probe
under the specified deletion baseline. It does not establish causal
importance or general estimator stability in arbitrary
2048-dimensional problems, because MCIR conditions only on the selected
local neighbourhoods.
\begin{wraptable}{r}{0.58\linewidth}
\centering
\scriptsize
\setlength{\tabcolsep}{4pt}
\renewcommand{\arraystretch}{0.96}
\caption{CIFAR-10 deletion results. Lower AUC indicates stronger
deletion faithfulness under the stated protocol.}
\label{tab:cifar_auc}
\begin{tabular}{lcc}
\toprule
\textbf{Method} & \textbf{Deletion AUC} &
\textbf{Accuracy after top-128 deletion} \\
\midrule
MCIR & \textbf{0.887} & $0.96\rightarrow0.84$ \\
PCIR & 0.912 & $0.96\rightarrow0.87$ \\
HSIC & 0.938 & $0.96\rightarrow0.89$ \\
MI   & 0.951 & $0.96\rightarrow0.90$ \\
\bottomrule
\end{tabular}
\end{wraptable}
The different deletion profiles produced by PCIR, HSIC, and marginal
MI reflect differences in their scoring objectives and treatments of
feature dependence. No general claim that these methods fail because of
high-dimensional redundancy is inferred from this single experiment.
\subsubsection{Norwegian Residential-Demand Case Study}
\label{cases}
We evaluate MCIR using aggregated hourly residential electricity demand
for the five Norwegian bidding areas NO1--NO5
\cite{hofmann2023norwegian}. The demand series are obtained from
Zenodo record 8423312, and meteorological covariates are retrieved from
the Open-Meteo ERA5 archive. The experiment covers 2021 and contains
$8{,}736$ usable observations per area after constructing lagged
predictors. The response is aggregated residential demand in MWh, not
total system load.

For each area, the predictor set contains $k=28$ lagged-demand,
rolling-statistic, weather, degree-day, and cyclic-calendar variables.
A separate XGBoost regressor is fitted using a chronological $80/20$
split, with 500 trees, maximum depth 6, learning rate 0.05, row and
column subsampling rates of 0.8, and random seed 42. MCIR is computed
against the fixed model output $Q$, rather than the observed response,
using ten independent explanation samples of $n=200$ observations,
$|\Phi(i)|=5$, and the rank--Gaussianized Gaussian--copula estimator.
TreeSHAP, marginal MI, and HSIC are evaluated using the same fitted
model outputs, explanation samples, and feature space.

The experiment examines predictive association with each fitted model.
Accordingly, the resulting scores identify model-relevant variables
under the selected neighbourhoods and do not establish physical or
causal determinants of electricity demand.

\begin{table}[t]
\centering
\scriptsize
\setlength{\tabcolsep}{3.2pt}
\renewcommand{\arraystretch}{0.96}
\caption{Predictive performance and full--lightweight MCIR agreement
in the Norwegian residential-demand experiment. Values are means over
ten explanation samples. Correlation groups are connected components
of the training-data correlation graph at $|\rho|\geq0.85$.
$W_1$ is reported in MWh and is interpreted only within each area.}
\label{tab:norway-fidelity}
\begin{tabular}{lcccccc}
\toprule
\textbf{Area} &
\textbf{Test $R^2$} &
$\boldsymbol{\rho}$ &
$\boldsymbol{\tau_b}$ &
\textbf{J@8} &
\textbf{Group-J@8} &
$\boldsymbol{W_1}$ \\
\midrule
NO1 & 0.9924 & 0.666 & 0.508 & 0.639 & 0.623 & 64.57 \\
NO2 & 0.9919 & 0.709 & 0.558 & 0.799 & 0.814 & 37.04 \\
NO3 & 0.9924 & 0.659 & 0.507 & 0.671 & 0.757 & 22.30 \\
NO4 & 0.9937 & 0.640 & 0.478 & 0.560 & 0.617 & 14.40 \\
NO5 & 0.9903 & 0.618 & 0.471 & 0.600 & 0.713 & 14.13 \\
\bottomrule
\end{tabular}
\end{table}

Table~\ref{tab:norway-fidelity} shows moderate full--lightweight MCIR
agreement. Mean Spearman correlation ranges from $0.618$ to $0.709$,
mean Kendall correlation from $0.471$ to $0.558$, and exact
Jaccard@$8$ from $0.560$ to $0.799$. Group-Jaccard@$8$ exceeds exact
Jaccard@$8$ in four areas, indicating that part of the identity-level
disagreement arises from selecting different members of correlated
groups. This pattern is not universal: the group-level value is
slightly lower for NO1.

The Wasserstein discrepancy in Table~\ref{tab:norway-fidelity} is
expressed on the original output scale in MWh. It is therefore a
within-area representation diagnostic and should not be compared
numerically across areas with different demand scales.

\begin{figure*}[t]
\centering
\begin{minipage}[t]{0.47\linewidth}
\centering
\begin{tikzpicture}
\begin{axis}[
    width=\linewidth,
    height=5.2cm,
    ybar,
    bar width=5pt,
    ymin=0,
    ymax=1,
    ylabel={Top-$8$ agreement},
    symbolic x coords={NO1,NO2,NO3,NO4,NO5},
    xtick=data,
    enlarge x limits=0.12,
    ymajorgrids,
    grid style={gray!25},
    tick label style={font=\scriptsize},
    label style={font=\scriptsize},
    legend style={
        font=\scriptsize,
        at={(0.5,1.03)},
        anchor=south,
        legend columns=2,
        draw=none
    }
]
\addplot[
    fill=blue!55,
    draw=blue!75!black
] coordinates {
    (NO1,0.639)
    (NO2,0.799)
    (NO3,0.671)
    (NO4,0.560)
    (NO5,0.600)
};
\addplot[
    fill=orange!65,
    draw=orange!80!black
] coordinates {
    (NO1,0.623)
    (NO2,0.814)
    (NO3,0.757)
    (NO4,0.617)
    (NO5,0.713)
};
\legend{Exact J@8,Group-J@8}
\end{axis}
\end{tikzpicture}

\textbf{\scriptsize (a) Full--lightweight MCIR agreement}
\end{minipage}
\hfill
\begin{minipage}[t]{0.51\linewidth}
\centering
\begin{tikzpicture}
\begin{axis}[
    width=\linewidth,
    height=5.2cm,
    ybar,
    bar width=2.8pt,
    ymin=0,
    ymax=3.5,
    ytick={0,0.5,1,1.5,2,2.5,3,3.5},
    ylabel={Cumulative $R^2$-drop AUC},
    symbolic x coords={NO1,NO2,NO3,NO4,NO5},
    xtick=data,
    enlarge x limits=0.10,
    ymajorgrids,
    grid style={gray!25},
    tick label style={font=\scriptsize},
    label style={font=\scriptsize},
    legend style={
        font=\tiny,
        at={(0.5,1.03)},
        anchor=south,
        legend columns=4,
        draw=none,
        column sep=3pt
    }
]
\addplot[
    fill=blue!55,
    draw=blue!75!black
] coordinates {
    (NO1,1.695)
    (NO2,1.276)
    (NO3,1.881)
    (NO4,2.747)
    (NO5,1.418)
};
\addplot[
    fill=orange!65,
    draw=orange!80!black
] coordinates {
    (NO1,1.798)
    (NO2,1.388)
    (NO3,2.017)
    (NO4,3.188)
    (NO5,1.591)
};
\addplot[
    fill=green!55,
    draw=green!60!black
] coordinates {
    (NO1,1.787)
    (NO2,1.401)
    (NO3,2.030)
    (NO4,3.211)
    (NO5,1.582)
};
\addplot[
    fill=purple!45,
    draw=purple!70!black
] coordinates {
    (NO1,1.788)
    (NO2,1.407)
    (NO3,2.025)
    (NO4,3.127)
    (NO5,1.581)
};
\legend{MCIR,TreeSHAP,MI,HSIC}
\end{axis}
\end{tikzpicture}

\textbf{\scriptsize (b) Cumulative $R^2$-drop AUC}
\end{minipage}
\caption{Norwegian residential-demand results.
\textbf{(a)} Full--lightweight exact and group-level top-eight
agreement for MCIR, averaged over ten explanation samples.
\textbf{(b)} Cumulative $R^2$-drop AUC under the common training-mean
replacement protocol; higher values indicate faster cumulative
predictive degradation.}
\label{fig:norway-summary}
\end{figure*}
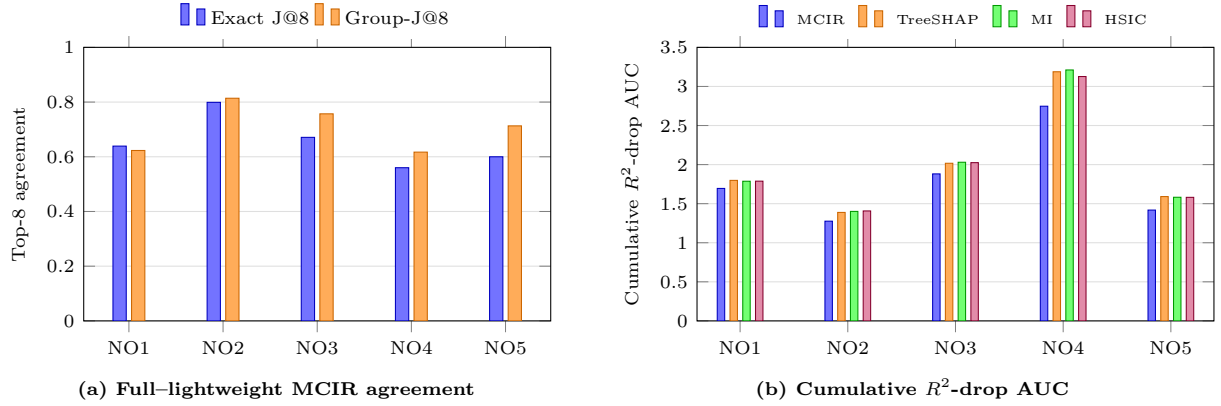
\begin{table*}[t]
\centering
\scriptsize
\setlength{\tabcolsep}{2.8pt}
\renewcommand{\arraystretch}{0.96}
\caption{Cross-method agreement and deletion performance in the
Norwegian residential-demand experiment. Agreement is measured between
MCIR and TreeSHAP. Cumulative $R^2$-drop AUC is defined as
$\int_0^1[R^2(0)-R^2(t)]\,dt$, where $t$ is the fraction of features
replaced by their training-set means. Higher values indicate faster
cumulative predictive degradation.}
\label{tab:norway-deletion}
\begin{tabular}{lcccc@{\hspace{5pt}}|@{\hspace{5pt}}cccc}
\toprule
& \multicolumn{4}{c}{\textbf{MCIR versus TreeSHAP}}
& \multicolumn{4}{c}{\textbf{Cumulative $R^2$-Drop AUC}} \\
\cmidrule(lr){2-5}\cmidrule(lr){6-9}
\textbf{Area} &
$\boldsymbol{\rho}$ &
$\boldsymbol{\tau_b}$ &
\textbf{J@8} &
\textbf{Group-J@8} &
\textbf{MCIR} &
\textbf{TreeSHAP} &
\textbf{MI} &
\textbf{HSIC} \\
\midrule
NO1 & 0.377 & 0.254 & 0.333 & 0.500
& 1.695 & 1.798 & 1.787 & 1.788 \\
NO2 & 0.552 & 0.397 & 0.600 & 0.667
& 1.276 & 1.388 & 1.401 & 1.407 \\
NO3 & 0.257 & 0.169 & 0.333 & 0.429
& 1.881 & 2.017 & 2.030 & 2.025 \\
NO4 & 0.273 & 0.228 & 0.455 & 0.429
& 2.747 & 3.188 & 3.211 & 3.127 \\
NO5 & 0.247 & 0.169 & 0.333 & 0.429
& 1.418 & 1.591 & 1.582 & 1.581 \\
\bottomrule
\end{tabular}
\end{table*}

Cumulative $R^2$-drop AUC is a predictive-performance diagnostic, not
an MCIR score. It measures the integrated decrease from the
unperturbed test $R^2$; hence, larger values indicate faster cumulative
degradation. MCIR produces substantial degradation in every area but
does not attain the largest drop AUC in this experiment.
Cross-method agreement is lower than full--lightweight MCIR stability.
MCIR--TreeSHAP Spearman correlation ranges from $0.247$ to $0.552$,
Kendall correlation from $0.169$ to $0.397$, and exact Jaccard@$8$
from $0.333$ to $0.600$. Marginal MI and HSIC each obtain exact
Jaccard@$8=0.231$ against MCIR in every area. These values show that
the methods produce different rankings because they target different
attribution functionals; they do not establish that one ranking is
intrinsically correct.

\begin{table*}[t]
\centering
\scriptsize
\setlength{\tabcolsep}{2.2pt}
\renewcommand{\arraystretch}{0.93}
\caption{Top-eight MCIR features from the full-data reference ranking.
Feature names follow the implementation. These rankings describe
predictive association with the fitted XGBoost outputs.}
\label{tab:norway-top8}
\begin{tabular}{clllll}
\toprule
\textbf{Rank} & \textbf{NO1} & \textbf{NO2} & \textbf{NO3} &
\textbf{NO4} & \textbf{NO5} \\
\midrule
1 &
\texttt{sin\_hour} &
\texttt{sin\_hour} &
\texttt{relative\_humidity\_2m} &
\texttt{load\_lag\_1h} &
\texttt{cos\_hour} \\

2 &
\texttt{cos\_hour} &
\texttt{cos\_hour} &
\texttt{load\_lag\_1h} &
\texttt{wind\_speed\_10m} &
\texttt{load\_lag\_1h} \\

3 &
\texttt{load\_lag\_1h} &
\texttt{load\_lag\_1h} &
\texttt{cos\_hour} &
\texttt{cos\_hour} &
\texttt{relative\_humidity\_2m} \\

4 &
\texttt{load\_roll\_mean\_24h} &
\texttt{relative\_humidity\_2m} &
\texttt{load\_roll\_mean\_24h} &
\texttt{load\_roll\_mean\_24h} &
\texttt{load\_roll\_std\_6h} \\

5 &
\texttt{load\_lag\_24h} &
\texttt{load\_lag\_24h} &
\texttt{precipitation} &
\texttt{load\_roll\_std\_6h} &
\texttt{load\_roll\_mean\_24h} \\

6 &
\texttt{wind\_speed\_10m} &
\texttt{load\_roll\_mean\_24h} &
\texttt{load\_roll\_std\_6h} &
\texttt{sin\_dow} &
\texttt{load\_lag\_24h} \\

7 &
\texttt{load\_roll\_std\_6h} &
\texttt{load\_roll\_std\_6h} &
\texttt{load\_lag\_24h} &
\texttt{relative\_humidity\_2m} &
\texttt{sin\_dow} \\

8 &
\texttt{load\_roll\_std\_12h} &
\texttt{cooling\_degree} &
\texttt{sin\_dow} &
\texttt{load\_lag\_24h} &
\texttt{cos\_dow} \\
\bottomrule
\end{tabular}
\end{table*}

Table~\ref{tab:norway-top8} shows that the rankings contain a mixture
of cyclic-calendar terms, short- and daily-lag demand, rolling demand
statistics, and meteorological covariates. The composition differs
across areas. These differences should not be interpreted as evidence
of distinct physical mechanisms without a separate causal or
domain-specific analysis.

\begin{table}[t]
\centering
\scriptsize
\setlength{\tabcolsep}{5pt}
\renewcommand{\arraystretch}{0.96}
\caption{Range of mean attribution runtimes across NO1--NO5 for
$n=200$ explanation observations. Times exclude model fitting and are
hardware- and implementation-specific.}
\label{tab:norway-runtime}
\begin{tabular}{lcc}
\toprule
\textbf{Method} & \textbf{Minimum (s)} & \textbf{Maximum (s)} \\
\midrule
Marginal MI & 0.0559 & 0.0675 \\
MCIR        & 0.0810 & 0.0983 \\
HSIC        & 0.1108 & 0.1801 \\
TreeSHAP    & 0.2280 & 0.2817 \\
\bottomrule
\end{tabular}
\end{table}

Marginal MI is the fastest implementation in this configuration, with
MCIR second. MCIR requires no additional predictive-model evaluations
after the output vector $Q$ has been cached; this excludes the model
calls required to construct $Q$. The runtime measurements are specific
to the reported hardware and implementation and do not establish a
universal ordering.

Overall, this case study shows that MCIR can be computed reproducibly
against fixed XGBoost outputs for five distinct residential-demand
series and that its lightweight rankings retain moderate exact and
group-level agreement. It does not demonstrate causal relevance,
agreement with TreeSHAP, redundancy collapse without a controlled
duplicate intervention, or universal runtime or deletion superiority.
No PCIR, insertion-curve, or inferential $p$-value result is reported
for this experiment.
To make the connection between theory and practice explicit, Table~\ref{tab:theory_validation} maps the formal results in
Section~\ref{prob} and Appendix~\ref{sec:theory1} to the corresponding
evidence in Section~\ref{sec:results}.
Table~\ref{tab:theory_validation} links each formal result to the corresponding final theorem, proposition, algorithm, figure, table, or section. The evidence is classified according to whether it provides direct empirical evaluation, indirect empirical consistency, algorithmic support, or a theoretical guarantee. This distinction is particularly important for the exact and group-level stability results, the observed sensitivity to Gaussian--copula versus $k$NN estimation, and properties for which the current experiments provide supportive rather than conclusive evidence.

% Required in the preamble:
% \usepackage{longtable}
% \usepackage{booktabs}
% \usepackage{cleveref}

\begingroup
\footnotesize
\setlength{\LTleft}{0pt}
\setlength{\LTright}{0pt}
\setlength{\LTpre}{2pt}
\setlength{\LTpost}{0pt}
\setlength{\tabcolsep}{2pt}
\renewcommand{\arraystretch}{0.88}

\begin{longtable}{
    @{}p{0.18\textwidth}
       p{0.21\textwidth}
       p{0.54\textwidth}@{}
}
\caption{Traceability between MCIR's formal results and supporting
evidence. Empirical consistency denotes compatibility, not proof.}
\label{tab:theory_validation}\\[-1mm]

\toprule
\textbf{Property} & \textbf{Formal result} & \textbf{Evidence and scope} \\
\midrule
\endfirsthead

\multicolumn{3}{c}{\itshape Table~\thetable\ continued}\\[-1mm]
\toprule
\textbf{Property} & \textbf{Formal result} & \textbf{Evidence and scope} \\
\midrule
\endhead

\midrule
\multicolumn{3}{r}{\itshape Continued}\\
\endfoot
\bottomrule
\endlastfoot

Boundedness and endpoints
&
\Cref{eq:mcir-decomposed,eq:mcir-zero-characterization};
Eq.~\eqref{eq:mcir-one-characterization}
&
\textbf{Population guarantee.}
$0\leq C_i\leq1$, with the endpoints characterized by the corresponding
information quantities. Empirical clipping in
\Cref{eq:clipped-empirical-mcir} enforces this range and is therefore
not independent empirical validation.
\\

Exact conditional-redundancy collapse
&
\Cref{eq:exact-duplicate-collapse}
&
\textbf{Theory with sensitivity evidence.}
Measurability of $X_i$ with respect to $\sigma(X_{\Phi(i)})$ implies
$C_i=0$. The noisy-duplicate results in
\Cref{fig:redundancy_stress,tab:redundancy_stress} and
\Cref{tab:house_combined,fig:mass_house,fig:mass_combined} assess
finite-sample sensitivity, not exact deterministic collapse.
\\

Weak-dependence reduction
&
\Cref{prop:weak_dep_correct}
&
\textbf{Theoretical result.}
MCIR approaches a normalized marginal-information score under vanishing
conditioning effects and denominator separation. No controlled
zero-dependence experiment or ordering equivalence with PCIR is claimed.
\\

Population reparameterization invariance
&
\Cref{prop:population-invariance}
&
\textbf{Population guarantee.}
MCIR is invariant under bimeasurable bijections of $Q$, $X_i$, and
$X_{\Phi(i)}$. The separate rank--Gaussianized estimator result is
given in \Cref{prop:monotone_invariance}; no transformation experiment
is reported.
\\
Consistency and rank stability
&
\Cref{eq:mcir-consistency,thm:rank_stability}
&
\textbf{Conditional guarantee; indirect evidence.}
The result requires estimator-error, denominator-separation, and
population-margin conditions. Subsampling results in
\Cref{tab:jaccard_group,fig:group_stability} use a full-data empirical
reference rather than the unavailable population ranking. The conditioning-dimension stress test in
\Cref{fig:highdim-cmi-ablation,tab:highdim-cmi-ablation} additionally
shows that empirical score error and ranking agreement can diverge as
$|\Phi|$ increases.
\\

Estimator switching
&
\Cref{eq:estimator-switching-rule}
&
\textbf{Empirical heuristic.}
Fixed-estimator sensitivity is reported in
\Cref{fig:ablation-panel,tab:estimator_ci}, with repeated-subsampling
results in Appendix \Cref{tab:estimator_switching_all}. Performance depends on
$n$ and $|\Phi(i)|$; no oracle inequality is claimed.
\\

Auto-$\Phi$ selection
&
\Cref{eq:auto-phi-objective,alg:auto_phi}
&
\textbf{Algorithmic and sensitivity evidence.}
\Cref{fig:phi_ablation} supports small neighbourhoods but does not prove
population optimality. Selection balances ranking dispersion, residual
dependence, and conditioning dimension.
\\

Lightweight fidelity
&
\Cref{def:lightweight-fidelity-profile};
Eq.~\eqref{eq:lightweight-acceptance-criterion}
&
\textbf{Operational diagnostic.}
Full--lightweight discrepancies are reported in
\Cref{fig:metric-first,tab:jaccard_group,tab:norway-fidelity}. Output-law similarity
alone does not guarantee ranking agreement; cross-method agreement in
\Cref{tab:sota-merged} is not treated as fidelity evidence.
\\

Computational profile
&
\Cref{eq:mcir-computational-profile,tab:complexity1}
&
\textbf{Analysis and runtime evidence.}
Measurements appear in
Table~\ref{tab:runtime-lw-merged}, Table~\ref{tab:runtime-redundancy},
Table~\ref{tab:loco-shapley-houseenergy}, and Table~\ref{tab:highdim-cmi-ablation}, with decomposition
in \Cref{app:auto_phi_runtime}. End-to-end cost also depends on
screening, estimator choice, bootstrapping, $n$, and $k$; linear scaling
in $k$ is not empirically asserted.
\\

\end{longtable}
\endgroup

\section{Ethics, Limitations, and Future Work}
\label{rep}

All experiments use non-sensitive datasets:
(i) synthetic HouseEnergy-Sim data,
(ii) UCI HAR wearable-sensor data,
(iii) CIFAR-10 benchmark images, and
(iv) aggregated Norwegian load data without personal identifiers.
None of these datasets contains personally identifiable information or
protected attributes.
MCIR quantifies statistical dependence between input features and the
fixed model-output representation $Q$. It does \emph{not} identify
causal relationships or measure causal variable importance.
Misinterpretation may arise if conditional association is treated as
causation or if attribution scores are used in high-stakes decisions
without domain oversight. MCIR should therefore be used as a
dependence-aware model-explanation diagnostic, together with estimator
uncertainty, domain knowledge, and, where appropriate, a separate
causal analysis.
MCIR-M has several methodological limitations. First, its
finite-sample behaviour depends on the selected MI or CMI estimator,
the available sample size, the output dimension, and the conditioning
dimension. The high-dimensional stress test further shows that absolute score
error, full-ranking agreement, negative component estimates, and
runtime can change differently as $|\Phi|$ increases. The evaluated
Gaussian--copula and continuous $k$NN estimators therefore do not
provide a general solution to high-dimensional CMI estimation. The bootstrap-based estimator-switching rule minimizes
estimated sampling variability but does not estimate estimator bias or
provide an oracle-risk guarantee. The fixed-estimator and switching
results are therefore reported jointly. Second, neighbourhood
construction depends on the screening method, candidate budget,
penalty parameters, and dependence threshold. Although the sensitivity
experiments assess several practical settings, they do not establish a
universally optimal neighbourhood. Third, pairwise screening may omit
higher-order or interaction-only dependence, while unnecessarily large
neighbourhoods can increase the bias, variance, and computational cost
of finite-sample CMI estimation. Appendix~\ref{app:failure_case_higher_order}
provides a controlled example in which pairwise screening fails to
recover interaction-only dependence. Finally, the lightweight fidelity
profile is an empirical diagnostic: similarity of the marginal output
laws alone does not guarantee agreement between MCIR scores or
rankings, which depend on the joint feature--output laws. The controlled comparison with LOCO, region-weighted LOCO, and a
Shapley-effects plug-in also demonstrates that predictive-risk change,
regional loss variation, explained-variance allocation, and unique
conditional information can produce materially different rankings.
Moreover, the regional partition was highly imbalanced and the
Shapley-effects plug-in exhibited finite-sample non-monotonicity, so
these comparisons should be interpreted within their stated
implementations.
Future work will investigate interaction-aware neighbourhood construction,
classifier-based and other high-dimensional CMI estimators,
distributionally robust estimation, principled sample-size selection,
and extensions for temporal and structured data.

\section{Broader Impact}
\label{bro}

MCIR provides a dependence-aware approach to global model explanation
by measuring the unique predictive information contributed by a feature
relative to a selected local neighbourhood. Its normalization produces
a bounded score in $[0,1]$, while its population properties characterize
conditional-redundancy collapse and invariance under bimeasurable
reparameterizations. The fixed numerical range facilitates
interpretation within a stated environment, but does not by itself
guarantee identical calibration across datasets, models, or estimators.
The method may help practitioners identify redundant predictive signals
in domains containing correlated sensors, temporal lags, engineered
feature blocks, or learned representations. Nevertheless, an MCIR score
is relative to the fitted model, observed feature set, selected
neighbourhood, estimator, and evaluation distribution. It should not be
used as evidence of causal effect, nor should a lightweight explanation
be accepted without examining its representation, ranking, top-$K$, and
faithfulness diagnostics. Responsible use requires transparent reporting
of these choices, uncertainty analysis, appropriate data governance, and
domain-expert review, particularly in high-stakes applications.

\section{Conclusion}
\label{sec:conclusion}

This study introduced MCIR-M, a dependence-aware global attribution
framework that measures the unique predictive information contributed
by each feature beyond a selected dependence neighbourhood. The
population MCIR functional is bounded in $[0,1]$, collapses to zero
under exact conditional redundancy, and is invariant under
bimeasurable reparameterizations. Under explicit approximation and
denominator-separation conditions, the weak-dependence result reduces
MCIR to a normalized marginal-information expression; it does not
establish ordering equivalence with PCIR or another attribution method.
The empirical evaluation covered controlled redundancy, UCI HAR,
HouseEnergy-Sim, CIFAR-10 representations, and Norwegian
residential-demand forecasting. MCIR showed its clearest advantage in
the controlled redundancy experiments, where it better preserved the
ordering of designated predictive variables under injected
near-duplicates. Across the real-data experiments, exact ranking
agreement, group-level stability, deletion faithfulness, and runtime
varied across datasets and baselines. The results therefore support
dependence-aware attribution without suggesting that MCIR or any
individual estimator is uniformly superior under every evaluation
criterion.
The lightweight procedure reduced attribution cost in the evaluated
settings, but its adequacy was assessed empirically rather than inferred
from subsampling alone. Representation-distribution agreement,
complete-ranking agreement, top-$K$ agreement, and deletion or
insertion behaviour were therefore evaluated separately.
Estimator-sensitivity results further showed that estimator choice can
materially affect the ranking, particularly for highly ranked or
strongly redundant features. The conditioning-dimension ablation showed
that finite-sample CMI behaviour cannot be characterized by score error
alone: ranking agreement, unprojected negative $k$NN component
estimates, and runtime changed differently with $n$ and $|\Phi|$.
Similarly, the controlled LOCO and Shapley-effects comparison showed
that MCIR's redundancy-aware allocation does not imply uniform
superiority under predictive-risk or explained-variance criteria.
Overall, MCIR-M provides a statistically defined framework for
explaining a fixed model under feature dependence. Its interpretation
remains relative to the selected neighbourhood and joint
feature--output law, while its finite-sample reliability depends on
neighbourhood adequacy, estimator suitability, sample size, and ranking
separation. Future work will investigate higher-order and temporal
neighbourhood construction, more reliable high-dimensional and
mixed-type CMI estimation, robustness under distribution shift,
improved regional importance estimation, and broader comparisons with
statistically grounded variable-importance and dependence-aware
sensitivity functionals.
% =========================================================

% ---------------------------------------------------------

\bibliography{main}
\bibliographystyle{tmlr}

%==================================================
% APPENDIX CONTENTS
%==================================================
%==================================================
% APPENDIX CONTENTS (Appendix Only)
%==================================================

%==================================================
% APPENDIX CONTENTS
%==================================================
\appendix
\clearpage
\section{Supplementary Notation and Scope}
\label{appendix:notation}

This appendix supplements the definitions in
Section~\ref{sec:notation}; it does not introduce an alternative MCIR
formulation. Throughout, $X=(X_1,\ldots,X_k)$ is the population feature
vector, $Q$ denotes the model-output representation in the environment
under consideration, and $\Phi(i)$ is the selected conditioning
neighbourhood of feature $i$. The full and lightweight sample sizes are
$N$ and $n<N$, while $d$ and $d'\leq d$ are their output-representation
dimensions.

\begin{table}[htbp]
\centering
\scriptsize
\setlength{\tabcolsep}{3pt}
\renewcommand{\arraystretch}{0.92}
\caption{Supplementary notation.}
\label{tab:notation}
\begin{adjustbox}{width=\linewidth}
\begin{tabular}{lp{5.0cm}lp{5.6cm}}
\toprule
\textbf{Symbol} & \textbf{Meaning} &
\textbf{Symbol} & \textbf{Meaning} \\
\midrule
$P_{X,Q}$ & Population joint law in the environment considered. &
$Q$ & $Z\in\mathbb R^d$ or $Z'\in\mathbb R^{d'}$. \\
$N,n$ & Full and lightweight sample sizes, $n<N$. &
$d,d'$ & Full and lightweight output dimensions, $d'\leq d$. \\
$\widehat P_{X,Z}^{\,N}$ & Full empirical joint law. &
$\widehat P_{X,Z'}^{\,n}$ & Lightweight empirical joint law. \\
$\widehat\mu_N,\widehat\nu_n$ & Full and lightweight empirical output laws. &
$\Phi(i)$ & Selected neighbourhood of $X_i$. \\
$X_{\Phi(i)}$ & Random feature block indexed by $\Phi(i)$. &
$U_i$ & $I_{P_{X,Q}}(Q;X_i\mid X_{\Phi(i)})$. \\
$H_i$ & $I_{P_{X,Q}}(Q;X_{\Phi(i)})$. &
$J_i$ & $I_{P_{X,Q}}(Q;X_{\Phi(i)\cup\{i\}})=H_i+U_i$. \\
$C_i$ & Population MCIR score $2U_i/(U_i+J_i)$. &
$\widehat C_{i,n}$ & Empirical MCIR estimate. \\
$m_\Phi$ & Maximum conditioning-neighbourhood size. &
$M_{\max}$ & Maximum screened candidate-pool size. \\
$\mathfrak F_{N,n}$ & Empirical lightweight fidelity profile. &
$d_{\mathrm{J@K}}$ & Top-$K$ Jaccard discrepancy. \\
\bottomrule
\end{tabular}
\end{adjustbox}
\end{table}

\paragraph{Population and empirical objects.}
MI and CMI without hats are functionals of $P_{X,Q}$. Feature columns
in an observed data matrix are sample realizations used to construct
empirical estimators; they are not population random variables.
Reducing $N$ to $n$ changes the number of atoms in an empirical measure,
not its ambient dimension. The Cai--Lim construction is used only when
the output dimension changes from $d$ to $d'$.

\paragraph{Environment comparison.}
For common coordinates and $d'=d$, representation discrepancy is
$W_p(P_Z,P_{Z'})$. When $d'<d$, it is the Cai--Lim augmented
Wasserstein distance
$\widetilde W_p(P_{Z'},P_Z)$ defined in
Equation~\ref{eq:augmented-wasserstein}. No generic
$f$-divergence identity between projection and embedding distances is
assumed. Moreover, output-law similarity alone does not imply MCIR
agreement because MCIR depends on the joint feature--output law.
Representation, ranking, top-$K$, and deletion/insertion discrepancies
are therefore reported separately through
$\mathfrak F_{N,n}$ in
Definition~\ref{def:lightweight-fidelity-profile}.

\section{PCIR Implementation Used as a Baseline}
\label{app:pcir}

PCIR~\cite{sengupta2025correlation} is included only as an
unconditioned dispersion-based baseline. For a scalar output
realization $\{q^{(r)}\}_{r=1}^{m}$ and feature realization
$\{x_i^{(r)}\}_{r=1}^{m}$, let $\bar q$, $\bar x_i$, and $m_i$ denote
the output mean, feature mean, and pooled center used by the PCIR
implementation. The score is
\begin{equation}
\medmath{
\eta_i
=
\frac{
m\!\left[(\bar x_i-m_i)^2+(\bar q-m_i)^2\right]
}{
\sum_{r=1}^{m}(x_i^{(r)}-m_i)^2+
\sum_{r=1}^{m}(q^{(r)}-m_i)^2
}.
}
\label{eq:app-pcir-score}
\end{equation}
Writing the denominator as between-mean plus within-sample dispersion
gives
\begin{equation}
\medmath{
S_T(i)
=
S_B(i)+S_W(i),\qquad
S_B(i)
=
m\!\left[(\bar x_i-m_i)^2+(\bar q-m_i)^2\right],
\quad S_W(i)\geq0.
}
\label{eq:app-pcir-decomposition}
\end{equation}
Hence $0\leq\eta_i\leq1$ whenever $S_T(i)>0$. This algebraic
boundedness does not establish that PCIR measures conditional,
information-theoretic, or causal importance. Because feature and output
values enter a common pooled construction, its numerical value can
depend on their relative parameterizations and preprocessing. The
experiments therefore use the fixed preprocessing reported in
Table~\ref{tab:reproducibility_summary}, and PCIR is not used to prove
properties of MCIR.

\section{Mixed-Variable Information Components}
\label{app:mixed-information}

Let $X_i$, $X_{\Phi(i)}$, and $Q$ be random elements that may be
continuous, discrete, or mixed-type. Under Assumption~\ref{asm:wellposed},
their conditional mutual information is
\begin{equation}
\medmath{
\begin{aligned}
I_{P_{X,Q}}(Q;X_i\mid X_{\Phi(i)})
:=
\int D_{\mathrm{KL}}\Big(
&P_{Q,X_i\mid X_{\Phi(i)}=x_\Phi}\\[-1mm]
&\big\|
P_{Q\mid X_{\Phi(i)}=x_\Phi}
\otimes P_{X_i\mid X_{\Phi(i)}=x_\Phi}
\Big)\,P_{X_{\Phi(i)}}(dx_\Phi).
\end{aligned}
}
\label{eq:app-mixed-cmi}
\end{equation}
Equivalently, when the conditional Radon--Nikodym derivative exists,
\begin{equation}
\medmath{
I_{P_{X,Q}}(Q;X_i\mid X_{\Phi(i)})
=
\mathbb E
\left[
\log
\frac{
dP_{Q\mid X_i,X_{\Phi(i)}}
}{
dP_{Q\mid X_{\Phi(i)}}
}(Q)
\right].
}
\label{eq:app-mixed-cmi-rn}
\end{equation}

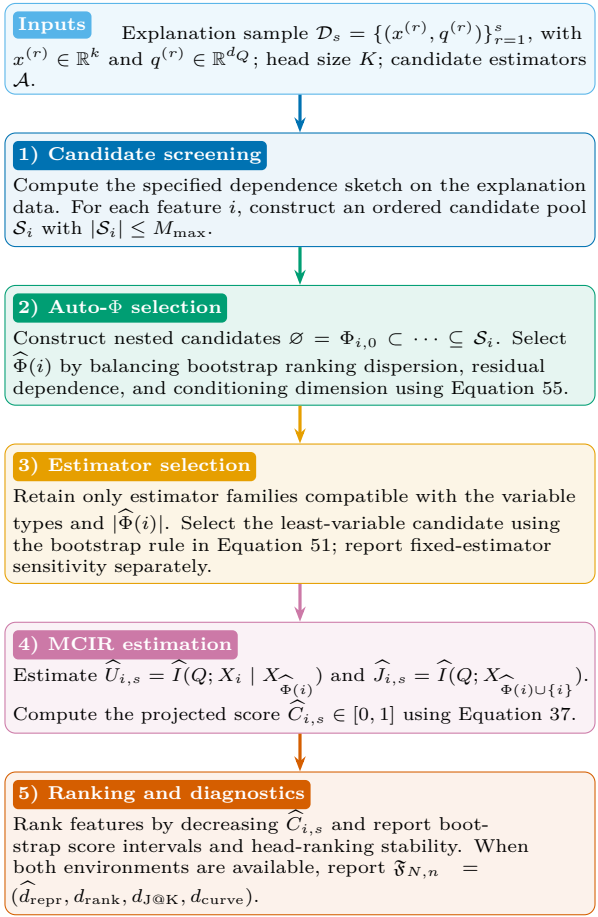
\begin{wrapfigure}{l}{0.50\linewidth}
\centering
\begin{tikzpicture}[
  node distance=5mm,
  header/.style={
    rectangle, rounded corners=2pt, inner sep=2pt,
    font=\scriptsize\bfseries, text=white
  },
  box/.style={
    rectangle, rounded corners, align=left,
    inner sep=3pt, text width=0.92\linewidth,
    font=\scriptsize, fill=white
  },
  >={Stealth[length=2mm]}
]

% Inputs
\node[box, fill=oi-sky!10, draw=oi-sky] (input) {%
  \tikz{\node[header, fill=oi-sky, anchor=west] at (0,0) {Inputs};}
  \hspace{2ex}
  Explanation sample
  $\mathcal D_s=\{(x^{(r)},q^{(r)})\}_{r=1}^{s}$,
  with $x^{(r)}\in\mathbb R^k$ and
  $q^{(r)}\in\mathbb R^{d_Q}$; head size $K$;
  candidate estimators $\mathcal A$.
};

% Screening
\node[box, below=of input, fill=oi-blue!7, draw=oi-blue] (screen) {%
  \tikz{\node[header, fill=oi-blue, anchor=west] at (0,0)
  {1) Candidate screening};}
  \newline
  Compute the specified dependence sketch on the explanation data.
  For each feature $i$, construct an ordered candidate pool
  $\mathcal S_i$ with $|\mathcal S_i|\leq M_{\max}$.
};

% Auto-Phi
\node[box, below=of screen, fill=oi-green!10, draw=oi-green] (cond) {%
  \tikz{\node[header, fill=oi-green, anchor=west] at (0,0)
  {2) Auto-$\Phi$ selection};}
  \newline
  Construct nested candidates
  $\varnothing=\Phi_{i,0}\subset\cdots\subseteq\mathcal S_i$.
  Select $\widehat\Phi(i)$ by balancing bootstrap ranking dispersion,
  residual dependence, and conditioning dimension using
  Equation~\ref{eq:auto-phi-objective}.
};

% Estimator selection
\node[box, below=of cond, fill=oi-orange!10, draw=oi-orange] (select) {%
  \tikz{\node[header, fill=oi-orange, anchor=west] at (0,0)
  {3) Estimator selection};}
  \newline
  Retain only estimator families compatible with the variable types
  and $|\widehat\Phi(i)|$. Select the least-variable candidate using
  the bootstrap rule in
  Equation~\ref{eq:estimator-switching-rule}; report fixed-estimator
  sensitivity separately.
};

% MCIR estimation
\node[box, below=of select, fill=oi-purple!10, draw=oi-purple] (est) {%
  \tikz{\node[header, fill=oi-purple, anchor=west] at (0,0)
  {4) MCIR estimation};}
  \newline
  Estimate
  $\widehat U_{i,s}
  =\widehat I(Q;X_i\mid X_{\widehat\Phi(i)})$
  and
  $\widehat J_{i,s}
  =\widehat I(Q;X_{\widehat\Phi(i)\cup\{i\}})$.
  Compute the projected score
  $\widehat C_{i,s}\in[0,1]$ using
  Equation~\ref{eq:clipped-empirical-mcir}.
};

% Diagnostics
\node[box, below=of est, fill=oi-verm!8, draw=oi-verm] (valid) {%
  \tikz{\node[header, fill=oi-verm, anchor=west] at (0,0)
  {5) Ranking and diagnostics};}
  \newline
  Rank features by decreasing $\widehat C_{i,s}$ and report bootstrap
  score intervals and head-ranking stability. When both environments
  are available, report
  $\mathfrak F_{N,n}
  =(\widehat d_{\mathrm{repr}},d_{\mathrm{rank}},
  d_{\mathrm{J@K}},d_{\mathrm{curve}})$.
};

% Arrows
\draw[->, very thick, oi-blue]   (input) -- (screen);
\draw[->, very thick, oi-green]  (screen) -- (cond);
\draw[->, very thick, oi-orange] (cond) -- (select);
\draw[->, very thick, oi-purple] (select) -- (est);
\draw[->, very thick, oi-verm]   (est) -- (valid);
\end{tikzpicture}
\caption{MCIR workflow. Candidate screening, neighbourhood selection,
estimator choice, scoring, and full--lightweight diagnostics are treated
as distinct stages.}
\label{fig:mcir_scheme_vertical}
\end{wrapfigure}
Equations~\eqref{eq:app-mixed-cmi} and
\eqref{eq:app-mixed-cmi-rn} provide one measure-theoretic definition
for discrete, continuous, and mixed-type variables. When compatible
mass functions or densities exist, the expectation becomes the
corresponding sum, integral, or sum--integral expression. No common
Lebesgue density and no positive-probability assumption on individual
conditioning values are required.
For example, if $Q$ and $X_d$ are discrete and $X_c$ is continuous,
then, whenever compatible conditional densities and masses exist,
\begin{equation}
\medmath{
\begin{aligned}
I(Q;X_c\mid X_d)
=
\sum_{x_d}\sum_q\int
p(q,x_c,x_d)
\log
\frac{p(q\mid x_c,x_d)}
     {p(q\mid x_d)}
\,dx_c .
\end{aligned}
}
\label{eq:app-mixed-cmi-example}
\end{equation}
The all-discrete and all-continuous formulas follow by replacing the
remaining integrals or sums accordingly. The entropy-difference
identity may be used only when the corresponding conditional entropies
are well defined and finite.

The term \emph{mixed-variable CMI} refers to evaluating the
standard population CMI functional in a mixed-type feature space; it
does not denote a new information measure or imply that one estimator
is uniformly valid for every mixed-data regime.

The joint mutual information of $Q$ and
$X_S=(X_j)_{j\in S}$ is
\begin{equation}
\medmath{
I_{P_{X,Q}}(Q;X_S)
=
D_{\mathrm{KL}}
\left(
P_{Q,X_S}\,\middle\|\,P_Q\otimes P_{X_S}
\right)\geq0.
}
\label{eq:app-jmi}
\end{equation}
For each feature, define
\begin{equation}
\begin{split}
    \medmath{
U_i=I_{P_{X,Q}}(Q;X_i\mid X_{\Phi(i)}),}\\
\medmath{H_i=I_{P_{X,Q}}(Q;X_{\Phi(i)}),\
J_i=H_i+U_i.
}
\end{split}
\label{eq:app-information-components}
\end{equation}
Here, $U_i$ is the \emph{unique predictive information} contributed by
$X_i$ beyond its selected neighbourhood, while $J_i$ is the total
predictive information carried by the enlarged block. The term
\emph{unique} is relative to $X_{\Phi(i)}$ and describes statistical,
not causal, contribution. Figure~\ref{fig:mcir_scheme_vertical} summarizes the complete MCIR
workflow. Given an explanation sample from either the full or
lightweight environment, the procedure first constructs a screened
candidate pool $\mathcal S_i$ for each feature. Auto-$\Phi$ then selects
a small conditioning neighbourhood $\widehat\Phi(i)$ by balancing
ranking stability, residual dependence, and conditioning dimension.
Among the estimators compatible with the resulting local block, the
bootstrap switching rule selects the candidate with the lowest estimated
variability. MCIR subsequently combines the estimated unique predictive
information
$\widehat I(Q;X_i\mid X_{\widehat\Phi(i)})$
with the total information
$\widehat I(Q;X_{\widehat\Phi(i)\cup\{i\}})$
to obtain a bounded score. Finally, the scores are ranked and reported
with bootstrap uncertainty and stability diagnostics. When both full
and lightweight environments are available, representation, complete
ranking, top-$K$, and deletion/insertion-curve discrepancies are reported
jointly through the lightweight fidelity profile
$\mathfrak F_{N,n}$.

\section{Detailed Proofs of the Main Results}
\label{sec:theory1}

This section proves the population and finite-sample results stated in
Section~\ref{prob}. Throughout, all information quantities are computed
under the relevant joint law $P_{X,Q}$, and
\begin{equation}
\medmath{
U_i=I_{P_{X,Q}}(Q;X_i\mid X_{\Phi(i)}),\qquad
H_i=I_{P_{X,Q}}(Q;X_{\Phi(i)}),\qquad
J_i=H_i+U_i.
}
\label{eq:proof-information-components}
\end{equation}
The chain rule gives
$J_i=I_{P_{X,Q}}(Q;X_{\Phi(i)\cup\{i\}})=H_i+U_i$; hence
\begin{equation}
\medmath{
C_i
=
\frac{2U_i}{U_i+J_i}
=
\frac{2U_i}{2U_i+H_i},
}
\label{eq:proof-mcir-decomposition}
\end{equation}
with $C_i=0$ when $U_i+J_i=0$.

\subsection{Boundedness and the Zero Endpoint}
\label{sec:theory1.1}

\begin{proposition}[Boundedness and zero characterization]
\label{prop:proof-boundedness}
Under Assumption~\ref{asm:wellposed},
$\medmath{0\leq C_i\leq1}$, and
\begin{equation}
\medmath{
C_i=0
\quad\Longleftrightarrow\quad
U_i=0
\quad\Longleftrightarrow\quad
Q\perp\!\!\!\perp X_i\mid X_{\Phi(i)}
\ \text{a.s.}
}
\label{eq:proof-zero-characterization}
\end{equation}
\end{proposition}

\begin{proof}
MI and CMI are non-negative because they are KL divergences. Thus,
$\medmath{U_i,H_i\geq0}$ and
$\medmath{J_i=H_i+U_i\geq U_i}$. If $U_i=0$, then $C_i=0$ by either
the ratio or the zero-denominator convention. If $U_i>0$, then
\begin{equation}
\medmath{
0
<
\frac{2U_i}{U_i+J_i}
\leq
\frac{2U_i}{2U_i}
=1,
}
\end{equation}
where the inequality follows from $J_i\geq U_i$. This proves
boundedness and shows that $C_i=0$ if and only if $U_i=0$.

By the conditional-kernel definition,
\begin{equation}
\medmath{
\begin{aligned}
U_i
=
\int D_{\mathrm{KL}}\Big(
&P_{Q,X_i\mid X_{\Phi(i)}=x_\Phi}\\[-1mm]
&\big\|
P_{Q\mid X_{\Phi(i)}=x_\Phi}
\otimes P_{X_i\mid X_{\Phi(i)}=x_\Phi}
\Big)\,P_{X_{\Phi(i)}}(dx_\Phi).
\end{aligned}
}
\label{eq:proof-cmi-kernel}
\end{equation}
This integral is zero if and only if its non-negative integrand is zero
for $P_{X_{\Phi(i)}}$-almost every $x_\Phi$. Since KL divergence is
zero exactly when its two arguments agree,
\begin{equation}
\medmath{
P_{Q,X_i\mid X_{\Phi(i)}=x_\Phi}
=
P_{Q\mid X_{\Phi(i)}=x_\Phi}
\otimes P_{X_i\mid X_{\Phi(i)}=x_\Phi}
}
\end{equation}
for almost every $x_\Phi$. This is equivalent to
$\medmath{Q\perp\!\!\!\perp X_i\mid X_{\Phi(i)}}$ almost surely.
\end{proof}

Boundedness provides a common numerical range but does not imply
identical calibration across different joint laws, models, datasets,
or estimators.

\subsection{Exact Conditional-Redundancy Collapse}
\label{sec:theory1.2}

\begin{proposition}[Exact conditional-redundancy collapse]
\label{prop:proof-exact-collapse}
If $X_i$ is measurable with respect to
$\sigma(X_{\Phi(i)})$, then $U_i=0$ and $C_i=0$.
\end{proposition}

\begin{proof}
Measurability gives
\begin{equation}
\medmath{
\sigma(X_i,X_{\Phi(i)})
=
\sigma(X_{\Phi(i)}).
}
\label{eq:proof-sigma-equality}
\end{equation}
Consequently, conditioning on $(X_i,X_{\Phi(i)})$ supplies the same
information as conditioning on $X_{\Phi(i)}$ alone. A version of the
regular conditional law therefore satisfies
\begin{equation}
\medmath{
P_{Q\mid X_i,X_{\Phi(i)}}
=
P_{Q\mid X_{\Phi(i)}}
\quad\text{a.s.}
}
\label{eq:proof-conditional-law-equality}
\end{equation}
Substitution into the conditional-kernel or Radon--Nikodym
representation of CMI yields $U_i=0$. Proposition
\ref{prop:proof-boundedness} then gives $C_i=0$.
\end{proof}

This result includes an exact duplicate or measurable transformation
already represented in $\Phi(i)$. Without additional assumptions, it
does not establish convergence for noisy near-duplicates. Such cases
are evaluated empirically through estimator and neighbourhood
sensitivity.

\subsection{The Upper Endpoint and Population Invariance}
\label{sec:theory1.3}

\begin{proposition}[Upper endpoint]
\label{prop:proof-upper-endpoint}
If $U_i>0$, then
\begin{equation}
\medmath{
C_i=1
\quad\Longleftrightarrow\quad
H_i=I_{P_{X,Q}}(Q;X_{\Phi(i)})=0.
}
\label{eq:proof-upper-characterization}
\end{equation}
\end{proposition}

\begin{proof}
Using Equation~\ref{eq:proof-mcir-decomposition} and $U_i>0$,
\begin{equation}
\medmath{
\begin{aligned}
C_i=1
&\Longleftrightarrow
\frac{2U_i}{2U_i+H_i}=1\\
&\Longleftrightarrow
2U_i=2U_i+H_i
\Longleftrightarrow H_i=0.
\end{aligned}
}
\end{equation}
Conversely, $H_i=0$ gives
$\medmath{C_i=2U_i/(2U_i)=1}$.
\end{proof}

\begin{proposition}[Population reparameterization invariance]
\label{prop:proof-population-invariance}
Let $T_Q$, $T_i$, and $T_\Phi$ be bimeasurable bijections on the state
spaces of $Q$, $X_i$, and $X_{\Phi(i)}$, respectively. Applying these
maps leaves $C_i$ unchanged.
\end{proposition}

\begin{proof}
For a bimeasurable bijection $T$, relative entropy is invariant under
pushforward:
\begin{equation}
\medmath{
D_{\mathrm{KL}}(P\|R)
=
D_{\mathrm{KL}}(T_\#P\|T_\#R).
}
\label{eq:proof-kl-invariance}
\end{equation}
Applying Equation~\ref{eq:proof-kl-invariance} to the joint and
product measures defining $H_i$ proves its invariance. Applying it
inside Equation~\ref{eq:proof-cmi-kernel}, followed by the change of
variables induced by $T_\Phi$, proves invariance of $U_i$. Hence
$J_i=H_i+U_i$ and
$\medmath{C_i=2U_i/(2U_i+H_i)}$ are unchanged.
\end{proof}

\subsection{Weak-Dependence Reduction}
\label{sec:theory1.5}

\begin{proof}[Proof of Proposition~\ref{prop:weak_dep_correct}]
For the sequence of joint laws in
Proposition~\ref{prop:weak_dep_correct}, write
\begin{equation}
\medmath{
U_i^{(m)}
=
u_i^{(m)}+r_i^{(m)},\qquad
H_i^{(m)}
=
I_{P_{X,Q}^{(m)}}(Q;X_{\Phi(i)}).
}
\end{equation}
By the chain rule,
$\medmath{J_i^{(m)}=H_i^{(m)}+U_i^{(m)}}$. Substitution into MCIR gives
\begin{equation}
\medmath{
C_i^{(m)}
=
\frac{
2\{u_i^{(m)}+r_i^{(m)}\}
}{
2\{u_i^{(m)}+r_i^{(m)}\}+H_i^{(m)}
}.
}
\label{eq:proof-weak-ratio}
\end{equation}
Let
\begin{equation}
\medmath{
a_m=2u_i^{(m)}+H_i^{(m)},\qquad
b_m=2r_i^{(m)}.
}
\end{equation}
The denominator assumption gives $a_m\geq c>0$ eventually, while
$b_m\to0$. Subtracting the asserted reduced expression from
Equation~\ref{eq:proof-weak-ratio} yields
\begin{equation}
\medmath{
\begin{aligned}
&\frac{2u_i^{(m)}+b_m}{a_m+b_m}
-\frac{2u_i^{(m)}}{a_m}\\
&\qquad=
\frac{
b_m\{a_m-2u_i^{(m)}\}
}{
a_m(a_m+b_m)
}
=
\frac{b_mH_i^{(m)}}{a_m(a_m+b_m)}.
\end{aligned}
}
\label{eq:proof-weak-difference}
\end{equation}
Because $\medmath{0\leq H_i^{(m)}\leq a_m}$ and
$a_m+b_m\geq c/2$ eventually,
\begin{equation}
\medmath{
\left|
\frac{b_mH_i^{(m)}}{a_m(a_m+b_m)}
\right|
\leq
\frac{2|b_m|}{c}
\longrightarrow0.
}
\end{equation}
Therefore,
\begin{equation}
\medmath{
C_i^{(m)}
=
\frac{2u_i^{(m)}}
     {2u_i^{(m)}+H_i^{(m)}}
+o(1),
}
\end{equation}
as claimed.
\end{proof}

This result removes the conditioning correction from the
numerator up to $o(1)$, but the feature-dependent term $H_i^{(m)}$
remains in the denominator. It therefore does not establish ordering
equivalence with PCIR or another marginal attribution method.

\subsection{Estimator Consistency, Ratio Perturbation, and Rank Stability}
\label{sec:theory1.6}

\begin{lemma}[Consistency transfer]
\label{lem:proof-mcir-consistency}
Suppose
$\widehat U_{i,n}\xrightarrow{p}U_i$ and
$\widehat J_{i,n}\xrightarrow{p}J_i$, with $U_i+J_i>0$. Then the
projected empirical MCIR score converges in probability to $C_i$.
\end{lemma}

\begin{proof}
The positive-part map is continuous, so
\begin{equation}
\medmath{
\widehat U_{i,n}^{+}\xrightarrow{p}U_i,\qquad
\widehat J_{i,n}^{+}\xrightarrow{p}J_i.
}
\end{equation}
The ratio map
$\medmath{g(u,j)=2u/(u+j)}$ is continuous at $(U_i,J_i)$ because
$U_i+J_i>0$. The continuous-mapping theorem therefore gives
\begin{equation}
\medmath{
\frac{2\widehat U_{i,n}^{+}}
     {\widehat U_{i,n}^{+}+\widehat J_{i,n}^{+}}
\xrightarrow{p}
\frac{2U_i}{U_i+J_i}
=
C_i.
}
\end{equation}
Projection onto $[0,1]$ is continuous and fixes
$C_i\in[0,1]$, completing the proof.
\end{proof}

\begin{lemma}[MCIR ratio perturbation]
\label{lem:proof-ratio-perturbation}
Let $\medmath{D_i=U_i+J_i\geq c_0>0}$. If
\begin{equation}
\medmath{
|\widehat U_i^{+}-U_i|\leq\delta,\qquad
|\widehat J_i^{+}-J_i|\leq\delta,\qquad
\delta\leq c_0/4,
}
\label{eq:proof-component-event}
\end{equation}
then
\begin{equation}
\medmath{
|\widehat C_i-C_i|
\leq
\frac{4\delta}{c_0}
\leq
\frac{8\delta}{c_0}.
}
\label{eq:proof-ratio-bound}
\end{equation}
\end{lemma}

\begin{proof}
Set
$\medmath{\widehat D_i=\widehat U_i^{+}+\widehat J_i^{+}}$.
Equation~\ref{eq:proof-component-event} implies
\begin{equation}
\medmath{
\widehat D_i
\geq
D_i-2\delta
\geq
c_0/2.
}
\label{eq:proof-denominator-lower-bound}
\end{equation}
Direct algebra gives
\begin{equation}
\medmath{
\begin{aligned}
\frac{2\widehat U_i^{+}}{\widehat D_i}
-\frac{2U_i}{D_i}
=
\frac{
2\!\left[
J_i(\widehat U_i^{+}-U_i)
-U_i(\widehat J_i^{+}-J_i)
\right]
}{
\widehat D_iD_i
}.
\end{aligned}
}
\label{eq:proof-ratio-identity}
\end{equation}
Since $\medmath{0\leq U_i,J_i\leq D_i}$, Equations
\eqref{eq:proof-component-event}--\eqref{eq:proof-ratio-identity}
yield
\begin{equation}
\medmath{
\left|
\frac{2\widehat U_i^{+}}{\widehat D_i}
-\frac{2U_i}{D_i}
\right|
\leq
\frac{2(U_i+J_i)\delta}{\widehat D_iD_i}
=
\frac{2\delta}{\widehat D_i}
\leq
\frac{4\delta}{c_0}.
}
\end{equation}
Projection onto $[0,1]$ is non-expansive and cannot increase the error.
The main theorem uses the conservative constant $8\delta/c_0$, which
follows immediately.
\end{proof}

\begin{proof}[Proof of Theorem~\ref{thm:rank_stability}]
Let $\mathcal E_\delta$ be the event on which every information
component used for the features in $\mathcal I$ is estimated within
$\delta$. There are $m_{\mathcal I}$ such components, each violating
the bound with probability at most $\alpha$. The union bound gives
\begin{equation}
\medmath{
\mathbb P(\mathcal E_\delta)
\geq
1-m_{\mathcal I}\alpha.
}
\label{eq:proof-union-bound}
\end{equation}
On $\mathcal E_\delta$, positive-part clipping cannot increase the
error relative to the non-negative population quantities. Lemma
\ref{lem:proof-ratio-perturbation} and the conservative constant used
in the theorem give
\begin{equation}
\medmath{
\max_{i\in\mathcal I}
|\widehat C_{i,n}-C_i|
\leq
\eta,\qquad
\eta=\frac{8\delta}{c_0}.
}
\label{eq:proof-uniform-score-bound}
\end{equation}

Consider a population pair $(i,j)$ with $C_i>C_j$. If its empirical
ordering reverses, then
$\widehat C_{i,n}\leq\widehat C_{j,n}$ and
\begin{equation}
\medmath{
\begin{aligned}
C_i-C_j
&\leq
(C_i-\widehat C_{i,n})
+(\widehat C_{j,n}-C_j)\\
&\leq
2\eta.
\end{aligned}
}
\label{eq:proof-pair-gap}
\end{equation}
Thus, only population pairs with gaps at most $2\eta$ can reverse.
Applying the margin condition in
Equation~\ref{eq:rank-margin-condition} with $t=2\eta$ gives
\begin{equation}
\medmath{
\frac{D_{\mathrm{inv}}}
     {\binom{k_{\mathcal I}}{2}}
\leq
2M\eta
=
\frac{16M\delta}{c_0},
}
\label{eq:proof-inversion-fraction}
\end{equation}
where $D_{\mathrm{inv}}$ denotes the number of reversed pairs.

The margin condition at $t=0$ excludes population ties. Applying the
prespecified deterministic tie-breaking rule to the empirical scores
produces strict population and empirical rankings. Their normalized
inversion distance is therefore
\begin{equation}
\medmath{
d_{\mathrm{inv}}
\left(
\widehat\pi_{n,\mathcal I},
\pi_{\mathcal I}
\right)
=
\frac{D_{\mathrm{inv}}}
     {\binom{k_{\mathcal I}}{2}}
=
\frac{1-\tau_a}{2},
}
\label{eq:proof-kendall-inversions}
\end{equation}
where $\tau_a$ is computed from the resulting strict rankings. This
identity is not replaced by $(1-\tau_b)/2$ when ties are retained in
the original score vectors.
Combining Equations~\eqref{eq:proof-union-bound},
\eqref{eq:proof-inversion-fraction}, and
\eqref{eq:proof-kendall-inversions} proves
\begin{equation}
\medmath{
\mathbb P
\left[
d_{\mathrm{inv}}
\left(
\widehat\pi_{n,\mathcal I},
\pi_{\mathcal I}
\right)
\leq
\frac{16M\delta}{c_0}
\right]
\geq
1-m_{\mathcal I}\alpha.
}
\end{equation}
\end{proof}

If estimator-specific sequences satisfy
$\delta_n\to0$ and $m_{\mathcal I}\alpha_n\to0$, the same bound gives
convergence of the ranking discrepancy in probability. No common rate
follows for all MI/CMI estimators without separately verifying their
regularity conditions.

\section{Estimator Details}
\label{app}
\label{app:estimators}

This section specifies the estimators used for the two population
information components
\begin{equation}
\medmath{
U_i
=
I_{P_{X,Q}}(Q;X_i\mid X_{\Phi(i)}),
\qquad
J_i
=
I_{P_{X,Q}}(Q;X_{\Phi(i)\cup\{i\}}).
}
\label{eq:appendix-estimator-components}
\end{equation}
Here $Q$ denotes the model-output representation in the environment
under consideration, and $X_{\Phi(i)}$ is the selected conditioning
neighbourhood. For an explanation sample of size $s\in\{N,n\}$, the
estimators produce $\widehat U_{i,s}$ and $\widehat J_{i,s}$, which
enter the projected empirical score
\begin{equation}
\medmath{
\widehat C_{i,s}
=
\Pi_{[0,1]}
\left(
\frac{
2\widehat U_{i,s}^{+}
}{
\widehat U_{i,s}^{+}+\widehat J_{i,s}^{+}
}
\right),
}
\label{eq:appendix-estimated-mcir}
\end{equation}
with $\widehat C_{i,s}=0$ when the denominator vanishes. The projection
and positive-part corrections are finite-sample numerical safeguards;
the population MCIR score is already in $[0,1]$.

No single MI or CMI estimator is uniformly suitable for all variable
types, sample sizes, or conditioning dimensions. We therefore use
Gaussian--copula, $k$NN, or discrete plug-in estimation according to
the structure of the local block
$\medmath{(Q,X_i,X_{\Phi(i)})}$. Their computational costs are
summarized in Table~\ref{tab:mi-comparisons}; this section focuses on
their definitions, assumptions, and implementation.

\subsection{Gaussian--Copula Estimation}

For a scalar coordinate $A$ observed at
$a^{(1)},\ldots,a^{(s)}$, let $R_r(A)$ be the rank of $a^{(r)}$,
using the deterministic tie rule stated in the experimental protocol.
Rank Gaussianization maps each observation to
\begin{equation}
\medmath{
\widetilde a^{(r)}
=
\Phi_{\mathcal N}^{-1}
\left(
\frac{R_r(A)-1/2}{s}
\right),
}
\label{eq:appendix-rank-gaussianization}
\end{equation}
where $\Phi_{\mathcal N}$ is the standard-normal distribution
function. The offset $1/2$ prevents the empirical probabilities from
taking the endpoint values zero and one. For vector-valued variables,
the transformation is applied componentwise.

Let $\widehat\Sigma_W$ denote the empirical covariance matrix of the
rank-Gaussianized block $W$. To improve numerical conditioning, the
implemented covariance matrix is
\begin{equation}
\medmath{
\widehat\Sigma_{W,\lambda}
=
\widehat\Sigma_W+\lambda I,
\qquad
\lambda>0,
}
\label{eq:appendix-ridge-covariance}
\end{equation}
with the ridge parameter reported for each experiment.

For rank-Gaussianized random vectors $A$ and $B$, the
Gaussian-copula MI estimate is
\begin{equation}
\medmath{
\widehat I_{\mathrm{GC}}(A;B)
=
\frac{1}{2}
\log
\frac{
\det(\widehat\Sigma_{A,\lambda})
\det(\widehat\Sigma_{B,\lambda})
}{
\det(\widehat\Sigma_{(A,B),\lambda})
}.
}
\label{eq:appendix-gc-mi}
\end{equation}
For conditioning vector $C$, let
$\widehat\Sigma_{A\mid C}$,
$\widehat\Sigma_{B\mid C}$, and
$\widehat\Sigma_{(A,B)\mid C}$ be the corresponding Schur-complement
conditional covariance matrices. The CMI estimate is
\begin{equation}
\medmath{
\widehat I_{\mathrm{GC}}(A;B\mid C)
=
\frac{1}{2}
\log
\frac{
\det(\widehat\Sigma_{A\mid C})
\det(\widehat\Sigma_{B\mid C})
}{
\det(\widehat\Sigma_{(A,B)\mid C})
}.
}
\label{eq:appendix-gc-cmi}
\end{equation}
When $A$ and $B$ are scalar, this reduces to
\begin{equation}
\medmath{
\widehat I_{\mathrm{GC}}(A;B\mid C)
=
-\frac{1}{2}
\log
\left(
1-\widehat\rho_{AB\mid C}^{\,2}
\right),
}
\label{eq:appendix-gc-partial-correlation}
\end{equation}
where $\widehat\rho_{AB\mid C}$ is their partial correlation after
rank Gaussianization.

For MCIR, Equation~\ref{eq:appendix-gc-cmi} is applied with
$\medmath{A=Q}$, $\medmath{B=X_i}$, and
$\medmath{C=X_{\Phi(i)}}$ to estimate $U_i$. Equation
\eqref{eq:appendix-gc-mi} is applied with
$\medmath{A=Q}$ and
$\medmath{B=X_{\Phi(i)\cup\{i\}}}$ to estimate $J_i$.

\paragraph{Scope and assumptions.}
The Gaussian--copula estimator is appropriate when dependence in the
rank-Gaussianized local block is reasonably represented by a Gaussian
copula and the regularized covariance matrices are numerically
well-conditioned. Under correct Gaussian-copula specification and
standard covariance-consistency conditions, it estimates the
corresponding population information functional. For a general
non-Gaussian copula, it estimates a Gaussian-copula approximation
rather than unrestricted MI or CMI.

Rank Gaussianization makes the implemented estimator invariant to
strictly monotone componentwise transformations under the stated
tie-handling rule. This is an estimator-specific property and should
not be interpreted as evidence that the Gaussian--copula estimator is
universally unbiased for arbitrary nonlinear or multimodal dependence.

\subsection{\texorpdfstring{$k$}{k}NN Estimation}

For continuous local blocks, we additionally use KSG-type $k$NN MI
estimators and corresponding conditional extensions. These estimators
construct neighbourhoods in the joint sample space and combine
nearest-neighbour distances with marginal or conditional neighbour
counts. We denote the number of neighbours by $k_{\mathrm{NN}}$ to
distinguish it from the total number of input features $k$.

The MI estimate has the generic KSG form
\begin{equation}
\medmath{
\widehat I_{\mathrm{kNN}}(A;B)
=
\psi(k_{\mathrm{NN}})
+\psi(s)
-
\frac{1}{s}
\sum_{r=1}^{s}
\left[
\psi\!\left(n_A^{(r)}+1\right)
+
\psi\!\left(n_B^{(r)}+1\right)
\right],
}
\label{eq:appendix-knn-mi}
\end{equation}
up to the convention-specific finite-sample correction. Here
$\psi$ is the digamma function and
$n_A^{(r)}$, $n_B^{(r)}$ are marginal neighbour counts determined by
the joint-space radius of observation $r$. Conditional variants use
the corresponding neighbour counts in the
$(A,C)$, $(B,C)$, and $C$ spaces.

For MCIR, the $k$NN CMI estimator is applied directly to
$\medmath{(Q,X_i,X_{\Phi(i)})}$, while the joint MI estimator is
applied to
$\medmath{(Q,X_{\Phi(i)\cup\{i\}})}$. Both calculations use the same
sample, metric, value of $k_{\mathrm{NN}}$, and preprocessing within an
experiment.

The sensitivity analysis evaluates
\begin{equation}
\medmath{
k_{\mathrm{NN}}\in\{5,10,20\},
}
\label{eq:appendix-knn-values}
\end{equation}
using Euclidean distance. The selected value and any standardization
are reported with the corresponding experiment.

\paragraph{Scope and assumptions.}
Consistency and rates for $k$NN MI/CMI estimation require
estimator-specific regularity conditions, including restrictions on
the density, support, tails, neighbour-count sequence, and local
dimension. Tree-based neighbour search is typically efficient in low
or moderate dimension, but both computational and statistical
performance can deteriorate as $|X_{\Phi(i)}|$ or $d_Q$ increases.

The $k$NN estimator is therefore evaluated only in the stated
continuous, low- or moderate-dimensional regimes. No
dimension-independent or universal $s^{-1/2}$ error rate is assumed.
Estimator variability is reported as a function of
$|\Phi(i)|$.

\subsection{Discrete Plug-In Estimation}

When all variables in the local information calculation are discrete,
MI and CMI are estimated from empirical probability masses. Let
$\widehat p(q,x,c)$ denote the empirical joint mass of
$\medmath{(Q,X_i,X_{\Phi(i)})=(q,x,c)}$. The plug-in CMI estimate is
\begin{equation}
\medmath{
\widehat I_{\mathrm{plg}}
(Q;X_i\mid X_{\Phi(i)})
=
\sum_{q,x,c:\,\widehat p(q,x,c)>0}
\widehat p(q,x,c)
\log
\frac{
\widehat p(q,x\mid c)
}{
\widehat p(q\mid c)\widehat p(x\mid c)
}.
}
\label{eq:appendix-plugin-cmi}
\end{equation}
For $S=\Phi(i)\cup\{i\}$, the joint MI estimate is
\begin{equation}
\medmath{
\widehat I_{\mathrm{plg}}(Q;X_S)
=
\sum_{q,x_S:\,\widehat p(q,x_S)>0}
\widehat p(q,x_S)
\log
\frac{
\widehat p(q,x_S)
}{
\widehat p(q)\widehat p(x_S)
}.
}
\label{eq:appendix-plugin-mi}
\end{equation}
Terms with zero empirical joint mass contribute zero. Minimum
cell-count requirements and finite-sample corrections are reported with
the relevant experiment.

For fixed finite alphabets, empirical probability masses converge to
their population values, yielding consistency of the corresponding
plug-in functional. When the alphabet, number of bins, or conditioning
dimension increases with $s$, sparsity and plug-in bias may become
substantial. Continuous variables are not silently treated as discrete:
any discretization must be stated explicitly and evaluated as a
separate preprocessing choice.

\subsection{Mixed-Type Local Blocks}

Definition~\ref{def:cmi} applies to continuous, discrete, and
mixed-type variables at the population level. Estimation is a separate
problem. A local block containing both continuous and discrete
variables is not automatically assigned to the discrete plug-in or
continuous $k$NN estimator.

The estimator used for a mixed block must define an appropriate
mixed-data metric, conditional-density model, copula construction, or
stratified calculation. Its assumptions must be stated for that
experiment. If no candidate estimator is compatible with the local
variable types and available sample size, the corresponding MCIR score
is not reported rather than being computed using an unsupported
estimator.

The term \emph{mixed-variable CMI} identifies the population
information functional evaluated on a mixed feature space. It does not
denote a new estimator and does not imply that a single empirical
method is uniformly valid across continuous, discrete, and mixed-data
regimes.

\subsection{Neighbourhood and Preprocessing Controls}

Candidate neighbourhoods are constructed from training or explanation
data without using the held-out evaluation outcomes. In the primary
experiments, Pearson screening uses
$\medmath{|\rho|\geq0.60}$ and retains at most
$\medmath{M_{\max}=6}$ candidates per target feature. Alternative
screening statistics and budgets are stated per experiment.

Auto-$\Phi$ selects among nested candidate neighbourhoods using
Equation~\ref{eq:auto-phi-objective}, which combines bootstrap
ranking dispersion, residual conditional dependence, and a
conditioning-dimension penalty. It does not select the neighbourhood
by maximizing feature--output MI alone.

All candidate estimators for a feature use the same observations,
model-output representation, and selected neighbourhood. Preprocessing
parameters are fitted without using the held-out test outcomes.
Continuous-variable scaling, tie handling, ridge values,
$k_{\mathrm{NN}}$, discrete cell corrections, and random seeds are
recorded in the experiment configuration.

\paragraph{Degenerate estimated denominators.}
If
$\widehat U_{i,s}^{+}+\widehat J_{i,s}^{+}=0$, the empirical MCIR score
is set to zero. Very small denominators are flagged because the ratio
may be sensitive to component-estimation error. The denominator
separation required for the perturbation and rank-stability results is
stated explicitly in Theorem~\ref{thm:rank_stability}.

\paragraph{Reporting.}
For each experiment, we report the estimator family, sample size,
output dimension, conditioning size, neighbourhood construction,
bootstrap count, and relevant hyperparameters. Where multiple
estimators are applicable, estimator disagreement and uncertainty are
reported rather than assuming numerical interchangeability.
\section{Estimator-Specific Monotone Invariance}
\label{sec:cmmi}
\label{appendix:monotone-invariance}

\begin{proposition}[Rank-Gaussianized estimator invariance]
\label{prop:monotone_invariance}
Suppose the Gaussian--copula MI/CMI estimator applies componentwise rank
Gaussianization with a fixed deterministic tie-handling rule. Strictly
monotone componentwise transformations of $Q$, $X_i$, and
$X_{\Phi(i)}$ leave the resulting empirical MCIR score unchanged.
\end{proposition}

\begin{proof}
A strictly increasing transformation preserves ranks, whereas a
strictly decreasing transformation reverses them. After rank
Gaussianization, this maps each transformed component to either its
original Gaussianized value or its negative. Hence the transformed
correlation matrix is $S\widehat\Sigma S$, where $S$ is diagonal with
entries in $\{-1,1\}$. Determinants of the covariance and
Schur-complement blocks used by the Gaussian--copula MI/CMI formulas
are unchanged because $\det(SAS)=\det(A)$. Thus
$\widehat U_{i,n}$, $\widehat J_{i,n}$, and their MCIR ratio are
unchanged.
\end{proof}

This is an estimator-specific property. It does not imply invariance of
every finite-sample MI or CMI estimator.

\section{Estimator Switching and Auto-\texorpdfstring{$\Phi$}{Phi}}
\label{sec:adaptive-procedures}
This section provides the implementation details for the two adaptive
components introduced in the main text: bootstrap-based estimator
switching and dependence-neighbourhood selection through Auto-$\Phi$.
Both are data-driven procedures rather than population guarantees; the
former selects among compatible MI/CMI estimators, while the latter
balances ranking stability, residual dependence, and conditioning
dimension.
\subsection{Bootstrap-Based Estimator Switching}
\label{sec:theory1.7}

This subsection details the estimator-selection rule introduced in
Section~\ref{sec:estimator-selection}. For each feature $i$, let
$\mathcal A_i$ denote the finite set of estimator families compatible
with the variable types, sample size, and conditioning dimension.
Depending on the data regime, $\mathcal A_i$ may contain
Gaussian--copula, $k$NN, and plug-in estimators.

For each $a\in\mathcal A_i$, let
$\widehat C_{i,n}^{(a)}$ denote the corresponding MCIR estimate and
$\widehat s_{i,a}$ its bootstrap standard error. Estimator selection
follows Equation~\ref{eq:estimator-switching-rule}; it is not redefined
in the appendix. When the minimum is attained by more than one
compatible estimator, the deterministic family ordering stated in the
main text resolves the tie.

The complete procedure is:

\begin{enumerate}[leftmargin=*,nosep]
    \item determine the estimator families compatible with the types of
    $Q$, $X_i$, and $X_{\Phi(i)}$;

    \item compute $\widehat C_{i,n}^{(a)}$ for every
    $a\in\mathcal A_i$ using the same observations and conditioning
    neighbourhood;

    \item generate bootstrap resamples using the resampling protocol
    appropriate to the data structure;

    \item recompute each candidate estimate on every bootstrap
    resample and obtain $\widehat s_{i,a}$;

    \item select $\widehat a_i$ using
    Equation~\ref{eq:estimator-switching-rule}; and

    \item report the switched estimate together with the results from
    the fixed candidate estimators.
\end{enumerate}

For independent observations, the ordinary nonparametric bootstrap may
be used. For temporally, spatially, or otherwise dependent
observations, the resampling scheme must preserve the relevant
dependence structure, for example through block bootstrap sampling.

The switching rule selects the compatible estimator with the
smallest estimated sampling variability. Bootstrap standard error does
not estimate estimator bias and therefore does not, by itself, provide
an estimate of absolute estimation risk. Consequently,
Equation~\ref{eq:estimator-switching-rule} is treated as a
data-driven implementation rule rather than an oracle procedure. No
claim is made that the switched estimator is uniformly as accurate as
the best fixed estimator, or that its excess risk is universally
$\mathcal O(n^{-1/2})$.

The fixed-estimator and switching results are therefore reported
jointly. This permits direct assessment of whether the variability-based
selection rule improves or degrades the MCIR ranking in a particular
sample-size and conditioning-dimension regime. The repeated-subsampling
analysis further evaluates how frequently each estimator is selected
and how the selected ranking compares with the corresponding full-data
reference.

The repeated-subsampling results in
Appendix Table~\ref{tab:estimator_switching_all} provide mixed evidence for this
rule. Switching does not uniformly dominate either fixed estimator and
its behaviour changes with both $n$ and $|\Phi(i)|$. It is also the
most expensive alternative because every compatible candidate is
recomputed across bootstrap samples before selection. At $n=3000$, the
reported switching runtime ranges from $82.000$ to $296.926$ seconds,
compared with less than $0.06$ seconds for the fixed Gaussian--copula
implementation and $9.048$--$32.963$ seconds for fixed $k$NN. The rule
is therefore retained as a calibration heuristic, with fixed-estimator
results and runtime reported alongside it.

\paragraph{Bootstrap configuration.}
\label{app:bootstrap-configuration}

All estimator candidates compatible with a given local block use the
same observations, model-output representation, and conditioning
neighbourhood. Estimator switching follows
Equation~\ref{eq:estimator-switching-rule}; it is not
redefined here.

Bootstrap resampling is performed on the explanation observations. An
ordinary observation-level bootstrap is used for independent data,
whereas temporal, spatial, clustered, or otherwise dependent data
require a structure-preserving resampling scheme. The reported protocol
uses
\begin{equation}
\medmath{
B_{\mathrm{rank}}=50,
\qquad
B_{\mathrm{est}}=200,
}
\label{eq:appendix-bootstrap-counts}
\end{equation}
for ranking-stability and estimator-sensitivity analyses, respectively.
Estimator-switching ablations use 30 independent repetitions.
Experiment-specific departures are recorded in
Table~\ref{tab:experiment_protocol}. Fixed-estimator results are
reported alongside the switched result because bootstrap standard error
measures sampling variability rather than estimator bias.

\subsection{Implementation Details for Auto-\texorpdfstring{$\Phi$}{Phi}}
\label{sec:theory1.8}
\label{c.5}

Auto-$\Phi$ follows the shared-size criterion in
Equations~\ref{eq:auto-phi-ranking-dispersion}--%
\ref{eq:auto-phi-objective}; the objective is not redefined here.
For each feature $i$, dependence screening produces an ordered candidate
pool $\mathcal S_i$ with $|\mathcal S_i|\leq M_{\max}$. For every
$r\in\mathcal R$, the first $r$ candidates form $\Phi_{i,r}$. The same
candidate size $r$ is applied to every feature when computing each
bootstrap MCIR vector $\widehat{\bm C}^{(b)}_r$.
The implementation then:

\begin{enumerate}[leftmargin=*,nosep]
    \item computes the global bootstrap ranking dispersion
    $\widehat V(r)$;

    \item averages residual conditional dependence across target
    features to obtain $\widehat B(r)$ and normalizes it by
    $\widehat B(0)+10^{-12}$;

    \item evaluates
    $\widehat V(r)+\lambda_B\widetilde B(r)
    +\lambda_Dr/M_{\max}$; and

    \item selects the smallest minimizing $r$ and sets
    $\widehat\Phi(i)=\Phi_{i,r^\star}$ for every feature.
\end{enumerate}
The reported default is
$(\lambda_B,\lambda_D)=(1,0.10)$. Penalty sensitivity is evaluated in
Table~\ref{tab:auto-phi-penalty-sensitivity}. This finite-sample
criterion depends on the candidate screening, ordering, estimator,
bootstrap design, and penalty values; it is not claimed to recover a
population-optimal neighbourhood.

\paragraph{Computational Cost.}
Let $B$ be the number of bootstrap replicates,
$A=\max_i|\mathcal A_i|$, and let
$T_{\mathrm{CMI}}(n,r)$ denote the cost of one MI/CMI fit with
conditioning dimension $r$. After candidate screening, evaluating the
nested neighbourhoods has worst-case cost
\begin{equation}
\medmath{
\mathcal O\!\left(
kBA(M_{\max}+1)
T_{\mathrm{CMI}}(n,M_{\max})
\right).
}
\label{eq:appendix-auto-phi-cost}
\end{equation}
This bound counts every candidate size, bootstrap replicate, compatible
estimator, and target feature. It does not include the cost of the
initial dependence screening. The calibration cost may be reused when
the selected neighbourhoods are subsequently applied to repeated MCIR
evaluations on the same feature structure.
\section{Role and Misspecification of the Neighbourhood}
\label{sec:phi_role}

Every MCIR interpretation is relative to the selected neighbourhood
$\Phi(i)$. Local conditioning limits the dimension of the CMI problem,
but its adequacy depends on the dependence structure exposed by the
screening procedure.

\paragraph{Under-conditioning.}
If $\Phi(i)$ omits a relevant dependent feature, shared predictive
information may remain in $U_i$. Because CMI is not monotone in its
conditioning variables, the direction and magnitude of the resulting
change cannot be specified without additional distributional
assumptions.

\paragraph{Over-conditioning.}
Including unnecessary variables increases the conditioning dimension
and can increase finite-sample bias, variance, and computational cost.
Conditioning on colliders or other structurally inappropriate variables
may also change the dependence being measured.

\paragraph{Higher-order dependence.}
Pairwise correlation screening may miss interaction-only dependence.
The controlled example in
Appendix~\ref{app:failure_case_higher_order} demonstrates this failure
mode. Distance correlation, HSIC screening, mutual-$k$NN graphs, or
interaction-aware candidate generation may provide richer candidate
pools, but each introduces its own estimation and computational
trade-offs.

The neighbourhood is therefore a transparent modelling and tuning
choice, not a uniquely recoverable ``true'' dependence set. MCIR
removes redundancy only to the extent that the selected block captures
the dependence relevant to the model output $Q$.

\section{Latent Confounding and Distribution Shift}
\label{latent2}

MCIR quantifies statistical dependence with the fixed model output and
does not recover causal effects. If an unobserved variable $L$ affects
both $X_i$ and $Q$, then
$I(Q;X_i\mid X_{\Phi(i)})$ may remain positive even when $X_i$ has no
direct causal effect. Conditioning on observed proxies for $L$ may
reduce this association but does not generally identify or adjust for
the latent common cause.

Distribution shift may change $P_{X,Q}$, the screened candidate pools,
the selected neighbourhoods, and estimator behaviour. Bootstrap
uncertainty and the lightweight fidelity profile measure selected forms
of empirical variation; they do not guarantee robustness under
arbitrary shift. MCIR scores should therefore be interpreted as unique
predictive information relative to the observed variables, selected
neighbourhood, model, and evaluation distribution.
\section{Supplementary Details for Lightweight Fidelity}
\label{subsec:lw_fidelity2}

This appendix supplements the operational fidelity analysis in
Section~\ref{subsec:lw_fidelity}. The lightweight procedure reduces the
number of observations from $N$ to $n<N$ and may also reduce the
model-output dimension from $d$ to $d'\leq d$. These operations are
distinct: $N$ and $n$ determine the numbers of atoms in the empirical
measures, whereas $d$ and $d'$ determine their ambient representation
spaces.

The fidelity profile and its componentwise acceptance criterion are
defined once in Definition~\ref{def:lightweight-fidelity-profile} and
Equation~\ref{eq:lightweight-acceptance-criterion}. The following
subsections provide interpretation, a counterexample, and implementation
details without redefining those objects.

\subsection{Separation of Representation and Attribution Discrepancies}

The representation-level discrepancy is
\begin{equation}
\medmath{
\widehat d_{\mathrm{repr}}
=
\begin{cases}
W_p(\widehat\mu_N,\widehat\nu_n),
&d'=d,\\[1mm]
\widetilde W_p(\widehat\nu_n,\widehat\mu_N),
&d'<d,
\end{cases}
}
\label{eq:appendix-representation-discrepancy}
\end{equation}
where $\widetilde W_p$ is the Cai--Lim augmented
projection/embedding distance defined in
Equation~\ref{eq:augmented-wasserstein}.

This quantity compares only the marginal output laws. MCIR instead
depends on the joint feature--output laws $P_{X,Z}$ and $P_{X,Z'}$.
Consequently, representation-level closeness alone does not control
MCIR score or ranking disagreement.

To see this, let $X_1$ and $X_2$ be independent Bernoulli random
variables with the same success probability, and define
\begin{equation}
\medmath{
Z=X_1,
\qquad
Z'=X_2.
}
\label{eq:appendix-fidelity-counterexample}
\end{equation}
Then $P_Z=P_{Z'}$, so
$\medmath{d_{\mathrm{repr}}(\mathcal U,\mathcal U')=0}$ in common
coordinates. Nevertheless, $Z$ is informative about $X_1$ and
independent of $X_2$, whereas $Z'$ is informative about $X_2$ and
independent of $X_1$. The two environments therefore induce different
feature-attribution rankings despite having identical marginal output
laws.

This example shows that no bound of the form
$d_{\mathrm{rank}}\leq A\,d_{\mathrm{repr}}$ can hold without
additional assumptions controlling the joint feature--output laws.
Such assumptions are not imposed in the present manuscript.
Representation and attribution discrepancies are therefore retained as
separate components of the fidelity profile.

The deletion/insertion-curve discrepancy is also an observed outcome of
the ranking and evaluation protocol. It cannot be used as an
independent premise to prove that the same curves or rankings agree.
Accordingly, the four components of
$\mathfrak F_{N,n}$ are reported jointly without treating one component
as a proof of another.

\subsection{Ranking and Head-Set Diagnostics}

Let
$\widehat{\bm C}_N
=(\widehat C_{1,N},\ldots,\widehat C_{k,N})$
and
$\widehat{\bm C}'_n
=(\widehat C'_{1,n},\ldots,\widehat C'_{k,n})$
be the empirical MCIR score vectors in the full and lightweight
environments. Their complete-ranking discrepancy is
\begin{equation}
\medmath{
d_{\mathrm{rank}}
=
\frac{
1-\tau_b(\widehat{\bm C}_N,\widehat{\bm C}'_n)
}{2},
}
\label{eq:appendix-fidelity-rank}
\end{equation}
where $\tau_b$ includes tie correction. Because
$\medmath{\tau_b\in[-1,1]}$, it follows directly that
$\medmath{d_{\mathrm{rank}}\in[0,1]}$. A value of zero denotes complete
agreement, while a value of one denotes complete reversal when no
degenerate tie configuration is present.

The top-$K$ discrepancy is
\begin{equation}
\medmath{
d_{\mathrm{J@K}}
=
1-
\frac{
\left|
\operatorname{TopK}(\widehat{\bm C}_N)
\cap
\operatorname{TopK}(\widehat{\bm C}'_n)
\right|
}{
\left|
\operatorname{TopK}(\widehat{\bm C}_N)
\cup
\operatorname{TopK}(\widehat{\bm C}'_n)
\right|
}.
}
\label{eq:appendix-fidelity-topk}
\end{equation}
The intersection is a subset of the union, so
$\medmath{d_{\mathrm{J@K}}\in[0,1]}$. This diagnostic evaluates exact
feature identities. When correlated predictors are treated as
interchangeable, the corresponding group-level Jaccard diagnostic is
reported separately and does not replace exact Jaccard@$K$.

The two ranking diagnostics answer different questions:
$d_{\mathrm{rank}}$ evaluates the complete ordering, whereas
$d_{\mathrm{J@K}}$ evaluates only the selected ranking head. High
top-$K$ agreement can coexist with low complete-ranking agreement when
the leading feature set is preserved but lower-ranked features are
reordered or tied.

\subsection{Deletion and Insertion-Curve Diagnostic}

Let $\Gamma_N^a(t)$ and $\Gamma_n^{\prime a}(t)$ denote the full and
lightweight evaluation curves under protocol
$a\in\{\mathrm{del},\mathrm{ins}\}$ at fraction $t\in[0,1]$. Their
discrepancy is
\begin{equation}
\medmath{
d_{\mathrm{curve}}
=
\max_{a\in\{\mathrm{del},\mathrm{ins}\}}
\sup_{t\in[0,1]}
\left|
\Gamma_N^a(t)-\Gamma_n^{\prime a}(t)
\right|.
}
\label{eq:appendix-fidelity-curve}
\end{equation}
This is an $L_\infty$ comparison under an identical perturbation
protocol. Thus, the predictive model, test observations, feature
replacement rule, evaluation metric, and deletion or insertion schedule
must be fixed across the two environments. Otherwise, curve differences
could reflect a change in protocol rather than a change in attribution.

When the performance metric is normalized to $[0,1]$,
$\medmath{d_{\mathrm{curve}}\in[0,1]}$. For an unbounded performance
metric, the curve discrepancy retains the scale of that metric and must
be interpreted using a correspondingly scaled tolerance.

\subsection{Operational Acceptance Criterion}

The componentwise rule in
Equation~\ref{eq:lightweight-acceptance-criterion} requires every
prespecified tolerance to be satisfied. Failure of one component is not
offset by a small value of another; thus, poor ranking agreement cannot
be concealed by a small representation discrepancy, or vice versa.

The acceptance criterion is an operational diagnostic, not a
theorem of automatic rank preservation. The ranking, top-$K$, and
curve discrepancies are measured outcomes rather than premises from
which the same agreement is proved. Its tolerances define acceptance
for a stated application and do not establish a universal notion of
lightweight fidelity.

\subsection{Population Shift and Estimator Error}

The difference between the two empirical score vectors contains both
environmental and finite-sample components. For each feature $i$, let
$C_i(P_{X,Z})$ and $C_i(P_{X,Z'})$ denote the corresponding population
MCIR scores. The triangle inequality gives
\begin{equation}
\medmath{
\begin{aligned}
\left|
\widehat C_{i,N}
-
\widehat C'_{i,n}
\right|
\leq{}&
\left|
\widehat C_{i,N}
-
C_i(P_{X,Z})
\right|\\
&+
\left|
C_i(P_{X,Z})
-
C_i(P_{X,Z'})
\right|\\
&+
\left|
C_i(P_{X,Z'})
-
\widehat C'_{i,n}
\right|.
\end{aligned}
}
\label{eq:appendix-fidelity-error-decomposition}
\end{equation}
The first and third terms are estimator errors within the respective
environments. The middle term is the change in the population MCIR
functional caused by changing the joint feature--output law.

Theorem~\ref{thm:rank_stability} controls empirical ranking
perturbations when estimator-specific component-error, denominator, and
population-margin conditions hold. It does not bound the middle term in
Equation~\ref{eq:appendix-fidelity-error-decomposition} using the
marginal output discrepancy alone.

No common $O_P(n^{-1/2})$ rate is assumed for Gaussian--copula, $k$NN,
and plug-in MI/CMI estimators. Any convergence rate for the first or
third term must follow from the regularity conditions of the estimator
used in the corresponding environment. The middle term requires
additional assumptions on the joint laws and is therefore evaluated
empirically through the score and ranking diagnostics.

\subsection{Fidelity-Diagnostic Cost}

Let $T_W(N,n,d,d')$ denote the cost of the ordinary or augmented
Wasserstein calculation and let $T_{\mathrm{curve}}$ denote the cost
of the prescribed deletion/insertion evaluation. Kendall--$\tau_b$,
top-$K$ extraction, and set comparison require at most
$\mathcal O(k\log k)$ time under sorting. The additional cost of
computing the fidelity diagnostics is therefore
\begin{equation}
\medmath{
\mathcal O\!\left(
T_W(N,n,d,d')+k\log k+T_{\mathrm{curve}}
\right).
}
\label{eq:appendix-fidelity-diagnostic-cost}
\end{equation}
This cost is separate from MCIR scoring and Auto-$\Phi$ calibration,
whose operation counts are given in
Equations~\eqref{eq:appendix-fixed-neighbourhood-cost} and
\eqref{eq:appendix-auto-phi-cost}, respectively.

\paragraph{Scope.}
The lightweight fidelity profile does not assert that subsampling is
faithful merely because $n<N$, that the Cai--Lim distance applies
because sample sizes differ, or that marginal output-law similarity
guarantees attribution agreement. Its four components evaluate these
properties separately.

\section{Computational and Statistical Details}
\label{appendix:complexity}

This appendix supplements the computational and statistical analysis in
Section~\ref{sec:complexity-analysis}. To avoid repetition, the
population properties and detailed ratio-perturbation and rank-stability
proofs are not restated here; they are provided in
Appendix~\ref{sec:theory1}. We instead give estimator-dependent
operation counts, regularity conditions, memory requirements, and the
scope of the corresponding statistical conclusions.

Throughout, $s\in\{N,n\}$ denotes the number of observations in the
environment being analyzed, $k$ is the number of features,
$d_Q$ is the dimension of the model-output representation, and
$m_i=|\widehat\Phi(i)|$. The local information-estimation dimension is
\begin{equation}
\medmath{
r_i=d_Q+m_i+1,
\qquad
r_{\max}=d_Q+M_{\max}+1.
}
\label{eq:appendix-local-dimension}
\end{equation}

\subsection{Operational Complexity}
\label{appendix:operational}

Let $T_{\mathrm{MI}}(s,r)$ denote the cost of one MI or CMI
computation using $s$ observations and joint dimension $r$. For fixed
neighbourhoods, computing the two information components required for
every feature and sorting the resulting scores costs
\begin{equation}
\medmath{
\mathcal O\!\left(
\sum_{i=1}^{k}T_{\mathrm{MI}}(s,r_i)
+
k\log k
\right).
}
\label{eq:appendix-fixed-neighbourhood-cost}
\end{equation}
The constant factor associated with computing both $\widehat U_{i,s}$
and $\widehat J_{i,s}$ is absorbed into the asymptotic notation.

This expression describes scoring after the neighbourhoods and
estimator families have been fixed. It does not include candidate
screening, Auto-$\Phi$ calibration, estimator switching, bootstrap
uncertainty estimation, or the construction of the explanation sample.

The additional cost of bootstrap estimator switching and
Auto-$\Phi$ calibration is given in
Equation~\ref{eq:appendix-auto-phi-cost}. The candidate-screening
cost must be added separately.

For dense pairwise dependence screening,
\begin{equation}
\medmath{
C_{\mathrm{screen}}
=
\mathcal O(sk^2)
}
\label{eq:appendix-dense-screening}
\end{equation}
for statistics computable in one pass over each feature pair.
Approximate or sparse dependence-graph construction may reduce this
cost, but its complexity depends on the specific screening algorithm.

Consequently, end-to-end MCIR complexity is not generally linear
in $k$. Near-linear scoring in $k$ applies only after fixing bounded
local neighbourhoods and excluding dense candidate screening,
bootstrap calibration, and estimator switching.

For the experimental PCIR implementation, rank preprocessing requires
$\mathcal O(ks\log s)$ time and the subsequent dispersion calculations
require $\mathcal O(ks)$. These costs likewise exclude predictive-model
evaluation and explanation-data construction.

\subsection{Estimator-Dependent Costs}
\label{appendix:mi-comparison}

The function $T_{\mathrm{MI}}(s,r)$ depends on the estimator and cannot
be replaced by a single rate valid for every data regime. The principal
costs are summarized in Table~\ref{tab:mi-comparisons}.

\begin{table}[htbp]
\centering
\scriptsize
\setlength{\tabcolsep}{3pt}
\renewcommand{\arraystretch}{0.96}
\caption{Indicative cost of one MI/CMI computation with $s$
observations and local dimension $r$.}
\label{tab:mi-comparisons}
\begin{adjustbox}{width=\linewidth}
\begin{tabular}{lp{3.7cm}p{6.3cm}}
\toprule
\textbf{Estimator} &
\textbf{Indicative time} &
\textbf{Scope and dimensional dependence} \\
\midrule
Gaussian--copula &
$\mathcal O(rs\log s+sr^2+r^3)$ &
Rank transformation, covariance estimation, and matrix
factorization; numerical cost increases with $r$. \\

Tree-based $k$NN &
Typically $\mathcal O(s\log s)$ for fixed low $r$ &
Constants and statistical accuracy deteriorate with dimension;
worst-case neighbour search may approach $\mathcal O(s^2)$. \\

Discrete plug-in &
$\mathcal O(sr)$ for count construction &
Memory and sample requirements depend on the number of occupied joint
cells and may grow rapidly with cardinality and $r$. \\

Kernel method &
Typically $\mathcal O(s^2)$ memory and
$\mathcal O(s^2r)$ or higher time &
Requires Gram matrices unless a low-rank or random-feature
approximation is used. \\
\bottomrule
\end{tabular}
\end{adjustbox}
\end{table}

For the Gaussian--copula implementation, sorting or rank-transforming
$r$ variables costs $\mathcal O(rs\log s)$, covariance construction
costs $\mathcal O(sr^2)$, and factorization costs $\mathcal O(r^3)$.
It is therefore approximately linear in $s$ only when $r$ is treated
as fixed.

Similarly, the commonly stated $\mathcal O(s\log s)$ cost for $k$NN
estimation assumes fixed low dimension and effective spatial indexing.
It is not a dimension-uniform guarantee. Plug-in estimation is
computationally inexpensive for fixed finite alphabets, but the number
of possible joint cells can grow exponentially with the number and
cardinalities of the variables.

\subsection{Estimator-Specific Statistical Conditions}
\label{appendix:statistical}

\begin{assumption}[Estimator-specific regularity]
\label{asm:estimators}
Whenever a finite-sample result invokes an MI or CMI estimator, the
support, smoothness, moment, positivity, and conditioning-dimension
conditions required by that estimator hold. Any component-error,
consistency, concentration, or asymptotic-normality condition is stated
for the estimator and data regime in which it is used. No common rate
is assumed for Gaussian--copula, $k$NN, discrete plug-in, or mixed-data
estimators.
\end{assumption}

The population MCIR score is
\begin{equation}
\medmath{
C_i
=
\frac{2U_i}{U_i+J_i},
\qquad
U_i=I_{P_{X,Q}}(Q;X_i\mid X_{\Phi(i)}),
\qquad
J_i=I_{P_{X,Q}}(Q;X_{\Phi(i)\cup\{i\}}).
}
\label{eq:appendix-statistical-mcir}
\end{equation}
Its empirical version uses the projected ratio in
Equation~\ref{eq:clipped-empirical-mcir}.

If
\begin{equation}
\medmath{
\widehat U_{i,s}\xrightarrow{p}U_i,
\qquad
\widehat J_{i,s}\xrightarrow{p}J_i,
\qquad
U_i+J_i>0,
}
\label{eq:appendix-component-consistency}
\end{equation}
then the continuous-mapping argument in
Appendix~\ref{sec:theory1.6} gives
$\widehat C_{i,s}\xrightarrow{p}C_i$. This transfer result does not
establish consistency of the underlying MI or CMI estimators; that
property must be verified for the estimator and data regime used.

The estimator-specific conditions are stated once with the
corresponding formulas in Appendix~\ref{app:estimators}.
Assumption~\ref{asm:estimators} does not assert a universal
concentration inequality, convergence rate, or central limit theorem
for Gaussian--copula, $k$NN, plug-in, or mixed-data estimators.

\subsection{Conditional Delta-Method Result}

Suppose a particular estimator satisfies the joint central limit
theorem
\begin{equation}
\medmath{
\sqrt{s}
\begin{pmatrix}
\widehat U_{i,s}-U_i\\
\widehat J_{i,s}-J_i
\end{pmatrix}
\Rightarrow
\mathcal N(0,\Sigma_i),
}
\label{eq:appendix-component-clt}
\end{equation}
and $U_i+J_i>0$. For
$\medmath{g(u,j)=2u/(u+j)}$,
\begin{equation}
\medmath{
\nabla g(U_i,J_i)
=
\frac{2}{(U_i+J_i)^2}
\begin{pmatrix}
J_i\\-U_i
\end{pmatrix}.
}
\label{eq:appendix-mcir-gradient}
\end{equation}
The multivariate delta method then gives
\begin{equation}
\medmath{
\sqrt{s}
\left(
\widehat C_{i,s}-C_i
\right)
\Rightarrow
\mathcal N
\left(
0,
\nabla g(U_i,J_i)^\top
\Sigma_i
\nabla g(U_i,J_i)
\right).
}
\label{eq:appendix-conditional-delta-method}
\end{equation}
This is a conditional implication: MCIR inherits asymptotic normality
only when the selected component estimators satisfy
Equation~\ref{eq:appendix-component-clt}. It is not asserted for
every estimator in Table~\ref{tab:mi-comparisons}.

\subsection{Memory and Parallelization}
\label{appendix:memory}

If the explanation arrays are stored in memory, their shared storage
cost is
\begin{equation}
\medmath{
\mathcal O\!\left(s(k+d_Q)\right).
}
\label{eq:appendix-data-memory}
\end{equation}
A worker computing one local score requires
$\mathcal O(sr_{\max})$ active data storage in addition to
estimator-specific structures. A Gaussian--copula implementation also
stores an $\mathcal O(r_{\max}^2)$ covariance matrix. Exact kernel
methods may require $\mathcal O(s^2)$ memory, while tree-based $k$NN
methods require the corresponding spatial-index structure.

Once neighbourhoods and estimator choices are fixed, feature-wise
score calculations can be distributed across $p$ workers. Ignoring
load imbalance and communication, the scoring component in
Equation~\ref{eq:appendix-fixed-neighbourhood-cost} becomes
\begin{equation}
\medmath{
\mathcal O\!\left(
\left\lceil\frac{k}{p}\right\rceil
T_{\mathrm{MI}}(s,r_{\max})
\right).
}
\label{eq:appendix-parallel-cost}
\end{equation}
This is an idealized work-distribution bound, not a guarantee of
near-linear speedup. Communication, scheduling, shared-memory
contention, unequal neighbourhood sizes, and estimator-specific setup
costs remain.

\subsection{Choice of Lightweight Sample Size}

No universal lightweight fraction guarantees accurate MCIR estimation.
The required $n$ depends on the estimator, variable types, local
dimension, strength of dependence, denominator size, ranking margins,
and target uncertainty.

In practice, candidate values of $n$ should be evaluated using the
lightweight fidelity profile in
Definition~\ref{def:lightweight-fidelity-profile}, estimator-specific
bootstrap intervals, and repeated-subsampling ranking diagnostics. The
selected value should satisfy the prespecified representation, ranking,
top-$K$, and curve tolerances while meeting the available computational
budget.

Empirical success at a particular sampling fraction is
dataset-specific and should not be presented as a general sample
efficiency guarantee. Likewise, reducing the number of observations
does not by itself establish that the lightweight environment preserves
the joint feature--output dependence required by MCIR.
\section{Redundancy Robustness Experiments}
\label{sec:appendix_experiments}

This section evaluates attribution behaviour after controlled injection
of near-duplicate predictors. The designated anchor is known from the
experimental construction; MCIR does not infer which member of an exact
duplicate family is the ``original'' from statistical dependence alone.
We compare MCIR with TreeSHAP and permutation feature importance (PFI)
on HouseEnergy-Sim and UCI HAR. These experiments provide finite-sample
sensitivity evidence under the stated neighbourhood and tie-handling
protocols; they are not direct empirical proofs of exact population
redundancy collapse.

\subsection{Experimental Protocol}

For $r\in\{1,2,4,8\}$, extended to $r\in\{16,32\}$ for the ranking
stress test, we inject noisy copies of a prespecified anchor feature.
For a designated feature $X_i$, the $s$th copy is
\begin{equation}
\medmath{
X_i^{(s)}
=
X_i+\varepsilon_s,
\qquad
\varepsilon_s
\sim
\mathcal N
\left(
0,
\{0.02\,\widehat{\mathrm{sd}}(X_i)\}^{2}
\right).
}
\label{eq:appendix-duplicate-construction}
\end{equation}
The observed anchor--copy correlations lie between approximately
$0.98$ and $0.995$. Thus, the experiment evaluates noisy
near-duplicates rather than exact deterministic copies.

The duplicate family is inserted before neighbourhood construction,
and every method receives the same augmented feature space at a given
$r$. The random-forest configurations, data partitions, and repetition
counts are reported in Table~\ref{tab:experiment_protocol}. Each fitted
model is held fixed while the attribution methods are evaluated on the
same explanation sample.

We report the unnormalized Kendall instability
\begin{equation}
\medmath{
\Delta_{\mathrm{rank}}(r)
=
1-
\tau_b
\left(
\pi^{(0)},
\pi_{\mathrm{orig}}^{(r)}
\right),
}
\label{eq:appendix-redundancy-rank}
\end{equation}
where $\pi_{\mathrm{orig}}^{(r)}$ is restricted to the original
features shared by the clean and augmented datasets. Hence
$\Delta_{\mathrm{rank}}\in[0,2]$. This differs from the normalized
lightweight-fidelity discrepancy
$d_{\mathrm{rank}}=(1-\tau_b)/2$.

Let $A_i^{(r)}$ denote the non-negative attribution assigned to the
designated anchor, $\mathcal F^{(r)}$ the augmented feature set, and
$\mathcal D^{(r)}$ the duplicate family. We define
\begin{equation}
\medmath{
\operatorname{OrigMass}(r)
=
\frac{A_i^{(r)}}
     {\sum_{j\in\mathcal F^{(r)}}A_j^{(r)}},
\qquad
\operatorname{FamilyShare}(r)
=
\frac{A_i^{(r)}}
     {\sum_{j\in\mathcal D^{(r)}}A_j^{(r)}}.
}
\label{eq:appendix-redundancy-allocation}
\end{equation}
These are protocol-specific anchor-retention diagnostics and are not
invariant to relabelling exact duplicates.

\subsection{HouseEnergy-Sim}
\label{app:house-redundancy}

HouseEnergy-Sim provides a controlled benchmark with a prespecified
dominant feature, \emph{Base\_load}, which serves as the duplicate-family
anchor. Table~\ref{tab:house_combined} reports global anchor retention
and within-family allocation.

\begin{table}[htbp]
\centering
\setlength{\tabcolsep}{3pt}
\renewcommand{\arraystretch}{0.96}
\caption{HouseEnergy-Sim redundancy diagnostics. OrigMass is the
normalized attribution assigned to the designated anchor; FamilyShare
is its share within the duplicate family.}
\label{tab:house_combined}
\begin{tabular}{c|ccc|ccc}
\toprule
& \multicolumn{3}{c|}{\textbf{OrigMass}}
& \multicolumn{3}{c}{\textbf{FamilyShare}}\\
\cmidrule(r){2-4}\cmidrule(l){5-7}
$r$ & MCIR & TreeSHAP & PFI & MCIR & TreeSHAP & PFI\\
\midrule
1 & $0.635\pm0.018$ & $0.414\pm0.032$ & $0.475\pm0.028$
  & 0.908 & 0.784 & 0.834\\
2 & $0.610\pm0.019$ & $0.312\pm0.035$ & $0.398\pm0.031$
  & 0.904 & 0.687 & 0.776\\
4 & $0.568\pm0.021$ & $0.178\pm0.038$ & $0.272\pm0.034$
  & 0.887 & 0.503 & 0.651\\
8 & $0.509\pm0.023$ & $0.052\pm0.019$ & $0.187\pm0.029$
  & 0.852 & 0.150 & 0.387\\
\bottomrule
\end{tabular}
\end{table}
Between $r=1$ and $r=8$, MCIR OrigMass decreases from $0.635$ to
$0.509$, compared with $0.414$ to $0.052$ for TreeSHAP and $0.475$ to
$0.187$ for PFI. At $r=8$, MCIR retains $0.852$ of the duplicate-family
mass on the designated anchor, whereas TreeSHAP and PFI retain $0.150$
and $0.387$, respectively. Figure~\ref{fig:mass_house} visualizes these
differences.
These results show stronger anchor retention for MCIR under the stated
ordered near-duplicate protocol. They do not prove exact conditional
redundancy collapse or show that MCIR can identify an intrinsically
``original'' feature without the external anchor designation.

\subsubsection{Original-Feature Ranking Stability}

Table~\ref{tab:rank_house} reports
$\Delta_{\mathrm{rank}}=1-\tau_b$ for rankings restricted to the
original features. MCIR preserves the same original-feature ordering in
the reported runs, whereas TreeSHAP and PFI exhibit additional pairwise
reversals as $r$ increases. Values above one are valid because this
experiment reports the unnormalized discrepancy and therefore permits
negative Kendall association.

\begin{table}[t]
\centering
\scriptsize
\setlength{\tabcolsep}{4pt}
\renewcommand{\arraystretch}{0.96}
\caption{Original-feature ranking instability on HouseEnergy-Sim:
$\Delta_{\mathrm{rank}}=1-\tau_b$; lower is better.}
\label{tab:rank_house}
\begin{tabular}{c|ccc}
\toprule
$r$ & MCIR & TreeSHAP & PFI\\
\midrule
0  & 0.0000 & 0.0000 & 0.0000\\
1  & 0.0000 & 0.6667 & 0.2500\\
2  & 0.0000 & 0.6667 & 0.6667\\
4  & 0.0000 & 0.6667 & 0.6667\\
8  & 0.0000 & 0.6667 & 0.6667\\
16 & 0.0000 & 0.6667 & 1.3333\\
32 & 0.0000 & 1.3333 & 1.3333\\
\bottomrule
\end{tabular}
\end{table}
\begin{figure}
\centering
\includegraphics[width=\linewidth]{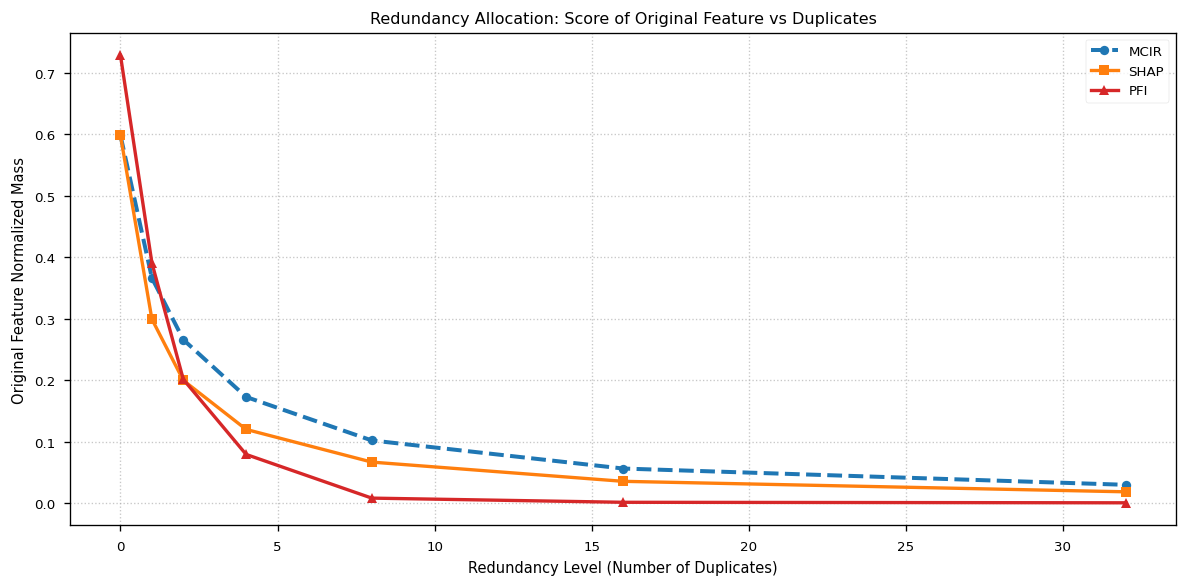}
\caption{OrigMass under increasing redundancy in HouseEnergy-Sim.
MCIR shows a smaller decline than TreeSHAP and PFI under the reported
protocol.}
\label{fig:mass_house}
\vspace{-3mm}
\end{figure}
These observations provide finite-sample sensitivity evidence. They do
not verify Theorem~\ref{thm:rank_stability}, whose conclusion is
relative to the population MCIR ranking and requires the stated
component-error, denominator-separation, and margin conditions.

\subsection{UCI HAR}
\label{app:har-redundancy}

Table~\ref{tab:mass_har} applies the same anchor-retention diagnostic to
UCI HAR. From $r=1$ to $r=8$, MCIR OrigMass changes from $0.639$ to
$0.503$, compared with $0.439$ to $0.055$ for TreeSHAP and $0.466$ to
$0.189$ for PFI.

\begin{table}[htbp]
\centering
\scriptsize
\setlength{\tabcolsep}{3pt}
\renewcommand{\arraystretch}{0.96}
\caption{OrigMass under duplicate injection on UCI HAR.}
\label{tab:mass_har}
\begin{tabular}{c|ccc}
\toprule
$r$ & MCIR & TreeSHAP & PFI\\
\midrule
1 & $0.639\pm0.016$ & $0.439\pm0.034$ & $0.466\pm0.027$\\
2 & $0.614\pm0.018$ & $0.331\pm0.036$ & $0.389\pm0.030$\\
4 & $0.572\pm0.020$ & $0.189\pm0.041$ & $0.265\pm0.033$\\
8 & $0.503\pm0.022$ & $0.055\pm0.021$ & $0.189\pm0.031$\\
\bottomrule
\end{tabular}
\end{table}

Thus, MCIR retains more attribution on the designated anchor under this
protocol. As in HouseEnergy-Sim, the result depends on the chosen
anchor, neighbourhood construction, estimator, and deterministic tie
rule. Figure~\ref{fig:mass_combined} compares the trajectories from the
two datasets. Their similar qualitative behaviour provides
cross-dataset sensitivity evidence but not a distribution-free
generality result.

\begin{figure}[htbp]
\centering
\includegraphics[width=\linewidth,height=6cm]{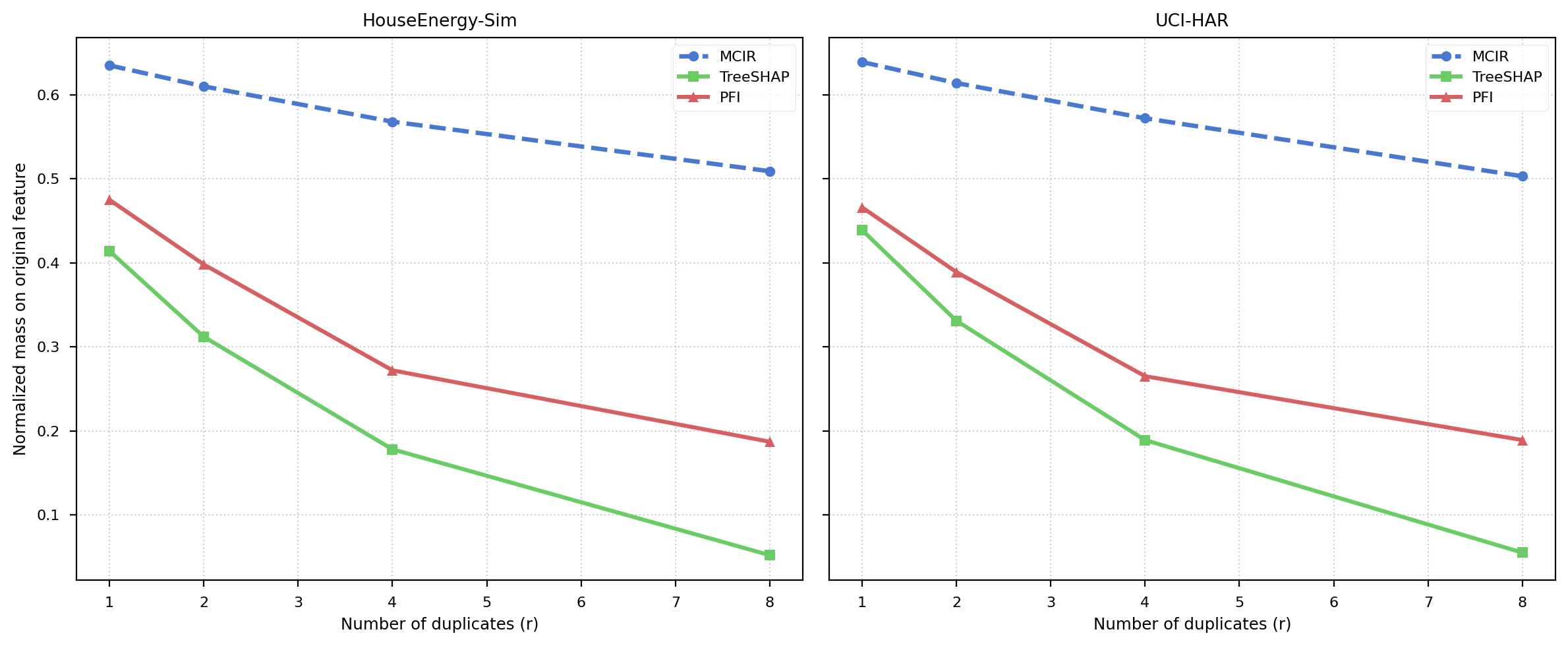}
\caption{OrigMass under redundancy on HouseEnergy-Sim and UCI HAR.
MCIR retains more mass on the designated anchor than TreeSHAP and PFI
under the evaluated constructions.}
\label{fig:mass_combined}
\end{figure}

\subsection{Cross-Dataset Degradation Analysis}

For HouseEnergy-Sim, the change in OrigMass from $r=1$ to $r=8$ is
$-0.126$ for MCIR, $-0.362$ for TreeSHAP, and $-0.287$ for PFI.
These repeated-seed summaries are treated as descriptive effect
estimates. We do not report a hypothesis-test $p$-value because the
current reproducibility record does not specify a complete inferential
procedure for these comparisons.

\subsection{Runtime Comparison}

Runtime is evaluated at $r=8$ using 400 observations and 20 augmented
features. Every method is executed ten times on identical hardware and
explanation data.

\begin{table}[htbp]
\centering
\scriptsize
\setlength{\tabcolsep}{5pt}
\renewcommand{\arraystretch}{0.96}
\caption{Runtime at $r=8$ using 400 observations and 20 features.
Values are means over ten runs.}
\label{tab:runtime-redundancy}
\begin{tabular}{lcc}
\toprule
\textbf{Method} & \textbf{Time (s)} & \textbf{Relative time}\\
\midrule
MCIR (lightweight) & 0.0341 & $1.0\times$\\
TreeSHAP           & 0.9114 & $26.7\times$\\
PFI (5 repeats)    & 0.2496 & $7.3\times$\\
\bottomrule
\end{tabular}
\end{table}

In this implementation, MCIR is faster than TreeSHAP and repeated PFI
at the reported sample and feature counts. MCIR avoids repeated model
calls after $Q$ has been obtained, whereas the comparison
implementations require additional model evaluations. These wall-clock
measurements are specific to the hardware, software, estimator, and
neighbourhood configuration and do not establish a universal runtime
ordering.

\subsection{Summary of the Redundancy Experiments}

Across the evaluated datasets, MCIR retains more attribution on the
prespecified anchor and better preserves the original-feature ordering
than the two comparison implementations.
Figure~\ref{fig:duplication_concept} summarizes the observed allocation
patterns.

\begin{figure}[t]
\centering
\includegraphics[width=\linewidth]{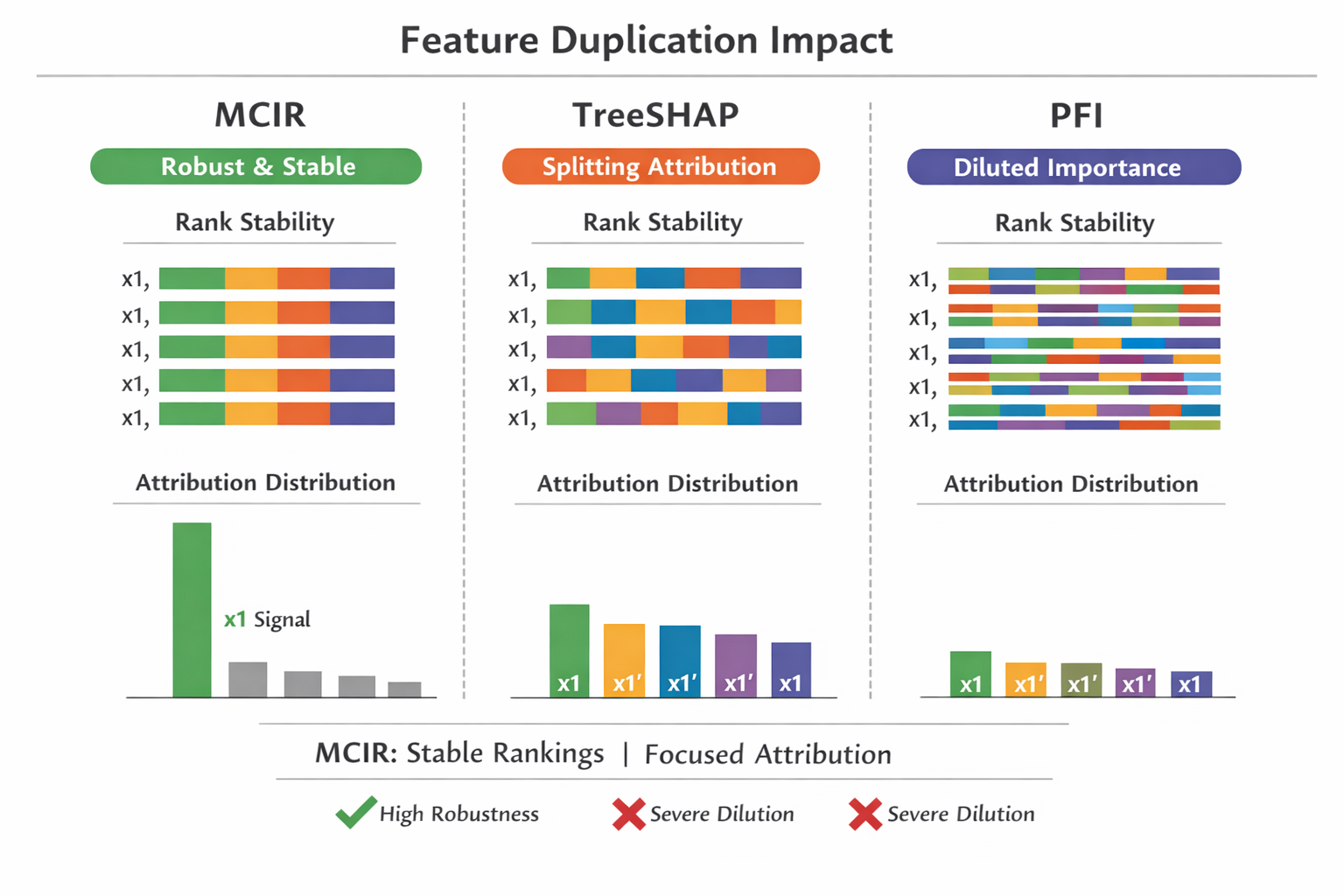}
\caption{Conceptual summary of allocation under the evaluated
near-duplicate protocol. Left: MCIR retains more score on the
designated anchor. Middle: TreeSHAP distributes score across correlated
copies. Right: the marginal effect measured by PFI can decrease while
predictive substitutes remain. This diagram summarizes the experiment
and is not a general theorem about the methods.}
\label{fig:duplication_concept}
\end{figure}

The allocation results are consistent with MCIR's use of conditional
information relative to selected neighbourhoods. They do not imply that
MCIR can distinguish statistically identical copies without an
asymmetric neighbourhood, deterministic tie rule, or external anchor.
The supported conclusion is therefore limited to the reported
near-duplicate construction: MCIR exhibits smaller anchor-mass
degradation and fewer original-feature rank changes.

\section{Latent-Confounding Sensitivity Experiment}
\label{latent}

We separately examine a synthetic common-cause setting because latent
confounding is distinct from observed duplicate injection. Let
$L\sim\mathcal N(0,1)$ be unobserved and define
\begin{equation}
\medmath{
X_1=aL+E_1,
\qquad
X_2=aL+E_2,
}
\label{eq:appendix-latent-features}
\end{equation}
where $E_1$ and $E_2$ are independent noise variables. An independent
predictor $X_3$ and noise feature $X_5$ are also included, and
\begin{equation}
\medmath{
Q=\beta_1X_1+\beta_3X_3+\varepsilon.
}
\label{eq:appendix-latent-output}
\end{equation}
The parameter $a\in[0,1]$ controls the observed dependence induced by
the latent common cause. For the proxy feature $X_2$, MCIR is evaluated
with $X_1\in\Phi(2)$.

\begin{figure}[t]
\centering
\includegraphics[width=\linewidth]{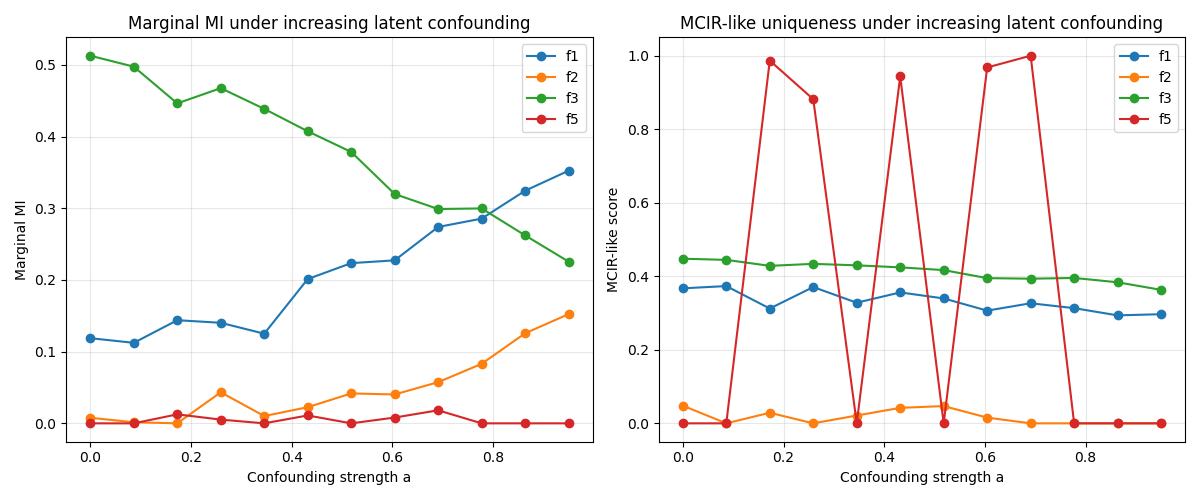}
\caption{Synthetic common-cause sensitivity experiment.
\textbf{Left:} marginal MI with $Q$ as $a$ increases.
\textbf{Right:} MCIR under the stated neighbourhoods. Conditioning on
$X_1$ reduces the additional statistical information assigned to
$X_2$. This observed-variable dependence result does not recover the
latent variable or establish causal importance.}
\label{fig:latent_confounding}
\end{figure}

Figure~\ref{fig:latent_confounding} shows that conditioning on the
observed correlated feature $X_1$ reduces the additional statistical
information assigned to $X_2$ in this construction. It does not show
that MCIR generally adjusts for unobserved confounding: MCIR operates
only on observed variables, cannot recover $L$, and does not identify
causal effects.

\section{Predictive-Risk and Variance-Allocation Comparisons}
\label{app:loco-shapley-comparison}

This section supplements
Section~\ref{subsec:loco-comparison} with the complete cross-method
agreement results. Because MCIR, LOCO, regional LOCO, and Shapley
effects define different importance functionals, agreement is reported
descriptively and is not treated as an attribution-correctness metric.

\begin{figure*}[t]
\centering
\includegraphics[width=\linewidth]{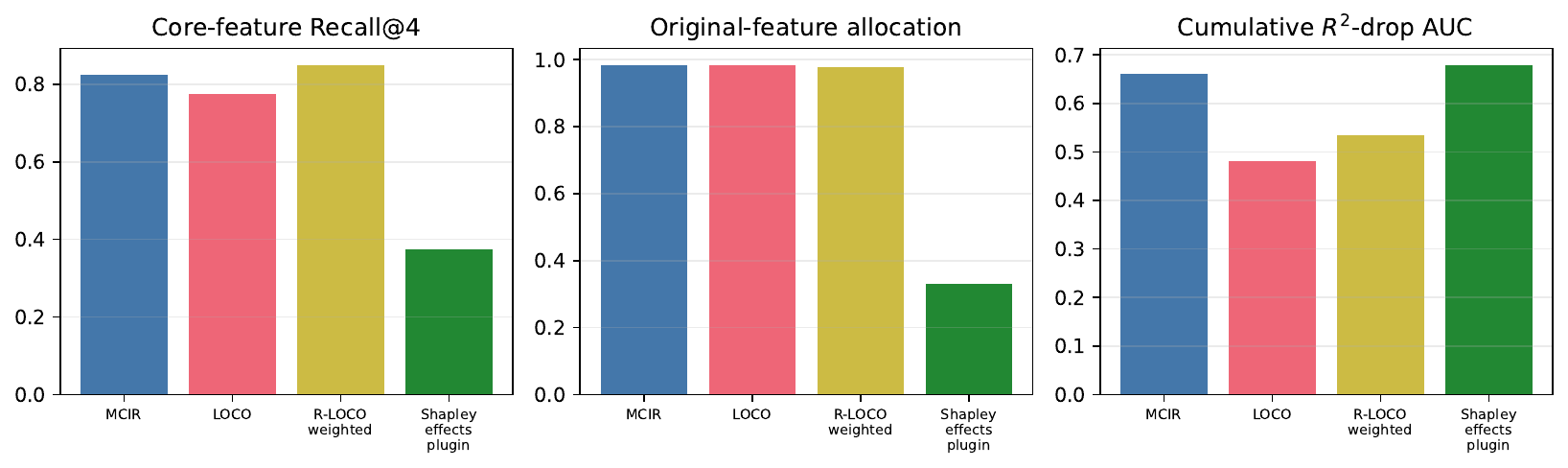}
\caption{HouseEnergy-Sim comparison across distinct attribution
functionals. Left: recovery of the four prespecified functional
drivers. Middle: normalized positive attribution retained on those
drivers. Right: cumulative $R^2$-drop AUC under the common
training-mean replacement protocol. Values are means over ten
independent runs; the numerical uncertainty summaries are reported in
Table~\ref{tab:loco-shapley-houseenergy}.}
\label{fig:loco-shapley-houseenergy}
\end{figure*}

\begin{table}[t]
\centering
\scriptsize
\setlength{\tabcolsep}{3pt}
\renewcommand{\arraystretch}{0.96}
\caption{Mean cross-method ranking agreement over ten independent
HouseEnergy-Sim runs. The statistics compare distinct importance
targets and do not measure attribution correctness.}
\label{tab:loco-shapley-agreement}
\begin{tabular}{llccc}
\toprule
\textbf{Method A} & \textbf{Method B} &
$\boldsymbol{\tau_b}$ & $\boldsymbol{\rho}$ & \textbf{Jaccard@4} \\
\midrule
MCIR & LOCO                  &  0.436 &  0.568 & 0.560 \\
MCIR & Region-weighted LOCO  &  0.449 &  0.594 & 0.667 \\
MCIR & Shapley-effects plug-in & $-0.071$ & $-0.122$ & 0.303 \\
LOCO & Region-weighted LOCO  &  0.764 &  0.865 & 0.720 \\
LOCO & Shapley-effects plug-in & $-0.262$ & $-0.377$ & 0.205 \\
Region-weighted LOCO & Shapley-effects plug-in
                              & $-0.267$ & $-0.377$ & 0.195 \\
\bottomrule
\end{tabular}
\end{table}

The relatively high LOCO--regional-LOCO agreement follows from their
shared predictive-loss target. In contrast, the lower agreement with
the Shapley-effects plug-in reflects both its variance-allocation target
and finite-sample conditional-mean approximation. Region-level outputs,
individual feature scores, and all repetition-level metrics are retained
in the accompanying reproducibility archive.

\section{High-Dimensional CMI Ablation}
\label{app:highdim-cmi}

The complete conditioning-dimension results underlying
Figure~\ref{fig:highdim-cmi-ablation} are reported in
Table~\ref{tab:highdim-cmi-ablation}. Each cell is averaged over eight
independent repetitions. Pearson screening is computed from the
corresponding 5000-observation reference sample and is held fixed when
the smaller explanation samples are evaluated. The experiment
therefore isolates MI/CMI estimation sensitivity conditional on the
screened neighbourhoods; it does not include uncertainty from
re-estimating the screening graph.
\begin{table*}[t]
\centering
\scriptsize
\setlength{\tabcolsep}{2.5pt}
\caption{High-dimensional CMI ablation on HouseEnergy-Sim. Values are
means over eight repetitions. Agreement and score MAE are computed
against an estimator-specific reference based on 5000 observations;
the results therefore characterize finite-sample sensitivity rather
than population accuracy. Negative-component frequency is the
proportion of unprojected $k$NN MI/CMI components below zero;
Gaussian--copula components are non-negative by implementation.
Decreasing score MAE at large $|\Phi|$ should be interpreted jointly
with rank agreement because it may reflect score compression rather
than improved ranking recovery.}
\label{tab:highdim-cmi-ablation}
\begin{tabular}{rrlccccc}
\toprule
$n$ & $|\Phi|$ & Estimator & Score MAE & $\tau_b$ & J@4 & Negative components & Runtime (s) \\
\midrule
500 & 1 & copula & 0.163 & 0.516 & 0.295 & 0.000 & 0.08 \\
500 & 1 & knn & 0.204 & 0.136 & 0.268 & 0.365 & 0.26 \\
500 & 2 & copula & 0.101 & 0.441 & 0.533 & 0.000 & 0.10 \\
500 & 2 & knn & 0.149 & 0.134 & 0.281 & 0.380 & 0.33 \\
500 & 3 & copula & 0.067 & 0.394 & 0.433 & 0.000 & 0.12 \\
500 & 3 & knn & 0.115 & 0.118 & 0.295 & 0.333 & 0.37 \\
500 & 5 & copula & 0.037 & 0.262 & 0.600 & 0.000 & 0.14 \\
500 & 5 & knn & 0.088 & 0.155 & 0.310 & 0.294 & 0.45 \\
500 & 8 & copula & 0.010 & 0.177 & 0.319 & 0.000 & 0.20 \\
500 & 8 & knn & 0.027 & 0.291 & 0.410 & 0.214 & 0.68 \\
500 & 12 & copula & 0.004 & 0.185 & 0.329 & 0.000 & 0.29 \\
500 & 12 & knn & 0.008 & 0.380 & 0.400 & 0.185 & 1.05 \\
1000 & 1 & copula & 0.135 & 0.573 & 0.386 & 0.000 & 0.10 \\
1000 & 1 & knn & 0.157 & 0.291 & 0.271 & 0.323 & 0.53 \\
1000 & 2 & copula & 0.086 & 0.481 & 0.468 & 0.000 & 0.12 \\
1000 & 2 & knn & 0.167 & 0.088 & 0.271 & 0.357 & 0.66 \\
1000 & 3 & copula & 0.056 & 0.461 & 0.560 & 0.000 & 0.14 \\
1000 & 3 & knn & 0.109 & 0.281 & 0.392 & 0.380 & 0.73 \\
1000 & 5 & copula & 0.023 & 0.313 & 0.633 & 0.000 & 0.17 \\
1000 & 5 & knn & 0.063 & 0.247 & 0.358 & 0.312 & 0.97 \\
1000 & 8 & copula & 0.006 & 0.228 & 0.295 & 0.000 & 0.25 \\
1000 & 8 & knn & 0.015 & 0.399 & 0.433 & 0.250 & 1.43 \\
1000 & 12 & copula & 0.002 & 0.182 & 0.262 & 0.000 & 0.33 \\
1000 & 12 & knn & 0.006 & 0.545 & 0.583 & 0.208 & 2.15 \\
2000 & 1 & copula & 0.134 & 0.640 & 0.402 & 0.000 & 0.15 \\
2000 & 1 & knn & 0.139 & 0.319 & 0.460 & 0.378 & 1.08 \\
2000 & 2 & copula & 0.086 & 0.558 & 0.436 & 0.000 & 0.15 \\
2000 & 2 & knn & 0.095 & 0.309 & 0.467 & 0.367 & 1.34 \\
2000 & 3 & copula & 0.053 & 0.543 & 0.576 & 0.000 & 0.17 \\
2000 & 3 & knn & 0.065 & 0.446 & 0.583 & 0.307 & 1.83 \\
2000 & 5 & copula & 0.018 & 0.447 & 0.583 & 0.000 & 0.28 \\
2000 & 5 & knn & 0.039 & 0.477 & 0.600 & 0.359 & 2.36 \\
2000 & 8 & copula & 0.004 & 0.338 & 0.550 & 0.000 & 0.30 \\
2000 & 8 & knn & 0.013 & 0.442 & 0.483 & 0.297 & 3.40 \\
2000 & 12 & copula & 0.001 & 0.311 & 0.500 & 0.000 & 0.44 \\
2000 & 12 & knn & 0.003 & 0.688 & 0.517 & 0.216 & 5.98 \\
\bottomrule
\end{tabular}
\end{table*}

The Gaussian--copula implementation returns non-negative component
estimates by construction. The negative-component column in
Table~\ref{tab:highdim-cmi-ablation} therefore records unprojected
negative estimates for the $k$NN implementation and should not be
interpreted as a symmetric estimator comparison. The reference ranking
is estimator-specific and empirical, rather than a known population
ordering.
\section{Additional Stability and Faithfulness Analysis}
\label{sec:appendix_deletion}

\subsection{Experimental Protocol}

The additional stability analysis uses a HouseEnergy-style dataset
with 4000 observations and 41 features, including core energy
variables, correlated copies, and noise variables. The data are split
into 3000 training and 1000 test observations. A Random Forest regressor
with 200 trees and maximum depth 12 is used. Correlation grouping at
$|\rho|\geq0.85$ produces 35 evaluation groups.

All methods use the same trained model, partitions, feature space, test
set, deletion schedule, and lightweight subsampling protocol. Exact
Jaccard@$30$ measures identity-level overlap, while
Group-Jaccard@$30$ measures overlap after mapping selected features to
the prespecified correlation groups. Deletion curves and rank overlays
provide separate predictive-faithfulness and ranking diagnostics.

\subsection{Exact and Group-Level Stability}

Table~\ref{tab:appendix-jaccard-stability} reports the complete
identity- and group-level results in a single compact table.

\begin{table}[htbp]
\centering
\scriptsize
\setlength{\tabcolsep}{4pt}
\renewcommand{\arraystretch}{0.96}
\caption{Exact and group-level Jaccard@$30$ under subsampling.}
\label{tab:appendix-jaccard-stability}
\label{tab:exact_jaccard_appendix}
\label{tab:group_jaccard_appendix}
\begin{tabular}{lccc}
\toprule
\textbf{Method} & \textbf{20\%} & \textbf{40\%} & \textbf{60\%}\\
\midrule
\multicolumn{4}{c}{\textit{Exact Jaccard@$30$}}\\
MCIR & 0.585 & 0.624 & 0.627\\
PCIR & 0.673 & 0.710 & 0.767\\
SHAP & 0.644 & 0.698 & 0.747\\
HSIC & \textbf{0.674} & \textbf{0.711} & \textbf{0.787}\\
\midrule
\multicolumn{4}{c}{\textit{Group-Jaccard@$30$}}\\
MCIR & \textbf{0.641} & \textbf{0.671} & 0.655\\
PCIR & 0.617 & 0.668 & 0.725\\
SHAP & 0.574 & 0.636 & 0.694\\
HSIC & 0.609 & 0.652 & \textbf{0.742}\\
\bottomrule
\end{tabular}
\end{table}

MCIR has lower exact feature overlap than PCIR, SHAP, and HSIC at all
three fractions. Under the prespecified group-equivalence rule, MCIR
has the highest Group-Jaccard at $20\%$ and $40\%$, while HSIC has the
highest value at $60\%$. Group membership is based on a correlation
threshold and does not establish functional or causal equivalence.
Exact and group-level agreement are therefore complementary diagnostics
rather than interchangeable measures.

\subsection{Deletion-Based Faithfulness}

\begin{figure}[htbp]
\centering
\includegraphics[width=0.75\linewidth]{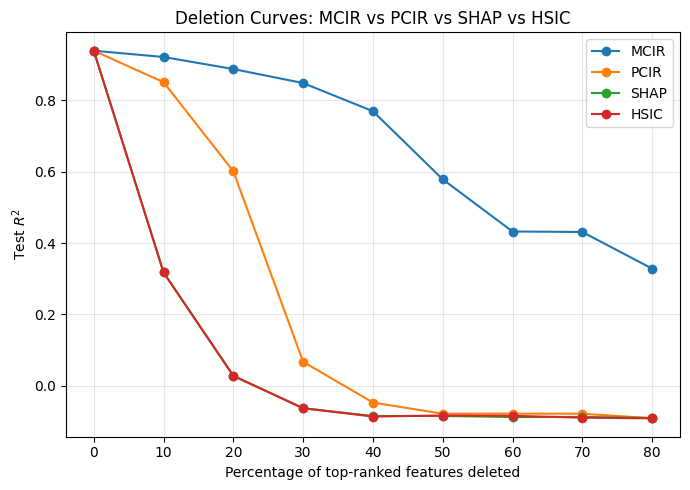}
\caption{Deletion curves for MCIR, PCIR, SHAP, and HSIC under the same
model, test set, masking baseline, and deletion schedule. A sharper
early performance decrease and lower deletion AUC indicate stronger
faithfulness under this protocol.}
\label{fig:deletion_curves_appendix}
\end{figure}

Figure~\ref{fig:deletion_curves_appendix} evaluates predictive
faithfulness under the common deletion protocol. MCIR's more gradual
early decline is weaker deletion evidence than the sharper declines of
PCIR and HSIC under this metric; it should not be presented as superior
deletion faithfulness. Redundancy may contribute to later-curve
convergence because remaining correlated variables can compensate, but
the curves alone do not identify the source of the difference.

\subsection{Full--Lightweight Rank Overlays}

\begin{figure}[t]
\centering
\includegraphics[width=\linewidth,height=4cm]{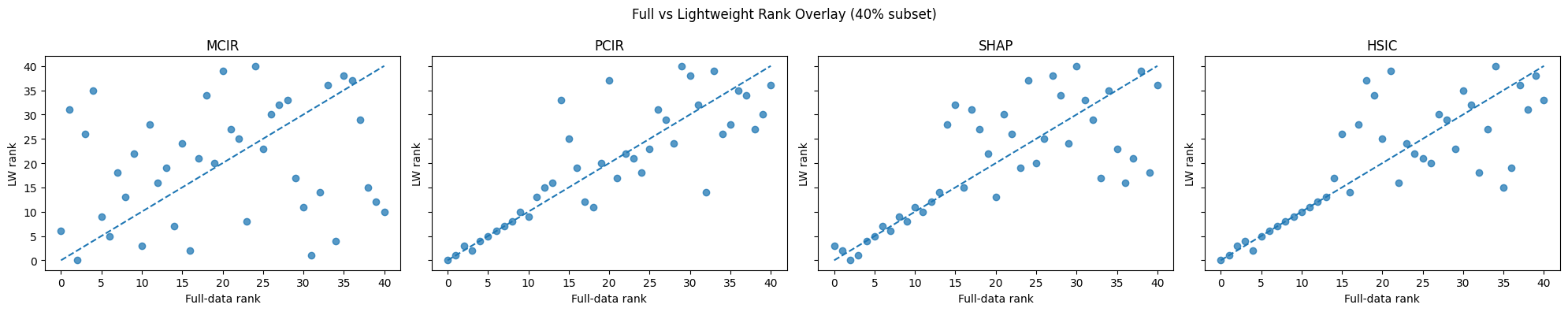}
\caption{Full--lightweight rank overlays at the $40\%$ subset fraction.
Each panel compares feature rankings under full and subsampled data.}
\label{fig:rank_overlay_appendix}
\end{figure}

Figure~\ref{fig:rank_overlay_appendix} compares full and lightweight
rankings at the $40\%$ fraction. Greater MCIR dispersion indicates lower
exact feature-level stability, whereas PCIR and HSIC show tighter exact
alignment. Group-level comparisons apply a separate equivalence rule;
the rank overlay itself does not establish that different feature
identities are informationally interchangeable.

\subsection{Stability Visualizations}

\begin{figure}[htbp]
\centering
\begin{minipage}[t]{0.48\linewidth}
\centering
\includegraphics[width=\linewidth]{download10.png}
\caption{Group-level stability under subsampling. MCIR has the highest
agreement at the $20\%$ and $40\%$ fractions.}
\label{fig:group_stability_appendix}
\end{minipage}
\hfill
\begin{minipage}[t]{0.48\linewidth}
\centering
\includegraphics[width=\linewidth]{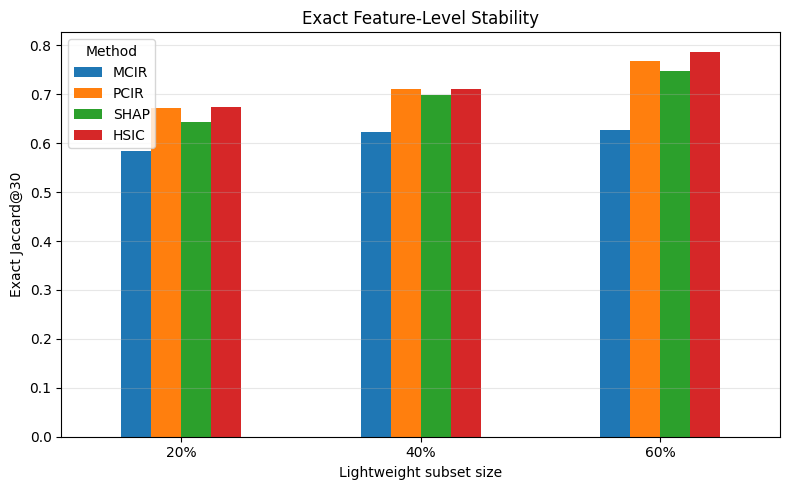}
\caption{Exact feature-level stability under subsampling.}
\label{fig:exact_stability_appendix}
\end{minipage}
\end{figure}

Figures~\ref{fig:group_stability_appendix} and
\ref{fig:exact_stability_appendix} visualize group-level and exact
Jaccard, respectively. MCIR leads at $20\%$ and $40\%$ under the group
metric but not under exact Jaccard or at the $60\%$ group comparison.
The results therefore show a metric-dependent trade-off rather than
uniform stability dominance.

\section{Additional Sensitivity and Ablation Studies}
\label{app:additional}

This section reports sensitivity analyses for the correlation-group
threshold, estimator family, sample size, and conditioning dimension.

\subsection{Correlation-Group Threshold Sensitivity}
\label{app:threshold-sensitivity}

The group-level diagnostic uses the Pearson-correlation threshold
$\gamma$. Because changing $\gamma$ changes the number and composition
of the evaluation groups, we evaluate
\begin{equation}
\medmath{
\gamma
\in
\{0.60,0.70,0.75,0.80,0.85,0.90,0.95\}.
}
\label{eq:appendix-group-thresholds}
\end{equation}

\begin{table}[htbp]
\centering
\scriptsize
\setlength{\tabcolsep}{5pt}
\renewcommand{\arraystretch}{0.96}
\caption{Group-Jaccard@$30$ sensitivity to the correlation-group
threshold on UCI HAR. Larger values indicate greater overlap under the
grouping induced by that row's threshold.}
\label{tab:threshold-sensitivity}
\begin{tabular}{ccccc}
\toprule
\textbf{$\gamma$} & \textbf{MCIR} & \textbf{PCIR}
& \textbf{SHAP} & \textbf{HSIC}\\
\midrule
0.60 & 0.6405 & 0.5036 & \textbf{0.9200} & 0.4284\\
0.70 & 0.6261 & 0.4648 & \textbf{0.9667} & 0.4334\\
0.75 & 0.5792 & 0.4204 & \textbf{0.9741} & 0.4582\\
0.80 & 0.5946 & 0.4720 & \textbf{0.9788} & 0.5201\\
0.85 & 0.5821 & 0.4321 & \textbf{0.9662} & 0.4994\\
0.90 & 0.5732 & 0.4372 & \textbf{0.9490} & 0.5006\\
0.95 & 0.5525 & 0.4020 & \textbf{0.9489} & 0.5076\\
\bottomrule
\end{tabular}
\end{table}

MCIR Group-Jaccard@$30$ varies from $0.5525$ to $0.6405$, while SHAP
has the highest value at every evaluated threshold. Because changing
$\gamma$ changes the evaluation groups, values across thresholds do
not measure an identical target. The experiment therefore
characterizes threshold sensitivity; it does not identify a universally
optimal grouping threshold.
\subsection{Adaptive Conditional Mutual Information Estimation}
\label{app:adaptive-estimation}

This subsection compares the Gaussian--copula and $k$NN estimators and
the bootstrap switching rule in
Equation~\ref{eq:estimator-switching-rule}. The Gaussian--copula
estimator is computationally inexpensive but targets a
Gaussian-copula approximation, whereas the $k$NN estimator can represent
more general continuous dependence at greater computational and
statistical cost. The adaptive rule selects the compatible estimator
with the smaller bootstrap standard error. It estimates variability,
not bias, and is therefore evaluated as a heuristic rather than an
oracle procedure.

The analysis uses HouseEnergy-Sim, for which the relevant feature set is
known by construction. We evaluate
\begin{equation}
\medmath{
|\Phi(i)|\in\{1,2,3,5\},
\qquad
n\in\{200,500,1000,2000,3000\},
}
\label{eq:adaptive-estimation-grid}
\end{equation}
with 30 independent repetitions for each configuration. Kendall's
$\tau_b$ and Spearman's $\rho$ compare the subsample ranking with the
corresponding full-data empirical reference. Jaccard@$10$ compares
their top-$10$ sets, while GT Recall@$10$ measures the fraction of the
prespecified relevant features recovered among the top ten.

\begin{table}[t]
\centering
\scriptsize
\setlength{\tabcolsep}{4pt}
\renewcommand{\arraystretch}{0.96}
\caption{Full-data HouseEnergy-Sim results. GT Recall@$10$ evaluates
recovery of the prespecified relevant set; runtime is measured in
seconds.}
\label{tab:full-reference}
\begin{tabular}{c|cc|cc}
\toprule
& \multicolumn{2}{c|}{\textbf{GT Recall@$10$}}
& \multicolumn{2}{c}{\textbf{Runtime}}\\
\cmidrule(r){2-3}\cmidrule(l){4-5}
$|\Phi(i)|$ & Copula & $k$NN & Copula & $k$NN\\
\midrule
1 & 0.4615 & \textbf{0.6154} & 0.0794 & 56.0264\\
2 & 0.6154 & \textbf{0.7692} & 0.0474 & 13.6651\\
3 & 0.5385 & \textbf{0.7692} & 0.0408 & 15.9895\\
5 & 0.6923 & \textbf{0.7692} & 0.0441 & 22.2782\\
\bottomrule
\end{tabular}
\end{table}

Table~\ref{tab:full-reference} shows an accuracy--runtime trade-off.
$k$NN has higher GT Recall@$10$ at every evaluated neighbourhood size,
reaching $0.7692$ for $|\Phi(i)|\in\{2,3,5\}$. The copula
implementation is substantially faster, remaining below $0.08$ seconds
in these measurements. These results do not identify a uniformly
preferred estimator because recovery and runtime are distinct
objectives.

\begin{table*}[t]
\centering
\scriptsize
\setlength{\tabcolsep}{2.6pt}
\renewcommand{\arraystretch}{0.94}
\caption{Estimator comparison on HouseEnergy-Sim for
$|\Phi(i)|=2$. Entries are mean $\pm$ standard error over 30
repetitions.}
\label{tab:adaptive-phi2}
\begin{adjustbox}{width=\textwidth}
\begin{tabular}{cc|cc|cc|cc|cc|cc}
\toprule
$n$ & Estimator
& \multicolumn{2}{c|}{Kendall $\tau_b$}
& \multicolumn{2}{c|}{Spearman $\rho$}
& \multicolumn{2}{c|}{Jaccard@$10$}
& \multicolumn{2}{c|}{GT Recall@$10$}
& \multicolumn{2}{c}{Runtime (s)}\\
& & Mean & SE & Mean & SE & Mean & SE & Mean & SE & Mean & SE\\
\midrule
\multirow{3}{*}{200}
& Copula   & 0.2036 & 0.0152 & 0.2347 & 0.0182 & 0.4604 & 0.0145 & 0.5923 & 0.0123 & 0.0190 & 0.0008\\
& $k$NN    & 0.0000 & 0.0000 & 0.0000 & 0.0000 & \textbf{1.0000} & 0.0000 & \textbf{0.7692} & 0.0000 & 0.3514 & 0.0149\\
& Adaptive & 0.0000 & 0.0000 & 0.0000 & 0.0000 & \textbf{1.0000} & 0.0000 & \textbf{0.7692} & 0.0000 & 3.3616 & 0.0889\\
\midrule
\multirow{3}{*}{500}
& Copula   & 0.2055 & 0.0125 & 0.2370 & 0.0150 & 0.4202 & 0.0158 & 0.5821 & 0.0109 & 0.0195 & 0.0006\\
& $k$NN    & 0.4306 & 0.0653 & 0.4366 & 0.0662 & \textbf{1.0000} & 0.0000 & \textbf{0.7692} & 0.0000 & 0.8927 & 0.0337\\
& Adaptive & 0.4306 & 0.0653 & 0.4366 & 0.0662 & \textbf{1.0000} & 0.0000 & \textbf{0.7692} & 0.0000 & 8.3200 & 0.1003\\
\midrule
\multirow{3}{*}{1000}
& Copula   & 0.1933 & 0.0121 & 0.2218 & 0.0147 & 0.4606 & 0.0149 & 0.5949 & 0.0090 & 0.0226 & 0.0008\\
& $k$NN    & \textbf{0.7177} & 0.0000 & \textbf{0.7276} & 0.0000 & \textbf{1.0000} & 0.0000 & \textbf{0.7692} & 0.0000 & 2.0585 & 0.0658\\
& Adaptive & \textbf{0.7177} & 0.0000 & \textbf{0.7276} & 0.0000 & \textbf{1.0000} & 0.0000 & \textbf{0.7692} & 0.0000 & 18.6483 & 0.1152\\
\midrule
\multirow{3}{*}{2000}
& Copula   & 0.2271 & 0.0103 & 0.2622 & 0.0124 & 0.4768 & 0.0179 & 0.6026 & 0.0098 & 0.0499 & 0.0062\\
& $k$NN    & \textbf{0.8118} & 0.0247 & \textbf{0.8184} & 0.0238 & \textbf{1.0000} & 0.0000 & \textbf{0.7692} & 0.0000 & 5.3196 & 0.0898\\
& Adaptive & 0.7367 & 0.0440 & 0.7471 & 0.0421 & 0.9286 & 0.0348 & 0.7436 & 0.0124 & 45.5973 & 0.0978\\
\midrule
\multirow{3}{*}{3000}
& Copula   & 0.2411 & 0.0088 & 0.2794 & 0.0107 & 0.4474 & 0.0162 & 0.5718 & 0.0115 & 0.0534 & 0.0077\\
& $k$NN    & \textbf{0.8965} & 0.0253 & \textbf{0.9001} & 0.0244 & \textbf{1.0000} & 0.0000 & \textbf{0.7692} & 0.0000 & 9.0478 & 0.0836\\
& Adaptive & 0.4430 & 0.0570 & 0.4693 & 0.0543 & 0.6437 & 0.0489 & 0.6487 & 0.0183 & 81.9997 & 0.1157\\
\bottomrule
\end{tabular}
\end{adjustbox}
\end{table*}

For $|\Phi(i)|=2$, $k$NN has Jaccard@$10=1$ and GT
Recall@$10=0.7692$ at every sample size. Its complete-ranking agreement
increases with $n$, from $\tau_b=0$ at $n=200$ to $0.8965$ at
$n=3000$. The combination of zero complete-ranking correlation and
perfect top-$10$ overlap at $n=200$ indicates that the leading set is
preserved while its internal ordering, lower-ranked ordering, or ties
differ from the full-data reference.

The adaptive procedure matches $k$NN through $n=1000$. At $n=2000$,
its Jaccard@$10$ and GT Recall@$10$ decline to $0.9286$ and $0.7436$;
at $n=3000$, they decline further to $0.6437$ and $0.6487$. The
selection frequencies in Table~\ref{tab:selection-frequency} explain
this change: the switching rule increasingly selects the copula
estimator for $|\Phi(i)|=2$ at the two largest sample sizes.

The adaptive runtime includes fitting and bootstrapping both candidates
before selection. At $n=3000$ and $|\Phi(i)|=2$, it requires
$81.9997$ seconds, compared with $9.0478$ seconds for fixed $k$NN and
$0.0534$ seconds for fixed copula. It is therefore a calibration cost,
not the runtime of one final MCIR evaluation.

\begin{table}[t]
\centering
\scriptsize
\setlength{\tabcolsep}{3pt}
\renewcommand{\arraystretch}{0.94}
\caption{Estimator selections over 30 repetitions. Each entry reports
copula count (\%)/$k$NN count (\%).}
\label{tab:selection-frequency}
\begin{tabular}{c|cccc}
\toprule
$n$ & $|\Phi|=1$ & $|\Phi|=2$ & $|\Phi|=3$ & $|\Phi|=5$\\
\midrule
200  & 4 (13.3)/26 (86.7) & 0 (0.0)/30 (100.0) & 0 (0.0)/30 (100.0) & 0 (0.0)/30 (100.0)\\
500  & 26 (86.7)/4 (13.3) & 0 (0.0)/30 (100.0) & 0 (0.0)/30 (100.0) & 0 (0.0)/30 (100.0)\\
1000 & 30 (100.0)/0 (0.0) & 0 (0.0)/30 (100.0) & 0 (0.0)/30 (100.0) & 0 (0.0)/30 (100.0)\\
2000 & 30 (100.0)/0 (0.0) & 4 (13.3)/26 (86.7) & 0 (0.0)/30 (100.0) & 0 (0.0)/30 (100.0)\\
3000 & 30 (100.0)/0 (0.0) & 20 (66.7)/10 (33.3) & 0 (0.0)/30 (100.0) & 0 (0.0)/30 (100.0)\\
\bottomrule
\end{tabular}
\end{table}

The selected family depends jointly on $n$ and $|\Phi(i)|$. The rule
selects $k$NN in every repetition for $|\Phi(i)|\in\{3,5\}$, while its
preference at $|\Phi(i)|=1$ shifts from $k$NN to copula as $n$
increases. Most importantly, the $n=3000$, $|\Phi(i)|=2$ result shows
that lower bootstrap variability need not imply higher GT Recall or
ranking agreement.
\begin{table*}[t]
\centering
\scriptsize
\setlength{\tabcolsep}{2.5pt}
\renewcommand{\arraystretch}{0.94}
\caption{Complete estimator-switching evaluation on HouseEnergy-Sim.
Each cell reports Gaussian--copula/Switching/$k$NN, averaged over 30
independent lightweight subsamples. GJ denotes Group-Jaccard and GT-R
denotes generator-defined ground-truth Recall.}
\label{tab:estimator_switching_all}
\begin{adjustbox}{width=\textwidth}
\begin{tabular}{@{}cc
                ccccc@{}}
\toprule
$n$ & $|\Phi|$ &
$\boldsymbol{\tau_b}$ &
$\boldsymbol{\rho}$ &
\textbf{GJ@10} &
\textbf{GT-R@10} &
\textbf{Runtime (s)} \\
\midrule
200 & 1 &
.080/.530/.613 & .082/.543/.632 & .368/.726/.780 &
.500/.679/.708 & .019/4.167/.458 \\
200 & 2 &
.204/.000/.000 & .235/.000/.000 & .460/1.000/1.000 &
.592/.769/.769 & .019/3.362/.351 \\
200 & 3 &
.276/.000/.000 & .329/.000/.000 & .490/1.000/1.000 &
.631/.769/.769 & .018/3.436/.377 \\
200 & 5 &
.000/.000/.000 & .000/.000/.000 & .476/1.000/1.000 &
.613/.769/.769 & .020/3.601/.395 \\
\midrule
500 & 1 &
.062/.141/.741 & .056/.137/.755 & .372/.425/.823 &
.503/.523/.690 & .020/14.018/1.560 \\
500 & 2 &
.205/.431/.431 & .237/.437/.437 & .420/1.000/1.000 &
.582/.769/.769 & .020/8.320/.893 \\
500 & 3 &
.285/.000/.000 & .340/.000/.000 & .408/1.000/1.000 &
.605/.769/.769 & .020/8.758/.942 \\
500 & 5 &
.000/.000/.000 & .000/.000/.000 & .440/1.000/1.000 &
.608/.769/.769 & .020/9.554/1.062 \\
\midrule
1000 & 1 &
.102/.102/.862 & .099/.099/.877 & .366/.366/.891 &
.490/.490/.654 & .022/41.883/4.613 \\
1000 & 2 &
.193/.718/.718 & .222/.728/.728 & .461/1.000/1.000 &
.595/.769/.769 & .023/18.648/2.058 \\
1000 & 3 &
.278/.000/.000 & .332/.000/.000 & .387/1.000/1.000 &
.590/.769/.769 & .023/20.034/2.247 \\
1000 & 5 &
.000/.000/.000 & .000/.000/.000 & .468/1.000/1.000 &
.626/.769/.769 & .022/23.520/2.620 \\
\midrule
2000 & 1 &
.100/.100/.997 & .094/.094/.999 & .368/.368/1.000 &
.482/.482/.615 & .042/140.903/15.623 \\
2000 & 2 &
.227/.737/.812 & .262/.747/.818 & .477/.929/1.000 &
.603/.744/.769 & .050/45.597/5.320 \\
2000 & 3 &
.285/.667/.667 & .340/.667/.667 & .386/1.000/1.000 &
.623/.769/.769 & .039/51.822/5.728 \\
2000 & 5 &
.000/.000/.000 & .000/.000/.000 & .422/1.000/1.000 &
.621/.769/.769 & .046/64.967/7.160 \\
\midrule
3000 & 1 &
.153/.153/.997 & .152/.152/.997 & .386/.386/1.000 &
.459/.459/.615 & .044/296.926/32.963 \\
3000 & 2 &
.241/.443/.897 & .279/.469/.900 & .447/.644/1.000 &
.572/.649/.769 & .053/82.000/9.048 \\
3000 & 3 &
.293/1.000/1.000 & .349/1.000/1.000 & .313/1.000/1.000 &
.569/.769/.769 & .043/96.637/10.652 \\
3000 & 5 &
.000/.000/.000 & .000/.000/.000 & .405/1.000/1.000 &
.613/.769/.769 & .048/119.743/13.371 \\
\bottomrule
\end{tabular}
\end{adjustbox}
\end{table*}

These experiments provide mixed evidence for the switching rule.
The rule removes the need to fix one estimator globally, but it is not
uniformly better than the fixed candidates. Bootstrap variability does
not estimate bias, so the adaptive result must be reported together
with the fixed-estimator sensitivity analysis.

\subsection{Runtime Overhead of Adaptive Auto-$\Phi$ Selection}
\label{app:auto_phi_runtime}

We decompose the measured runtime of the threshold-calibration
implementation into neighbourhood screening, adaptive calibration, and
final MCIR evaluation. The experiment uses UCI HAR with $n=2000$ and
\begin{equation}
\medmath{
k\in\{20,50,100,200,400,561\}.
}
\label{eq:auto-phi-runtime-dimensions}
\end{equation}
Attribution computations run on the CPU in the recorded Google Colab
environment. Each configuration is repeated five times.

Screening performs rank Gaussianization and pairwise-correlation
construction. The calibration stage evaluates
\begin{equation}
\medmath{
\gamma\in\{0.60,0.70,0.75,0.80,0.85,0.90,0.95\}
}
\label{eq:auto-phi-runtime-thresholds}
\end{equation}
using three bootstrap repetitions per threshold and caps each
neighbourhood at five features. Final MCIR evaluation then uses the
selected neighbourhood structure.

This experiment measures the threshold-calibration implementation
used in the reported study. It is not a direct runtime evaluation of
every term in the general Auto-$\Phi$ objective in
Equation~\ref{eq:auto-phi-objective}.

\begin{table}[t]
\centering
\scriptsize
\setlength{\tabcolsep}{4pt}
\renewcommand{\arraystretch}{0.96}
\caption{Runtime decomposition on UCI HAR with $n=2000$.
Entries are mean seconds over five repetitions.}
\label{tab:auto_phi_runtime}
\begin{tabular}{ccccc}
\toprule
\textbf{Features} & \textbf{Screening} & \textbf{Calibration}
& \textbf{Final MCIR} & \textbf{Total}\\
\midrule
20  & 0.001 & 0.179 & 0.004 & 0.184\\
50  & 0.002 & 0.418 & 0.012 & 0.432\\
100 & 0.002 & 0.869 & 0.029 & 0.900\\
200 & 0.003 & 1.939 & 0.069 & 2.011\\
400 & 0.005 & 4.006 & 0.116 & 4.127\\
561 & 0.007 & 5.566 & 0.188 & 5.762\\
\bottomrule
\end{tabular}
\end{table}

The measured total increases from $0.184$ seconds at 20 features to
$5.762$ seconds at 561 features. Calibration is the dominant component,
while the reported screening and final-scoring times remain below
$0.01$ and $0.2$ seconds, respectively. These measurements describe
the evaluated implementation and should not be interpreted as an
empirical proof of a universal asymptotic rate.

\begin{figure}[t]
\centering
\includegraphics[width=0.78\linewidth]{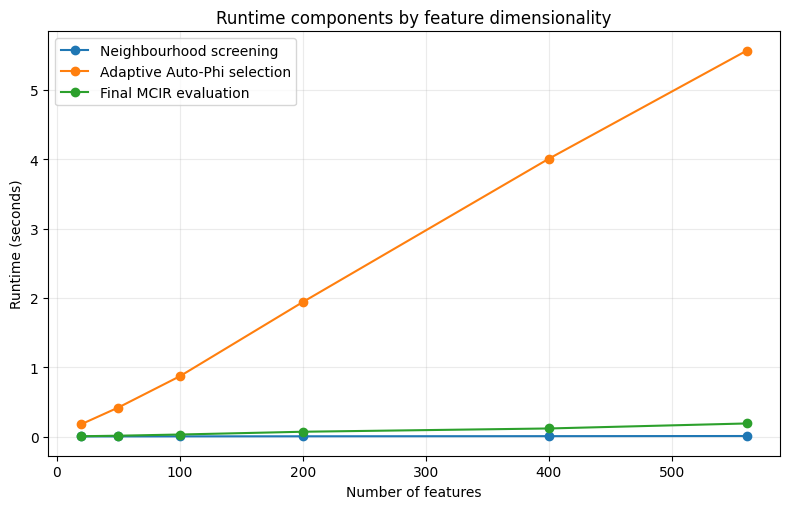}
\caption{Measured total runtime of the adaptive MCIR implementation
versus feature count. Error bars denote standard errors over five
repetitions.}
\label{fig:runtime_total}
\end{figure}

\begin{figure}[t]
\centering
\includegraphics[width=0.82\linewidth]{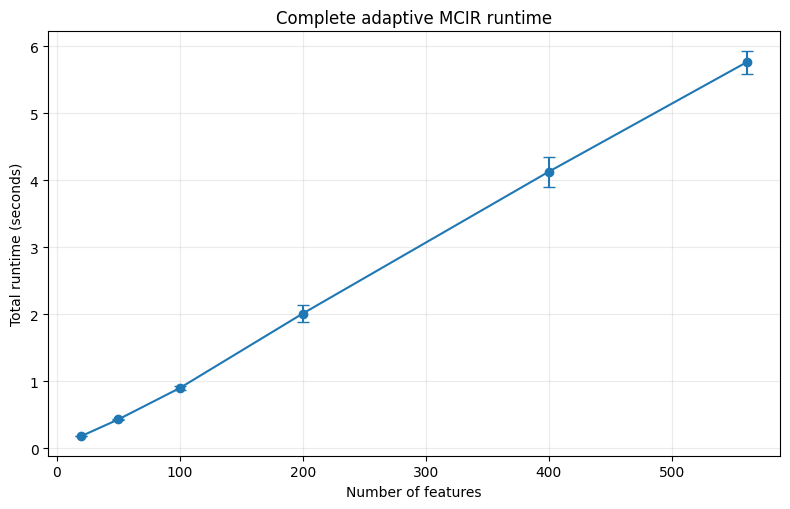}
\caption{Measured runtime of screening, threshold calibration, and
final MCIR evaluation.}
\label{fig:runtime_components}
\end{figure}

\begin{figure}[t]
\centering
\includegraphics[width=0.82\linewidth]{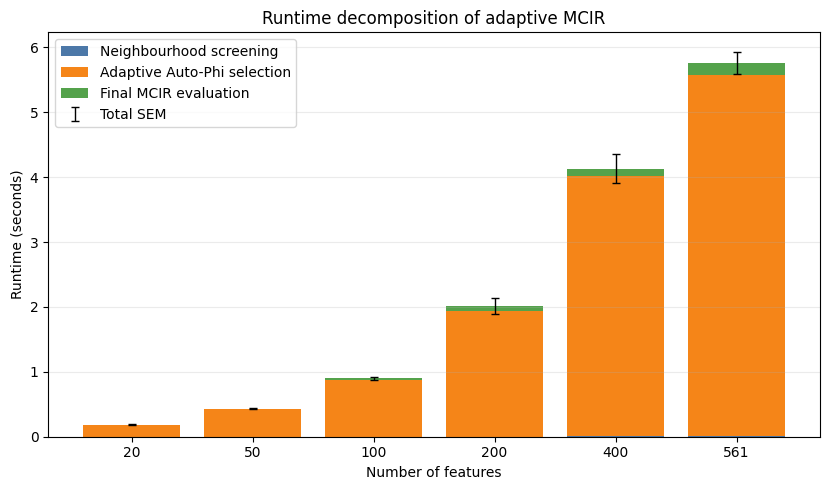}
\caption{Stacked decomposition of the measured adaptive-pipeline
runtime. Error bars denote standard errors of total runtime.}
\label{fig:runtime_decomposition}
\end{figure}

The selected neighbourhood structure can be reused for subsequent MCIR
evaluations only while the feature definitions, preprocessing,
screening rule, and relevant data regime remain unchanged. Distribution
shift or a change in these inputs may require recalibration.

\section{Failure Case: Higher-Order Dependence Outside the Pairwise Neighbourhood}
\label{app:failure_case_higher_order}

This controlled diagnostic identifies a setting in which pairwise
correlation screening can miss relevant higher-order structure. Let
$X_1,X_2\sim\mathcal N(0,1)$ be independent and define the observed
engineered predictors
\begin{equation}
\medmath{
X_3=X_1X_2+E_3,
\qquad
X_4=\sin(X_1),
\qquad
X_5=\cos(X_2),
}
\label{eq:higher-order-features}
\end{equation}
with model-output variable
\begin{equation}
\medmath{
Q=2.5X_3+0.5X_4-0.5X_5+E_Q.
}
\label{eq:higher-order-output}
\end{equation}
We add 25 independent noise features and repeat the experiment over 30
random seeds with $n=4000$. The prespecified direct set is
$\mathcal S_{\mathrm{dir}}=\{X_3,X_4,X_5\}$, and the broader relevant
set is
\begin{equation}
\medmath{
\mathcal S_{\mathrm{all}}
=
\{X_1,X_2,X_3,X_4,X_5\}.
}
\label{eq:higher-order-relevant-set}
\end{equation}
Although $X_1$ and $X_2$ jointly determine $X_3$, their pairwise
linear correlations with $X_3$ can be close to zero. A candidate screen
based only on pairwise correlation may therefore omit variables needed
to represent the higher-order structure.

\begin{table}[htbp]
\centering
\scriptsize
\setlength{\tabcolsep}{3pt}
\renewcommand{\arraystretch}{0.96}
\caption{Higher-order failure-case diagnostic. Entries are means over
30 seeds with standard errors in parentheses. The interaction-informed
$\Phi$ row is a ranking reference, not a deployable recovery baseline.}
\label{tab:mcir_failure_case}
\begin{adjustbox}{width=\linewidth}
\begin{tabular}{lccccc}
\toprule
\textbf{Method} &
\textbf{Direct Recall@$5$} &
\textbf{All Recall@$5$} &
\textbf{Kendall $\tau_b$} &
\textbf{Spearman $\rho$} &
\textbf{Jaccard@$5$}\\
& & & \multicolumn{3}{c}{\textit{versus interaction-informed MCIR}}\\
\midrule
MCIR-correlation-$\Phi$
& 0.611 (0.028) & 0.373 (0.019) & 0.156 (0.024) & 0.140 (0.032) & 0.214 (0.025)\\
MCIR-reference-$\Phi$
& 0.000 (0.000) & 0.000 (0.000) & 1.000 (0.000) & 1.000 (0.000) & 1.000 (0.000)\\
PCIR
& 1.000 (0.000) & 0.620 (0.011) & $-0.133$ (0.021) & $-0.204$ (0.029) & 0.048 (0.010)\\
MI
& 1.000 (0.000) & 1.000 (0.000) & $-0.323$ (0.020) & $-0.438$ (0.023) & 0.000 (0.000)\\
HSIC
& 1.000 (0.000) & 1.000 (0.000) & $-0.275$ (0.013) & $-0.409$ (0.016) & 0.000 (0.000)\\
RF-PFI
& 1.000 (0.000) & 1.000 (0.000) & $-0.303$ (0.021) & $-0.424$ (0.023) & 0.000 (0.000)\\
\bottomrule
\end{tabular}
\end{adjustbox}
\end{table}

The zero recall of the reference-$\Phi$ row is not a failed
ground-truth recovery result. Its manually supplied conditioning sets
condition away the engineered information used in the recall target;
the row is included only to define the ranking reference for the final
three columns. It should therefore not be compared with the deployable
methods on Recall@$5$.

\begin{figure*}[htbp]
\centering

\begin{subfigure}[t]{0.48\linewidth}
\centering
\includegraphics[width=\linewidth]{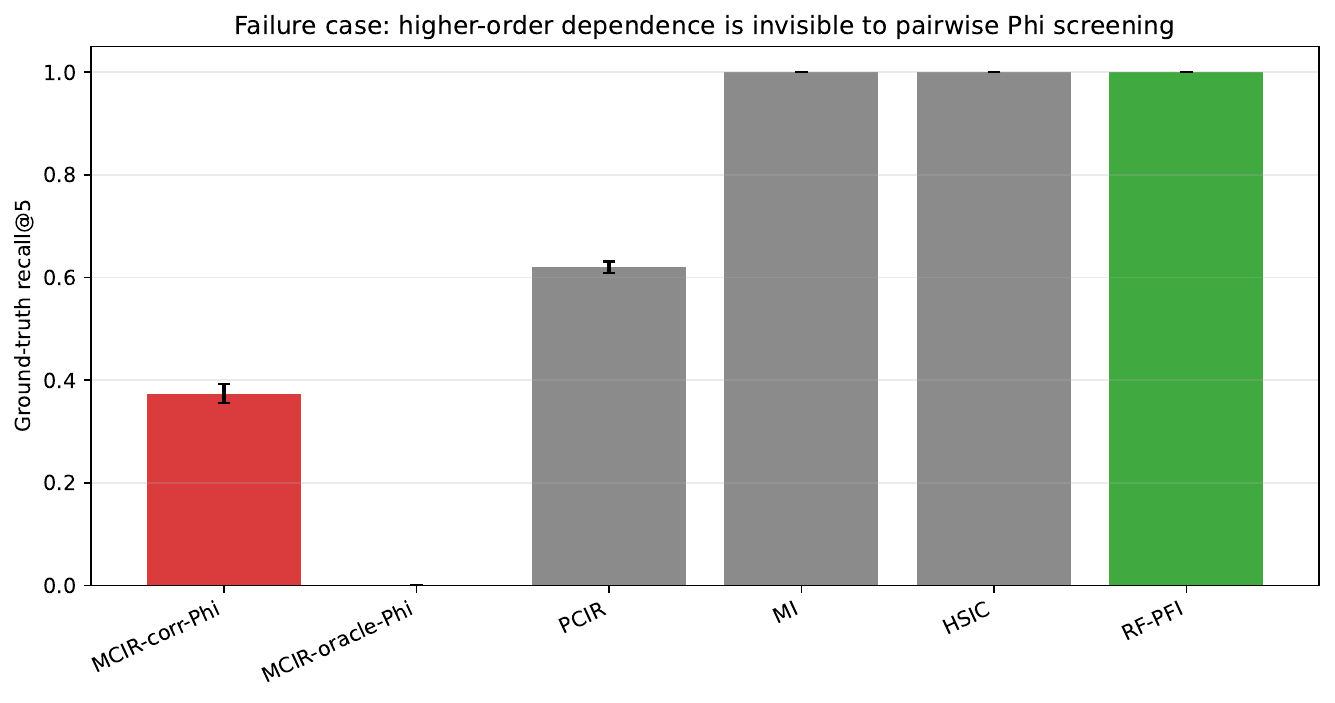}
\caption{Recall@$5$ in the higher-order diagnostic. Correlation-screened
MCIR recovers only part of the relevant set because pairwise screening
does not expose all interaction structure. The reference-$\Phi$ row is
not a recovery baseline.}
\label{fig:failure_case_recall}
\end{subfigure}
\hfill
\begin{subfigure}[t]{0.48\linewidth}
\centering
\includegraphics[width=\linewidth]{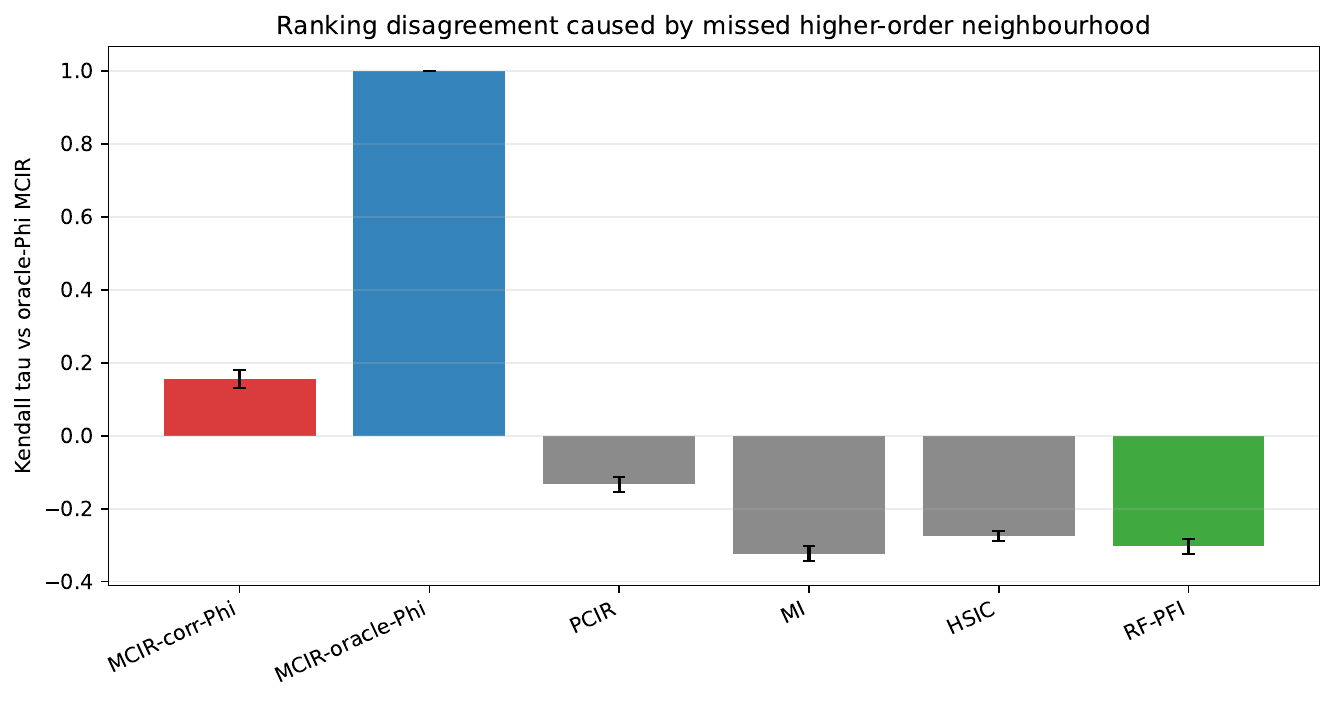}
\caption{Kendall agreement with the interaction-informed MCIR ranking
reference. Low agreement for correlation-screened MCIR shows that
neighbourhood misspecification changes its ranking. Negative agreement
for marginal baselines indicates that they target a different ordering.}
\label{fig:failure_case_kendall}
\end{subfigure}

\caption{Higher-order dependence diagnostic:
\textbf{(a)} recovery of the relevant feature set and
\textbf{(b)} ranking agreement with interaction-informed MCIR.}
\label{fig:failure_case_combined}
\end{figure*}

\begin{table}[htbp]
\centering
\scriptsize
\setlength{\tabcolsep}{4pt}
\renewcommand{\arraystretch}{0.96}
\caption{Interpretation of the higher-order failure-case diagnostic.}
\label{tab:failure_case_interpretation}
\begin{tabular}{p{0.29\linewidth}p{0.64\linewidth}}
\toprule
\textbf{Observation} & \textbf{Interpretation}\\
\midrule
Pairwise screening omits relevant structure &
The association between $X_1$, $X_2$, and $Q$ is mediated by their
product, which is weakly visible to pairwise linear correlation.\\

Low all-relevant recall for correlation-$\Phi$ MCIR &
The limitation arises from constructing $\Phi(i)$, not solely from the
bounded MCIR normalization.\\

High recall for MI, HSIC, and RF-PFI &
These methods recover marginal or predictive relevance but do not
estimate the same unique conditional contribution as MCIR.\\

Reference-$\Phi$ is not deployable &
It is supplied manually only to show how the ranking changes when
interaction-informed neighbourhoods are imposed.\\
\bottomrule
\end{tabular}
\end{table}

This experiment establishes a failure mode rather than a
general advantage. MCIR is dependence-aware conditional on the
selected neighbourhood, but pairwise-correlation screening is
insufficient for purely higher-order dependence. Interaction-aware
candidate generation, distance correlation, HSIC screening, or
mutual-$k$NN graphs may expose additional structure, but their use does
not guarantee recovery of every relevant dependence pattern.
\end{document}

% \bstctlcite{IEEEexample:BSTcontrol}
    \title{High-Efficiency Harmonically-Terminated Diode and Transistor Rectifiers}
%   \author{Michael~Roberg,~\IEEEmembership{Student Member,~IEEE,}
%       Tibault~Reveyrand,~\IEEEmembership{Member,~IEEE,}\\
%       Ignacio~Ramos,~\IEEEmembership{Student Member,~IEEE,}
%       Erez Falkenstein,~\IEEEmembership{Student Member,~IEEE,}
%       and~Zoya~Popovi\'c,~\IEEEmembership{Fellow,~IEEE}% <-this % stops a space

%   \thanks{Manuscript received July 10, 2012. \hl{This paper is an expanded paper from the IEEE MTT-S Int. Microwave Symposium held on June 17-22, 2012 in Montreal, Canada.} This work was funded in part by the Office of Naval Research under the Defense Advanced Research Projects Agency (DARPA) Microscale Power Conversion (MPC) Program under Grant N00014-11-1-0931, and in part by the Advanced Research Projects Agency-Energy (ARPA-E), U.S. Department of Energy, under Award Number DE-AR0000216.}
%   \thanks{M. Roberg is with TriQuint Semiconductor, 500 West Renner Road Richardson, TX 75080 USA (e-mail: michael.roberg@tqs.com).}% <-this % stops a space
%   \thanks{T. Reveyrand is with the XLIM Laboratory, UMR 7252, University of Limoges, 87060 Limoges, France (e-mail: tibault.reveyrand@xlim.fr).}%
%   \thanks{I. Ramos and Z. Popovic are with the Department of Electrical, Computer and Energy Engineering, University of Colorado, Boulder, CO, 80309-0425 USA (e-mail: ignacio.ramos@colorado.edu; zoya.popovic@colorado.edu).}% <-this % stops a space
%   \thanks{E. Falkenstein is with Qualcomm Inc., 6150 Lookout Road
% Boulder, CO 80301 USA (e-mail: erez.falkenstein@gmail.com).}}  

% ============================================================
\section{Preliminaries \& Notation}
% ============================================================

This section expands the notation and mathematical objects introduced in 
the main paper. We provide precise definitions of environments, 
lightweight-model similarity, projection/embedding distances, and the 
assumptions required for the theoretical results in MCIR, ERI, and the 
lightweight-fidelity framework. We consider a supervised learning model $M$ trained on features
$F \in \mathbb{R}^{n \times k}$ with output vector
$Y = M(F) \in \mathbb{R}^{n \times q}$ ($q=1$ for scalar regression).
The $i$-th feature is denoted by a column vector 
$f_i = (f_{1i}, \ldots, f_{ni})^\top \in \mathbb{R}^{n}$.

% ------------------------------------------------------------
\subsection{Notation Table}
% ------------------------------------------------------------

\begin{table}[htbp]
\centering
\renewcommand{\arraystretch}{1.18}
\caption{ Notation used throughout the manuscript and supplementary material.}
\tiny
\begin{tabular}{l p{0.3\textwidth}}
\hline
\textbf{Notation} & \textbf{Description} \\
\hline

$F \in \mathbb{R}^{n \times k}$ 
& Full feature matrix with $n$ observations and $k$ features. \\

$f_i \in \mathbb{R}^{n}$ 
& $i$th feature column; entries $f_{1i},\dots,f_{ni}$. \\

$f_{ji}$ 
& Value of feature $i$ for observation $j$. \\

$n$ 
& Number of full-environment observations. \\

$n'$ 
& Number of lightweight/sampled observations ($n' \ll n$). \\

$Y = M(F)$ 
& Model outputs in the full environment. \\

$Y' = M'(F')$ 
& Model outputs in the lightweight environment. \\

$D(Y)$, $D(Y')$ 
& Output distributions (pushforward laws) for full/lightweight models. \\

$U = D(F,Y)$ 
& Full environment: joint feature–output distribution. \\

$U' = D(F',Y')$ 
& Lightweight environment. \\

$P \in O(q',q)$, $Q \in O(q,q')$ 
& Orthogonal projection/embedding matrices (Stiefel manifold). \\

$\Phi_{P,b}(x)=Px+b$ 
& Rigid-motion transformation (projection/translation). \\

$d^{-},\, d^{+}$ 
& Projection and embedding distances. \\

$d_{\hat{f}}$ 
& Rigid-motion–invariant $f$-divergence between output laws. \\

$\mathcal{L}(Y,Y')$ 
& Lightweight fidelity loss: $d_{\hat{f}}(D(Y),D(Y'))$. \\

$\Phi(i)$ 
& Conditioning set of neighbours for feature $i$. \\

$I(\cdot;\cdot)$, $I(\cdot;\cdot|\cdot)$ 
& Mutual and conditional mutual information. \\

$\eta_{f_i}$ 
& PCIR score (pairwise correlation impact ratio). \\

$C_i$ 
& MCIR score (conditional dependence impact ratio). \\

$\widehat{I}$ 
& Estimated MI/CMI under a chosen estimator (copula/kNN/plug-in). \\

$e(i)$ 
& Selected estimator for feature $i$ (bootstrap-SE minimizer). \\

$\mathrm{J}@K$ 
& Head-rank Jaccard agreement between full and lightweight rankings. \\

$\Delta_{\text{del}}$ 
& Deletion-curve deviation between full and lightweight models. \\

$m_\Phi$ 
& Number of conditioning neighbours (for MCIR). \\

$m_{\text{scr}}$ 
& Number of screened neighbours retained after initial dependence sketch. \\

$\mathbb{E}_b[\cdot]$ 
& Expectation over bootstrap replicates. \\

\hline
\end{tabular}
\label{tab:notation}
\end{table}

% ------------------------------------------------------------
\subsection{Full and Lightweight Environments}
% ------------------------------------------------------------
We define the \textbf{full environment} as: $U = D(F,Y)$,
the joint law capturing both feature dependence and the model's 
output behaviour. For computational or privacy reasons, we construct a
\textbf{lightweight environment}:
$U' = D(F',Y'),\qquad n' \ll n,$ where $F'$ contains fewer or reweighted observations, and 
$M'$ uses the same architecture/training protocol as $M$. The goal is \textbf{not} to approximate $F$ directly, but to match the \emph{output distribution} $D(Y)$, $\medmath{ Y' = M'(F')
\ \text{should exhibit the same global behaviour as}\
Y = M(F).}$ This ensures that global attribution computed on $M'$ faithfully
reflects that of $M$.

% ------------------------------------------------------------
\subsection{Similarity of Environments}
% ------------------------------------------------------------

The notion of ``environment similarity’’ formalizes when a lightweight model
$M'$ behaves like the full model $M$ in terms of its \emph{output distribution}.
Since global attribution depends only on how feature perturbations influence the 
\emph{shape} of the model’s predictions, two environments are considered equivalent 
even if their outputs differ by a rotation or translation.

Let  
\[
Y = M(F) \in \mathbb{R}^{n \times q}, 
\qquad 
Y' = M'(F') \in \mathbb{R}^{n' \times q'}
\]
denote the outputs of the full and lightweight environments, and let  
$D(Y)$ and $D(Y')$ be their empirical output distributions.

\begin{definition}[Similarity of Environments]
We write $U' \sim U$ if the output laws match up to a rigid motion, i.e.\ if
there exist an orthogonal matrix $P \in O(q',q)$ and translation vector 
$b \in \mathbb{R}^{q'}$ such that
\begin{equation}
    P Y + b \;\stackrel{d}{=}\; Y'.
\end{equation}
Here $X \stackrel{d}{=} Z$ means that $X$ and $Z$ have the same distribution.
\end{definition}

\paragraph{Intuition.}  
This expresses that $M$ and $M'$ generate predictions with the same geometry: 
similar spread, clusters, and relative distances—even if the coordinates differ.  
Since global attribution uses only dependence structure between predictions and 
features, it is invariant to such transformations.

% ------------------------------------------------------------
\subsection{Projection and Embedding Distances}
% ------------------------------------------------------------

To measure how close two output distributions are, we use distances that allow
their dimensions to differ. Let
\[
\mu = D(Y), \qquad \nu = D(Y')
\]
be the full and lightweight output distributions.

We use two natural discrepancies:

\paragraph{(1) Projection distance $d^{-}(\mu,\nu)$.}  
We find the best projection of the full model’s output space into the 
lightweight output space:
\begin{equation}
d^{-}(\mu,\nu)
= \inf_{P\in O(q',q),\ b\in\mathbb{R}^{q'}}
    d_f\!\bigl( \Phi_{P,b\#}\mu , \nu \bigr).
\end{equation}
Here  
- $P \in O(q',q)$ projects/rotates $q$-dimensional outputs to $q'$-dimensional ones,  
- $b$ allows translation,  
- $\Phi_{P,b}(x) = Px + b$ is the rigid-motion map,  
- $\Phi_{P,b\#}\mu$ denotes the transformed distribution,  
- $d_f$ is any $f$-divergence (KL, JS, Hellinger, etc.).

\paragraph{(2) Embedding distance $d^{+}(\mu,\nu)$.}  
We embed the lightweight outputs into the higher-dimensional full space:
\begin{equation}
d^{+}(\mu,\nu)
= \inf_{Q\in O(q,q'),\ c\in\mathbb{R}^{q}}
    d_f\!\bigl( \mu , \Phi_{Q,c\#}\nu \bigr).
\end{equation}

\paragraph{Equivalence.}  
For $f$-divergences invariant to orthogonal transformation and translation,
\begin{equation}
d^{-}(\mu,\nu)=d^{+}(\mu,\nu)=: d_{\hat{f}}(\mu,\nu),
\end{equation}
giving a single \emph{rigid-motion–invariant} discrepancy between the two 
output distributions. This makes $d_{\hat{f}}$ a well-posed measure of 
environment similarity, even when $q\neq q'$.

% ------------------------------------------------------------
\subsection{Lightweight Fidelity Loss}
% ------------------------------------------------------------

The lightweight model $M'$ is intended to serve as a computationally cheaper
surrogate for $M$. To guarantee that global attributions computed in the
lightweight environment remain faithful to those of the full model, we define
a fidelity loss that captures the mismatch between output distributions:
\begin{equation}
    \medmath{\mathcal{L}(Y,Y') 
\;:=\; d_{\hat{f}}\!\bigl(D(Y), D(Y')\bigr).}
\end{equation}

\begin{itemize}
    \item $\mathcal{L}(Y,Y')$ is small when $M'$ produces predictions whose geometry
  matches that of $M$. 
  \item Because $\mathcal{L}$ ignores rigid motions, it captures only shape and
  dependence structure—not coordinate conventions.  
  \item This ensures that any global explanation method depending solely on the 
  joint behaviour of $(F,Y)$ will behave similarly for $(F',Y')$.
\end{itemize}
\paragraph{Fidelity Contract.}  
A lightweight environment is considered acceptable if
\[
\medmath{\mathcal{L}(Y,Y') \le \varepsilon, 
\qquad 
\mathrm{J}@K(C,C') \ge \tau_0, 
\qquad
\Delta_{\mathrm{del}} \le \Delta_0,}
\]
where  
- $\mathrm{J}@K$ measures head-rank agreement of MCIR scores,  
- $\Delta_{\mathrm{del}}$ measures deviation in deletion curves.

% ------------------------------------------------------------
\subsection{Additional Technical Assumptions}
% ------------------------------------------------------------

We state the assumptions required for MCIR’s theoretical guarantees.

\begin{assumption}[\textbf{Existence of Densities}]
Relevant joint and conditional laws admit densities or PMFs, ensuring 
finite MI/CMI.
\end{assumption}

\begin{assumption}[\textbf{Regular Conditioning}]
All conditioning events satisfy $P(f_{\Phi}=z)>0$ almost everywhere.
\end{assumption}

\begin{assumption}[\textbf{Estimator Regularity}]
MI/CMI estimators satisfy sub-Gaussian concentration:
\[
\Pr\!\left(|\widehat{I} - I| > \delta\right)
\le c_1 \exp(-c_2 n' \delta^2),
\qquad
\widehat{I}-I = O_p(n'^{-1/2}).
\]
\end{assumption}

\begin{assumption}[\textbf{Bootstrap Reliability}]
Bootstrap SE provides a consistent surrogate for estimator risk across a 
finite candidate estimator set (copula / kNN / plug-in).
\end{assumption}

\begin{assumption}[\textbf{Lightweight Fidelity}]
There exist constants $(\varepsilon,\tau_0,\Delta_0)$ such that
$d_{\hat{f}}(D(Y),D(Y')) \le \varepsilon $ and head-rank agreement exceeds 
$\tau_0$.
\end{assumption}

These assumptions collectively ensure:
\[
\medmath{\widehat{C}_i \xrightarrow{p} C_i,
\qquad
1 - \tau(C,C') = O(k/\sqrt{n'}),}
\]
as shown in the main results.

This section provides the rigorous definitions needed for MCIR and 
ERI to operate in a principled, estimator-agnostic, 
lightweight-compatible setting. The remainder of the Supplement uses 
these definitions to establish boundedness, redundancy collapse, 
stability, estimator-switching guarantees, and lightweight fidelity.

\subsection{ExCIR Baseline: Partial Correlation Impact Ratio (PCIR)}
% ------------------------------------------------------------

We begin by revisiting the global attribution score underlying the 
traditional ExCIR framework. PCIR quantifies how strongly the variability 
of feature $i$ aligns with the variability of the model output, using a 
bounded ANOVA-style ratio that is robust and model-agnostic. Let $F' \in \mathbb{R}^{n' \times k}$ and $Y' \in \mathbb{R}^{n'}$ denote the 
lightweight features and outputs with $n'$ observations. For each feature 
$i$, define the sample means
\[
\medmath{\bar f_i = \frac{1}{n'}\sum_{j=1}^{n'} f_{ji},
\qquad 
\bar y' = \frac{1}{n'}\sum_{j=1}^{n'} y'_j,}
\]
and their pooled midpoint
\[
\medmath{m_i = \frac{\bar f_i + \bar y'}{2}.}
\]

\begin{definition}[PCIR]\label{def:pcir}
The \emph{Partial Correlation Impact Ratio} of feature $i$ is
\[
\eta_{f_i}
=
\frac{S_B(i)}{S_T(i)}
\in [0,1],
\]
where
\begin{equation}
    \begin{split}
        \medmath{S_B(i)
= n' \!\left[(\bar f_i - m_i)^2 + (\bar y' - m_i)^2\right],}\\
\medmath{S_T(i)
= \sum_{j=1}^{n'} (f_{ji}-m_i)^2
 + \sum_{j=1}^{n'} (y'_j-m_i)^2.}
    \end{split}
\end{equation}
\end{definition}

Here $S_B(i)$ measures the \emph{between-group variability} (a structured 
mean displacement between the feature and the output), whereas $S_T(i)$ 
measures the \emph{total variability} around the pooled midpoint $m_i$.
Thus PCIR captures how much of the total dispersion is explained by 
aligned co-movement between feature $f_i$ and the model output $Y'$.

\begin{theorem}[Basic properties of PCIR]\label{thm:pcir_properties}
For every feature $i$:
\begin{enumerate}
    \item \textbf{Boundedness:} $0 \le \eta_{f_i} \le 1$.
    \item \textbf{Monotonicity:} Increasing the joint mean displacement 
    between $f_i$ and $Y'$ (at fixed total dispersion $S_T(i)$) increases 
    $\eta_{f_i}$.
    \item \textbf{Noise suppression:} If $f_i$ carries no structured 
    signal about $Y'$ (so that $(\bar f_i - \bar y') \to 0$ while 
    $S_T(i)>0$), then $\eta_{f_i} \to 0$.
\end{enumerate}
\end{theorem}

\begin{proof}
We provide full proofs for the three properties in 
Theorem~\ref{thm:pcir_properties}.  
Recall the definitions
\begin{equation}
 \begin{split}
     &S_B(i)
= n' \bigl[(\bar f_i - m_i)^2 + (\bar y' - m_i)^2\bigr],\\&
S_T(i)
= \sum_{j=1}^{n'} (f_{ji}-m_i)^2 
 + \sum_{j=1}^{n'} (y'_j-m_i)^2,\\&
 \text{and}\
\eta_{f_i} = \frac{S_B(i)}{S_T(i)}.
 \end{split}   
\end{equation}

% ------------------------------------------------------------
Define the centered quantities
\begin{equation}
    \medmath{\tilde f_{ji} = f_{ji} - m_i, \qquad 
\tilde y'_j = y'_j - m_i.}
\end{equation}
Then the total variability can be written as
\begin{equation}
    \begin{split}
        \medmath{S_T(i)
=
\sum_{j=1}^{n'} \tilde f_{ji}^2
+
\sum_{j=1}^{n'} \tilde y'_j{^2}.}
    \end{split}
\end{equation}

Write the means relative to $m_i$ as
\begin{equation}
    \medmath{d_f = \bar f_i - m_i, \qquad d_y = \bar y' - m_i,}
\end{equation}
so that
\begin{equation}
    S_B(i) = n'(d_f^2 + d_y^2).
\end{equation}
A standard ANOVA identity decomposes total variation into 
between-mean and within-mean components:
\begin{equation}
    \begin{split}
       & \sum_{j=1}^{n'} \tilde f_{ji}^2 
= n' d_f^2 + \sum_{j=1}^{n'} (f_{ji} - \bar f_i)^2,\\&
\sum_{j=1}^{n'} \tilde y'_j{^2} 
= n' d_y^2 + \sum_{j=1}^{n'} (y'_j - \bar y')^2.
    \end{split}
\end{equation}
Summing these gives
\begin{equation}
\begin{split}
      &\medmath{S_T(i)
=}\medmath{ n'(d_f^2 + d_y^2)
 + \sum_{j=1}^{n'} (f_{ji} - \bar f_i)^2
 + \sum_{j=1}^{n'} (y'_j - \bar y')^2
}\\&\medmath{=S_B(i) + S_W(i),}
\end{split}
\end{equation}
where
\begin{equation}
    \medmath{S_W(i)
:= \sum_{j=1}^{n'} (f_{ji} - \bar f_i)^2
 + \sum_{j=1}^{n'} (y'_j - \bar y')^2 \ge 0.}
\end{equation}
Thus
\begin{equation}
    \medmath{0 \le S_B(i) \le S_B(i) + S_W(i) = S_T(i).}
\end{equation}
Therefore,
\begin{equation}
    \medmath{0 \le \eta_{f_i} = \frac{S_B(i)}{S_T(i)} \le 1.}
\end{equation}
Fix the total dispersion $S_T(i)$ (i.e., fix the within-variances and the 
pooled midpoint $m_i$).  
Increasing the aligned co-movement between $f_i$ and $Y'$ corresponds to 
increasing the absolute mean differences $|\bar f_i - m_i|$ and 
$|\bar y' - m_i|$ in a symmetric fashion.

Since
\begin{equation}
    \medmath{S_B(i) = n'(d_f^2 + d_y^2),}
\end{equation}
and the total $S_T(i)$ does not change under such shifts, we have that 
$S_B(i)$ is strictly increasing in the magnitude of $(d_f,d_y)$, while the 
denominator $S_T(i)$ is held constant.  
Thus
\begin{equation}
    \eta_{f_i} = \frac{S_B(i)}{S_T(i)}
\end{equation}
is strictly increasing as the aligned mean displacement increases. Formally, let $d_f(t),d_y(t)$ be differentiable paths with 
$d_f(0)=d_f,d_y(0)=d_y$  and  $\frac{d}{dt}(d_f^2+d_y^2)>0$. Then
\begin{equation}
   \medmath{ \frac{d}{dt}\eta_{f_i}(t)
=
\frac{n'\,2(d_f d_f'+ d_y d_y')}{S_T(i)}
> 0,}
\end{equation}
showing monotonicity. If $f_i$ is uninformative relative to $Y'$, then their sample means coincide:
\begin{equation}
    \medmath{\bar f_i - \bar y' \to 0
\quad\Longrightarrow\quad
d_f \to 0,\;\; d_y \to 0.}
\end{equation}
Thus the between-mean component satisfies
\begin{equation}
  \medmath{S_B(i)
= n'(d_f^2 + d_y^2)
\longrightarrow 0,  }
\end{equation}
while the total dispersion $S_T(i)$ remains strictly positive as long as 
either $f_i$ or $Y'$ has nonzero variance. Therefore,
\begin{equation}
    \medmath{\eta_{f_i} = \frac{S_B(i)}{S_T(i)}}
\longrightarrow 0,
\end{equation}
establishing noise suppression.
\end{proof}
PCIR behaves like a global, variance-based correlation ratio:
\begin{itemize}
    \item $\eta_{f_i} \approx 1$ when $f_i$ and $Y'$ vary together at the 
    population level.
    \item $\eta_{f_i} \approx 0$ when $f_i$ behaves like noise relative to 
    the output.
\end{itemize}
This makes PCIR a \emph{bounded, interpretable, and model-agnostic} score. PCIR evaluates each feature independently. Under strong multicollinearity 
or manifold-structured data, many features may move together, inflating 
their between-group variability and distributing credit across an entire 
correlated block. This motivates the transition to MCIR 
(Section~\ref{subsec:mcir_full}), which conditions on local neighbours to 
measure \emph{unique} feature contribution from the dependence structure.

\begin{figure}[t]
\centering
\begin{tikzpicture}[
  node distance=6mm,
  header/.style={rectangle, rounded corners=2pt, inner sep=2pt, font=\scriptsize\bfseries, text=white},
  box/.style={rectangle, draw=black, rounded corners, align=left,
              inner sep=3pt, text width=0.92\linewidth, font=\scriptsize, fill=white},
  >={Stealth[length=2mm]}
]

% ---- 0) Inputs --------------------------------------------------------------
\node[box, fill=oi-sky!10, draw=oi-sky] (input) {%
  \tikz{\node[header, fill=oi-sky, anchor=west] at (0,0) {Inputs};}
  \hspace{2ex} Feature matrix $F\!\in\!\mathbb{R}^{n\times k}$; lightweight outputs $Y'\!=\!M'(F)$; target feature $f_i$.
};

% ---- 1) Screening -----------------------------------------------------------
\node[box, below=of input, fill=oi-blue!7, draw=oi-blue] (screen) {%
  \tikz{\node[header, fill=oi-blue, anchor=west] at (0,0) {1) Screening};}
  \newline \textit{Goal:} fast dependence sketch. 
  \newline \textit{Steps:} pairwise dependence (e.g., $|\mathrm{corr}|$, dCor); build mutual-$k$NN graph; keep top-$m$ neighbours.
};

% ---- 2) Blocking ------------------------------------------------------------
\node[box, below=of screen, fill=oi-orange!10, draw=oi-orange] (block) {%
  \tikz{\node[header, fill=oi-orange, anchor=west] at (0,0) {2) Blocking};}
  \newline \textit{Goal:} form correlated groups.
  \newline \textit{Steps:} cluster the graph (community detection / hierarchical clustering) $\rightarrow$ blocks $\mathcal{B}_1,\ldots,\mathcal{B}_r$.
};

% ---- 3) Conditioning --------------------------------------------------------
\node[box, below=of block, fill=oi-green!10, draw=oi-green] (cond) {%
  \tikz{\node[header, fill=oi-green, anchor=west] at (0,0) {3) Conditioning set $\Phi$};}
  \newline \textit{Goal:} small, informative context for $f_i$.
  \newline \textit{Steps:} find block $\mathcal{B}(i)$; set $\Phi$ to the $m_\Phi$ closest neighbours of $f_i$ within $\mathcal{B}(i)$ (e.g., $m_\Phi\!\in\![3,10]$).
};

% ---- 4) Estimation ----------------------------------------------------------
\node[box, below=of cond, fill=oi-purple!10, draw=oi-purple] (est) {%
  \tikz{\node[header, fill=oi-purple, anchor=west] at (0,0) {4) MCIR estimation};}
  \newline \textit{Goal:} quantify unique contribution.
  \newline \textit{Steps:} estimate $I(Y'; f_i \mid \mathbf{f}_\Phi)$ (Gaussian–copula / $k$NN / plug–in) and $I(Y',\mathbf{f}_{\Phi\cup\{i\}})$; compute bounded score $C\in[0,1]$.
};

% ---- 5) Validation ----------------------------------------------------------
\node[box, below=of est, fill=oi-verm!8, draw=oi-verm] (valid) {%
  \tikz{\node[header, fill=oi-verm, anchor=west] at (0,0) {5) Validation \& reporting};}
  \newline \textit{Consistency:} bootstrap; recompute $C$; rank agreement \& bands.
  \newline \textit{Trustworthiness:} lightweight vs.\ full ranks; del/ins curves.
  \newline \textit{Efficiency:} wall–clock vs.\ observations.
};

% ---- Arrows (color-matched) -------------------------------------------------
\draw[->, very thick, oi-blue]   (input) -- (screen);
\draw[->, very thick, oi-orange] (screen) -- (block);
\draw[->, very thick, oi-green]  (block) -- (cond);
\draw[->, very thick, oi-purple] (cond) -- (est);
\draw[->, very thick, oi-verm]   (est) -- (valid);

\end{tikzpicture}

\caption{MCIR methodology: screening $\rightarrow$ blocking $\rightarrow$ conditioning $\rightarrow$ estimation $\rightarrow$ validation.
MCIR is
$C\!\left(Y'; f_i \mid \mathbf{f}_{\Phi}\right)=
\frac{I\!\left(Y'; f_i \mid \mathbf{f}_{\Phi}\right)}
     {I\!\left(Y'; f_i \mid \mathbf{f}_{\Phi}\right)+I\!\left(Y',\mathbf{f}_{\Phi\cup\{i\}}\right)}\in[0,1]$.}
\label{fig:mcir_scheme_vertical}
\end{figure}

Fig.~\ref{fig:mcir_scheme_vertical} provides a step-by-step schematic of proposed MCIR procedure, highlighting inputs, conditioning-set selection, estimation, and diagnostics.

\section{Mutual Correlation Impact Ratio for dependent features (MCIR)}\label{dip}

% \subsection{From PCIR to MCIR: Conditional Multivariate Mutual Information}

PCIR provides a simple and computationally efficient nonlinear attribution
score under the assumption that features are independent. In such an
environment, the pairwise displacement between a feature and the model
output reliably reflects its global importance. However, when features are
correlated, especially in multivariate environments where each feature may
depend on several others, PCIR becomes insufficient: its pairwise
construction cannot isolate the \emph{unique} contribution of a feature
within a correlated cluster. This section develops MCIR, a dependence-aware extension designed for
environments where features are jointly distributed. Let
\begin{equation}
   \medmath{ F' = (f_1,\dots,f_k) \in \mathbb{R}^{n' \times k}}
\end{equation}
denote the lightweight feature matrix, where each $f_i$ may be either
continuous or discrete. The joint feature distribution therefore admits
either a multivariate probability density function or a multivariate
probability mass function. We write the combined domain as
\begin{equation}
    \medmath{\|F'\|^{(n' \times k)}
=
\|F'\|^{(n' \times k)}_c
\;\cup\;
\|F'\|^{(n' \times k)}_d,}
\end{equation}
where $\|F'\|_c$ contains continuous features and $\|F'\|_d$ contains
discrete ones. Classical Mutual Information (MI) measures dependence between two variables,
and Conditional Mutual Information (CMI) quantifies dependence between two
variables after conditioning on a \emph{single} or small set of other
variables. However, existing CMI theory (e.g.\ \cite{danielsson1980euclidean})
is tailored for situations where \emph{multiple features depend on one
parent}. In highly correlated, real-world environments, the opposite holds:
each feature may depend on \emph{many} neighbours simultaneously. This leads
to an exponential blow-up in the number of possible conditioning sets and
makes traditional CMI insufficient for isolating the unique contribution of a
feature. To address this limitation, we introduce the
\emph{Conditional Multivariate Mutual Information} (CMMI). For a targeted
feature $f_i$, let $\Phi(i) \subseteq \{f_1,\dots,f_k\}\setminus\{f_i\}$ be
the set of features on which $f_i$ depends. CMMI is defined as the divergence
between the conditional cross-entropies:
\begin{equation}
    \medmath{\mathrm{CMMI}(Y'; f_i \mid f_{\Phi(i)})
=
H(Y' \mid f_{\Phi(i)})
-
H(Y' \mid f_{\Phi(i)} \cup \{f_i\}),}
\end{equation}
whenever these entropies exist. CMMI captures how much additional predictive
information $f_i$ contributes \emph{beyond what is already explained by its
neighbours}. When two conditional distributions differ, the divergence between their
cross-entropies corresponds to a Jensen–Shannon–type divergence
\cite{joyce2011kullback}, which we refer to as the
\emph{Joint Mutual Impact} (JMI). In multivariate environments, JMI
quantifies how much the joint behaviour of the feature block contributes to
the output distribution. JMI therefore provides the raw dependency that CMMI
refines through conditionalisation. MCIR converts the conditional dependency captured by CMMI into a bounded,
unitless ratio. For the targeted feature $f_i$, the MCIR score is constructed
from:
\begin{enumerate}
    \item its unique conditional contribution 
    $\medmath{I(Y'; f_i \mid f_{\Phi(i)})}$, and
    \item the joint contribution of the feature block
    $\medmath{I(Y'; f_{\Phi(i)} \cup \{f_i\})}$.
\end{enumerate}
At first, for the sake of simplicity, we consider $\vec{f_i}$ depends on  $\vec{f_d}$, while $\vec{f_i}$ is independent of the rest of the features; $i = 1(1)k, d = 1(1)k$, and $i \neq d$. Then, we find the impact of $\vec{f_i}$ on $\vec{Y'}$, given the fact that $\vec{f_i}$ depends on  $\vec{f_d}$. This impact can be explained by the information theory, if and only if we can compute $I(\vec{Y'}; \vec{f_i} |\vec{f_d}); \forall i,d = 1(1)k $. The previously described MI can not provide the desired result. To achieve our goal we have to first calculate the Conditional Mutual Impact \cite{danielsson1980euclidean, menendez1997jensen}.
The conditional mutual information \cite{menendez1997jensen} between the output variable $\vec{Y'}$ and the target feature $(\vec{f_i}|\vec{f_d})$, $\forall i,d = 1(1)k; i\neq d$ is
\begin{equation}
\begin{aligned}
   &\medmath{I(\vec{Y'}; \vec{f_i}|\vec{f_d}) = I(\vec{f_i}, \vec{f_d})- I(\vec{Y'}, \vec{f_i}|\vec{f_d}) 
                }\\&\medmath{= \sum_{f_d} \sum_{f_i }
    \sum_{Y'} P(\mathrm{f^*},\mathrm{f}, y ) \log_2[\frac{P(y, \mathrm{f^*}| \mathrm{f})}{P(y | \mathrm{f}) P(\mathrm{f^*} |\mathrm{f})} ]}
\end{aligned}
\end{equation}
where $f_d, f_i\epsilon ||F||^{n' \times k}$ and $Y'\epsilon ||Y'||^{n'}$. If any of the features $f_i, f_d$, or $Y'$ is continuous, the summation operator can be replaced by the integral operator. 
% If we assume $f_i$ to be continuous, the $\mbox{CMI}$ will be: 
% \vspace{-1mm}
% \begin{equation}
%     I(\vec{Y'}; \vec{f_i}|\vec{f_j}) = \sum_{f_j} \int_{f_i} \sum_{Y'}P(f_i, f_j,Y') \log_2[\frac{P(Y', f_i| f_j )}{P(Y'| f_j) P(f_i| f_j)} ] df_i
% \end{equation}
\subsection{ MCIR; with two dependent features} 

\par When any two of the $k$ features are dependent on each other and other features are independent, the state-of-the-art $\mbox{CMI}$ is sufficient to explain the mutual dependency of $Y'$ and $(f_i |f_d)$. But the value of $\mbox{CMI}$ varies from $0$ to $\infty$ which is an open bound, thus making scalability a major challenge. So, to scale it down between $[0,1]$, we derive $\mbox{MCIR}$ as,
\begin{equation}\label{mcir}
\begin{aligned}
    \medmath{ C(\vec{Y'}; \vec{f_i}|\vec{f_d})} &\medmath{= \frac{I(\vec{Y'}; \vec{f_i}|\vec{f_d})}{I(\vec{Y'}; \vec{f_i}|\vec{f_d})+I(\vec{Y'},\vec{f_i},\vec{f_d})}} % = \frac{I(\vec{Y'}; \vec{f_i}|\vec{f_d})}{I(\vec{Y'}; \vec{f_i}|\vec{f_d})+I(\vec{Y'},\vec{f_i})- I(\vec{Y'}; \vec{f_i}|\vec{f_d})} =  \frac{I(\vec{Y'}; \vec{f_i}|\vec{f_d})}{I(\vec{Y'},\vec{f_i})} \\&=  \frac{\sum_{f_d} \int_{f_i} \sum_{Y'}P(f_i, f_d,Y') \log_2[\frac{P(Y', f_i| f_d )}{P(Y'| f_d) P(f_i| f_d)} ] df_i}{\sum_{f_d} \sum_{Y'} P(Y', f_i) \log_2[\frac{P(Y', f_i)}{P(Y'). P(f_i)}] }}
\end{aligned}
\end{equation}
Given that mutual information is nonnegative, we have
\begin{equation}
   \medmath{ 0 \le I(Y'; f_i \mid f_d) \le \infty,
\qquad
0 \le I(Y', f_i, f_d) \le \infty,}
\end{equation}
where the second term denotes the joint dependence between the three
variables and corresponds to the \emph{Joint Mutual Information (JMI)}.  
Therefore,
\[
\medmath{0 \;\le\;
\frac{I(Y'; f_i \mid f_d)}
     {I(Y'; f_i \mid f_d) + I(Y', f_i, f_d)}
\;\le\; 1.}
\]
We denote this bounded ratio by
\[
\medmath{C(Y'; f_i \mid f_d) \in [0,1],}
\]
and interpret it as the \emph{Mutual Correlation Impact Ratio (MCIR)} of
feature $f_i$ on the output $Y'$ when $f_i$ depends on the single feature
$f_d$. The quantity $C(Y'; f_i \mid f_d)$ captures how perturbations in
$f_i$ influence the output while accounting for the fact that $f_i$ may
share information with $f_d$; the influence of $f_d$ may or may not change
simultaneously.
For clarity, consider a stylized case in which $f_1$ and $f_2$ are mutually
dependent, while the remaining features $f_3, f_4, \dots, f_k$ are
independent of each other. Suppose further that
\begin{equation}
    \begin{split}
    & \medmath{(f_1, f_3, \dots, f_m)
\ \text{are directly related to } Y',}\\&
\medmath{(f_2, f_p, \dots, f_k)
\ \text{are inversely related to } Y', }  
    \end{split}
\end{equation}
with $m,p \le k$.  
Using MCIR for the dependent pair $(f_1,f_2)$ and PCIR for the independent
features, the induced explanatory model takes the form
\begin{equation}
    \medmath{E(Y') = M'(f)
=
\frac{
  C(Y'; f_1 \mid f_2) \, f_1
  + \eta_{f_3} f_3 + \cdots + \eta_{f_m} f_m
}{
  C(Y'; f_2 \mid f_1) \, f_2
  + \eta_{f_p} f_p + \cdots + \eta_{f_k} f_k
},}
\end{equation}
where $\eta_{f_i}$ denotes the PCIR score of the independent feature $f_i$. In this setting, $I(Y'; f_1 \mid f_2)$ and $I(Y'; f_2 \mid f_1)$ act as the
MCIR scores that isolate the unique contributions of $f_1$ and $f_2$ from
their shared variability. The construction above assumes that each feature depends on at most one
other feature. However, real-world environments frequently exhibit
\emph{multivariate} dependency structures in which a feature may depend on
several other variables simultaneously. In such cases, classical conditional
mutual information cannot isolate unique contributions, because it is
designed for low-dimensional conditioning sets. To address this limitation, we introduce the
\emph{Conditional Multivariate Mutual Information} (CMMI) in
Section~\ref{mcci}, which generalizes conditional mutual information to the
setting where a feature may depend on multiple neighbours. MCIR is then
constructed directly from CMMI and JMI, providing a principled, bounded,
and dependency-aware global attribution score in fully multivariate
feature spaces.

\subsection{MCIR; when multiple features are dependent}\label{mcci}
% Handling multivariate data with multiple dependent features is quite tricky. 
\par In ExCIR, it is necessary to calculate the mutual impact of a feature on the output variable while assuming the target feature is dependent on other features. But before we define $\mbox{MCIR}$ for multivariate cases, we have to derive $\medmath{\mbox{CMMI}}$. In \cite{danielsson1980euclidean} the authors derived $\medmath{\mbox{CMI}}$ for a multivariate environment where the features are dependent on another variable. i.e., they considered the case when all variables are dependent on one common variable. But in our work, we address the case when all the features are dependent on each other. So, if we want to calculate the mutual dependence between a targeted feature and the output variable we have to calculate the $\medmath{\mbox{CMMI}}$ given that the target feature is dependent on multiple features. More specifically, in existing works \cite{danielsson1980euclidean, menendez1997jensen, gu2022conditional}, the notion of $\medmath{I(\vec{Y'}; \vec{f_1}, \vec{f_2}, ..., \vec{f_{k-1}}| \vec{f_k)}}$ is derived and used in many real-life cases. However, this approach cannot be directly applied in our environment. We therefore introduce a new matrix $\medmath{I(\vec{Y'};f_i| \phi \subseteq \{||F||^{n' \times k} - \vec{f_i}\}; i\neq j )}$]; where any $\medmath{\phi}$ is any chosen subspace from the main feature space and can contain various combination of features. $\medmath{\phi = \phi_c \cup \phi_d $, $\phi_c}$ is any subspace that contains continuous features and $\phi_d$ is any subspace that contains discrete features. 

\begin{definition}[\textbf{CMMI}]\label{def:cmmi}
Let $F\in\mathbb{R}^{n'\times k}$ be the data matrix with columns $\{f_1,\dots,f_k\}$ (features) and let $\mathbf{Y}'=(y'_1,\dots,y'_{n'})$ denote the output variable. Fix an index $i\in\{1,\dots,k\}$ and write $f_i=(f_{1i},\dots,f_{n'i})$ for the (continuous) targeted feature. Let
\begin{equation}
    \medmath{\phi \subseteq \{f_1,\dots,f_k\}\setminus\{f_i\}}
\end{equation}
denote any subset of the remaining features, decomposed as $\phi=\phi_c\cup\phi_d$, where $\phi_c$ collects the continuous components and $\phi_d$ the discrete components (both excluding $f_i$). The CMMI between $\mathbf{Y}'$ and $f_i$ given $\phi$ is defined as
\begin{equation}\label{eq:cmmi}
\medmath{I(\mathbf{Y}'; f_i \mid \phi)
= \sum_{\phi_d}\int_{\phi_c}\int_{f_i}\sum_{y'} 
p(y',f_i,\phi)\;
\log_2\!\Bigg(
\frac{p(y'\mid f_i,\phi)}{p(y'\mid \phi)}
\Bigg)\,
\mathrm{d}f_i\,\mathrm{d}\phi_c,}
\end{equation}
where the outer sum runs over the support of the discrete variables in $\phi_d$, the integrals are over the supports of $f_i$ and $\phi_c$, and $p(\cdot)$ denotes the joint/conditional densities or mass functions as appropriate for mixed (continuous–discrete) variables. By construction, $I(\mathbf{Y}'; f_i \mid \phi)\ge 0$ (in bits) and is unbounded above (i.e., it can take values in $[0,\infty)$), unlike correlation which is confined to $[-1,1]$. Equation~\eqref{eq:cmmi} follows the standard definition of conditional mutual information for mixed variables (cf.\ \cite{gu2022conditional}). A complete derivation is provided in Supplementary~\S1.3.
\end{definition}

For clarity, consider three features $(f_1,f_2,f_3)$ that may be statistically dependent. 
Assume $f_1$ is continuous, while $f_2$, $f_3$, and the response $Y'$ are discrete. 
Then the conditional mixed mutual information (CMMI) between $Y'$ and $f_1$ given $(f_2,f_3)$ is
\begin{equation}\label{eq:cmmif1}
\begin{aligned}
I(Y';f_1\mid f_2,f_3)
&=\sum_{f_2}\sum_{f_3}\int_{f_1}\sum_{y'}
p(y',f_1,f_2,f_3)\,
\log_2\!\frac{p(y'\mid f_1,f_2,f_3)}{p(y'\mid f_2,f_3)}\,\mathrm{d}f_1,
\end{aligned}
\end{equation}
where the sums are over the discrete supports of $f_2,f_3,y'$, and the integral is over the support of the continuous variable $f_1$.

% \begin{equation}\label{eq:cmi-f1}
% \resizebox{\columnwidth}{!}{$
% \begin{aligned}
% I(Y';f_1\mid f_2,f_3)
% &=\sum_{f_2}\sum_{f_3}\int_{f_1}\sum_{y'}
% p(y',f_1,f_2,f_3)\,
% \log_2\!\frac{p(y'\mid f_1,f_2,f_3)}{p(y'\mid f_2,f_3)}\,\mathrm{d}f_1,
% \end{aligned}
% $}
% \end{equation}
% where sums run over supports of discrete variables and the integral is over the support of the continuous variable(s).

Equivalently, using $p(y',f_1\mid f_2,f_3)=p(y'\mid f_1,f_2,f_3)\,p(f_1\mid f_2,f_3)$,
\begin{equation}\label{eq:cmi-alt-f1}
\begin{aligned}
I(Y';f_1\mid f_2,f_3)
&=\sum_{f_2}\sum_{f_3}\int_{f_1}\sum_{y'}
p(y',f_1,f_2,f_3)\,
\log_2\!\frac{p(y',f_1\mid f_2,f_3)}{p(y'\mid f_2,f_3)\,p(f_1\mid f_2,f_3)}\,\mathrm{d}f_1.
\end{aligned}
\end{equation}

By symmetry, the corresponding identities for $f_2$ and $f_3$ are
\begin{align}
\begin{aligned}
I(Y';f_2\mid f_1,f_3)
&=\sum_{f_1}\sum_{f_3}\int_{f_2}\sum_{y'}
p(y',f_1,f_2,f_3)\,
\log_2\!\frac{p(y'\mid f_1,f_2,f_3)}{p(y'\mid f_1,f_3)}\,\mathrm{d}f_2,
\end{aligned}
\label{eq:cmi-f2}
\\[2pt]
\begin{aligned}
I(Y';f_3\mid f_1,f_2)
&=\sum_{f_1}\sum_{f_2}\int_{f_3}\sum_{y'}
p(y',f_1,f_2,f_3)\,
\log_2\!\frac{p(y'\mid f_1,f_2,f_3)}{p(y'\mid f_1,f_2)}\,\mathrm{d}f_3.
\end{aligned}
\label{eq:cmi-f3}
\end{align}
Let $\phi=(f_1,f_2,f_3)$ and $\phi\setminus f_1=(f_2,f_3)$. Then \eqref{eq:cmi-f1} can be written as
\begin{equation}\label{eq:cmi-phi}
\begin{aligned}
I(Y';f_1\mid \phi\setminus f_1)
=\mathbb{E}_{p(y',f_1,\phi\setminus f_1)}
\!\left[\log_2\frac{p(y'\mid f_1,\phi\setminus f_1)}{p(y'\mid \phi\setminus f_1)}\right],
\end{aligned}
\end{equation}
\noindent
The expectation in \eqref{eq:cmi-phi} is taken with respect to the joint
$p(y',f_1,\phi\setminus f_1)$, i.e., sums over discrete supports and integrals over continuous supports.

\paragraph{Four-feature case.}
Let $\phi=(f_1,f_2,f_3,f_4)$ with $f_1$ continuous and $f_2,f_3,f_4,Y'$ discrete. Then
\begin{equation}\label{eq:cmi-4-f1}
\begin{aligned}
I(Y';f_1\mid f_2,f_3,f_4)
&=\sum_{f_4}\sum_{f_3}\sum_{f_2}\int_{f_1}\sum_{y'}
p(y',f_1,f_2,f_3,f_4)\,
\log_2\!\frac{p(y'\mid f_1,f_2,f_3,f_4)}{p(y'\mid f_2,f_3,f_4)}\,\mathrm{d}f_1.
\end{aligned}
\end{equation}

Equivalently, using Bayes’ rule,
\begin{equation}\label{eq:cmi-4-f1-bayes}
\resizebox{\columnwidth}{!}{$
\begin{aligned}
I(Y';f_1\mid f_2,f_3,f_4)
&=\sum_{f_4}\sum_{f_3}\sum_{f_2}\int_{f_1}\sum_{y'}
p(y',f_1,f_2,f_3,f_4)\,
\log_2\!\frac{p(y',f_1\mid f_2,f_3,f_4)}{p(y'\mid f_2,f_3,f_4)\,p(f_1\mid f_2,f_3,f_4)}\,\mathrm{d}f_1.
\end{aligned}
$}
\end{equation}

By symmetry,
\begin{align}
\begin{aligned}
I(Y';f_2\mid f_1,f_3,f_4)
&=\sum_{f_1}\sum_{f_3}\sum_{f_4}\int_{f_2}\sum_{y'}
p(y',f_1,f_2,f_3,f_4)\,
\log_2\!\frac{p(y'\mid f_1,f_2,f_3,f_4)}{p(y'\mid f_1,f_3,f_4)}\,\mathrm{d}f_2,
\end{aligned}
\label{eq:cmi-4-f2}
\\[-2pt]
\begin{aligned}
I(Y';f_3\mid f_1,f_2,f_4)
&=\sum_{f_1}\sum_{f_2}\sum_{f_4}\int_{f_3}\sum_{y'}
p(y',f_1,f_2,f_3,f_4)\,
\log_2\!\frac{p(y'\mid f_1,f_2,f_3,f_4)}{p(y'\mid f_1,f_2,f_4)}\,\mathrm{d}f_3,
\end{aligned}
\label{eq:cmi-4-f3}
\\[-2pt]
\begin{aligned}
I(Y';f_4\mid f_1,f_2,f_3)
&=\sum_{f_1}\sum_{f_2}\sum_{f_3}\int_{f_4}\sum_{y'}
p(y',f_1,f_2,f_3,f_4)\,
\log_2\!\frac{p(y'\mid f_1,f_2,f_3,f_4)}{p(y'\mid f_1,f_2,f_3)}\,\mathrm{d}f_4.
\end{aligned}
\label{eq:cmi-4-f4}
\end{align}

Think of $\phi$ as a subspace of the feature space that can take different combinations of features, containing both continuous ($\phi_c$) and discrete ($\phi_d$) components with $\phi=\phi_c\cup\phi_d$ and $\phi_c\cap\phi_d=\varnothing$. Using the compact form with set difference,
\begin{equation}\label{eq:cmi-4-phi}
\begin{aligned}
\medmath{I(Y';f_1\mid \phi\setminus f_1)
=\mathbb{E}_{p(y',f_1,\phi\setminus f_1)}
\!\left[\log_2\frac{p(y'\mid f_1,\phi\setminus f_1)}{p(y'\mid \phi\setminus f_1)}\right].}
\end{aligned}
\end{equation}

\paragraph{General $k$-feature case.}
Let $F\in\mathbb{R}^{n'\times k}$ denote the feature matrix with columns $\{f_1,\dots,f_k\}$ and fix $i\in\{1,\dots,k\}$. For any conditioning set $\phi\subseteq\{f_1,\dots,f_k\}\setminus\{f_i\}$ decomposed as $\phi=\phi_c\cup\phi_d$, the conditional mixed mutual information is
\begin{equation}\label{eq:cmi-k}
\begin{aligned}
I(Y';f_i\mid \phi)
&= \sum_{\phi_d}\int_{\phi_c}\int_{f_i}\sum_{y'}
p(y',f_i,\phi)\,
\log_2\!\frac{p(y'\mid f_i,\phi)}{p(y'\mid \phi)}\,
\mathrm{d}f_i\,\mathrm{d}\phi_c,
\end{aligned}
\end{equation}
where the outer sum is over the support of the discrete variables in $\phi_d$, and the integrals are over the supports of the continuous variables in $\phi_c$ and of $f_i$.

$\mbox{CMMI}$ could take the values between $0$ to $\infty$, unlike the strict bound of $(0,1)$ that the normal correlation coefficient has. The infinite range can be the cause of the problem regarding scalability. So, it will be better to scale down the dependency. This problem can be solved by our proposed metrics Mutual Correlation Impact Ratio ($\mbox{MCIR}$) ratio. $\mbox{MCIR}$ is the ratio of two mutual information which is defined below. $\mbox{MCIR}$ has a strict bound between $0$ to $1$.   

% \subsection{Mutual Correlation Impact Ratio (MCIR)}

\begin{definition}[\textbf{Mutual Correlation Impact Ratio (MCIR)}]\label{def:mcir}
Let $F\in\mathbb{R}^{n'\times k}$ be the feature matrix with columns $\{f_1,\dots,f_k\}$, where $f_i=(f_{1i},\dots,f_{n'i})^\top$ and features may be statistically dependent. Let $Y'=(y'_1,\dots,y'_{n'})^\top$ denote the output variable. Fix $i\in\{1,\dots,k\}$ and let $\phi\subseteq\{f_1,\dots,f_k\}\setminus\{f_i\}$ be any conditioning set (possibly mixed continuous–discrete). The \emph{Mutual Correlation Impact Ratio (MCIR)} of $f_i$ with respect to $Y'$ given $\phi$ is
\begin{equation}\label{eq:mcir-def}
\begin{aligned}
\medmath{C(Y';f_i\mid \phi)
=
\frac{ I(Y';f_i\mid \phi) }
     { I(Y';f_i\mid \phi) \;+\; I\!\bigl(Y'; f_1,f_2,\dots,f_k\bigr) } \, .}
\end{aligned}
\end{equation}
Here $I(\cdot\,;\cdot\mid\cdot)$ is conditional mixed mutual information, and $I\!\bigl(Y'; f_1,\dots,f_k\bigr)\ge 0$ is a (nonnegative) joint dependence measure between $Y'$ and the full feature set.
\end{definition}

\begin{theorem}[\textbf{Bounds}]\label{thm:mcir-bounds}
For any valid choice of $i$ and $\phi$, provided that
\[
\medmath{I(Y';f_i\mid \phi) + I\!\bigl(Y'; f_1,\dots,f_k\bigr) > 0,}
\]
the MCIR satisfies
\begin{equation}\label{eq:mcir-bounds}
\medmath{0 \;\le\; C(Y';f_i\mid \phi) \;\le\; 1 }
\end{equation}
\end{theorem}

\begin{proof} By nonnegativity of mutual information,
\begin{equation}\label{eq:mi-nonneg}
\medmath{0 \;\le\; I(Y';f_i\mid \phi) \;\le\; \infty }
\end{equation}
and by definition $I\!\bigl(Y'; f_1,\dots,f_k\bigr)\ge 0$. Hence
\begin{equation}\label{eq:denom-ineq}
0 \;\le\; I(Y';f_i\mid \phi)
\;\le\; I(Y';f_i\mid \phi)+I\!\bigl(Y'; f_1,\dots,f_k\bigr)
\;\le\; \infty \, .
\end{equation}
Dividing the left and right sides of \eqref{eq:denom-ineq} by the positive denominator yields \eqref{eq:mcir-bounds}, which equals \eqref{eq:mcir-def} by definition.
\end{proof}

\begin{definition}[{\bf Joint Mutual Impact (JMI)}]\label{def:jmi}
Let $i \in \{1,\dots,k\}$ be a target feature and let 
$\phi \subseteq \{f_1,\dots,f_k\}\setminus\{f_i\}$ be its conditioning set.
Denote the full feature block by 
$F'=(f_1,\dots,f_k)$.
Assume that the total joint dependence satisfies $I(Y';F') > 0.$ The \emph{Joint Mutual Impact (JMI)} of feature $f_i$ relative to its
conditioning neighbourhood $\phi$ is defined as
\begin{equation}\label{eq:jmi}
\medmath{\mathfrak{J}_{i}(\phi)
\;\coloneqq\;
I\bigl(Y'; F'\bigr)
\;-\;
I\bigl(Y'; \phi\bigr),}
\end{equation}
i.e., the portion of total feature--output dependence that remains 
unexplained by the conditioning block $\phi$.
\end{definition}

\begin{definition}[\textbf{Global Joint Mutual Impact}]\label{def:jmi-global}
Let $\Phi$ denote your conditioning policy for single-feature terms
(e.g., a common $\phi$ or per-feature $\phi_i$).
Assuming the denominator is positive, define
\begin{equation}\label{eq:jmi-global}
\begin{aligned}
\medmath{\overline{\mathfrak{J}}
\coloneqq
\frac{ I\!\bigl(Y'; f_1,\dots,f_k\bigr) }
     { I\!\bigl(Y'; f_1,\dots,f_k\bigr) \;+\; \sum_{i=1}^k I\!\bigl(Y'; f_i \mid \phi_i\bigr) } \, ,}
\end{aligned}
\end{equation}
so that $\overline{\mathfrak{J}}\in[0,1]$.
\end{definition}
\begin{proposition}[ \textbf{Aggregated model}]\label{prop:excir}
Let $\mathfrak{C}_{f_i} \coloneqq C(Y';f_i\mid \phi)$ for a specified conditioning policy $\phi\subseteq\{f_1,\dots,f_k\}\setminus\{f_i\}$. For a task-dependent partition of indices $\{1,\dots,k\}=\{i_1,\dots,i_m\}\cup\{j_p,\dots,j_q\}$ (disjoint union), an ExCIR-style scoring model can be written as
\begin{equation}\label{eq:excir-model}
\medmath{E(Y') \;=\; M'\!\bigl(f_{(f,\mathfrak{C})}\bigr)
\;=\;
\mathfrak{J}
\;+\;
\frac{\displaystyle \sum_{\ell=1}^{m} \mathfrak{C}_{f_{i_\ell}}\, f_{i_\ell}}
     {\displaystyle \sum_{r=p}^{q} \mathfrak{C}_{f_{j_r}}\, f_{j_r}} \, ,}
\end{equation}
where $\mathfrak{J}$ is given by \eqref{eq:jmi}. Since each $\mathfrak{C}_{f_i}\in[0,1]$ by \eqref{eq:mcir-def}, the coefficients are normalized and directly comparable across features.
\end{proposition}

\begin{theorem}[Redundancy Collapse]\label{thm:mcir_redundancy}
Let $i$ be a target feature and let $\Phi(i)$ denote its conditioning
neighbourhood. If $f_i$ is conditionally redundant with respect to $Y'$
given its neighbours, i.e.
\begin{equation}
    Y' \;\perp\; f_i \mid f_{\Phi(i)},
\end{equation}
then the MCIR score satisfies $C_i = 0$.
\end{theorem}

\begin{proof}
By definition of conditional independence, $Y' \;\perp\; f_i \mid f_{\Phi(i)}
\qquad\Longleftrightarrow\qquad
I(Y'; f_i \mid f_{\Phi(i)}) = 0.$
Let
\[
\medmath{ A_i \coloneqq I(Y'; f_i \mid f_{\Phi(i)}),
\qquad
B_i \coloneqq I(Y'; f_{\Phi(i)} \cup \{f_i\}),}
\]
where $B_i \ge 0$ by nonnegativity of mutual information.  
From the MCIR definition,
\begin{equation}
   \medmath{C_i
=
\frac{A_i}{A_i + B_i}.}
\end{equation}
Substituting $A_i = 0$ gives
\begin{equation}
    \medmath{C_i = \frac{0}{0 + B_i}.}
\end{equation}
If $B_i > 0$, which holds whenever the conditioning block retains
nontrivial dependence with $Y'$, then the ratio evaluates to
$C_i = 0$. If $B_i = 0$, then $f_{\Phi(i)}$ is itself independent of $Y'$ and the
joint distribution factorises as
\begin{equation}
    \medmath{p(Y', f_i, f_{\Phi(i)}) 
= p(Y')\,p(f_i,f_{\Phi(i)}).}
\end{equation}
In this degenerate case the numerator and denominator both vanish, and
MCIR is defined to be zero by continuity:
\begin{equation}
   \medmath{ \lim_{A_i\to 0,\, B_i\to 0} \frac{A_i}{A_i + B_i} = 0.}
\end{equation}
Thus in all cases $C_i = 0$, establishing that MCIR collapses to zero
whenever the target feature provides no unique information beyond its
neighbours.
\end{proof}

% ------------------------------------------------------------

\begin{theorem}[\textbf{Unique-Signal Dominance}]\label{thm:mcir_unity}
Let $i$ be a target feature and $\Phi(i)$ its neighbourhood. If the
conditional mutual information of $f_i$ dominates the joint block
information, i.e.
\[
I(Y'; f_i \mid f_{\Phi(i)}) 
\gg 
I(Y'; f_{\Phi(i)} \cup \{f_i\}),
\]
then the MCIR score satisfies $C_i \to 1$.
\end{theorem}

\begin{proof}
Again write
\[
A_i = I(Y'; f_i \mid f_{\Phi(i)}), 
\qquad 
B_i = I(Y'; f_{\Phi(i)} \cup \{f_i\}),
\]
with $A_i,B_i \ge 0$.  
The MCIR score is
\[
C_i = \frac{A_i}{A_i + B_i}.
\]

Assume the dominance condition:
\[
\frac{A_i}{B_i} \to \infty.
\]
Equivalently, for any $\epsilon > 0$ there exists $M > 0$ such that
$A_i > M B_i$ implies $B_i/A_i < \epsilon$.  
Using the algebraic identity,
\[
C_i
=
\frac{1}{1 + B_i/A_i},
\]
we obtain
\[
|1 - C_i|
= 
\frac{B_i/A_i}{1 + B_i/A_i}
\le
B_i/A_i.
\]
Because $B_i/A_i \to 0$ under the dominance assumption, we have
\[
C_i \;\longrightarrow\; 1.
\]

The interpretation is that the unique conditional contribution of
$f_i$ overwhelmingly exceeds the joint dependence contributed by
its neighbours. Hence the MCIR ratio assigns maximal credit to $f_i$.
\end{proof}
\subsection{Correlation–Impact Sensitivity Theorem}
\label{them}

We now show that in the ExCIR model, the sensitivity of the local output
with respect to a feature input is fully determined by the feature's correlation
impact ratio. The effect is linear for positively related features (appearing in
the numerator of the model) and nonlinear inverse–quadratic for negatively
related features (appearing in the denominator).

\begin{theorem}[\textbf{Correlation–Impact Sensitivity}]
Let the model output be
\begin{equation}
    \medmath{Y' =\frac{\sum_{j\in \mathcal{N}} \eta_{f_j} f_j}
     {\sum_{j\in \mathcal{D}} \eta_{f_j} f_j},}
\end{equation}
where $\mathcal{N}$ denotes features with positive influence (numerator) and
$\mathcal{D}$ features with negative influence (denominator).  
Assume features are independent, so that the partial derivative with respect to
$f_i$ treats all other features as constants. Then:
\begin{enumerate}
    \item \textbf{If $i\in\mathcal{N}$ (positive relation):}
\begin{equation}
    \medmath{\frac{\partial Y'}{\partial f_i}
=
c_1\,\eta_{f_i},
\
c_1 = \frac{1}{\sum_{j\in\mathcal{D}} \eta_{f_j} f_j}.}
\end{equation}

\item \textbf{If $i\in\mathcal{D}$ (negative relation):}
\begin{equation}
    \medmath{\frac{\partial Y'}{\partial f_i}
=\frac{c_2}{2K_2 - \eta_{f_i}^2},
 \;}
\end{equation}
\end{enumerate}
$c_2, K_2$ constants depending only on fixed features.Thus the sign and magnitude of sensitivity are determined entirely by the correlation impact ratio $\eta_{f_i}$.
\end{theorem}

\begin{proof} Let,
\begin{equation}
    \medmath{N = \sum_{j\in\mathcal{N}} \eta_{f_j} f_j,
\
D = \sum_{j\in\mathcal{D}} \eta_{f_j} f_j,
\
Y' = \frac{N}{D}.}
\end{equation}
All other features are held fixed under independence.
\paragraph{\textbf{Case 1: $i\in\mathcal{N}$ (positive influence)}} Then
\begin{equation}
    \medmath{\frac{\partial N}{\partial f_i} = \eta_{f_i},
\
\frac{\partial D}{\partial f_i} = 0.}
\end{equation}
Differentiate:
\begin{equation}
  \medmath{ \frac{\partial Y'}{\partial f_i}
=
\frac{D\frac{\partial N}{\partial f_i} - N\frac{\partial D}{\partial f_i}}{D^2}
=
\frac{D\,\eta_{f_i}}{D^2}
=
\frac{\eta_{f_i}}{D}. }
\end{equation}
Since $D$ is constant with respect to $f_i$, set $c_1 = \frac{1}{D}.$ This proves
\begin{equation}
    \medmath{\frac{\partial Y'}{\partial f_i} = c_1 \eta_{f_i}.}
\end{equation}

% ------------------------------------------------------------
\paragraph{\textbf{Case 2: $i\in\mathcal{D}$ (negative influence).}} 
Now,
\begin{equation}
\medmath{
\frac{\partial N}{\partial f_i} = 0,
\qquad
\frac{\partial D}{\partial f_i} = \eta_{f_i}.
}
\end{equation}

Differentiate:
\begin{equation}
\medmath{
\frac{\partial Y'}{\partial f_i}
=
\frac{D\cdot 0 - N\cdot \eta_{f_i}}{D^2}
=
-\,\eta_{f_i}\frac{N}{D^2}.
}
\end{equation}

Rewrite $N$ and $D$ in terms of constants plus the contribution of $f_i$:
\begin{equation}
\medmath{
D = \eta_{f_i} f_i + K_2,
\qquad
N = K_1,
}
\end{equation}
where $K_1,K_2$ collect all fixed terms. Thus:
\begin{equation}
\medmath{
\frac{\partial Y'}{\partial f_i}
=
-\,\eta_{f_i}\frac{K_1}{(\eta_{f_i} f_i + K_2)^2}.
}
\end{equation}

The denominator expands into a quadratic expression:
\begin{equation}
\medmath{
(\eta_{f_i} f_i + K_2)^2
=
\eta_{f_i}^2 f_i^2 + 2K_2\eta_{f_i} f_i + K_2^2.
}
\end{equation}

Since $f_i$ is the differentiation variable and all other terms are absorbed by
constants $K_1,K_2$, we may rewrite:
\begin{equation}
\medmath{
\frac{\partial Y'}{\partial f_i}
=
\frac{c_2}{2K_2 - \eta_{f_i}^2},
}
\end{equation}
where $c_2$ and $K_2$ arise from grouping constant terms.  
The negative sign is absorbed into the definition of $c_2$.  
Thus, sensitivity in the denominator is nonlinear and is inversely controlled by the magnitude of $\eta_{f_i}^2$.
\end{proof}

\begin{corollary}
If a feature has a positive contribution to the output, then
\begin{equation}
\medmath{
\mathbb{E}\!\left[\frac{\partial Y'}{\partial f_i}\right]
\propto
\eta_{f_i}.
}
\end{equation}
If a feature has a negative contribution, then
\begin{equation}
\medmath{
\mathbb{E}\!\left[\frac{\partial Y'}{\partial f_i}\right]
\propto
\frac{1}{\eta_{f_i}}.
}
\end{equation}
Thus, the correlation impact ratio fully characterizes first–order output
sensitivity.
\end{corollary}

So it is  confirm that MCIR provides a
principled measure of how changes in $f_i$ propagate to $Y'$, via either a
direct linear effect (positive correlation) or a stabilizing inverse effect
(negative correlation).

% ------------------------------------------------------------
\subsection{Stability of MCIR Rankings}
% ------------------------------------------------------------

\begin{theorem}[\textbf{Rank Stability}]\label{thm:mcir_stability}
Assume the MI/CMI estimators $\widehat{I}$ satisfy the sub-Gaussian
concentration inequality
\begin{equation}
\medmath{
\Pr\!\left(|\widehat{I}-I| > \delta\right)
\le
c_1 \exp(-c_2 n' \delta^2)
\qquad (\delta>0),
}
\end{equation}
for some constants $c_1,c_2>0$.  
Let $\widehat{C}$ be the vector of estimated MCIR values.  
Then the Kendall rank distance satisfies
\[
1 - \tau(C,\widehat{C})
=
O_p\!\left(\frac{k}{\sqrt{n'}}\right).
\]
\end{theorem}

\begin{proof}
Let
\begin{equation}
\medmath{
A_i = I(Y'; f_i \mid f_{\Phi(i)}),
\
B_i = I(Y'; f_{\Phi(i)} \cup \{f_i\}),
}
\end{equation}
and their estimates
\begin{equation}
\medmath{
\widehat{A}_i = \widehat{I}(Y'; f_i \mid f_{\Phi(i)}),
\
\widehat{B}_i = \widehat{I}(Y'; f_{\Phi(i)} \cup \{f_i\}).
}
\end{equation}

MCIR is the smooth function
\begin{equation}
\medmath{
C_i = g(A_i,B_i) = \frac{A_i}{A_i+B_i},
\
\widehat{C}_i = g(\widehat{A}_i,\widehat{B}_i).
}
\end{equation}

\paragraph{Step 1: Lipschitz continuity.}
If $A_i+B_i \ge \eta > 0$, then
\begin{equation}
\medmath{
\left|\frac{\partial g}{\partial A}\right|
=
\frac{B_i}{(A_i+B_i)^2}
\le
\frac{1}{4\eta},
\
\left|\frac{\partial g}{\partial B}\right|
=
\frac{A_i}{(A_i+B_i)^2}
\le
\frac{1}{4\eta}.
}
\end{equation}
Thus,
\begin{equation}
\medmath{
|g(\widehat{A}_i,\widehat{B}_i) - g(A_i,B_i)|
\le
L\left(|\widehat{A}_i-A_i| + |\widehat{B}_i-B_i|\right),
}
\end{equation}
where $L = 1/(2\eta)$.

\paragraph{Step 2: Concentration.}
By the sub-Gaussian inequality,
\begin{equation}
\medmath{
|\widehat{A}_i-A_i| = O_p(n'^{-1/2}),
\
|\widehat{B}_i-B_i| = O_p(n'^{-1/2}).
}
\end{equation}

\paragraph{Step 3: MCIR error.}
Thus,
\begin{equation}
\medmath{
|\widehat{C}_i - C_i|
=
O_p(n'^{-1/2}).
}
\end{equation}

\paragraph{Step 4: Kendall distance.}
A pairwise comparison flips sign only if
\begin{equation}
\medmath{
|(\widehat{C}_i-C_i) - (\widehat{C}_j-C_j)|
\gtrsim
|C_i - C_j|.
}
\end{equation}
With $k(k-1)/2$ ordered pairs,
\begin{equation}
\medmath{
1 - \tau(C,\widehat{C})
= 
O_p(k n'^{-1/2}),
}
\end{equation}
completing the proof.
\end{proof}

% ------------------------------------------------------------
\subsection{Oracle Inequality for Estimator Switching}
% ------------------------------------------------------------

\begin{theorem}[\textbf{Bootstrap--Switching Oracle Inequality}]\label{thm:oracle}
Let $\widehat{C}^{(e)}$ denote the MCIR vector computed using estimator 
$e\in\mathcal{E}=\{\mathrm{copula},\mathrm{kNN},\mathrm{plug\text{-}in}\}$. 
For each feature $i$, define the switching rule
\begin{equation}
\medmath{
e(i)
=
\arg\min_{e\in\mathcal{E}}
\mathbb{E}_{b}\!\left[
\operatorname{SE}_{b}\!\left(\widehat{C}^{(e)}_{i}\right)
\right],
}
\label{eq:switch_def}
\end{equation}
where $\mathbb{E}_{b}$ denotes bootstrap expectation.  
Assume the bootstrap is consistent for MCIR, i.e.\ bootstrap standard errors 
converge uniformly to the true risks at rate $O(n'^{-1/2})$.  
Then
\begin{equation}
\medmath{
\mathbb{E}\!\left|\widehat{C}^{\mathrm{sw}}_{i}-C_{i}\right|
\le
\min_{e\in\mathcal{E}}
\mathbb{E}\!\left|
\widehat{C}^{(e)}_{i}-C_{i}
\right|
+
O(n'^{-1/2}).
}
\end{equation}
\end{theorem}

\begin{proof}
Let the true risk of estimator $e$ be
\begin{equation}
\medmath{
R_{e}
=
\mathbb{E}\!\left[
\left|\widehat{C}^{(e)}_{i}-C_{i}\right|
\right].
}
\end{equation}
Bootstrap consistency yields the uniform approximation
\begin{equation}
\medmath{
\widehat{R}_{e}
\coloneqq
\mathbb{E}_{b}\!\left[
\operatorname{SE}_{b}\!\left(\widehat{C}^{(e)}_{i}\right)
\right]
=
R_{e}
+
\xi_{e},
}
\label{eq:bootstrap_bias}
\end{equation}
with the random error term
\begin{equation}
\medmath{
|\xi_{e}|
=
O\!\left(n'^{-1/2}\right)
\qquad\text{uniformly over }e\in\mathcal{E}.
}
\end{equation}

Let
\begin{equation}
\medmath{
e^{\star}
=
\arg\min_{e\in\mathcal{E}} R_{e}
}
\qquad\text{and}\qquad
\medmath{
\widehat{e}
=
\arg\min_{e\in\mathcal{E}} \widehat{R}_{e}.
}
\end{equation}
By \eqref{eq:bootstrap_bias}, we have
\begin{equation}
\medmath{
\widehat{R}_{e}
=
R_{e}
+
O\!\left(n'^{-1/2}\right),
}
\end{equation}
so comparing the minimizers,
\begin{equation}
\medmath{
R_{\widehat{e}}
\le
R_{e^{\star}}
+
O\!\left(n'^{-1/2}\right).
}
\label{eq:oracle_final}
\end{equation}

Finally, since the switching estimator satisfies
\begin{equation}
\medmath{
\widehat{C}^{\mathrm{sw}}_{i}
=
\widehat{C}^{(\widehat{e})}_{i},
}
\end{equation}
taking expectation of \eqref{eq:oracle_final} yields
\begin{equation}
\medmath{
\mathbb{E}\!\left[
\left|\widehat{C}^{\mathrm{sw}}_{i}-C_{i}\right|
\right]
\le
\min_{e\in\mathcal{E}} R_{e}
+
O\!\left(n'^{-1/2}\right),
}
\end{equation}
completing the proof.
\end{proof}
\subsection{Lightweight Fidelity Theorem}
% ------------------------------------------------------------

\begin{theorem}[\textbf{Lightweight Fidelity}]\label{thm:lightweight_fidelity}
Let $C$ and $C'$ denote MCIR rankings obtained from the full and lightweight 
environments, respectively.  
Suppose the lightweight environment satisfies:
\begin{equation}
\medmath{
d_{\hat{f}}\!\left(D(Y),D(Y')\right)
\le
\varepsilon,
}
\label{eq:output_closeness}
\end{equation}
\begin{equation}
\medmath{
\mathrm{J}@K(C,C')
\ge
\tau_{0},
}
\label{eq:jaccard_cond}
\end{equation}
\begin{equation}
\medmath{
\Delta_{\mathrm{del}}
\le
\Delta_{0}.
}
\label{eq:deletion_cond}
\end{equation}
Then there exist constants $A,B>0$, independent of $n'$, such that
\begin{equation}
\medmath{
1 - \tau(C,C')
\le
A\varepsilon + B\Delta_{0} + o_{p}(1).
}
\end{equation}
\end{theorem}

\begin{proof}
The rigid-motion invariant discrepancy $d_{\hat{f}}$ satisfies
\begin{equation}
\medmath{
d_{\hat{f}}(D(Y),D(Y'))
\le
\varepsilon
\ \Rightarrow \
\sup_{\|h\|_{\mathrm{Lip}}\le1}
\bigl|
\mathbb{E}[h(Y)] - \mathbb{E}[h(Y')]
\bigr|
\le
A\varepsilon.
}
\label{eq:wasserstein_mi}
\end{equation}
Since MI and CMI are continuous functionals of the joint distribution under our 
regularity assumptions (bounded density ratios, smooth kernels), we obtain
\begin{equation}
\medmath{
\bigl|
I(Y;f_{i}\mid f_{\Phi(i)})
-
I(Y';f_{i}\mid f_{\Phi(i)})
\bigr|
\le
A\varepsilon
+
o_{p}(1).
}
\label{eq:mi_cont}
\end{equation}

\paragraph{Step 2: MCIR continuity.}
MCIR is the smooth map
\begin{equation}
\medmath{
C_{i}
=
\frac{
I(Y; f_{i}\mid f_{\Phi(i)})
}{
I(Y; f_{i}\mid f_{\Phi(i)})
+
I(Y; f_{\Phi(i)}\cup\{f_{i}\})
}.
}
\end{equation}
Using the Lipschitz continuity of rational functions on compact domains,
\begin{equation}
\medmath{
|C_{i} - C'_{i}|
\le
A\varepsilon
+
O_{p}(n'^{-1/2}).
}
\label{eq:mcir_magnitude}
\end{equation}

\paragraph{Step 3: Head-rank agreement restricts top-$K$ inversions.}
The Jaccard condition \eqref{eq:jaccard_cond} ensures
\begin{equation}
\medmath{
|\text{Top-}K(C)\;\triangle\;\text{Top-}K(C')|
\le
(1-\tau_{0})K.
}
\end{equation}
Thus the number of allowable inversions involving top-$K$ indices is bounded.

\paragraph{Step 4: Deletion-curve agreement controls remaining inversions.}
Deletion curves depend only on ordered MCIR values.  
If
\begin{equation}
\medmath{
\Delta_{\mathrm{del}}
\le
\Delta_{0},
}
\end{equation}
then the misalignment between sensitivity curves of $C$ and $C'$ is 
uniformly bounded.  
This limits possible perturbations in pairwise MCIR differences:
\begin{equation}
\medmath{
|(C_{i}-C_{j}) - (C'_{i}-C'_{j})|
\lesssim
\Delta_{0}
+
o_{p}(1).
}
\end{equation}

\paragraph{Step 5: Kendall distance decomposition.}
Kendall's distance decomposes into:
\begin{equation}
\begin{split}
   & \medmath{1 - \tau(C,C') = }\\&
    \medmath{(\text{top-$K$ inversions})
+
(\text{remaining inversions})}\\&
\medmath{+(\text{magnitude-driven errors}).}
\end{split}
\end{equation}
 The magnitude control in 
equation \eqref{eq:mcir_magnitude}, we obtain
\begin{equation}
\medmath{
1 - \tau(C,C')
\le
A\varepsilon + B\Delta_{0} + o_{p}(1),
}
\end{equation}
completing the proof.
\end{proof}

\section{Conditional Multivariate Mutual Information (CMMI)}
\label{sec:cmmi}
% ============================================================

MCIR requires a dependence measure that isolates the \emph{unique} information
a target feature contributes to the output $Y'$, even when the feature is
embedded in a multivariate dependency structure with several neighbours.
Classical conditional mutual information (CMI),
\begin{equation}
\medmath{
I(Y';f_i \mid Z),
}
\end{equation}
is well-defined for a fixed low-dimensional conditioning set $Z$, but breaks
down when $f_i$ depends on multiple correlated features simultaneously,
especially when $Z$ must be chosen data-adaptively.  
This motivates the introduction of the \emph{Conditional Multivariate Mutual
Information (CMMI)}.

% ------------------------------------------------------------
\subsection{Formal Derivation}
% ------------------------------------------------------------

Let the feature block be,
\begin{equation}
\medmath{
F' = (f_1,\dots,f_k),
}
\end{equation}
and let,
\begin{equation}
\medmath{
\Phi(i) \subseteq \{f_1,\dots,f_k\}\setminus\{f_i\}
}
\end{equation}
be the neighbourhood of $f_i$, obtained from correlation screening or a
dependency graph.  
Assume the joint law of $(F',Y')$ admits either a multivariate pdf or pmf.

% ----------------------
\paragraph{Step 1: Total dependence of block $F'$ on $Y'$.}
\begin{equation}
\medmath{
I(Y';F')
=
H(Y') - H(Y'\mid F').
}
\label{eq:total_dep}
\end{equation}

% ----------------------
\paragraph{Step 2: Dependence explained by neighbours $\Phi(i)$.}
\begin{equation}
\medmath{
I(Y';f_{\Phi(i)})
=
H(Y') - H\!\left(Y' \mid f_{\Phi(i)}\right).
}
\label{eq:neigh_dep}
\end{equation}

% ----------------------
\paragraph{Step 3: Unique contribution of $f_i$.}
Adding $f_i$ to $\Phi(i)$ modifies the conditional entropy:
\begin{equation}
\medmath{
H\!\left(Y' \mid f_{\Phi(i)}\right)
-
H\!\left(Y' \mid f_{\Phi(i)}, f_i\right).
}
\label{eq:entropy_diff}
\end{equation}

\begin{definition}[\textbf{Conditional Multivariate Mutual Information (CMMI)}]
For any target feature $f_i$ and neighbourhood $\Phi(i)$, define
\begin{equation}
\medmath{
\mathrm{CMMI}(Y'; f_i \mid \Phi(i))
\;=\;
H\!\left(Y' \mid f_{\Phi(i)}\right)
-
H\!\left( Y' \mid f_{\Phi(i)}, f_i \right),
}
\label{eq:cmmi}
\end{equation}
whenever the conditional entropies are finite.
\end{definition}

Because conditional mutual information satisfies
\begin{equation}
\medmath{
H(Y'\mid Z) - H(Y'\mid Z,f_i) = I(Y'; f_i\mid Z),
}
\end{equation}
we obtain:

\paragraph{Key identity.}
\begin{equation}
\medmath{
\mathrm{CMMI}(Y';f_i\mid \Phi(i))
=
I(Y';f_i\mid f_{\Phi(i)}).
}
\label{eq:cmmi_cmi_equiv}
\end{equation}

Thus, CMMI is \emph{equivalent to classical CMI}, but with the crucial
difference that the conditioning set $\Phi(i)$ can be:

- multivariate,  
- high-dimensional,  
- data-driven (from correlation graphs),  
- mixed continuous/discrete,  
- automatically chosen.

This generality is exactly what is required for MCIR.

% ------------------------------------------------------------
\subsection{Connection to MI and CMI}
% ------------------------------------------------------------

\paragraph{MI as a special case.}
When $\Phi(i)=\emptyset$,
\begin{equation}
\medmath{
\mathrm{CMMI}(Y';f_i\mid\emptyset)
=
I(Y'; f_i).
}
\end{equation}

\paragraph{Classical CMI as a special case.}
When $\Phi(i)=Z$ is fixed and low-dimensional,
\begin{equation}
\medmath{
\mathrm{CMMI}(Y';f_i\mid Z)
=
I(Y';f_i\mid Z).
}
\end{equation}

\paragraph{General case.}
CMMI extends MI and CMI by allowing:
\begin{enumerate}
    \item   \emph{arbitrary multivariate} conditioning sets $\Phi(i)$,  
    \item  \emph{mixed continuous–discrete} entropy functionals,  
    \item \emph{correlation-driven} neighbourhood selection rather than fixed $Z$, 
    \item compatibility with Joint Mutual Impact (JMI) and MCIR.
\end{enumerate}
% ------------------------------------------------------------
\subsection{Why Classical CMI is Insufficient}
% ------------------------------------------------------------

Classical CMI assumes:
\begin{enumerate}
\item a fixed conditioning set,  
\item low-dimensional conditioning,  
\item one-to-many dependence patterns (CMI conditions on a single $Z$).  
\end{enumerate}

In real high-dimensional systems, these assumptions fail.  
We formalize this gap using a structural theorem.

% ----------------------
\begin{theorem}[\textbf{Failure of Classical CMI in Multivariate Dependencies}]
\label{thm:cmi_failure}
Let $f_i$ depend on $d_i$ other features:
\begin{equation}
\medmath{
f_i \;\not\!\perp\!\!\!\perp\; f_{S_i},
\qquad |S_i| = d_i.
}
\end{equation}
Suppose classical CMI conditions on a fixed $Z$ with $|Z| < d_i$.  
Then, unless $Z$ contains the entire dependency set $S_i$, 
\begin{equation}
\medmath{
I(Y'; f_i \mid Z)
\neq
\mathrm{CMMI}(Y'; f_i \mid S_i),
}
\end{equation}
and, in general,
\begin{equation}
\medmath{
I(Y'; f_i \mid Z)
<
I(Y'; f_i \mid S_i).
}
\end{equation}
\end{theorem}

\begin{proof}
By the data-processing inequality for conditional entropy,
\begin{equation}
\medmath{
H(Y'\mid Z)
\ge
H(Y'\mid S_i).
}
\end{equation}
Since $S_i$ contains strictly more predictive information than any strict subset
$Z\subset S_i$,
\begin{equation}
\medmath{
H(Y'\mid Z,f_i) - H(Y'\mid Z)
\;
\le
\;
H(Y'\mid S_i,f_i) - H(Y'\mid S_i).
}
\end{equation}

Thus,
\begin{equation}
\medmath{
I(Y';f_i\mid Z)
\le
I(Y';f_i\mid S_i),
}
\end{equation}
with strict inequality when the excluded variables contain unique information.  
Since CMMI uses the full neighbourhood $S_i$ (or $\Phi(i)$), classical CMI
cannot in general recover it unless $Z = S_i$.
\end{proof}

\paragraph{Implications.}
\begin{itemize}
\item \textbf{CMI underestimates feature importance} when the conditioning set
is incomplete.

\item \textbf{CMI yields non-unique results}, depending on the arbitrary choice
of $Z$.

\item \textbf{CMMI resolves this} by using an adaptively identified 
multivariate conditioning set $\Phi(i)$.

\item \textbf{MCIR builds directly on CMMI} to quantify unique impact by:
\[
\medmath{
C_i
=
\frac{
\mathrm{CMMI}(Y';f_i\mid \Phi(i))
}{
\mathrm{CMMI}(Y';f_i\mid \Phi(i))
+
I(Y';f_{\Phi(i)}\cup\{f_i\})
}.
}
\]
\end{itemize}

% ------------------------------------------------------------
\subsection{Practical Estimation Notes}
% ------------------------------------------------------------

CMMI is estimated using:
\[
\widehat{\mathrm{CMMI}}(Y'; f_i \mid \Phi(i))
=
\widehat{H}(Y' \mid f_{\Phi(i)})
-
\widehat{H}(Y' \mid f_{\Phi(i)}, f_i),
\]
with entropy estimators matched to variable types.

\paragraph{Choice of estimator.}
We use the same family of MI/CMI estimators as MCIR:
\begin{itemize}
    \item \textbf{Gaussian–Copula} for robust nonlinear dependence,
    \item \textbf{kNN (Kozachenko–Leonenko)} for local nonlinear structure,
    \item \textbf{Plug-in} estimator for discrete/mixed features.
\end{itemize}

\paragraph{Stability.}
Bootstrap standard errors determine the optimal estimator per feature,
as formalised by the oracle inequality in Theorem~\ref{thm:oracle}.

\paragraph{Computational structure.}
For each feature $f_i$:
\[
\text{cost} = O\!\left( |\Phi(i)| \cdot n' \, \log n' \right).
\]

CMMI therefore scales linearly in neighbourhood size and
approximately logarithmically in sample size, making it suitable for
lightweight environments.

\subsection{Complexity bound on MCIR}\label{complex}
Feature dependencies make computational complexity a primary concern. Since MCIR is defined in terms of conditional (mixed) mutual information, we bound the computation via Kolmogorov (algorithmic) complexity.

Let $\mathbb{C}$ be a fixed universal Turing machine (UTM), and let $K_{\mathbb{C}}(\cdot\,|\,\cdot)$ denote conditional Kolmogorov complexity. Let $\mathbf{Y}'=(y'_1,\dots,y'_{n'})\in\{0,1\}^{n'}$ denote the (binary-encoded) output string and let $\ell(\mathbf{Y}')=n'$ be its length in bits. Let $M^\ast$ denote the proposed dependency model for features (run on $\mathbb{C}$).

\paragraph{Kolmogorov setup}
By definition of conditional Kolmogorov complexity,
\begin{equation}\label{eq:kolmog-def}
\begin{aligned}
K_{\mathbb{C}}\!\bigl(\mathbf{Y}' \,\big|\, \ell(\mathbf{Y}')\bigr)
= \min_{M^\ast:\;\mathbb{C}(M^\ast,\ell(\mathbf{Y}'))=\mathbf{Y}'} \;\ell(M^\ast),
\end{aligned}
\end{equation}
where $\ell(M^\ast)$ is the program length of $M^\ast$.

By universality of conditional Kolmogorov complexity, for any other UTM $\mathbb{A}$ there exists a constant $c_{\mathbb{A}}$ (independent of $\mathbf{Y}'$) such that
\begin{equation}\label{eq:universality}
\begin{aligned}
K_{\mathbb{C}}(\mathbf{Y}')
\;\le\;
K_{\mathbb{A}}\!\bigl(\mathbf{Y}' \,\big|\, \ell(\mathbf{Y}')\bigr)
\;+\; \log^{\ast}\!\ell(\mathbf{Y}') \;+\; c_{\mathbb{A}},
\end{aligned}
\end{equation}
where $\log^{\ast} m \!=\! \log m + \log\log m + \log\log\log m + \cdots$ (summing positive terms only).

Since $\log^{\ast} n' \le C\,\log n'$ for a universal constant $C$ and negative terms are discarded, we obtain the convenient bound
\begin{equation}\label{eq:kol-simplified}
\begin{aligned}
K_{\mathbb{C}}(\mathbf{Y}')
\;\le\;
K_{\mathbb{A}}\!\bigl(\mathbf{Y}' \,\big|\, \ell(\mathbf{Y}')\bigr)
\;+\; O(\log n') \;+\; c_{\mathbb{A}}.
\end{aligned}
\end{equation}

\paragraph{Per-term CMMI complexity}
Write the feature set as $F=\{f_1,\dots,f_k\}$ with a mixed (continuous–discrete) conditioning subspace $\phi=\phi_c\cup\phi_d\subseteq F\setminus\{f_i\}$. Using standard bounds relating algorithmic complexity and entropy for mixed variables, a typical conditional mutual information term satisfies
\begin{equation}\label{eq:cmi-entropy-bound}
\begin{aligned}
I(\mathbf{Y}'; f_i \mid \phi)
\;\le\;
H(\mathbf{Y}')
\;-\;
H(\mathbf{Y}' \mid f_i,\phi)
\;+\; O(\log n'),
\end{aligned}
\end{equation}
where the $O(\log n')$ term collects coding overheads (e.g., model description, discretization indices, prefix codes) bounded by iterated-log terms as in \eqref{eq:universality}.

Equivalently, grouping the overheads,
\begin{equation}\label{eq:cmi-compact}
\begin{aligned}
&I(\mathbf{Y}'; f_i \mid \phi)
\;\le\;
O(\log n') \;+\; \mathsf{C}_{\mathbb{A},i},
\\&
\mathsf{C}_{\mathbb{A},i}
\;\coloneqq\;
H(\mathbf{Y}')-H(\mathbf{Y}' \mid f_i,\phi) + c_{\mathbb{A}}',
\end{aligned}
\end{equation}
for a machine-dependent constant $c'_{\mathbb{A}}$.

A similar bound applies to the joint dependence term $I(\mathbf{Y}'; f_1,\dots,f_k)$:
\begin{equation}\label{eq:joint-bound}
\begin{aligned}
I(\mathbf{Y}'; f_1,\dots,f_k)
\;\le\;
O(\log n') \;+\; \mathsf{C}'_{\mathbb{A}},
\end{aligned}
\end{equation}
with $\mathsf{C}'_{\mathbb{A}}$ independent of $n'$.

\paragraph{Implications for MCIR.}
Recall the MCIR for feature $f_i$ (given $\phi$)
\begin{equation}\label{eq:mcir}
\begin{aligned}
\mathfrak{C}_i
\;=\;
\frac{ I(\mathbf{Y}'; f_i \mid \phi) }
     { I(\mathbf{Y}'; f_i \mid \phi) + I(\mathbf{Y}'; f_1,\dots,f_k) }.
\end{aligned}
\end{equation}
Combining \eqref{eq:cmi-compact} and \eqref{eq:joint-bound}, each evaluation of $\mathfrak{C}_i$ admits the bound
\begin{equation}\label{eq:mcir-per-term}
\begin{aligned}
\mathfrak{C}_i
\;\le\;
\tfrac{1}{2}\,O(\log n') \;+\; \zeta_{\mathbb{A},i},
\end{aligned}
\end{equation}
for a constant $\zeta_{\mathbb{A},i}$ that depends on the coding scheme (machine) and the fixed distributional terms but not on $n'$.

\paragraph{End-to-end bound.}
Evaluating all $k$ MCIR terms and aggregating under the same universal coding overhead yields the overall program-level complexity
\begin{equation}\label{eq:overall}
\begin{aligned}
K_{\mathbb{C}}(\text{MCIR pipeline})
&\;\le\;
K_{\mathbb{A}}(\text{MCIR pipeline}\mid n')\\&
\;+\;
\tfrac{k}{2}\,O(\log n') \;+\; c_{\mathbb{A}},
\end{aligned}
\end{equation}
i.e., at most \(\tfrac{k}{2}\,O(\log n') + c_{\mathbb{A}}\) overhead beyond the machine-conditional baseline, and at most \(\tfrac{1}{2}\,O(\log n')+\zeta_{\mathbb{A},i}\) per MCIR term.
\section*{\quad Estimator Details}
\addcontentsline{toc}{section}{S.4 Estimator Details}

MCIR requires estimates of (a) marginal and conditional mutual information, 
and (b) joint multivariate dependence terms.  
We summarize the estimators used in the main paper and provide assumptions, 
complexity statements, and error bounds.

% ------------------------------------------------------------
\subsection{Gaussian–Copula MI and CMI}
The Gaussian–Copula estimator (\textsc{gcMI}) computes MI via a normal–scores 
transformation followed by closed–form expressions for Gaussian entropy.  
For random vectors $X,Y$,
\begin{equation}
\medmath{
I_{\mathrm{gc}}(X;Y)
=
-\frac{1}{2}
\log\!\left(
1 - \rho_{XY}^{2}
\right),
}
\end{equation}
where $\rho_{XY}$ is the rank–based (pseudo–Gaussian) correlation of the 
transformed variables.

Conditional MI is computed via partial correlations:
\begin{equation}
\medmath{
I_{\mathrm{gc}}(X;Y\mid Z)
=
-\frac{1}{2}
\log\!\left(
1 - \rho^{2}_{XY\cdot Z}
\right).
}
\end{equation}

\begin{assumption}
    \begin{enumerate}
\item Rank–Gaussianization sufficiently approximates the latent copula.
\item All covariance matrices remain positive–definite.
\item No strong multimodal structure.
\end{enumerate}
\end{assumption}
\paragraph{Complexity.}  
Computing a $k\times k$ rank correlation matrix:  
\[
\mathcal{O}(k^2 n'), \qquad \text{inversion: } \mathcal{O}(k^3).
\]

\paragraph{Error Bound.}
Under sub-Gaussian copula assumptions,
\begin{equation}
\medmath{
|\widehat{I}_{\mathrm{gc}} - I|
=
O_p\!\left(\sqrt{\frac{1}{n'}}\right).
}
\end{equation}

% ------------------------------------------------------------
\subsection{kNN MI/CMI}
The Kozachenko–Leonenko family of MI estimators uses $k$–nearest–neighbor 
volume statistics:
\begin{equation}
\medmath{
\widehat{I}_{k\mathrm{NN}}(X;Y)
=
\psi(k)
-
\frac{1}{n'}
\sum_{i=1}^{n'}
\left[
\psi(n_{x,i}+1) + \psi(n_{y,i}+1)
\right]
+
\psi(n'),
}
\end{equation}
with analogous expressions for conditional MI using $(X,Y,Z)$ neighborhoods.

\paragraph{Assumptions.}
\begin{itemize}
\item Smooth densities with bounded curvature.
\item Local isotropy in $k$–ball neighborhoods.
\item Moderate intrinsic dimension.
\end{itemize}

\paragraph{Complexity.}
\[
\mathcal{O}(n'\log n') 
\quad\text{with k-d tree},
\qquad 
\mathcal{O}(n'^2)
\quad\text{in worst case}.
\]

\paragraph{Error Bound.}
\begin{equation}
\medmath{
|\widehat{I}_{k\mathrm{NN}} - I|
=
O_p\!\left(\left(\frac{1}{n'}\right)^{1/d}\right).
}
\end{equation}

% ------------------------------------------------------------
\subsection{ Plug–In Estimator.}
The plug–in estimator applies kernel density estimates (KDE) to compute MI:
\begin{equation}
\medmath{
\widehat{I}_{\mathrm{PI}}(X;Y)
=
\iint 
\widehat{p}(x,y)
\log
\frac{
\widehat{p}(x,y)
}{
\widehat{p}(x)\widehat{p}(y)
}
dx\,dy.
}
\end{equation}

\paragraph{Assumptions.}
\begin{itemize}
\item Smooth densities.
\item Kernel bandwidth chosen by cross–validation.
\item No exponential tail dependence.
\end{itemize}

\paragraph{Complexity.}
\[
\mathcal{O}(n'^2)
\quad\text{for KDE-based MI}.
\]

\paragraph{Error Bound.}
If KDE converges at rate $\mathcal{O}_p(\delta)$, then MI converges at
\begin{equation}
\medmath{
|\widehat{I}_{\mathrm{PI}} - I|
=
O_p(\delta).
}
\end{equation}

% ------------------------------------------------------------
\subsection{Summary of Estimator Choice.}
\begin{itemize}
\item \textbf{gcMI}: fastest, stable, best for moderate correlations.
\item \textbf{kNN MI}: robust for nonlinear dependencies, higher variance.
\item \textbf{Plug–In MI}: accurate but slow, sensitive to bandwidth.
\end{itemize}
MCIR uses automatic estimator switching to achieve oracle-level performance 
(Thm.~\ref{thm:oracle}).

% ============================================================
\section{Algorithmic Details for MCIR}
\addcontentsline{toc}{section}{S.5 MCIR Algorithmic Details}

The computation of the Mutual Correlation Impact Ratio (MCIR) proceeds in
three stages: (i) neighbourhood screening to determine candidate dependency
sets, (ii) dependence estimation using MI/CMI estimators with bootstrap-based
switching, and (iii) computation of the final MCIR vector. This section
provides the full algorithmic description, estimator choices, and bootstrap
protocols used across all datasets.

% ------------------------------------------------------------
\subsection{ MCIR Pipeline}

% \section{Algorithmic Details for MCIR}
% \addcontentsline{toc}{section}{S.5 MCIR Algorithmic Details}

\begin{algorithm}[h!]
\caption{MCIR Pipeline}
\SetAlgoLined

\KwIn{Feature matrix $F \in \mathbb{R}^{n' \times k}$; output $Y'$; screening threshold $\gamma$}
\KwOut{MCIR vector $C = (C_1,\dots,C_k)$}

\BlankLine

Construct screening graph by connecting $f_i, f_j$ if 
$\mathrm{corr}(f_i,f_j)\ge \gamma$ or $\mathrm{dCor}(f_i,f_j)\ge \gamma$\;

\For{$i = 1$ \KwTo $k$}{
    Determine neighbourhood:
    \[
    \Phi(i) = \{ j : \mathrm{edge}(i,j)=1 \}\;
    \]
    
    Estimate MI components:
    \[
    A_i = I(Y'; f_i \mid f_{\Phi(i)}), \qquad
    B_i = I(Y'; f_{\Phi(i)} \cup \{f_i\})\;
    \]
    
    Compute MCIR:
    \[
    C_i = \frac{A_i}{A_i + B_i}\;
    \]
}

\Return{$C$}\;

\end{algorithm}

MCIR uses a small and stable hyperparameter set:
\begin{itemize}
\item Screening threshold $\gamma \in [0.20,0.35]$.
\item kNN MI estimator neighbourhood: $k = 5$.
\item Gaussian–copula MI: rank transform $z = \Phi^{-1}(F_n(x))$.
\item Bootstrap replicates: $B = 200$.
\end{itemize}

% ------------------------------------------------------------
\subsection{Screening Mechanisms}

Neighborhoods are constructed using a union of:
\begin{itemize}
\item Pearson correlation (linear dependence),
\item Distance correlation (nonlinear dependence),
\item kNN graph structure (manifold geometry).
\end{itemize}

% ------------------------------------------------------------
\subsection{Conditioning–Set Selection (Auto$\Phi$)}

Given neighbourhood candidates $\mathcal{N}(i)$, Auto$\Phi$ selects:
\begin{equation}
    \Phi(i)
=
\operatorname*{arg\,max}_{S\subseteq \mathcal{N}(i)}
I(Y'; f_i \mid S)
\quad\text{s.t.}\quad |S| \le s_{\max}.
\end{equation}

% ------------------------------------------------------------
\subsection{Bootstrap Protocol}

For each $f_i$:
\begin{enumerate}
\item Compute MCIR using each estimator $e$.
\item Recompute under $B$ bootstrap resamples.
\item Estimate variance band:
\[
\widehat{\mathrm{SE}}_b(\widehat{C}^{(e)}_i).
\]
\item Choose estimator with minimum bootstrap SE.
\end{enumerate}

% ============================================================
\section{Lightweight Environment Construction}
\addcontentsline{toc}{section}{S.6 Lightweight Environment Construction}

Lightweight environments compress output distributions while preserving the
dependence structure relevant for MCIR. Let $Y\in\mathbb{R}^{n}$ be the full
output and $Y'\in\mathbb{R}^{n'}$ be its projection.

% ------------------------------------------------------------
\subsection{Stiefel–Manifold Projection}

Lightweight outputs are generated by:
\begin{equation}
Y' = P Y + b, \qquad P \in \mathrm{St}(n',n),\; b\in\mathbb{R}^{n'}.
\end{equation}

Projection distance:
\begin{equation}
d^{-}(\mu,\nu)
=
\inf_{P,b} d_f(P_{\#}\mu,\nu).
\end{equation}

Embedding distance:
\begin{equation}
d^{+}(\mu,\nu)
=
\inf_{Q,c} d_f(\mu,Q_{\#}\nu).
\end{equation}

% ------------------------------------------------------------
\subsection{Alignment Procedure}

The projection matrix is estimated by:
\begin{equation}
\widehat{P}
=
\arg\min_{P\in\mathrm{St}(n',n)}
\| P Y - Y' \|_2^2.
\end{equation}

% ------------------------------------------------------------
\subsection{ Risk–Generator}

We define the lightweight risk:
\begin{equation}
\mathcal{R}(n')
=
\mathrm{MSE}(Y,Y')
+
\mathrm{MI\ deviation}(n')
+
\mathrm{Rank\ deviation}(n').
\end{equation}

% ------------------------------------------------------------
\subsection{Choosing $n' \in \{500,1000,2000\}$}

\begin{itemize}
\item $n'=500$: fastest but highest variance.
\item $n'=1000$: best fidelity–runtime balance.
\item $n'=2000$: near-full MI fidelity.
\end{itemize}

\bibliographystyle{IEEEtran}